\documentclass[10pt]{article}
\usepackage[margin=1in]{geometry}
\usepackage[round,compress]{natbib}

\usepackage[utf8]{inputenc} %
\usepackage[T1]{fontenc}    %

\usepackage{parskip}

\usepackage{amsmath}
\usepackage{amsxtra,amsfonts,amscd,amssymb,bm,bbm,mathrsfs}
\usepackage{nicefrac}       %
\usepackage{amsthm,thmtools,thm-restate}
\usepackage{mathtools}
\usepackage{mathabx}
\newcommand\labelAndRemember[2]
  {\expandafter\gdef\csname labeled:#1\endcsname{#2}%
   \label{#1}#2}
\newcommand\recallLabel[1]
   {\csname labeled:#1\endcsname\tag{\ref{#1}}}
\newcommand\labelr[2]
  {\expandafter\gdef\csname labeled:#1\endcsname{#2}%
   \label{#1}#2}
\newcommand\recall[1]
   {\csname labeled:#1\endcsname}

\usepackage[algo2e,linesnumbered,vlined,ruled]{algorithm2e}
\usepackage{algorithm}
\usepackage[noend]{algorithmic}

\usepackage{url}
\usepackage[breaklinks=true]{hyperref}
\usepackage{breakcites}
\hypersetup{
  colorlinks,
  citecolor=blue!70!black,
  linkcolor=blue!70!black,
  urlcolor=blue!70!black
}
\usepackage[capitalise]{cleveref}
\usepackage[titletoc,title]{appendix}
\usepackage{cancel}
\usepackage{enumerate}
\usepackage{enumitem}
\usepackage{comment}
\crefformat{equation}{#2(#1)#3}
\Crefformat{equation}{#2(#1)#3}

\usepackage{booktabs}       %
\usepackage{multirow}
\usepackage{multicol}
\usepackage{makecell}
\usepackage{array}
\newcolumntype{H}{>{\setbox0=\hbox\bgroup}c<{\egroup}@{}}
\newcolumntype{Z}{>{\setbox0=\hbox\bgroup}c<{\egroup}@{\hspace*{-\tabcolsep}}}

\usepackage[normalem]{ulem}
\usepackage{exscale}
\usepackage{tikz}
\usepackage{float}
\usepackage{graphicx,graphics,subfig}
\graphicspath{ {./fig/} }
\allowdisplaybreaks

\usepackage{chngcntr}
\usepackage{apptools}
\makeatletter
\def\thm@space@setup{\thm@preskip=3pt
\thm@postskip=0pt}
\makeatother
\newtheorem{theorem}{Theorem}
\AtAppendix{\counterwithin{theorem}{section}}
\newtheorem{lemma}[theorem]{Lemma}
\AtAppendix{\counterwithin{lemma}{section}}
\newtheorem{definition}[theorem]{Definition}
\AtAppendix{\counterwithin{definition}{section}}
\newtheorem{corollary}[theorem]{Corollary}
\AtAppendix{\counterwithin{corollary}{section}}
\newtheorem{proposition}[theorem]{Proposition}
\AtAppendix{\counterwithin{proposition}{section}}

\theoremstyle{definition}
\newtheorem{remark}[theorem]{Remark}
\AtAppendix{\counterwithin{remark}{section}}
\newtheorem{example}[theorem]{Example}
\AtAppendix{\counterwithin{example}{section}}
\newcommand{\exend}{{~$\Diamond$}}

\newenvironment{proof-sketch}{\noindent{\bf Proof Sketch}
  \hspace*{1em}}{\qed\bigskip\\}
\newenvironment{proof-idea}{\noindent{\bf Proof Idea}
  \hspace*{1em}}{\qed\bigskip\\}
\newenvironment{proof-of}[1][{}]{\noindent{\bf Proof of \cref{#1}}
  \hspace*{1em}}{\qed\bigskip\\}
\newenvironment{proof-of-lemma}[1][{}]{\noindent{\bf Proof of Lemma {#1}}
  \hspace*{1em}}{\qed\bigskip\\}
\newenvironment{proof-of-proposition}[1][{}]{\noindent{\bf
    Proof of Proposition {#1}}
  \hspace*{1em}}{\qed\bigskip\\}
\newenvironment{proof-of-theorem}[1][{}]{\noindent{\bf Proof of Theorem {#1}}
  \hspace*{1em}}{\qed\bigskip\\}
\newenvironment{inner-proof}{\noindent{\bf Proof}\hspace{1em}}{
  $\bigtriangledown$\medskip\\}
\newenvironment{proof-attempt}{\noindent{\bf Proof Attempt}
  \hspace*{1em}}{\qed\bigskip\\}
\newenvironment{proofof}[1][{}]{\noindent{\bf Proof of \cref{#1}}
  \hspace*{1em}}{\qed\bigskip\\}

\newcommand{\revise}[1]{{#1}}

\def\shownotes{1}  %
\ifnum\shownotes=1
\newcommand{\authnote}[2]{{\scriptsize $\ll$\textsf{#1 notes: #2}$\gg$}}
\else
\newcommand{\authnote}[2]{}
\fi

\renewcommand{\hat}{\widehat}
\renewcommand{\epsilon}{\varepsilon}

\usepackage{pifont}

\newcommand{\yes}{\ding{51}}

\newcommand{\tM}{\tilde{M}}

\newcommand{\RR}{\mathbb{R}}

\newcommand{\eps}{\varepsilon}

\newcommand{\II}{\mathbb{I}}
\newcommand{\EE}{\mathbb{E}}
\newcommand{\PP}{\mathbb{P}}
\newcommand{\QQ}{\mathbb{Q}}

\newcounter{cnt}
\foreach \num in {1,2,...,26}{%
  \stepcounter{cnt}%
  \expandafter\xdef \csname c\Alph{cnt}\endcsname {\noexpand\mathcal{\Alph{cnt}}}%
  \expandafter\xdef \csname b\Alph{cnt}\endcsname {\noexpand\mathbb{\Alph{cnt}}}%
}
\renewcommand{\bR}{\mathbf{R}}
\newcommand{\br}{\mathbf{r}}
\newcommand{\tr}{\mathrm{tr}}

\DeclareMathOperator*{\argmin}{arg\,min}
\DeclareMathOperator*{\argmax}{arg\,max}

\newcommand{\TV}{\operatorname{TV}}

\newcommand{\reg}{\mathbf{Reg}}

\newcommand{\leqsim}{\lesssim}
\newcommand{\geqsim}{\gtrsim}

\newcommand{\T}{\top}  
\newcommand{\iprod}[2]{\left\langle #1, #2 \right\rangle}
\newcommand{\nrm}[1]{\left\|#1\right\|}
\newcommand{\minop}[1]{\min\left\{#1\right\}}
\newcommand{\abs}[1]{\left|#1\right|}

\newcommand{\cond}[2]{\mathbb{E}\left[\left.#1\right|#2\right]}

\newcommand{\bigO}[1]{\mathcal{O}\left(#1\right)}
\newcommand{\tbO}[1]{\widetilde{\mathcal{O}}\left(#1\right)}

\newcommand{\tbOm}[1]{\tilde{\Omega}\left(#1\right)}
\newcommand{\Om}[1]{\Omega\left(#1\right)}

\newcommand{\ceil}[1]{\left\lceil #1\right\rceil}
\newcommand{\floor}[1]{\left\lfloor #1\right\rfloor}
\DeclarePairedDelimiterX{\ddiv}[2]{(}{)}{%
  #1\;\delimsize\|\;#2%
}
\newcommand{\KL}{\operatorname{KL}\ddiv}

\newcommand{\Unif}{\mathrm{Unif}}

\newcommand{\poly}{\operatorname{poly}}

\newcommand{\hV}{\hat{V}}
\newcommand{\opi}{\overline{\pi}}
\newcommand{\hpi}{\hat{\pi}}

\newcommand{\Algo}{\mathsf{Alg}}

\newcommand{\norm}[1]{\left\|{#1}\right\|} %
\newcommand{\lone}[1]{\norm{#1}_1} %
\newcommand{\ltwo}[1]{\norm{#1}_2} %
\newcommand{\linf}[1]{\norm{#1}_\infty} %

\renewcommand{\cO}{\mathcal{O}}
\newcommand{\tO}{\widetilde{\cO}}
\newcommand{\wt}{\widetilde}

\newcommand{\<}{\left\langle}
\renewcommand{\>}{\right\rangle}

\newcommand{\sign}{\operatorname{sign}}

\renewcommand{\II}{\mathbbm{1}}

\newcommand{\setto}{\leftarrow}

\newcommand{\unif}{{\rm Unif}}

\newcommand{\hM}{\widehat{M}}
\renewcommand{\hV}{\widehat{V}}
\newcommand{\oM}{{\overline{M}}}
\newcommand{\omu}{\overline{\mu}}

\newcommand{\dec}{\operatorname{dec}}
\newcommand{\psc}{\operatorname{psc}}
\newcommand{\omlec}{\operatorname{mlec}}
\newcommand{\mlec}{\operatorname{mlec}}
\newcommand{\odec}{\dec}
\newcommand{\eec}{\operatorname{pacdec}}
\newcommand{\oeec}{\eec}

\newcommand{\eecg}{\operatorname{edec}_{\gamma}}

\newcommand{\dH}{D_{\mathrm{H}}}
\newcommand{\dTV}{D_{\mathrm{TV}}}
\newcommand{\wtdH}{D_{\operatorname{RL}}}
\renewcommand{\DH}[1]{D_{\mathrm{H}}^2\left(#1\right)}
\newcommand{\tDH}[1]{D_{\mathrm{RL}}^2\left(#1\right)}
\newcommand{\DTV}[1]{D_{\mathrm{TV}}\left(#1\right)}
\newcommand{\DTVt}[1]{D_{\mathrm{TV}}^2\left(#1\right)}

\newcommand{\subopt}{\mathbf{SubOpt}}

\newcommand{\co}{\operatorname{co}}
\newcommand{\regdm}{\mathbf{Reg}_{\mathbf{DM}}}

\newcommand{\eluder}{\mathfrak{e}}
\newcommand{\starn}{\mathfrak{s}}
\newcommand{\est}{\operatorname{Est}}
\newcommand{\Alg}{\mathbf{Alg}}
\newcommand{\Est}{\mathbf{Est}}
\newcommand{\EstwtH}{\mathbf{Est}_{\operatorname{RL}}}
\newcommand{\EstRL}{\mathbf{Est}_{\operatorname{RL}}}

\newcommand{\Ms}{{M^{\star}}}
\newcommand{\Vs}{V_\star}

\newcommand{\expl}{\mathrm{exp}}
\newcommand{\out}{\mathrm{out}}

\newcommand{\doac}{\mathrm{do}}
\usetikzlibrary{matrix,decorations.pathreplacing,calc}

\newcommand{\MP}{\mathsf{P}^M}
\newcommand{\oMP}{\mathsf{P}^{\oM}}
\newcommand{\oMR}{\mathsf{R}^{\oM}}

\newcommand{\etod}{\textsc{E2D-TA}}
\newcommand{\etodfull}{\textsc{Estimation-to-Decisions with Tempered Aggregation}}
\newcommand{\etodva}{\textsc{E2D-VA}}
\newcommand{\Vovkalg}{Tempered Aggregation}

\newcommand{\tPP}{\widetilde{\PP}}
\newcommand{\etap}{\eta_{\mathrm{p}}}
\newcommand{\etar}{\eta_{\mathrm{r}}}
\newcommand{\mops}{\textsc{MOPS}}
\newcommand{\omle}{\textsc{OMLE}}

\newcommand{\dcp}{\operatorname{dcp}}
\newcommand{\dcpha}{\dcp^{h,\alpha}}
\newcommand{\dcpa}{\dcp^{\alpha}}

\newcommand{\eetod}{\textsc{PAC E2D}}
\newcommand{\pexp}{p_{\mathrm{exp}}}
\newcommand{\pout}{p_{\mathrm{out}}}
\newcommand{\hatpout}{\hat{p}_{\mathrm{out}}}
\newcommand{\pac}{{\rm pac}}

\newcommand{\rfec}{\operatorname{rfdec}}
\newcommand{\rfecg}{\operatorname{rfdec}_{\gamma}}
\newcommand{\orfec}{\rfec}
\newcommand{\orfecg}{\orfec_{\gamma}}
\newcommand{\rfalg}{\textsc{Reward-Free E2D}}

\newcommand{\rf}{{\rm rf}}

\newcommand{\suboptrf}{\mathbf{SubOpt}^{\mathbf{rf}}}

\newcommand{\Pm}{\mathsf{P}}

\newcommand{\oPm}{\overline{\mathsf{P}}}
\newcommand{\hPm}{\hat{\mathsf{P}}}

\newcommand{\err}{\operatorname{Err}}
\newcommand{\PPs}{\PP^\star}

\newcommand{\piest}{\pi^{\exp}}

\newcommand{\btheta}{\boldsymbol{\theta}}

\newcommand{\ltwot}[1]{\ltwo{#1}^2}
\newcommand{\dr}{d}
\newcommand{\Rs}{\bR_{\star}}

\newcommand{\ba}{\mathbf{a}}
\renewcommand{\circ}{\diamond}
\newcommand{\Gap}{{\rm Gap}}
\newcommand{\EQ}{{\tt EQ}}

\newcommand{\NE}{{\tt NE}}
\newcommand{\CE}{{\tt CE}}
\newcommand{\CCE}{{\tt CCE}}

\newcommand{\mealg}{\textsc{All-Policy Model-Estimation E2D}}
\newcommand{\MEalg}{\textsc{All-Policy Model-Estimation E2D}}

\newcommand{\ME}{\mathrm{me}}
\newcommand{\muout}{\mu_{\out}}
\renewcommand{\opi}{\bar{\pi}}
\newcommand{\DTVPi}[1]{\wdTV^{\Pi}\left(#1\right)}
\newcommand{\mdec}{\operatorname{amdec}}
\newcommand{\medec}{\operatorname{amdec}}
\newcommand{\omdec}{\mdec}
\newcommand{\MDEC}{AMDEC}

\newcommand{\wdTV}{\wt{D}_{\rm RL}}
\newcommand{\wDTV}[1]{\wdTV\left(#1\right)}

\newcommand{\DTVPid}[1]{\wdTV^{\Pi^{\det}}\left(#1\right)}

\newcommand{\Bpara}{B-representation}

\newcommand{\arev}{\alpha_{\sf rev}}

\newcommand{\nrmonetwo}[1]{\nrm{#1}_{1, 2}}

\newcommand{\stab}{\Lambda_{\sf B}}

\newcommand{\dum}{\rm dum}

\newcommand{\cerr}{\cE}

\newcommand{\BB}{\mathbf{B}}
\newcommand{\bq}{\mathbf{q}}

\newcommand{\QAh}{\mathcal{U}_{A,h}}

\newcommand{\Uone}{\cU_1}
\newcommand{\Uh}{{\cU_h}}
\newcommand{\Uhp}{{\cU_{h+1}}}

\newcommand{\UAh}{\mathcal{U}_{A,h}}

\newcommand{\nUA}{U_A}
\newcommand{\nUAh}{\abs{\UAh}}

\newcommand{\pis}{\pi_\star}

\newcommand{\LM}{L}

\newcommand{\dimc}{\dim_{\rm c}}

\newcommand{\dtr}{q}

\newcommand{\vgap}{G}

\newcommand{\odimc}{\overline{\dim}_{\rm c}}
\newcommand{\oeluder}{\overline{\mathfrak{e}}}

\newcommand{\conv}{\mathrm{conv}}
\newcommand{\normal}[1]{\textsf{N}\paren{#1}}

\newcommand{\belrep}{decouplable representation}
\newcommand{\Belrep}{Decouplable representation}
\newcommand{\mbelrep}{strong decouplable representation}
\newcommand{\mBelrep}{Strong decouplable representation}

\newcommand{\Pia}{\Pi}
\newcommand{\Piab}{\Pi_{\sf Pb}}
\newcommand{\cMpb}{\cM_{\sf Pb}}
\newcommand{\bpi}{\boldsymbol{\pi}}

\newcommand{\decpb}{\pbdec}

\newcommand{\Regpb}{\reg^{\bf pb}}
\newcommand{\pbrlalg}{\textsc{Preference-based E2D}}
\newcommand{\pbdec}{{\rm pbdec}}
\newcommand{\opbdec}{\pbdec}

\newcommand{\ssp}{\mathbb{S}}
\newcommand{\ssps}{\ssp^{\GOAL}}
\newcommand{\ssprf}{\ssp^{\rm rf}}
\newcommand{\sspme}{\ssp^{\rm me}}
\newcommand{\ssppb}{\ssp^{\rm pb}}
\newcommand{\suboptme}{\subopt^{\bf me}}
\newcommand{\suboptpac}{\subopt^{\bf pac}}

\newcommand{\csp}{\mathbb{D}}
\newcommand{\csps}{\csp^{\GOAL}}
\newcommand{\csppb}{\csp^{\rm pb}}

\newcommand{\GOAL}{\mathsf{G}}
\newcommand{\cspreg}{\csp_{\sf reg}}
\newcommand{\csppac}{\csp_{\sf pac}}
\newcommand{\DEC}{\GOAL\!\operatorname{-dec}}
\newcommand{\DECt}{$\GOAL$-DEC}
\newcommand{\getod}{$\GOAL$-\textsc{E2D}}
\newcommand{\subopts}{\subopt^{\GOAL}}
\newcommand{\Regs}{\reg^{\GOAL}}

\newcommand{\suboptpb}{\subopt^{\mathbf{pb}}}
\newcommand{\Mtraj}{M_{\sf traj}}
\newcommand{\oMtraj}{\oM_{\sf traj}}
\newcommand{\cMtraj}{\cM_{\sf traj}}
\newcommand{\Cmp}{\C}

\newcommand{\bSigma}{\mathbf{\Sigma}}

\newcommand{\byparts}[2]{#1#2}
\newcommand{\bypart}[2]{#1#2}

\newcommand{\dimG}{\dim}

\newcommand{\paren}[1]{{\left( #1 \right)}}
\newcommand{\brac}[1]{{\left[ #1 \right]}}
\newcommand{\set}[1]{{\left\{ #1 \right\}}}

\newcommand{\defeq}{\mathrel{\mathop:}=}
\newcommand{\eqdef}{=\mathrel{\mathop:}}
\newcommand{\vect}[1]{\ensuremath{\mathbf{#1}}}
\newcommand{\mat}[1]{\ensuremath{\mathbf{#1}}}

\renewcommand{\det}{\mathrm{det}}
\newcommand{\rank}{\mathrm{rank}}

\newcommand{\polylog}{\mathrm{polylog}}

\newcommand{\E}{\mathbb{E}}

\renewcommand{\P}{\mathbb{P}}
\newcommand{\Q}{\mathbb{Q}}

\newcommand{\Z}{\mathbb{Z}}

\newcommand{\R}{\mathbb{R}}

\newcommand{\C}{\mathbb{C}}

\newcommand{\B}{\mat{B}}

\newcommand{\M}{\mat{M}}

\newcommand{\e}{\vect{e}}

\renewcommand{\v}{\vect{v}}
\renewcommand{\a}{\vect{a}}
\renewcommand{\o}{\vect{o}}

\renewcommand{\O}{\mathbb{O}}
\renewcommand{\M}{\mathbb{M}}
\renewcommand{\T}{\mathbb{T}}

\newcommand{\tmu}{\tilde{\mu}}

\renewcommand{\PPs}{\PP^\star}
\renewcommand{\Rs}{\bR^\star}
\newcommand{\xt}{\pi^t}
\newcommand{\xs}{\pi^s}
\newcommand{\xp}{\pi}
\newcommand{\EEt}[1]{\EE_t\brac{#1}}

\newcommand{\op}{\overline{p}}
\newcommand{\Arew}{\cA_{0}}
\newcommand{\Arev}{\cA_{\sf rev}}
\newcommand{\Bern}{\mathrm{Bern}}
\newcommand{\acrev}[1]{a_{\sf{rev}, i}}

\newcommand{\RM}{\mathsf{R}^M}

\newcommand{\oeps}{\underline{\epsilon}}

\newcommand{\GBE}{\cG_{\rm BE}}

\newcommand{\gammaT}{\gamma(T)}

\usepackage{etoolbox}
\newtoggle{aos}
\togglefalse{aos}
\newcommand{\aos}[2]{\iftoggle{aos}{#1}{#2}}
\newcommand{\crefa}[2]{\aos{#1 of the Supplementary Material~\citep{chen2024supp}}{\cref{#2}}}
\newcommand{\crefapp}[2]{\crefa{Appendix #1}{#2}}

\usepackage{breakcites}

\title{Unified Algorithms for RL with Decision-Estimation Coefficients: PAC, Reward-Free, Preference-Based Learning, and Beyond}

\author{
  Fan Chen\thanks{Massachusetts Institute of Technology. Email: \texttt{fanchen@mit.edu}}
 \and
  Song Mei\thanks{UC Berkeley. Email: \texttt{songmei@berkeley.edu}} \footnotemark[4] 
  \and
  Yu Bai\thanks{Salesforce Research. Email: \texttt{yu.bai@salesforce.com}} \thanks{Equal contribution.} 
}

\begin{document}
\maketitle

\begin{abstract}

Modern Reinforcement Learning (RL) is more than just learning the optimal policy; Alternative learning goals such as exploring the environment, estimating the underlying model, and learning from preference feedback are all of practical importance. While provably sample-efficient algorithms for each specific goal have been proposed, these algorithms often depend strongly on the particular learning goal and thus admit different structures correspondingly. It is an urging open question whether these learning goals can rather be tackled by a single unified algorithm.

We make progress on this question by developing a unified algorithm framework for a large class of learning goals, building on the Decision-Estimation Coefficient (DEC) framework. Our framework handles many learning goals such as no-regret RL, PAC RL, reward-free learning, model estimation, and preference-based learning, all by simply instantiating the same generic complexity measure called ``Generalized DEC'', and a corresponding generic algorithm. The generalized DEC also yields a sample complexity lower bound for each specific learning goal. As applications, we propose ``decouplable representation'' as a natural sufficient condition for bounding generalized DECs, and use it to obtain many new sample-efficient results (and recover existing results) for a wide range of learning goals and problem classes as direct corollaries. 
Finally, as a connection, we re-analyze two existing optimistic model-based algorithms based on Posterior Sampling and Maximum Likelihood Estimation, showing that they enjoy sample complexity bounds under similar structural conditions as the DEC.

\end{abstract}

\section{Introduction}

Reinforcement Learning (RL) has achieved immense success in modern artificial intelligence. As RL agents typically require an enormous number of samples to train in practice~\citep{mnih2015human,silver2016mastering}, \emph{sample-efficiency} has been an important question in RL research. This question has been studied extensively in theory, with provably sample-efficient algorithms established for many concrete RL problems. This includes tabular Markov Decision Processes (MDPs)~\citep{brafman2002r,azar2017minimax,agrawal2017optimistic,jin2018q,dann2019policy,zhang2020almost}, as well as MDPs with various types of linear structures~\citep{yang2019sample,jin2020provably,zanette2020learning,ayoub2020model,zhou2021nearly,wang2021optimism}.

Towards a more unifying theory, recent work seeks general structural conditions and unified algorithms that encompass as many known sample-efficient RL problems as possible. Many such structural conditions have been identified, including Bellman rank~\citep{jiang2017contextual}, Witness rank~\citep{sun2019model}, Eluder dimension~\citep{russo2013eluder,wang2020reinforcement}, Bilinear Class~\citep{du2019provably}, and Bellman-Eluder dimension~\citep{jin2021bellman}.~\revise{Intuitively, these conditions commonly require a generalized low-rank structure in the Bellman errors.} The recent work of~\citet{foster2021statistical} proposes the Decision-Estimation Coefficient (DEC) as a quantitative complexity measure governing the statistical complexity of model-based RL. Roughly speaking, the DEC measures the optimal trade-off---achieved by any policy---between exploration (gaining information) and exploitation (being a near-optimal policy itself) when the true model could be any model within the model class. \citet{foster2021statistical} establish upper and lower bounds showing that, for any RL problem identified with a model class, a bounded DEC is necessary and sufficient for online learning with low regret. This constitutes a significant step towards a unified understanding of sample-efficient RL.

Despite this progress, there still lacks an essential understanding of important learning goals beyond no-regret learning. For a broad range of RL applications, the objective of the agent is not to minimize regret, but rather to explore the environment sufficiently to collect enough information. Such learning goals include: (1) PAC learning~\citep{dann2017unifying}, where the objective is to ensure an output policy with small sub-optimality; (2) reward-free learning~\citep{jin2020reward}, where the agent explores without knowing the reward function so that a near-optimal policy can be computed for any reward after interactions; and (3) model estimation \citep{kumar2015stochastic}, where the objective is to estimate the \emph{model} of environment. Furthermore, some goals cannot be directly characterized by the standard notion of regret. One of the most prominent examples is preference-based learning~\citep{wirth2017survey}, where the performance of the agent is not measured by rewards but instead by comparison (e.g. human preference). Preference-based learning is also known as reinforcement learning from human feedback (RLHF) and recently achieved massive success in large language models \citep{ouyang2022training}. 

Previously, all the aforementioned learning goals have been studied in a problem-specific manner, i.e., for a specific learning goal and a specific problem class with certain structural conditions. In this paper, we present a \emph{unified} study of all these seemingly different RL learning goals under a generalized DEC framework. We do this by developing one principal complexity measure and essentially one unified algorithm.

\begin{figure}[t]
\center
\includegraphics[width=.9\hsize]{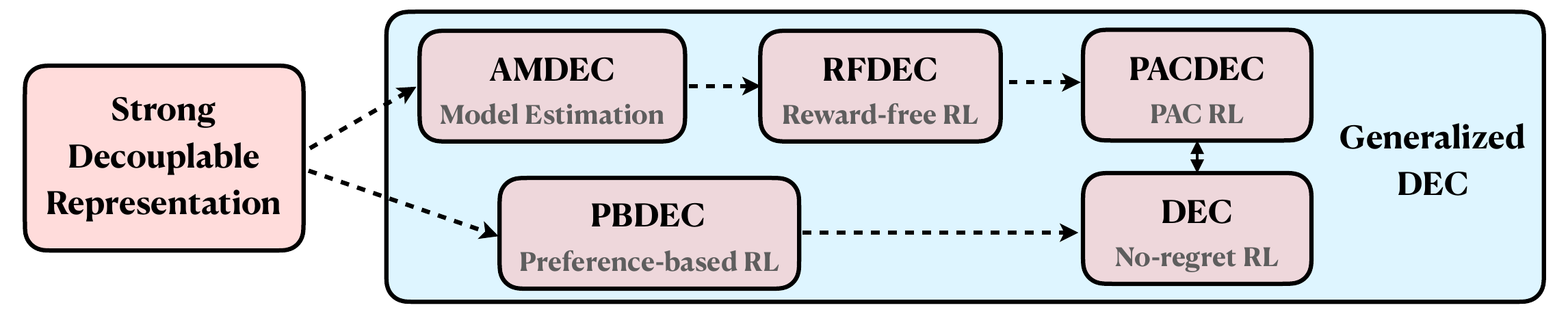}
\caption{
\small A conceptual diagram of implications between various \DECt s and (strong) decouplable representation.~\revise{{\rm PACDEC} can be bounded by {\rm RFDEC}, which can be further bounded by {\rm AMDEC}; {\rm PACDEC} and {\rm Regret DEC} can be converted to each other; {\rm Regret DEC} can be bounded by {\rm PBDEC}. } The implications between \DECt s are discussed in the corresponding sections (cf. \cref{sec:gen-dec}), and the bounds on \DECt s in terms of (strong) \belrep~are presented in \cref{section:bellman-rep} (\cref{prop:belrep-pac}, \ref{prop:belrep-am}, \ref{prop:belrep-rf}, and \ref{prop:belrep-pb}). %
}
\label{figure:relate}
\end{figure}

Our contributions can be summarized as follows.
\begin{itemize}[leftmargin=2em]

\item We extend the DEC framework to handle a generic family of learning goals (\cref{sec:gen-dec}). For any general learning goal, we present a unified meta-algorithm (\cref{alg:E2D-gen}) with complexity measured by a generalized DEC. We show that these generalized DECs also serve as lower bounds for the generic learning goals. 
\item We study the concrete learning goals of PAC learning (\cref{section:eec}), reward-free learning (\cref{section:rfec}), model estimation (\cref{section:mdec}), and preference-based learning (\cref{section:pbrl}). By specifying generalized DECs and generalized E2D algorithms, we derive a unified algorithm for each goal with complexity characterized by task-specific DECs, and provide corresponding lower bounds. We further examine the connections of DECs across different tasks\revise{, which implies the relative difficulty of tasks (as illustrated in \cref{figure:relate})}. As a further extension, we give a unified sample-efficient algorithm for learning equilibrium in Markov Games (\crefapp{G}{appendix:markov-game}). 
\item We apply our results to give sample complexity guarantees for reinforcement learning with low-complexity~\emph{\belrep},~\revise{which is a generalization of several known conditions for RL (e.g., Bilinear Class~\citep{du2021bilinear} and Bellman-Eluder dimension~\citep{jin2021bellman}).} Our results recover many existing and yield new guarantees when specialized to concrete RL problems (\cref{figure:inclusion} and \cref{section:bellman-rep}). 
\item We establish connections between E2D and two existing model-based algorithm principles: Model-Based Optimistic Posterior Sampling (MOPS) \citep{agarwal2022model}, and Optimistic Maximum-Likelihood Estimation (OMLE) \citep{liu2022partially}. We show these algorithms enjoy sub-optimality bounds similar to \eetod~under similar structural conditions (\crefapp{H}{section:optimistic}).%
\end{itemize}

\revise{We note that our unified framework and algorithms for general learning goals are a generalization of the pioneering work of \citet{foster2021statistical}. Rather than focusing on tightening the upper and lower bounds of the DEC framework, as explored in concurrent work \citep{foster2023tight}, our main contribution is the development of a unified approach that generalizes across various learning goals, incorporating them under the DEC framework. 
A detailed discussion of our technical innovations over \citet{foster2021statistical} is presented in \crefapp{A.1}{appdx:compare}.
}

\begin{figure}[t]
\center
\includegraphics[width=.6\hsize]{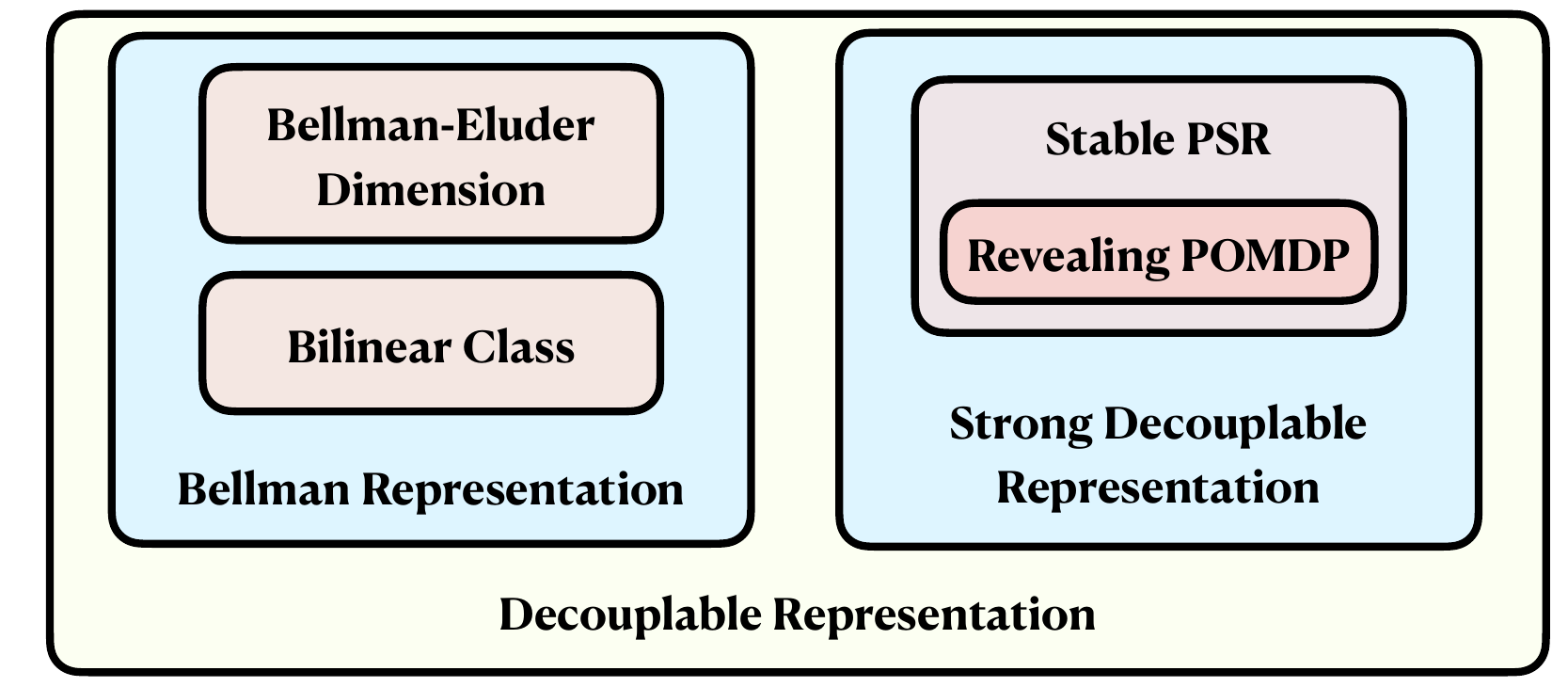}
\caption{
\small Illustration of how the \belrep~recovers existing generic structural conditions, including the model-based version of Bilinear class \citep{du2021bilinear}, Bellman-Eluder dimension \citep{jin2021bellman}, and stable PSR \citep{chen2022partially}. As we discuss in \cref{section:bellman-rep}, \mbelrep~also encompasses various concrete MDP model classes (see e.g. \cref{tab:examples}). 
}
\label{figure:inclusion}
\end{figure}

\subsection{Related work}

\paragraph{Sample-efficient reinforcement learning}
Sample-efficient RL has been extensively studied in the basic model of tabular MDPs~\citep{kearns2002near, brafman2002r, jaksch2010near, dann2015sample, azar2017minimax, agrawal2017optimistic, jin2018q, russo2019worst, dann2019policy, zanette2019tighter, zhang2020almost, domingues2021episodic}. The minimax sample complexity for finite-horizon tabular MDPs has been achieved by both model-based and model-free approaches~\citep{azar2017minimax,zhang2020almost}. When function approximation is involved, the sample complexity of RL has been studied under concrete assumptions about the function class and/or the MDP, such as various forms of linear or low-rank MDPs~\citep{yang2019sample,du2020is,jin2020provably,zanette2020learning,cai2020provably,lattimore2020learning,agarwal2020flambe,ayoub2020model,modi2020sample,zhou2021nearly}, generalized linear function approximation~\citep{wang2021optimism}, Block MDPs~\citep{du2019provably, misra2020kinematic}, parametric MDPs~\citep{kakade2020information, chowdhury2021reinforcement,li2022exponential} and others. More general structural conditions and algorithms have been studied~\citep{russo2013eluder,jiang2017contextual,sun2019model,wang2020reinforcement} and later unified by frameworks such as Bilinear Class~\citep{du2021bilinear} and Bellman-Eluder dimension~\citep{jin2021bellman}~\revise{based on the complexity of Bellman errors, as we illustrate in \cref{figure:inclusion}.}

\paragraph{Decision-estimation coefficient} \citet{foster2021statistical} propose the DEC as a complexity measure for interactive decision-making problems and develop the E2D meta-algorithm as a general model-based algorithm for problems within their DMSO framework, which covers bandits and RL. The DEC framework is further generalized in~\citet{foster2022complexity} to capture adversarial decision-making problems. %
The DEC has close connections to the modulus of continuity \citep{donoho1987geometrizing, donoho1991geometrizing, donoho1991geometrizing3}, information ratio \citep{russo2016information, russo2018learning, lattimore2021mirror}, and Exploration-by-optimization \citep{lattimore2020exploration}. Our work also builds on and extends the DEC framework: we generalize the DEC notions to capture general learning goals, including PAC learning, reward-free learning, all-policy model estimation, and preference-based learning. %

\paragraph{Other general algorithms} 
Posterior sampling (or Thompson Sampling) is another general purpose algorithm for interactive decision making \citep{thompson1933likelihood,russo2019worst, agrawal2017optimistic,zanette2020frequentist,zhang2022feel,agarwal2022model,agarwal2022non}. Frequentist regret bounds for posterior sampling are established in tabular MDPs~\citep{agrawal2017optimistic, russo2019worst} and linear MDPs~\citep{russo2019worst,zanette2020frequentist}. \citet{zhang2022feel} proves regret bounds of a posterior sampling algorithm for RL with general function approximation, which is then generalized in \cite{agarwal2022model,agarwal2022non}.
\crefapp{H.1}{section:mops} discusses the connection between the MOPS algorithm of~\citet{agarwal2022model} and \etod. The OMLE (Optimistic Maximum Likelihood Estimation) algorithm is studied in \citep{liu2022partially, liu2022sample} for Partially Observable Markov Decision Processes; however, the algorithm itself is general and can be used for any problem within the DMSO framework. We provide such a generalization and discuss the connections in~\crefapp{H.2}{section:omle}. Maximum-likelihood-based algorithms for RL are also studied in~\citep{mete2021reward,agarwal2020flambe, uehara2021representation}.

\paragraph{Reward-free learning, model estimation, and preference-based RL}
The reward-free learning framework, introduced by~\citep{jin2020reward}, has been extensively studied in both tabular and function approximation settings~\citep{jin2020reward,zhang2020task,kaufmann2021adaptive, menard2021fast, wang2020reward, zanette2020provably,agarwal2020flambe, liu2021sharp, modi2021model, zhang2021reward,zhang2021near,qiu2021reward,wagenmaker2022reward}. Recent work by~\citet{chen2022statistical} provides a general algorithm for problems with low (reward-free version of) Bellman-Eluder dimension. Our Reward-Free DEC framework generalizes many of these results by offering a unified structural condition and algorithm for reward-free RL with a model class. 

The learning goal of model estimation is also known as system identification in the literature of learning dynamical systems \citep[etc.]{schoukens2019nonlinear, mania2020active}. While model estimation is a stronger learning goal than both PAC RL and reward-free learning, for most RL problem classes this learning goal has not been considered. %
Our All-policy Model-estimation DEC framework advances the understanding of model estimation in RL and provides guarantees for a surprisingly broad range of RL problems.

A line of recent work studies preference-based RL, which has also been used in large-scale practice such as Reinforcement Learning from Human Feedback (RLHF)~\citep{ouyang2022training}.
Several recent works have explored preference-based RL with specific problem structures, including dueling bandit~\citep{dudik2015contextual, novoseller2020dueling, bengs2021preference}, tabular MDP~\citep{xu2020preference, pacchiano2021dueling}, linear mixture MDP~\citep{chen2022human}, and offline preference-based learning~\citep{zhu2023principled}. Our Preference-based DEC framework unifies all these results for online preference-based RL under the unified algorithm \pbrlalg.

\paragraph{Other problems covered by DMSO} 
Besides multi-armed bandits and RL, the DMSO framework of~\citep{foster2021statistical} (and thus all our theories as well) can handle other problems such as contextual bandits~\citep{auer2002nonstochastic, langford2007epoch, chu2011contextual, beygelzimer2011contextual, agarwal2014taming, foster2020beyond, foster2020instance}, contextual reinforcement learning~\citep{abbasi2014online,modi2018markov, dann2019policy, modi2020no}, online convex bandits~\citep{kleinberg2004nearly, bubeck2015bandit,bubeck2016multi, lattimore2020improved}, and non-parametric bandits~\citep{kleinberg2004nearly, auer2007improved,kleinberg2013bandits}. Instantiating our theories in these settings would be an interesting direction for future work. 

\revise{
\paragraph{Concurrent work}
Parallel to this paper, \citet{foster2023tight} propose the constrained version of DECs for \emph{reward-based} PAC learning and no-regret learning. Based on these \emph{constrained DECs}, they derive tighter lower and upper bounds for these two setting, which are remarkable as they are matching up to the factor of the model class complexity and polylogarthmic factors. In particular, the constrained DEC provides a characterization better than our results for PAC RL (\cref{section:eec}). On the other hand, as the proof technique of \citet{foster2023tight} is specific to the reward structure, it may not be directly generalized to the general learning goals studied in this paper. A detailed discussion is deferred to \crefapp{A.2}{appdx:constrained-dec} due to space constraints.
}

\paragraph{Subsequent works}
Since the initial appearance of our work, several related works have built upon the connections we point out in this paper. \citet{chen2022partially} develop a unified complexity measure for partially observable RL and demonstrated that the algorithmic principles (E2D, MOPS, and OMLE) can all be applied with sample-efficient guarantees. In parallel, \citet{liu2023optimistic} also identify OMLE as a generic algorithm for both PAC learning and model estimation across a range of model-based problem classes. \citet{foster2022note} demonstrate that posterior sampling with minimax policy optimization can be characterized by an optimistic variant of DEC, which can also handle model-free RL. Later, \citet{zhong2022gec} propose a general posterior sampling framework combining both model-based and model-based RL. 
For specific learning goals, %
\citet{foster2023complexity} introduce DEC framework for partial monitoring and learning equilibrium in multi-agent settings, which are both generalized PAC learning goals. Furthermore, \citet{wang2023rlhf} study the preference-based learning with OMLE under certain eluder-type structure conditions.

\section{Preliminaries}\label{sec:prelim}

\paragraph{RL as Decision Making with Structured Observations} In this paper, we adopt the general framework of Decision Making with Structured Observations (DMSO)~\citep{foster2021statistical}, which captures broad classes of interactive decision-making problems, including bandits and reinforcement learning.

In DMSO, the environment is described by a \emph{model} $M=(\MP, \bR^M)$, where $\MP \in \Delta(\cO)$ specifies the distribution of the \emph{observation} $o\in\cO$ (which may be a sequence of state-action pairs), and $\bR^M: \cO \to [0, 1]^H$ specifies the conditional mean rewards\footnote{Note that $R^M$ (and thus $M$) only specifies the conditional mean rewards instead of the reward distributions. This differs from the original DMSO framework of \cite{foster2021statistical} in which a model specifies the reward distribution. } of the stochastic \emph{reward vector} $\br\in[0,1]^H$, where $H \in \mathbb N$ is the horizon length. The learner interacts with a model using a \emph{policy} $\pi\in\Pi$. Upon executing $\pi$ in $M$, the learner observes an (observation, reward) tuple $(o, \br)\sim M(\pi)$ as follows:
\begin{enumerate}[leftmargin=2em]
    \item The learner first observes an observation $o\sim \MP(\pi)$ (also denoted as $\P^{M,\pi}(\cdot)\in\Delta(\cO)$) from the distribution specified by the environment $M$ and the policy $\pi$.
    \item Then, the learner receives a (stochastic) reward vector $\br=(r_h)_{h=1}^H \in [0, 1]^H$ with conditional mean rewards $\bR^M(o)=(R^M_h(o))_{h=1}^H\defeq \E_{\br\sim \RM(\cdot|o)}[\br]\in[0,1]^H$, and with independent entries conditioned on $o$. We also assume that $\sum_{h=1}^H R^M_h(o)\in[0,1]$ almost surely under any model $M$. 
\end{enumerate}
We let $f^M(\pi)\defeq \E^{M,\pi}[\sum_{h=1}^H r_h]$ denote the \emph{value} (expected cumulative reward) of $\pi$ under $M$, and let $\pi_M\defeq \argmax_{\pi\in\Pi} f^M(\pi)$ and $f^M(\pi_M)$ denote the optimal policy and optimal value for $M$, respectively.

Episodic Markov Decision Processes (MDPs) provide an example of a DMSO environment. An MDP $M=(H,\cS,\cA,\P^M,r^M)$ can be formulated as a DMSO problem as follows. The observation $o=(s_1,a_1,\dots,s_H,a_H)$ consists of the full state-action trajectory over the episode (so that the observation space is $\cO=(\cS\times\cA)^H$). Upon executing the Markov policy $\pi=\set{\pi_h:\cS\to\Delta(\cA)}_{h\in[H]}$ in $M$, the learner observes $o=(s_1,a_1\dots,s_H,a_H)\sim \MP(\pi)$, which sequentially samples $s_1\sim \P^M_0(\cdot)$, $a_h\sim \pi_h(\cdot|s_h)$, and $s_{h+1}\sim \P^M_{h}(\cdot|s_h, a_h)$ for all $h\in[H]$. The learner then receives a reward vector $\br=(r_h)_{h\in[H]}\in[0,1]^H$, where $r_h=r^M_h(s_h,a_h)$ is the (possibly random) instantaneous reward for the $h$-th step with conditional mean $\E^M[r_h|o]= R_h^M(o) \eqdef R^M_h(s_h,a_h)$ depending only on $(s_h,a_h)$.

\paragraph{Learning goals}
We examine the online decision-making problem under the DMSO framework, where the learner interacts with a fixed (but unknown) ground truth model $M^\star$ for $T$ episodes. Let $\pi^t\in\Pi$ denote the policy executed in the $t$-th episode. In general, $\pi^t$ may be sampled by the learner from a distribution $p^t\in\Delta(\Pi)$ before episode $t$ begins. 

In this paper, we consider a general class of learning goals formalized as follows. The agent is given an (abstract) strategy space $\ssp$. Its objective is to find an output strategy $\pout\in\ssp$ so that $\subopt_{\Ms}(\pout)$ is as small as possible, where $\subopt_M: \ssp \to \R$ is a non-negative and convex functional over the convex space $\ssp$ for model $M\in\cM$. 

The general framework described above encompasses the most well-studied learning goals of PAC reinforcement learning and no-regret learning, with strategy space $\ssp=\Delta(\Pi)$ and sub-optimality $\subopt_{M}(p)\defeq f^M(\pi_M)-\EE_{\pi\sim p} f^M(\pi)$. For PAC RL, the agent can choose any $\pout\in\Delta(\Pi)$ after all interactions are completed. For no-regret learning, the agent needs to choose $\pout=\frac{1}{T}\sum_{t=1}^T p^t$ as the output strategy, and this recovers the standard notion of regret, which measures the cumulative suboptimality of $\set{p^t}_{t\in[T]}$:
\begin{align*}
    \regdm \defeq \sum_{t=1}^T \E_{\pi^t\sim p^t}\brac{f^{M^\star}(\pi_{M^\star}) - f^{M^\star}(\pi^t)}.
\end{align*}
In \cref{sec:gen-dec}, we provide more examples of general learning goals, including reward-free learning, model estimation, and preference-based RL.

This paper focuses on model-based approaches in which we are given a \emph{model class} $\cM$, and we assume \emph{realizability}: $M^\star\in\cM$. Additionally, throughout the majority of the main text of this paper, we assume that the model class is finite: $\abs{\cM}<\infty$ for simplicity of presentation; this assumption can be relaxed using standard covering arguments (see e.g.~\crefapp{C.4}{appendix:ta-covering}), which we do when we instantiate our results to concrete RL problems in \cref{section:bellman-rep} (\cref{tab:examples}).

\paragraph{Divergences} The standard Hellinger distance between probability distributions $\P, \Q$ is $\dH^2(\P,\Q) \defeq \int (\sqrt{d \P /d\mu}-\sqrt{d \Q /d\mu})^2 d\mu$, where $\mu$ is a dominating measure. Based on the Hellinger distance, we define the following squared divergence:
\begin{align}
\label{equation:modified-hellinger}
    \wtdH^2(M(\pi), \oM(\pi)) \defeq \dH^2(\MP(\pi), \oMP(\pi)) + \E_{o\sim \MP(\pi)}\brac{ \big\|\bR^M(o) - \bR^{\oM}(o)\big\|^2_2 },
\end{align}
which is more suitable for studying the model class of interest in this paper. When we consider the goal of reward-free learning, there is no reward function, so that $\wtdH$ naturally degenerates to $\dH$.

The key feature of the divergence $\wtdH^2$ is its separate treatment of observations and rewards: it measures the observation distribution in Hellinger distance, but measures the reward only in the squared $L_2$ distance between the \emph{conditional mean rewards}. This asymmetric treatment is well-suited for RL problems, since estimating mean rewards is typically easier than estimating full reward distributions\footnote{\citet{foster2021statistical} primarily use the standard Hellinger distance (in the tuple $(o, \br)$) in their definition of the DEC, caring about full reward distributions (cf.~\crefapp{A.1.1}{appendix:comparison-dec}\aos{}{ for detailed discussions}).} and is also sufficient in most scenarios.

\subsection{DEC with randomized reference models}
The Decision-Estimation Coefficient (DEC), introduced by~\citet{foster2021statistical}, is a key quantity capturing the regret complexity of sequential decision-making problems. We adopt a specific definition of DEC with randomized reference models (henceforth ``DEC''), which instantiates the general definition of DECs in~\citet[Section 4.3]{foster2021statistical} by employing the divergence function $\wtdH^2$. 
\begin{definition}[DEC with randomized reference models]
\label{definition:dec}
The DEC of $\cM$ with respect to distribution $\omu\in\Delta(\cM)$ (with policy class $\Pi$ and parameter $\gamma>0$) is defined as
\begin{align}
\label{eqn:dec-def}
    \dec_\gamma(\cM, \omu) & \defeq \inf_{p\in\Delta(\Pi)} \sup_{M\in\cM} \E_{\pi\sim p} \E_{\oM\sim\omu}\brac{ f^M(\pi_M) - f^M(\pi) - \gamma\wtdH^2(M(\pi), \oM(\pi))}.
\end{align}
Further define $\odec_\gamma(\cM) \defeq \sup_{\omu\in\Delta(\cM)} \dec_\gamma(\cM, \omu)$. Above, $\wtdH^2$ is the squared divergence given by \cref{equation:modified-hellinger}. %

\end{definition}
The DEC measures the optimal trade-off between two terms: low suboptimality $f^M(\pi_M)-f^M(\pi)$, representing the gap between the optimal and learned policies under the true model $M$; and high information gain $\wtdH^2(M(\pi), \oM(\pi))$, measuring the divergence from the \emph{randomized} reference model $\oM\sim\omu$. 
\revise{To distinguish this DEC from our generalized DECs, we will refer to the DEC for no-regret learning as the ``Regret DEC'' in the subsequent sections. }

\subsection{No-regret algorithm: E2D with Tempered Aggregation}
\label{section:e2d}

\citet{foster2021statistical} propose the \emph{Estimation-to-Decisions} (E2D) algorithm as a meta-algorithm for no-regret learning in any model with a bounded DEC. We present an instantiation of E2D using the \emph{Tempered Aggregation} subroutine (\etod{}; \cref{alg:E2D-TA}).

\paragraph{Algorithm description}
In each episode $t$, \cref{alg:E2D-TA} maintains a randomized model estimator $\mu^t\in\Delta(\cM)$, and uses it to obtain a distribution of policies $p^t\in\Delta(\Pi)$ by minimizing the following risk function (cf. Line~\ref{line:e2d-pt}):
\begin{align}
\label{equation:vtp}
   \hV^{\mu^t}_\gamma(p) \defeq \sup_{M\in\cM}\E_{\pi \sim p} \E_{\oM\sim \mu^t} \brac{ f^M(\pi_M)-f^M(\pi) - \gamma\wtdH^2(M(\pi), \oM(\pi)) }.
\end{align}
The algorithm then samples a policy $\pi^t\sim p^t$, executes $\pi^t$, and observes $(o^t, r^t)$ from the environment (Line~\ref{line:e2d-execute}). Finally, the algorithm updates the randomized model estimator $\mu^t$ using the \emph{\Vovkalg} subroutine, which performs an exponential weights update on $\mu^t(M)$ using a linear combination of the log-likelihood $\log\P^{M,\pi^t}(o^t)$ for the observation, and the negative squared $L_2$ loss $-\| \br^t - \bR^M(o^t)\|_2^2$ for the reward (cf. Line~\ref{line:e2d-ta}).

A key feature of the Tempered aggregation subroutine is the learning rate $\etap\le 1/2$, which is smaller than the learning rate of $\etap=1$ used in Vovk's aggregating algorithm~\citep{vovk1995game}. This smaller learning rate allows Tempered Aggregation to achieve a slightly stronger estimation guarantee than Vovk's algorithm, making it suitable for our purpose. Intuitively, exponential weights with $\exp(\etap \log\P^{M,\pi^t}(o^t))=(\P^{M,\pi^t}(o^t))^{\etap}$ with $\etap\le 1/2$ is equivalent to computing the \emph{tempered posterior} in a Bayesian setting~\citep{bhattacharya2019bayesian,alquier2020concentration} (hence our name ``tempered''), whereas $\etap=1$ computes the \emph{exact} posterior (see~\crefapp{C.1}{appendix:discussion-aggregation} for a derivation).

\begin{algorithm}[t]
	\caption{\etod{}: \etodfull{}}
	\begin{algorithmic}[1]
	\label{alg:E2D-TA}
	\REQUIRE Parameter $\gamma>0$; Learning rate $\etap\in(0, \frac{1}{2})$, $\etar>0$.
	\STATE Initialize $\mu^1\setto {\rm Unif}(\cM)$.
	\FOR{$t=1,\ldots,T$}
    \STATE Set $p^t\setto \argmin_{p\in\Delta(\Pi)}\hV^{\mu^t}_{\gamma}(p)$, where $\hV^{\mu^t}_{\gamma}$ is defined in~\cref{equation:vtp}. \label{line:e2d-pt}
    \STATE Sample $\pi^t\sim p^t$. Execute $\pi^t$ and observe $(o^t,\br^t)$. \label{line:e2d-execute}
    \STATE Update randomized model estimator by \Vovkalg:
    \begin{align}
    \label{equation:tempered-aggregation}
        \mu^{t+1}(M) \; \propto_M \; \mu^{t}(M) \cdot \exp\paren{\etap \log \P^{M,\pi^t}(o^{t}) - \etar\ltwo{\br^t - \bR^M(o^t)}^2 }.
    \end{align}
    \label{line:e2d-ta}
    \ENDFOR
   \end{algorithmic}
\end{algorithm}

With the estimation guarantee of the Tempered aggregation subroutine, \cref{alg:E2D-TA} achieves the following regret guarantee; see~\crefapp{C.3}{appendix:general-e2d} for the proof.
\begin{proposition}[Regret guarantee for \etod{}]
\label{thm:E2D-TA}
Choosing $\etap=\etar=1/3$,~\cref{alg:E2D-TA} achieves the following with probability at least $1-\delta$:
\begin{align*}
    \regdm\leq T\odec_{\gamma}(\cM) + 10\gamma\cdot \log(\abs{\cM}/\delta).
\end{align*}
\end{proposition}
In typical problems (see examples in~\cref{section:bellman-rep}), the DEC scales as $\odec_\gamma(\cM)\le \tO(d/\gamma)$, where $d$ is some structural complexity (such as Eluder dimension) of $\cM$. In this case,~\cref{thm:E2D-TA} with the optimal choice of $\gamma>0$ implies a regret bound
\begin{align}
\label{eqn:E2D-TA-regret}
    \regdm\asymp \inf_{\gamma>0}\set{T\odec_\gamma(\cM) + \gamma\log(\abs{\cM}/\delta)} \le \tO\paren{ \sqrt{d\log\abs{\cM}\cdot T} }, 
\end{align}
which has the optimal $\sqrt{T}$ scaling in $T$, and depends on both the structural complexity $d$ as well as log-cardinality of the model class $\cM$. The proof of~\cref{thm:E2D-TA} builds upon the analysis of E2D meta-algorithms~\citep{foster2021statistical} combined with the online estimation guarantee for the \Vovkalg~subroutine (\crefapp{C.2}{corollary:restate-lemma-tempered-aggregation}). 

We remark that \etod~is slightly different from the instantiations of the E2D meta-algorithms in~\citet{foster2021statistical}, which use either Vovk's aggregating algorithm or problem-specific estimation subroutines outputting a \emph{deterministic} model estimator. While the original instantiation of E2D in \citet{foster2021statistical} may also achieve the upper bound of \cref{thm:E2D-TA}, we here introduce the modified algorithm \etod{} since it could easily generalize to learning goals other than no-regret learning (\cref{sec:gen-dec}), the main purpose of this paper.

\section{PAC reinforcement learning via PAC DEC}
\label{section:eec}

The original DEC definition and the E2D algorithm were designed for no-regret learning. In this paper, we generalize the DEC framework to support additional reinforcement learning goals beyond no-regret. One such goal is PAC learning, which requires the agent to identify a near-optimal policy after exploring for $T$ episodes. Unlike no-regret learning, PAC learning does not require the executed policies $\set{\pi^t}_{t=1}^T$ (the ``exploration policies'') during learning to be of high quality. To capture the complexity of PAC learning, we introduce the following definition PACDEC:
\begin{definition}[PACDEC]\label{def:edec}
The PAC Decision-Estimation Coefficient (PACDEC) of a model-class $\cM$ with respect to $\omu\in\Delta(\cM)$ and parameter $\gamma>0$ is defined as
\begin{align*}
    \eec_{\gamma}(\cM,\omu) \defeq \inf_{\substack{\pexp,\pout\in\Delta(\Pi)}}\sup_{M\in \cM} \bypart{ \EE_{\pi \sim \pout}\left[f^M(\pi_M)-f^M(\pi)\right] \vspace{-200pt}}{ -\gamma \EE_{\pi \sim \pexp,\oM \sim \omu}\left[ \wtdH^2(M(\pi),\oM(\pi))\right] }.
\end{align*}
Further, define $\oeec_{\gamma}(\cM) \defeq \sup_{\omu\in\Delta(\cM)} \eec_{\gamma}(\cM,\omu)$.
\end{definition}
The main distinction between the PACDEC and the Regret DEC (\cref{definition:dec}) is the infimum being taken over separate policy distributions $\pexp$ and $\pout$. Specifically, $\pexp$ (the ``exploration policy distribution'') is used in the information gain term, while $\pout$ (the ``output policy distribution'') appears in the suboptimality term. In contrast, Regret DEC constrains both terms to use the same policy distribution. This captures the differing goals between PAC learning and no-regret learning: \revise{PAC learning does not mandate that exploration policies achieve low suboptimality.}

\begin{algorithm}[t]
	\caption{\eetod} 
	\begin{algorithmic}[1]
	\label{alg:E2D-exp}
	\REQUIRE Parameter $\gamma>0$; Learning rate $\etap\in(0, \frac{1}{2})$, $\etar>0$.
	\STATE Initialize $\mu^1\setto {\rm Unif}(\cM)$.
	\FOR{$t=1,\ldots,T$}
    \STATE Set $(\pexp^t,\pout^t)\setto\argmin_{(\pexp,\pout)\in\Delta(\Pi)^2}\hV^{\mu^t}_{\pac,\gamma}(\pexp,\pout)$, where $\hV^{\mu^t}_{\pac}$ is defined in~\cref{equation:vtp-exp}.
    \label{line:ee2d-pt}
    \STATE Sample $\pi^t\sim \pexp^t$. Execute $\pi^t$ and observe $(o^t,\br^t)$. %
    \STATE Update randomized model estimator by \Vovkalg:
    \begin{align}
    \label{equation:tempered-aggregation-e2dexp}
        \mu^{t+1}(M) \; \propto_M \; \mu^{t}(M) \cdot \exp\paren{\etap \log \P^{M,\pi^t}(o^{t}) - \etar\ltwo{\br^t - \bR^M(o^t)}^2 }.
    \end{align}
    \label{line:ee2d-ta}
    \ENDFOR
    \ENSURE Policy $\hatpout:=\frac{1}{T}\sum_{t=1}^T \pout^t$.
   \end{algorithmic}
\end{algorithm}

\paragraph{Algorithm and theoretical guarantee}
We propose the following \eetod~algorithm (\cref{alg:E2D-exp}) for PAC RL. Define risk function $\hV_{\pac, \gamma}^{\mu^t}:\Delta(\Pi)\times\Delta(\Pi)\to\R$ as
\begin{align}\label{equation:vtp-exp}
\begin{aligned}
    \hV^{\mu^t}_{\pac,\gamma}(\pexp,\pout):=\sup_{M\in\cM}
    \bypart{\EE_{\pi \sim \pout}\left[f^M(\pi_M)-f^M(\pi)\right]}{-\gamma \EE_{\pi \sim \pexp,\hM^t\sim\mu^t}\left[ \wtdH^2(M(\pi),\hM^t(\pi))\right]}.
\end{aligned}
\end{align}
\eetod~(\cref{alg:E2D-exp}) is similar as \etod~(\cref{alg:E2D-TA}), except that in each iteration, \cref{alg:E2D-exp} finds $(\pexp^t,\pout^t)$ that jointly minimizes $\hV^{\mu^t}_{\pac,\gamma}(\cdot,\cdot)$ (Line~\ref{line:ee2d-pt}), executes $\pi^t\sim\pexp^t$ to collect data, and returns $\hatpout=\frac{1}{T}\sum_{t=1}^T \pout^t$ as the output policy after $T$ episodes. The randomized model estimators $\set{\mu^t}_{t=1}^T$ are updated using \Vovkalg, same as in~\cref{alg:E2D-TA}. We show that \eetod~achieves the following PAC sub-optimality bound.

\begin{theorem}[PAC RL with \eetod]
\label{thm:E2D-exp}
Choosing $\etap=\etar=1/3$, \cref{alg:E2D-exp} achieves the following PAC guarantee with probability at least $1-\delta$:
\begin{align*}
   \suboptpac \defeq  f^{M^\star}(\pi_{M^\star})-\EE_{\pi\sim \hatpout}\left[f^{M^\star}(\pi)\right]\leq \oeec_{\gamma}(\cM)+10\frac{\gamma \log(\abs{\cM}/\delta)}{T}.
\end{align*}
\end{theorem}
For problems with $\oeec_\gamma(\cM)\lesssim \tbO{d/\gamma}$,~\cref{thm:E2D-exp} shows that \eetod~achieves $\suboptpac\le \tbO{\sqrt{d\log\abs{\cM}/T}}$ (by tuning $\gamma$), which implies an $\tbO{d\log\abs{\cM}/\eps^2}$ sample complexity for learning an $\eps$ near-optimal policy.

In the literature, PAC RL algorithms that decouple exploration from output policies have been developed across problems \citep{jiang2017contextual,du2021bilinear,liu2022partially}. These methods typically design exploration policies manually, often by appending uniform actions or leveraging domain knowledge. In contrast, \eetod~automatically \emph{learns} the best exploration policy $p_{\exp}\in\Delta(\Pi)$ by minimizing~\cref{equation:vtp-exp}. This bypasses the need for specialized, hand-engineered exploration, considerably simplifying algorithm design.

\paragraph{Lower bound}
We show that PACDEC provides an information-theoretic lower bound for PAC RL. This result generalizes the regret lower bounds in \citet{foster2021statistical}.
\begin{proposition}[Lower bound for PAC RL]
\label{prop:eec-lower-bound-demo}
For any model class $\cM$, $T\in\Z_{\ge 1}$, and any $T$-round algorithm $\Algo$, there exists a $M^\star\in\cM$ such that
\begin{align*}
\EE^{M^\star,\Algo}\brac{\suboptpac} \geq \frac13 \cdot \eec_{\gammaT}(\cM),
\end{align*}
where $\gammaT\asymp T$ is defined in~\cref{prop:gen-lower-bound-demo}. 
\end{proposition}
The upper and lower bounds in \cref{thm:E2D-exp} and~\cref{prop:eec-lower-bound-demo} together demonstrate that a controlled PACDEC is both necessary and sufficient for PAC RL, analogous to the role of the Regret DEC in no-regret learning~\citep{foster2021statistical}. A more detailed discussion on the necessity and sufficiency of PACDEC (along with the proof of \cref{thm:E2D-exp} and \cref{prop:eec-lower-bound-demo}) is deferred to \cref{sec:gen-dec}, where we discuss the implications of our lower bounds for any general learning goal. \revise{We also note that tighter lower bounds can be derived based on a constrained version of the PACDEC~\citep{foster2023tight} (as discussed in \crefapp{A.2}{appdx:constrained-dec}).}

\subsection{Relationship with DEC and no-regret learning}

One potential approach for deriving PAC upper bound using the Regret DEC framework (as opposed to the PACDEC) is to first obtain a regret bound in terms of DEC as in~\cref{thm:E2D-TA}, and then apply the standard online-to-batch conversion (e.g.~\citet{jin2018q}). Conversely, we can also convert a PAC algorithm to a no-regret learning algorithm using the standard \emph{explore-then-commit} procedure. Therefore, up to polynomial factors, PAC learning and no-regret learning can be viewed as equivalent. A similar relationship holds between the PACDEC and Regret DEC, as formalized in the following proposition.

\begin{proposition}[Relationship between PACDEC and Regret DEC]\label{prop:eec-to-dec}
For any model class $\cM$, $\alpha\in(0,1), \gamma>0$ and $\omu\in\Delta(\cM)$, we have 
\begin{align*}
\oeec_\gamma(\cM) \stackrel{(i)}{\leq} \odec_{\gamma}(\cM) \stackrel{(ii)}{\leq} \inf_{\alpha\in(0,1)}\set{\alpha+(1-\alpha)\oeec_{\gamma\alpha/(1-\alpha)}(\cM)}.
\end{align*}
\end{proposition}

Inequality (i) states that for any model class with a bounded DEC, the same upper bound applies to the PACDEC. This means that \eetod~achieves a sample complexity no worse than that of \etod~(\cref{thm:E2D-exp} \&~\cref{thm:E2D-TA}). On the other hand, the converse inequality (ii) generally provides a \emph{lossy} conversion---For a model class with low PACDEC, its DEC will also be bounded but with a slightly worse rate. See \crefapp{D.2}{appendix:eec-vs-dec} for an in-depth discussion of the relationship between the PACDEC and Regret DEC. 

Although PACDEC and DEC can be viewed as equivalent up to polynomial factors, the PACDEC (rather than the DEC) is the natural complexity measure for PAC RL, as it \emph{tightly} captures the sample complexity of PAC RL (whereas the Regret DEC cannot). To illustrate, the following proposition demonstrates that for a simple class of structured bandits, the best achievable sample complexity through no-regret learning and bounding the Regret DEC is $\tbO{1/\eps^3}$, since the minimax optimal regret scales as $T^{2/3}$. In contrast, PAC RL achieved by bounding the PACDEC gives the tight sample complexity of $\tbO{1/\eps^2}$.

\begin{proposition}[Informal]\label{prop:rev-bandit}
    For every $d\geq 1$, there exists $\cM$ a class of ``bandits with revealing actions'' such that (up to logarithmic factors)
    \begin{align*}
        \eec_{\gamma}(\cM)\asymp \frac{d}{\gamma}, \qquad
        \dec_{\gamma}(\cM)\asymp \min\set{\sqrt{\frac{d}{\gamma}}, \frac{2^d}{\gamma}}.
    \end{align*}
    In particular, any algorithm on $\cM$ has regret of $\Om{\min\set{T^{2/3},\sqrt{2^d T}}}$. %
\end{proposition}
The details are contained in \crefapp{D.3}{appdx:proof-rev-bandit}.

\section{Generalizing DEC: A unifying framework for RL tasks}\label{sec:gen-dec}

In this section, we generalize the DEC framework to capture the intrinsic sample complexity across a range of reinforcement learning goals beyond no-regret RL and PAC RL. These general learning goals include reward-free RL, all-policy model-estimation tasks, and preference-based RL~\revise{(\cref{tab:goals})}. Before diving into the details of each specific learning goal in \cref{section:rfec} - \ref{section:pbrl}, we first introduce a unified framework encompassing all of these examples. 

As discussed in \cref{sec:prelim}, a general learning goal of RL (denoted as $\GOAL$) is associated with a convex strategy space $\ssps$ (which could depend on the problem class $(\cM,\Pi)$) and a convex functional $\subopts_M: \ssps \to \R$. To capture the sample complexity of such a learning goal, we define the following \DECt~(pronounced as ``generalized DEC''; here the prefix $\GOAL$ indicates that this DEC is tied to the general learning goal $\GOAL$): 
\begin{definition}\label{def:gen-dec}
Consider a general learning goal $\GOAL$, specified by a given sub-optimality measure $\subopts_M$ and a decision domain $\csps\subseteq \Delta(\Pi)\times \ssps$. For this learning goal, the \DECt~of a model class $\cM$ is defined as
\begin{align*}
    \DEC_{\gamma}(\cM,\omu) \defeq & \inf_{(\pexp,\pout)\in\csps}\sup_{M\in \cM}\subopts_{M}(\pout)-\gamma \EE_{\pi \sim \pexp,\oM \sim \omu}\brac{ \wtdH^2\paren{M(\pi),\oM(\pi)}},
\end{align*}
for any $\omu\in\Delta(\cM)$ and $\gamma>0$. %
Further, we define $\DEC_{\gamma}(\cM) \defeq \sup_{\omu\in\Delta(\cM)} \DEC_{\gamma}(\cM,\omu)$.
\end{definition}

\cref{def:gen-dec} generalizes the standard DEC (\cref{definition:dec}) as the complexity for no-regret RL, as well as PACDEC (\cref{def:edec}) as the complexity for PAC RL. 
\revise{
It measures the trade-off between two terms: suboptimality term $\subopts_{M}(\pout)$, representing the sub-optimality of the output policy $\pout$ under the model $M$; and the information gains $\wtdH^2(M(\pi), \oM(\pi))$ under $\pi\sim \pexp$, measured by the divergence from the reference model. The constraint $(\pexp,\pout)\in \csps$ corresponds to the learning goal.
}
In this paper, we will consider learning goals $\GOAL$ which fall into one of the following two categories: %
\begin{itemize}
\item[(1)] Generalized PAC learning: We set $\csps = \csppac:=\Delta(\Pi)\times \ssps$. In this scenario, the output strategy $\pout$ is not tied to the exploration policy $\pexp$. The goal of the agent is to find an output strategy $\pout\in\ssps$ minimizing the \emph{sub-optimality}:
\begin{align*}
    \subopts\defeq \subopt^{\mathsf{G}}_{\Ms}(\pout).
\end{align*}
An example is PAC RL from \cref{section:eec}, with $\subopts_{M}(p)\defeq f^M(\pi_M)-\EE_{\pi\sim p} f^M(\pi)$. 
\item[(2)] Generalized No-regret learning: We set the strategy space $\ssps=\Delta(\Pi)$ and we set $\csps = \cspreg :=\set{(\pexp,\pout)=(p, p): p\in\Delta(\Pi)}$. The goal of the agent is to minimize \emph{generalized regret}:
\begin{align}
\textstyle \Regs\defeq \sum_{t=1}^T \subopts_{\Ms}(p^t), 
\end{align}
where $p^t\in\Delta(\Pi)$ is the strategy the agent executes in the $t$-th round. An example is the standard no-regret RL from \cref{sec:prelim}, upon choosing $\subopts_{M}(p)\defeq f^M(\pi_M)-\EE_{\pi\sim p} f^M(\pi)$ (see also discussion in \crefapp{E.3}{appdx:specifying}).
\end{itemize}

\revise{For example, reward-free learning (\cref{section:rfec}) and model estimation (\cref{section:mdec}) are both generalized PAC learning goals, and we formulate preference-based learning naturally as a generalized no-regret learning goal (\cref{section:pbrl}). For these specific learning goals, we summarize the definitions of the corresponding strategy space $\ssp$ and suboptimality measure $\subopts$ in \cref{tab:goals}. }

\begin{table}[t]
\renewcommand{\arraystretch}{1.3}
    \centering
    \begin{tabular}{|c|c|c|c|c|}
        \hline
        Learning goal $\GOAL$ & \DECt & $\ssps$ & $\csps$ & $\subopts$ \\\hline
        No-regret RL & DEC & $\Delta(\Pi)$ & $\cspreg$ & \multirow{2}{*}{$f^M(\pi_M)-\EE_{\pi\sim p}\big[f^M(\pi)\big]$} \\\cline{1-4}
        PAC RL & \!PACDEC\! & $\Delta(\Pi)$ & $\csppac$ & \\\hline
        Reward-free & RFDEC & \!$\cR\to\Delta(\Pi)$\! & $\csppac$ & $\sup_{R\in\cR} \big\{ f^{M,R}(\pi_{M,R})-\EE_{\pi\sim p(R)}\big[f^{M,R}(\pi)\big] \big\}$ \\\hline
        \!Model-estimation\! & \!AMDEC\! & $\Delta(\cM)$ & $\csppac$ & $\max_{\pi}\EE_{\hM\sim p}\big[ \wdTV\big( M(\pi), \hM(\pi) \big) \big]$ \\\hline
        \!Preference-based\! & PBDEC & $\Delta(\Piab)$ & $\cspreg$ & \!$\sup_{\pis} \EE_{(\pi_1,\pi_2)\sim p}\big[ \Cmp^M(\pi_\star, \pi_{1})+\Cmp^M(\pi_\star, \pi_{2})-1 \big]$\! \\\hline %
    \end{tabular}
    \caption{Definition of the strategy space $\ssps$, decision domain $\csps$ and sub-optimality measure $\subopts$ for various learning goals $\GOAL$.~\revise{The relationship of these learning goals are illustrated in \cref{figure:relate}. } For details, see the corresponding sections and also the discussions in the \crefapp{E.3}{appdx:specifying}.  }
    \label{tab:goals}
\end{table}

\paragraph{Algorithm and theoretical guarantee} For the general learning goal $\GOAL$, we propose the \getod~algorithm (\cref{alg:E2D-gen}). We define risk function $\hV_{\GOAL,\gamma}^{\mu^t}:\csps\to\R$ as
\begin{align}\label{equation:vtp-gen}
   \hV^{\mu^t}_{\GOAL,\gamma}(\pexp,\pout):=\sup_{M\in \cM}\subopts_{M}(\pout)-\gamma \EE_{\pi \sim \pexp,\oM \sim \omu}\brac{ \wtdH^2\paren{M(\pi),\oM(\pi)}}.
\end{align}
\getod~(\cref{alg:E2D-gen}) generalizes both \etod~(\cref{alg:E2D-TA}) and \eetod~(\cref{alg:E2D-exp}), where in each iteration, \getod~finds $(\pexp^t,\pout^t)$ that jointly minimizes the objective $\hV^{\mu^t}_{\GOAL,\gamma}(\cdot,\cdot)$ (Line~\ref{line:ge2d-pt}), executes $\pi^t\sim\pexp^t$ to collect data, and returns $\hatpout=\frac{1}{T}\sum_{t=1}^T \pout^t$ as the output policy after $T$ episodes.

\begin{algorithm}[t]
	\caption{\getod} 
	\begin{algorithmic}[1]
	\label{alg:E2D-gen}
	\REQUIRE Learning goal $\GOAL$, model class $\cM$, parameter $\gamma>0$, learning rate $\etap, \etar\geq 0$.
	\FOR{$t=1,\ldots,T$}
    \STATE Set $(\pexp^t,\pout^t)\setto\argmin_{(\pexp,\pout)\in\csps}\hV^{\mu^t}_{\GOAL,\gamma}(\pexp,\pout)$, where $\hV^{\mu^t}_{\GOAL,\gamma}$ is defined in~\cref{equation:vtp-gen}.
    \label{line:ge2d-pt}
    \STATE Sample $\pi^t\sim \pexp^t$. Execute $\pi^t$ and observe $(o^t,\br^t)$.
    \STATE Update randomized model estimator $\mu^{t+1}$ by \Vovkalg~subroutine:
    \begin{align}
    \label{eqn:E2D-gen-TA}
        \mu^{t+1}(M) \; \propto_M \; \mu^{t}(M) \cdot \exp\paren{\etap \log \P^{M,\pi^t}(o^{t}) - \etar\ltwo{\br^t - \bR^M(o^t)}^2 }.
    \end{align}
    \label{line:ge2d-ta}
    \ENDFOR
    \ENSURE Strategy $\hatpout:=\frac{1}{T}\sum_{t=1}^T \pout^t$.
   \end{algorithmic}
\end{algorithm}

\begin{theorem}%
\label{thm:E2D-gen}
Consider \cref{alg:E2D-gen} which instantiates the \Vovkalg~subroutine \cref{eqn:E2D-gen-TA} with proper choice of learning rate $(\etap,\etar)$ (as in \crefa{Corollary C.2}{corollary:restate-lemma-tempered-aggregation}). Then, for generalized PAC learning goals, \cref{alg:E2D-gen} achieves the following guarantee with probability at least $1-\delta$:
\begin{align*}
   \subopts\leq \DEC_{\gamma}(\cM)+10\frac{\gamma \log(\abs{\cM}/\delta)}{T}.
\end{align*}
Furthermore, for generalized no-regret learning goals, \cref{alg:E2D-gen} achieves the following guarantee with probability at least $1-\delta$:
\begin{align*}
    \Regs\leq T\cdot \DEC_{\gamma}(\cM)+10\gamma \log(\abs{\cM}/\delta).
\end{align*}
\end{theorem}
To the best of our knowledge, \cref{alg:E2D-gen} offers the first template for designing algorithms for general learning goals with statistical efficiency guarantees. Consider a problem with $\DEC_\gamma(\cM)\leq \tbO{d/\gamma}$. \cref{thm:E2D-gen} then implies that for generalized PAC learning goals, \getod~achieves $\subopts\le\eps$ within $\tbO{d\log\abs{\cM}/\eps^2}$ episodes; for generalized no-regret learning goals, \getod~achieves a regret bound of $\Regs\leq\tO{\sqrt{d\log\abs{\cM}T}}$. The proof of \cref{alg:E2D-gen} is provided in \crefapp{E.1}{appendix:proof-gen-dec}.

\paragraph{Lower bound}
We show that \DECt~also gives an information-theoretic lower bound for the corresponding learning goal $\GOAL$ (proof in \crefapp{E.2}{appendix:proof-gen-lower-bound}). %

\begin{theorem}\label{prop:gen-lower-bound-demo}
Consider a general learning goal $\GOAL$, a model class $\cM$, and $T \geq 1$ a fixed integer. Suppose that the sub-optimality $\subopts_M(p)\in[0,1]$ for any model $M\in\cM$ and $p\in\ssp$. Define $\gammaT=c_0\log(2T)\cdot T$, where $c_0$ is a large absolute constant.
Then, for any $T$-round algorithm $\Algo$, the following holds:
\begin{enumerate}
\item[(1)] For $\csps=\csppac$ (generalized PAC learning), there exists a $M^\star\in\cM$ such that
\begin{align*}
\EE^{M^\star,\Algo}\brac{\subopts} \geq \frac13\cdot \DEC_{\gammaT}(\cM),
\end{align*}

\item[(2)] For $\csps=\cspreg$ (generalized no-regret learning), there exists a $M^\star\in\cM$ such that
\begin{align*}
\EE^{M^\star,\Algo}\brac{\Regs} 
\geq&~ \frac{T}{3}\cdot \DEC_{\gammaT}(\cM).
\end{align*}
\end{enumerate}
\end{theorem}
The upper and lower bounds in \cref{thm:E2D-gen} and~\cref{prop:gen-lower-bound-demo} together demonstrate that a controlled \DECt~is both necessary and sufficient for the corresponding learning goal $\GOAL$. Instantiated to concrete learning goals (such as PAC learning, reward-free learning, preference-based learning) also shows that the corresponding DEC is both necessary and sufficient for each learning goal. %

To see this, we can consider any generalized PAC learning goal $\GOAL$ and a model class $\cM$, and the following measure for the minimax-optimal sample complexity of learning an $\eps$-optimal strategy for problem $(\GOAL,\cM)$:
\begin{align*}
    T^\star(\GOAL,\cM,\eps)\defeq \inf\set{ T: \exists\text{ $T$-round algorithm $\Algo$ s.t. }\max_{M\in\cM}\EE^{M,\Algo}[\subopts]\leq \eps}.
\end{align*}
Letting $T^{\dec}(\GOAL,\cM,\eps)\defeq \inf\set{\gamma: \DEC_\gamma(\cM)\leq \eps}$ be the sample complexity induced by the \DECt. The above results show that $T^\star$ can be upper and lower bounded by $T^{\dec}$ as (omitting logarithmic factors)
\begin{align}\label{eqn:SC}
    T^{\dec}(\GOAL,\cM,3\eps)\leqsim T^\star(\GOAL,\cM,\eps) \leqsim T^{\dec}(\GOAL,\cM,\eps/2)\cdot \frac{\log\abs{\cM}}{\eps},
\end{align}
where the lower bound is implied by \cref{prop:gen-lower-bound-demo} and the upper bound follows from \cref{thm:E2D-gen}. Similar results also hold for any generalized no-regret learning goal. Therefore, up to a factor of $\eps^{-1}$ and the model class complexity $\log\abs{\cM}$ (which in model-based RL is \emph{assumed} to be tractable), the statistical complexity of any learning problem $(\GOAL,\cM)$ is completely characterized by the \DECt.

\paragraph{Improved lower bounds}
We remark that the lower bounds implied by \cref{prop:gen-lower-bound-demo} are possibly not tight for specific problem classes. For example, for model class $\cM$ with $\eec_\gamma(\cM) \asymp \frac{d}{\gamma}$, \cref{prop:gen-lower-bound-demo} implies a $\Tilde{\Omega}(d/\epsilon)$ sample complexity lower bound for PAC RL, which can be undesirable in many scenarios where we expect a lower bound of $\Omega(d/\epsilon^2)$ (e.g. tabular MDPs).
In \crefapp{E.4}{appdx:gen-lower-local}, we demonstrate how a localized version of \cref{prop:gen-lower-bound-demo} can be used to derive tighter lower bounds for specific problems. For example, by bounding the \emph{localized} PACDEC for tabular MDPs, we recover the known $\Omega(HSA/\eps^2)$ PAC sample complexity lower bound for tabular MDP~\citep{domingues2021episodic}. \revise{Further, for PAC RL, tighter lower bounds can be derived \citep{foster2023tight} based on a constrained version of the PACDEC. For a detailed discussion, see \crefapp{A.2}{appdx:constrained-dec}.}

\paragraph{Instantiations}
For the remainder of this section, we instantiate our general framework to three exemplary learning goals: reward-free learning, all-policy model estimation, and preference-based RL. For each learning goal, we present the corresponding definition of the \DECt, as well as the corresponding E2D algorithm. %
\cref{tab:goals} %
illustrates how the results of each example can be derived from the general framework.

\subsection{Reward-free learning via Reward-Free DEC}
\label{section:rfec}

The goal of \emph{reward-free RL}~\citep{jin2020reward} is to optimally explore the environment without observing rewards, so that after the exploration phase, a near-optimal policy for any given reward function can be computed using only the collected trajectory data, without further environment interaction. This setting is particularly important in scenarios where reward functions are iteratively engineered to encourage desired behavior via trial and error (e.g., constrained RL formulations). In such cases, repeatedly invoking the same reinforcement learning algorithm with different rewards can be highly sample inefficient, making reward-free learning a more efficient solution.

To formalize the reward-free RL setting within the DMSO framework, we consider in this section the model $M=\Pm$ which only specifies the distribution $\Pm$ over observations $o\in\cO$ (the trajectory), with an empty reward vector $\br=\emptyset$. In this case, for any model $\Pm$ and reward function $R: \cO \to [0, 1]$, we define the value function $f^{\Pm,R}(\pi)\defeq \E^{\Pm,\pi}[R(o)]$, and let $\pi_{\Pm,R}\defeq \argmax_{\pi\in\Pi} f^{\Pm,R}(\pi)$ and $f^{\Pm,R}(\pi_{\Pm,R})$ denote the optimal policy and optimal value for the pair $(\Pm,R)$, respectively.

We now formulate the reward-free RL goal within our general framework by regarding $\cP$ as a class of transition dynamics models. Given a class $\cR\subseteq \set{R:\cO\to[0,1]}$ of \emph{mean} reward functions, the strategy space is $\ssprf=\set{ p: \cR\to\Delta(\Pi)}$, where a strategy maps any reward function to a distribution over policies. The sub-optimality of any $p\in\ssprf$ under a model $\Pm\in\cP$ is measured by
\begin{align*}
    \suboptrf_{\Pm}(p)\defeq \sup_{R\in\cR} \set{ f^{\Pm,R}(\pi_{\Pm,R})-\EE_{\pi\sim p(R)}\brac{f^{\Pm,R}(\pi)} }.
\end{align*}

Based on this, we can define the Reward-Free DEC (RFDEC) to capture the complexity of reward-free learning, derived by specifying $\GOAL=\rm rf$ in \cref{def:gen-dec}. The detailed derivation of RFDEC is presented in \crefapp{E.3}{appdx:specifying}.

\begin{definition}[Reward-Free DEC]\label{def:rfdec}
The Reward-Free Decision-Estimation Coefficient (RFDEC) of model class $\cM =\cP$, along with $\cR$ the class of mean reward function, with respect to $\omu\in\Delta(\cP)$ and parameter $\gamma>0$, is defined as
\begin{align*}
\rfecg(\cP,\omu):=\inf_{\pexp\in\Delta(\Pi)}\sup_{R\in\cR}\inf_{\pout\in\Delta(\Pi)}\sup_{\Pm\in\cP} & \Big\{ \EE_{\pi\sim \pout}\left[f^{\Pm, R}(\pi_{\Pm, R})-f^{\Pm, R}(\pi)\right] \\
& \quad -\gamma \EE_{\pi \sim \pexp}\EE_{\oPm \sim \omu}\left[ \dH^2(\Pm(\pi),\oPm(\pi))\right] \Big\}.
\end{align*}
Furthermore, we define $\orfecg(\cP):=\sup_{\omu\in\Delta(\cP)}\rfecg(\cP,\omu)$.
\end{definition}
\revise{
The RFDEC can be interpreted as a modification of the PACDEC, where we further insert $\sup_{R\in\cR}$ to reflect that we care about the suboptimality of the output policy $\pout=\pout(R)$ under \emph{any} given reward function $R\in\cR$, and use $\dH^2(\Pm(\pi),\oPm(\pi))$ as the divergence to reflect that we only observe the state-action trajectories $o^t$ (without rewards).
}

\paragraph{Algorithm and theoretical guarantee}
We propose \rfalg~(full description in~\crefapp{E.5}{alg:E2D-RF}), a specification of \getod, for reward-free learning. \rfalg~works in two phases. In the exploration phase, in the $t$-th episode, the algorithm finds $\pexp^t\in\Delta(\Pi)$ that minimizes the sup-risk $\sup_{R\in\cR} \hV^{\mu^t}_{\rf,\gamma}(\cdot,R)$, where
\begin{align}
\label{equation:vtp-rf}
\begin{aligned}
    \hV^{\mu^t}_{\rf,\gamma}(\pexp,R) \defeq \inf_{\pout}\sup_{\Pm\in\cP}
    \bypart{
    \E_{\pi \sim p_{\out}}  \brac{ f^{\Pm,R}(\pi_{\Pm,R})-f^{\Pm,R}(\pi) } }{
    - \gamma \E_{\pi \sim \pexp, \hPm^t\sim \mu^t}\brac{ \dH^2(\Pm(\pi), \hPm^t(\pi)) } }.
\end{aligned}
\end{align}
Then, in the planning phase, for any given reward function $R^\star\in\cR$, the algorithm computes $\pout^t(R^\star)$ that minimizers $\inf_{p_{\out}}$ in $\hV^{\mu^t}_{\rf,\gamma}(\pexp^t,R^\star)$. Finally, it outputs the average policy $\hatpout(R^\star)\defeq \frac{1}{T}\sum_{t=1}^T p_{\out}^t(R^\star)$. The following theorem is an instantiation of \cref{thm:E2D-gen} to reward-free learning.

\begin{theorem}[Reward-Free upper bound]
  \label{thm:E2D-RF}
\rfalg~achieves the following with probability at least $1-\delta$: 
\begin{align*}
    \suboptrf_{\Pm^\star}(\hatpout)
    \byparts{=\sup_{R^\star\in\cR} \set{ f^{\Pm^\star,R^\star}(\pi_{\Pm^\star,R^\star})-\EE_{\pi\sim \hatpout(R^\star)}\left[f^{\Pm^\star,R^\star}(\pi)\right]} }{
    \leq \orfecg(\cP) + \gamma\frac{3\log(\abs{\cP}/\delta)}{T} }.
\end{align*}
\end{theorem}

For problems with $\orfec_\gamma(\cP)\lesssim \tbO{d/\gamma}$, by tuning $\gamma>0$, \cref{thm:E2D-RF} shows that \rfalg~achieves $\suboptrf\le\eps$ within $\tbO{d\log\abs{\cP}/\eps^2}$ episodes. The only known such general guarantee for reward-free RL is the recently proposed RFOlive algorithm of~\citet{chen2022statistical}, which achieves sample complexity $\tbO{{\rm poly}(H)\cdot d_{\rm BE}^2\log(\abs{\cF}\abs{\cR})/\eps^2}$ in the model-free setting\footnote{Here, $\cF$ denotes the value class, $\cR$ denotes the reward class, and $d_{\rm BE}$ denotes the Bellman-Eluder dimension of a certain class of \emph{reward-free Bellman errors} induced by $\cF$.}. \cref{thm:E2D-RF} can be seen as a generalization of this result to the model-based setting, with a more general \emph{form} of structural condition (RFDEC). Further, unlike~\citet{chen2022statistical}, our guarantee \emph{does not further} depend on the statistical complexity (e.g. log-cardinality) of $\cR$ once we assume bounded RFDEC.

\paragraph{Lower bound} An instantiation of \cref{prop:gen-lower-bound-demo} shows that RFDEC gives the following lower bound for reward-free learning. %
\begin{proposition}[Reward-free lower bound]
\label{prop:rfec-lower-bound-demo}
For any model class $\cP$ and reward function class $\cR$, $T\in\Z_{\ge 1}$, and any algorithm $\Algo$ with output $\hatpout\in\ssprf$, there exists a $\Pm^\star\in\cP$ such that
\begin{align*}
\EE^{\Pm^\star,\Algo}\brac{\suboptrf_{\Pm^\star}(\hatpout)} \geq \frac13 \cdot \rfec_{\gammaT}(\cM),
\end{align*}
where $\gammaT\asymp T$ is defined in~\cref{prop:gen-lower-bound-demo}.
\end{proposition}

\revise{
The upper and lower bounds in \cref{thm:E2D-RF} and~\cref{prop:rfec-lower-bound-demo} together demonstrate that the suboptimality of reward-free learning is characterized by the RFDEC.
}

\subsection{All-Policy Model-Estimation via AMDEC}
\label{section:mdec}

The goal of the all-policy model-estimation task is to estimate an approximate model that captures the true transition and rewards under \emph{any policy}\footnote{This is to be distinguished from the guarantee of the \Vovkalg~subroutine, which only achieves online model estimation guarantee on the \emph{deployed} policies $(\pi^1,\cdots,\pi^T)$.}. The all-policy model estimation provides a stronger guarantee compared to PAC RL and reward-free learning, because an accurately estimated model is enough for outputting a near-optimal under any reward function. We note that such a learning goal is common in dynamical system learning \citep{kumar2015stochastic}, but it is largely unknown in RL literatures, as estimating the model can be much more challenging than learning a near-optimal policy.

More specifically, for any model $M$, $M'$, and policy $\pi \in \Pi$, we define divergence functions
\begin{align}
\wDTV{ M(\pi), M'(\pi)} \defeq&~ \dTV( \MP(\pi), \Pm^{M'}(\pi))+ \E_{o\sim \MP(\pi)}\brac{ \big\|\bR^M(o) - \bR^{M'}(o)\big\|_1 }, \label{equation:def-wDTV}\\
\DTVPi{M, M'}\defeq&~ \max_{\pi\in\Pi}\wDTV{ M(\pi), M'(\pi) }.\label{equation:def-dtvpi}
\end{align}
The divergence $\wDTV{\cdot}$ is an $L_1$ variant of the squared divergence $\wtdH(\cdot)$ defined in \cref{equation:modified-hellinger}.
The divergence $\DTVPi{M, M'}$ measures how close the models $(M, M')$ are over all policies. The goal of all-policy model estimation is to interact with the true model $M^\star$ in $T$ rounds and output an estimated model $\hM$ such that $\DTVPi{M^\star, \hM} \le \eps$.

Towards this goal, we instantiate our framework with $\sspme=\Delta(\cM)$ (corresponding to an \emph{estimation of the model}), $\csp=\csppac$, and the objective is to minimize 
\begin{align*}
    \suboptme_M(p)\defeq \max_{\pi\in\Pi}\EE_{\hM\sim p}\big[ \wdTV\big( M(\pi), \hM(\pi) \big) \big], \qquad \forall p\in\sspme.
\end{align*}
With this specification, we now derive the following All-Policy Model-Estimation DEC (AMDEC) by specifying $\GOAL=\rm me$ in \cref{def:gen-dec}, which captures the complexity of the all-policy model-estimation task.

\begin{definition}[All-Policy Model-Estimation DEC]
\label{definition:mdec}
The All-policy Model-estimation DEC (\MDEC) of $\cM$ with respect to reference measure $\omu\in\Delta(\cM)$ is defined as
\begin{align}
\begin{split}
    \mdec_\gamma(\cM,\omu)\defeq \inf_{\substack{\pexp\in\Delta(\Pi),\\ \muout\in\Delta(\cM)}}\sup_{M\in\cM,\opi\in\Pi}\Big\{&~\E_{\hM \sim \muout}  \brac{ \wDTV{ M(\opi), \hM(\opi)} } \\
    &~- \gamma \E_{\pi \sim \pexp}\E_{\oM\sim \omu}\brac{ \wtdH^2(M(\pi), \oM(\pi)) } \Big\}.
\end{split}
\end{align}
Furthermore, we define $\omdec_\gamma(\cM)\defeq \sup_{\omu\in\Delta(\cM)} \mdec_\gamma(\cM, \omu)$. 
\end{definition}

\revise{The AMDEC can be regarded as a measure of the optimal trade-off between the estimation error of the output $\muout$ and the information gain of the exploration policies $\pexp$.}

\paragraph{Algorithm and theoretical guarantee}
We propose \mealg~(full description in~\crefapp{E.6}{alg:E2D-ME}), specified from \getod, for the all-policy model-estimation task. For each step $t$, the algorithm finds $(\pexp^t,\muout)$ that jointly minimizes the risk 
\begin{align}
\label{equation:vtp-me}
\begin{aligned}
    \hV^{\mu^t}_{\ME,\gamma}(\pexp,\muout) \defeq \sup_{M\in\cM}\sup_{\opi\in\Pi}
    \bypart{ \E_{\oM \sim \muout}  \brac{ \wDTV{ M(\opi), \oM(\opi)} } }{
    - \gamma \E_{\pi \sim \pexp}\E_{\hM^t\sim \mu^t}\brac{ \wtdH^2(M(\pi), \hM^t(\pi)) } }.
\end{aligned}
\end{align}
Then, the algorithm outputs $\hM$ which is the projection of $\muout=\frac1T\sum_{t=1}^T\muout^t\in\Delta(\cM)$ into $\cM$.\footnote{Note that the direct specification of \getod~outputs $\muout$, an \emph{improper} estimation. Here, the projection step ensures that the output $\hM\in\cM$.} The following theorem is an instantiation of \cref{thm:E2D-gen} to all-policy model-estimation task, with a slight adaption to provide guarantee of the \emph{proper} estimation $\hM$ (details in \crefapp{E.6}{appendix:mdec}). %

\begin{theorem}[Model-estimation upper bound]\label{thm:E2D-ME}
\mealg~achieves the following with probability at least $1-\delta$: 
\begin{align*}
    \suboptme_{\Ms}(\hM)=\DTVPi{M^\star, \hM}\leq 6\omdec_{\gamma}(\cM)+\gamma\frac{60\log(\abs{\cM}/\delta)}{T}.
\end{align*}
\end{theorem}

\paragraph{Lower bound} An instantiation of \cref{prop:gen-lower-bound-demo} shows that AMDEC gives the following lower bound for the all-policy model-estimation task. %
\begin{proposition}[Model-estimation lower bound]
\label{prop:mdec-lower-bound-demo}
For any model class $\cM$, $T\in\Z_{\ge 1}$, and any $T$-round algorithm $\Algo$~which outputs an estimation $\hatpout\in\Delta(\cM)$, there exists a $M^\star\in\cM$ such that %
\begin{align*}
\EE^{M^\star,\Algo}\brac{\suboptme_{\Ms}(\hatpout)} \geq \frac13 \cdot \mdec_{\gammaT}(\cM),
\end{align*}
where $\gammaT\asymp T$ is defined in~\cref{prop:gen-lower-bound-demo}.
\end{proposition}

\revise{
As we have argued in \cref{sec:gen-dec}, the upper and lower bounds in \cref{thm:E2D-ME} and~\cref{prop:mdec-lower-bound-demo} together demonstrate that a controlled AMDEC is both necessary and sufficient for model estimation.
}

\paragraph{Reward-free learning is simpler than all-policy model-estimation} We remark that a suitable version of the \mealg~can also perform reward-free learning. Indeed, for any transition dynamic class $\cP$ with bounded \MDEC, reward-free learning with $\cP\times\cR$ can be solved by all-policy model-estimation: Simply run \mealg~to obtain an estimated transition $\hPm$, and for any given reward $R\in\cR$, output $\pi=\pi_{\hPm, R}$ (i.e. planning on $(\hPm, R)$). This implies that reward-free learning is a simpler task than all-policy model-estimation. The proposition below illustrates this task complexity relationship by comparing the DEC of these two tasks\aos{.}{ (proof in \cref{appendix:relation-mdec-rfdec}).}

\begin{proposition}\label{prop:rfdec-to-mdec}
For any transition model class $\cP$ and any reward function class $\cR\subseteq \set{ R:\cO\to[0,1] }$, it holds that $\orfecg(\cP,\cR)\leq 2\omdec_{\gamma/2}(\cP)$. \footnote{Here we write $\orfecg(\cP,\cR)$ to emphasize that the RFDEC also depends on the reward function class (cf. \cref{def:rfdec}).}
\end{proposition}

\paragraph{Extension: learning equilibrium in Markov games} 
We show the AMDEC framework can be adapted to give unified sample-efficient algorithms for learning Nash Equilibria and (Coarse) Correlated Equilibria (NE/CE/CCE) in tabular/linear mixture/low-rank Markov Games, which we present in~\crefapp{G}{appendix:markov-game}. 

\subsection{Preference-based reinforcement learning} \label{section:pbrl}

The goal of preference-based reinforcement learning \citep{wirth2017survey} is to achieve near-optimal performance measured in terms of a certain \emph{preference} function. Such a learning goal formulates scenarios where the agent's performance cannot be directly measured by a reward function. A particularly important example is reinforcement learning with human feedback (RLHF) \citep{christiano2017deep, ouyang2022training}, which has achieved significant empirical success in training large language models. 

In Preference-based reinforcement learning (PbRL) with trajectory preferences, the environment specifies a model of the form $M=(\Mtraj,\Cmp^M)$. The \emph{transition model} $\Mtraj\in\cMtraj$ together with a policy $\pi \in \Pi$ specifies a distribution over \emph{trajectories} $\tau \in \cT$, with $\tau \sim \Mtraj(\pi)$. The \emph{comparison function} $\Cmp^M:\cT\times\cT\to[0,1]$ compares two trajectories, satisfying $\Cmp^M(\tau_1,\tau_2)+\Cmp^M(\tau_2,\tau_1)=1$ for all $\tau_1,\tau_2\in\cT$. That is, $\Cmp^M$ can be viewed as the probability that $\tau_1$ is preferable to $\tau_2$. The observation space in PbRL is $\cO=\cT \times \cT \times\set{0,1}$, where an observation $o = (\tau_1, \tau_2, b) \in \cO$ consists of two trajectories $(\tau_1, \tau_2) \in \cT \times \cT$ and a $1$-bit preference feedback $b \in \{0, 1\}$. The policy space in PbRL is $\Piab = \Pia^2$, where $\Pi$ is the \emph{single-trajectory} policy class for $\Mtraj$. Upon executing policy $\bpi = (\pi_1, \pi_2)$ in environment $M$, the agent observes $o = (\tau_1, \tau_2, b)$, where $(\tau_1, \tau_2) \sim \Mtraj(\pi_1) \times \Mtraj(\pi_2)$ and $b \sim \Cmp^M(\tau_1, \tau_2)$ is a Bernoulli variable indicating the preference. With a slight abuse of notations, we also denote $\Cmp^M(\pi_1, \pi_2) = \E_{\tau_1 \sim M(\pi_1), \tau_2 \sim M(\pi_2)}[ \Cmp^M(\tau_1, \tau_2) ]$. We denote $\cMpb$ as a model class of PbRL models.

\begin{example}[Example of comparison function $\Cmp$]\label{example:BTL}
Given any function $u:\cT\to\R$, it induces a comparison function $\Cmp(\tau_1,\tau_2)= \exp(u(\tau_1)) / [\exp(u(\tau_1))+\exp(u(\tau_2))]$, known as the Bradley-Terry-Luce model \citep{bradley1952rank}. This probabilistic choice model, which outputs the probability that $\tau_1$ is preferred over $\tau_2$, has been widely adopted in reinforcement learning with human feedback (RLHF) \citep{christiano2017deep, ouyang2022training}. 
\end{example}

\newcommand{\PbReg}{\mathbf{PbReg}}

In PbRL, we consider the following notion of regret $\PbReg$ of a sequence of policy couples $\{ p^t \}_{t \in [T]}\subseteq \Delta(\Piab)$: 
\begin{align}\label{eqn:def-pb-reg}
    \textstyle \Regpb =  \sum_{t = 1}^T \max_{\pis\in\Pi} \EE_{(\pi^{t}_{1}, \pi^{t}_{2})\sim p^t}\brac{ \sum_{i = 1}^2 \paren{ \Cmp(\pi_\star, \pi^{t}_{ i}) - 1/2 } }. 
\end{align}
$\Regpb$ is always non-negative, which can be regarded as a measure of how much an opponent can outperform the profile $\{ p^t \}_{t \in [T]}$. We note that $\Regpb$ is slightly stronger than the regret considered in prior works~\citep{xu2020preference, bengs2021preference, chen2022human}, which is defined with respect to a \emph{fix} opponent policy. %

PbRL can be framed as a generalized no-regret learning problem: define the strategy space as $\ssppb = \Delta(\Piab)$, the decision domain as $\csppb = \cspreg$, and the sub-optimality measure as%
\begin{align*}
\suboptpb_M(p)=\sup_{\pis\in\Pia}\EE_{(\pi_1,\pi_2)\sim p}\brac{ \Cmp^M(\pi_\star, \pi_{1})+\Cmp^M(\pi_\star, \pi_{2})-1 }, \qquad \forall p\in\Delta(\Piab). 
\end{align*}
Then, the generalized regret corresponding to $\suboptpb$ defined above agrees with $\Regpb$, and the following Preference-Based Decision-Estimation Coefficient (PBDEC) now follows from specifying $\GOAL=\rm pb$ in \cref{def:gen-dec}.

\begin{definition}[Preference-based DEC]\label{def:pbdec}
The Preference-Based Decision-Estimation Coefficient (PBDEC) of a model-class $\cMpb$ with respect to $\omu\in\Delta(\cMpb)$ and parameter $\gamma>0$ is defined as
\begin{align*}
    \pbdec_\gamma(\cMpb, \omu) & \defeq \inf_{p\in\Delta(\Piab)} \sup_{M\in\cMpb} \suboptpb_M(p)-\gamma\E_{\bpi\sim p} \E_{\oM\sim\omu}\brac{\dH^2(M(\bpi), \oM(\bpi))}.
\end{align*}
Furthermore, we define $\opbdec_\gamma(\cMpb) \defeq \sup_{\omu\in\Delta(\cMpb)} \pbdec_\gamma(\cMpb, \omu)$. 
\end{definition}

\revise{The PBDEC measures the optimal trade-off between the exploitation and exploration under policy distribution $p$. Different from the (reward-based) DEC~\cref{eqn:dec-def}, the suboptimality of $p$ is not a linear combination of the suboptimality of the individual $\bpi\sim p$, capturing the nature of the regret considered in~\cref{eqn:def-pb-reg}. }

\paragraph{Algorithm and theoretical guarantee} We propose the {\pbrlalg} algorithm (described in~\crefapp{E.7}{alg:E2D-PbRL}) for preference-based RL, which is a direct instantiation of \getod.
The risk function is defined as
\begin{align}
\label{equation:vtp-pbrl}
   \hV^{\mu^t}_{\mathrm{pb},\gamma}(p) \defeq \sup_{M\in\cMpb} \suboptpb_M(p)-\gamma\E_{\bpi\sim p} \E_{\oM\sim\mu^t}\brac{\dH^2(M(\bpi), \oM(\bpi))}.
\end{align}
The algorithm is similar as \etod~(\cref{alg:E2D-TA}), except that in each iteration, we observe two trajectories and one-bit feedback $o^t = (\tau_1^t, \tau_2^t, b^t)$, but not the reward vector. The following theorem is an instantiation of \cref{thm:E2D-gen} to preference-based reinforcement learning.

\begin{theorem}[Preference-based RL upper bound]\label{thm:E2D-PBRL}
\pbrlalg~achieves the following with probability at least $1-\delta$:
\begin{align*}
    \Regpb\leq T\cdot\opbdec_\gamma(\cMpb)+10 \gamma \cdot \log(\abs{\cMpb }/ \delta).
\end{align*}
\end{theorem}
For model classes with $\pbdec_\gamma(\cMpb)\lesssim \tbO{d/\gamma}$, by tuning $\gamma>0$, \cref{thm:E2D-PBRL} shows that \pbrlalg~achieves $\Regpb\leq \tO(\sqrt{d\log\abs{\cM}T})$. 

\paragraph{Lower bound} An instantiation of \cref{prop:gen-lower-bound-demo} shows that PBDEC gives the following regret lower bound for Preference-based RL. 
\begin{proposition}[Preference-based RL lower bound]
\label{prop:pbrl-lower-bound-demo}
For any model class $\cM$, $T\in\Z_{\ge 1}$, and any $T$-round algorithm $\Algo$, there exists a $M^\star\in\cP$ such that
\begin{align*}
\EE^{M^\star,\Algo}\brac{\Regpb} 
\geq&~ \max_{\pis\in\Pi}\,\EE^{M^\star,\Algo}\brac{ \sum_{t = 1}^T \sum_{i = 1}^2 \paren{ \Cmp(\pi_\star, \pi^{t}_{ i}) - 1/2 } }
\geq \frac{T}3 \cdot \pbdec_{\gammaT}(\cM),
\end{align*}
where $\gammaT\asymp T$ is defined in~\cref{prop:gen-lower-bound-demo}.
\end{proposition}
Prior to our work, such upper and lower bounds for preference-based RL with general problem classes are largely unknown.

\paragraph{Relationship with standard no-regret RL}  Preference-based reinforcement learning (PbRL) and standard no-regret reinforcement learning are related in the following way. Consider a standard DMSO model class $\cM$, where each model $M=(\P^M,R^M) \in \cM$ has a mean reward function $R^M:\cT\to [0,1]$. Any $M\in\cM$ could induce a PbRL model $M_{\sf Pb} = (\P^M, \Cmp^M) \in \cMpb$ with $\Cmp^M$ with comparison probabilities:
\begin{align}\label{eqn:Cmp-RL}
    \Cmp^M(\tau_1,\tau_2)=(1+R^M(\tau_1)-R^M(\tau_2) ) / 2.
\end{align}
This gives a PbRL sub-optimality gap:
\begin{align*}
\suboptpb_M(p)=f^M(\pi_M)-\EE_{(\pi_1,\pi_2)\sim p}\EE_{\pi\sim\unif(\pi_1,\pi_2)}\brac{f^M(\pi)}. 
\end{align*}
Thus, the PbRL sub-optimality gap is the sub-optimality gap of the average policy $(\pi_1 + \pi_2)/2$ in standard no-regret RL, and the PbRL model $M_{\sf Pb}$ can be regarded as the original $M$ with the observation $o$ being restricted, in the sense that $M_{\sf Pb}$ only reveals the difference between cumulative rewards in two consecutive episodes. This shows that standard no-regret RL can be reduced to PbRL, in the sense that any algorithm for the PbRL problem with comparison probability~\cref{eqn:Cmp-RL} solves standard no-regret RL.

Somewhat surprisingly, the converse statement is also true in certain sense: PbRL is \emph{not more difficult} than standard model-based RL on a series of related model classes. More concretely, the following theorem shows that, if the agent can beat any known opponent in $\Mtraj$, then no-regret PbRL is possible.~\revise{Hence, it provides a universal reduction from no-regret PbRL to the standard no-regret RL, which is new in the preference-based RL literature.}

\newcommand{\oaux}{o_{\rm aux}}
\begin{theorem}\label{thm:PbRL-to-RL}
Consider a given PbRL model class $\cMpb$. For any $q\in\Delta(\Pia)$ and $M=(\Mtraj,\Cmp^M)\in\cMpb$, we consider a standard DMSO model $M_q$ defined as follows: for any policy $\pi\in\Pia$, upon executing $\pi$ on $M_q$:
\begin{itemize}
    \item The learner first observes $\tau\sim \Mtraj(\pi)$.
    \item Then, the environment generates auxiliary observation $\oaux=(\pi_0,\tau_0)$ as $\pi_0\sim q$, $\tau_0\sim \Mtraj(\pi_0)$.
    \item The learner observes $\oaux$ and receives a Bernoulli reward $r\sim \C^M(\tau,\tau_0)$.
\end{itemize}
Define $\cM_q=\{M_q: M\in\cMpb\}$ for each $q\in\Delta(\Pi)$. Then as long as $\Pi$ is finite, it holds that
\begin{align*}
    \pbdec_\gamma(\cMpb)\leq 2\sup_{q\in\Delta(\Pi)} \dec_{\gamma/2}(\cM_q).
\end{align*}
\end{theorem}

\aos{}{The proof of \cref{thm:PbRL-to-RL} is contained in \cref{appdx:proof-PbRL-to-RL}.}

\section{Instantiation: DEC bounds for model classes with \belrep}
\label{section:bellman-rep}

In this section, we consider a broad class of RL models---models with low-complexity \belrep. \Belrep~ is a generalization of Bellman Representation introduced by \citet{foster2021statistical} that encapsulates many general decision-making processes, including bandit problems, MDPs, and partially observable RL. We bound the various DECs by the complexity of \belrep, and provide unified sample-efficient guarantees for various RL goals~\revise{(\cref{tab:examples})}. This recovers existing sample complexity results across a range of RL goals and models, while also yielding novel bounds. 
For succinctness, we defer most of the details, discussions, and proofs to \crefapp{F}{appendix:bellman-rep}.

\paragraph{Decoupling dimension}
To define \belrep~and its complexity, we first introduce the notion of \emph{decoupling dimension} for a function class with a class of distributions. 
\begin{definition}[Decoupling dimension]\label{def:dcpl-dim}
For any function class $\cF\subseteq ( \cX \rightarrow \R)$ and distribution class $\cQ \subseteq \Delta(\cX)$, the decoupling dimension $\dimc(\cF,\cQ,\gamma)$ for $\gamma>0$ is defined as
\begin{align*}
    \dimc(\cF,\cQ,\gamma):=\sup_{\nu \in\Delta(\cF\times\cQ)}\gamma\mathbb{E}_{(f,q)\sim\nu}\EE_{x\sim q}\brac{f(x)} -\gamma^2 \EE_{f\sim\nu}\EE_{q\sim\nu,x\sim q}\brac{\abs{f(x)}^2}.
\end{align*}
When $\cQ$ is clear from the context or $\cQ=\Delta(\cX)$, we may abbreviate $\dimc(\cF,\gamma):=\dimc(\cF,\cQ,\gamma)$. We also write $\odimc(\cF,\gamma)=\sup_{f\in\cF}\dimc(\cF-f,\gamma)$.
\end{definition}
Decoupling dimension is a fairly general complexity measure of the decomposition function class. For instance, it can be bounded for generalized linear function classes\aos{}{ (\cref{example:dc-linear-gen})}, 
and more generally for any function class with low Eluder dimension~\citep{russo2013eluder} or star number~\citep{foster2020instance}\aos{}{ (cf. \cref{example:dc-to-es})}. Furthermore, the decoupling dimension of $(\cF,\cQ)$ can also be bounded by the complexity of the distribution class $\cQ$ itself, e.g. coverability~\citep{xie2022role}\aos{}{ (cf. \cref{example:dc-to-cov})}. \aos{A detailed discussion is postponed to \crefapp{F.1}{}.}

\subsection{Decouplable representation}

We now define \emph{\belrep} and its complexity as follows. 
\begin{definition}[\Belrep~and its complexity]
\label{definition:err-rep}
The \emph{\belrep} $\cG$ of the model class $\cM$ is associated with general index sets $\{ \cT_h\}_{h \in [H]}$, a class of functions $\{ \cerr^{M;\oM}_h: \cT_h \to \R \}_{\oM, M\in\cM, h\in[H]}$, a class of distributions $\{ \dtr_h(M; \oM) \in \Delta(\cT_h) \}_{\oM, M\in\cM, h\in[H]}$, a class of exploration policies $\{ \piest_{M}\in\Delta(\Pi) \}_{M \in \cM}$\footnote{Note that here we slightly abuse the notation by considering the mixture of policy in $\Delta(\Pi)$.}, and a constant $L$, satisfying the following %
\begin{align}
&f^{M}(\pi_M)-f^{\oM}(\pi_M)\leq\sum_{h=1}^H \EE_{\tau_h\sim \dtr_h(M; \oM)}\brac{\cerr^{M; \oM}_h(\tau_h)}, && \forall M,\oM, \label{eqn:err-rep-decouple-1}\\
&\sum_{h=1}^H \EE_{\tau_h\sim \dtr_h(M'; \oM)}\abs{\cerr^{M; \oM}_h(\tau_h)}^2\leq  L\cdot  \tDH{\oM(\piest_{M'}), M(\piest_{M'}) } , &&\forall M,M',\oM. \label{eqn:err-rep-decouple-2}
\end{align}
Denote $\cG^{\oM}_h=\{ \cerr^{M;\oM}_h \}_{M\in\cM}$ and $\cQ_h^{\oM}=\{ \dtr_h(M; \oM) \}_{M\in\cM}$. The complexity of \belrep~$\cG$ is measured by $\dim(\cG,\gamma)\defeq\max_{\oM,h}\dimc(\cG^{\oM}_h,\cQ^{\oM}_h,\gamma)$. %
\end{definition}

\revise{The decouplable representation generalizes several known structural conditions for RL, e.g., Bellman-Eluder dimension and Bilinear class~\citep{jin2021bellman, du2021bilinear}. Intuitively, in a \belrep, condition \cref{eqn:err-rep-decouple-1} requires the value difference $f^{M}(\pi_M)-f^{\oM}(\pi_M)$ can be decomposed into the error terms $\cerr^{M; \oM}_h$ and the occupancy $\dtr_h(M; \oM)$ of $\pi_M$, while condition \cref{eqn:err-rep-decouple-2} requires the \emph{decoupled} errors can be in turn bounded by the Hellinger distances.} Therefore, \belrep~captures most of the general structural conditions for MDP model class $\cM$. In the following example, we first demonstrate that Bellman errors of MDPs induce a \belrep~of $\cM$.

\begin{example}[Bellman errors]\label{def:bellman-err}
For any MDP model class $\cM$, it admits a \belrep~$\GBE$ as follows: \footnote{For simplicity, here we only consider the case where the initial state distribution is the same across $\cM$.} 
\begin{itemize}%
\item (Index set) For each $h\in[H]$, the index set is $\cT_h=\cS\times\cA$.
\item (Error functions and distributions) For each $M,\oM\in\cM$,
\begin{align*}
    \dtr_h(M;\oM)\defeq&~\PP^{\oM,\pi_M}(s_h=\cdot,a_h=\cdot)\in\Delta(\cS\times\cA), \\
    \cerr_h^{M; \oM}(s_h,a_h) \defeq&~ Q^{M, \star}_h(s_h, a_h) - [\T^{\oM}_h V^{M,\star}_{h+1}](s_h,a_h),
\end{align*}
where $(Q^{M,\star}, V^{M,\star})$ is the optimal value functions of $M$, and $\T_h^{\oM}$ is the \emph{Bellman operator} of the MDP $\oM$, which maps any function $V: \cS\to \R$ to
\begin{align*}
    [\T^{\oM}_h V](s,a)=R^{\oM}_h(s,a) + \EE_{s'\sim \PP^{\oM}_h(\cdot|s,a)} V(s')
\end{align*}
\item The exploration policies are given by $\piest_M=\pi_M$ and the constant $L=4H$.
\end{itemize}
\end{example}

Representation $\GBE$ described above corresponds to the well-known performance difference lemma \citep{kakade2002approximately} (see also \citet{jin2021bellman, foster2021statistical}): %
\begin{align}\label{eqn:Bellman-decomp}
    f^M(\pi_M)-f^{\oM}(\pi_M)=\sum_{h=1}^H \EE^{\oM,\pi_M}\brac{ Q^{M, \star}_h(s_h, a_h) - [\T^{\oM}_h V^{M,\star}](s_h,a_h) },
\end{align}
which immediately implies that \cref{eqn:err-rep-decouple-1} holds for $\GBE$, and \cref{eqn:err-rep-decouple-2} follows from the definition (detailed proof is deferred to \crefapp{F.3}{appendix:GBE-proof}).

\revise{In particular, \belrep~with controlled complexity encompasses the model-based version of bilinear class \citep{du2021bilinear}\aos{}{ (\cref{appdx:bilinear})}, Bellman-Eluder dimension \citep{jin2021bellman}\aos{}{ (\cref{appdx:BE-dim})}, and also Bellman representation \citep{foster2021statistical}\aos{}{ (\cref{appdx:relation-bel-rep})}.} More specifically, under these structural conditions, the decoupling dimension can be bounded for certain choices of \belrep~(i.e., variants of $\GBE$, see discussions in \crefapp{F.2}{appdx:bilinear-eluder}). %

Furthermore, for explicit problem classes (e.g. \cref{tab:examples}), natural \belrep~can be directly written down. In this way, \belrep~also captures the general structural condition of partially observable RL~\citep{chen2022partially,liu2023optimistic}\aos{}{ (cf. \cref{appendix:POMDP})}, which itself encompasses a wide range of POMDP classes, e.g. revealing POMDPs~\citep{jin2020sample,liu2022partially}.

 To summarize, for a wide range of RL classes, we can devise a corresponding \belrep~$\cG$ such that its decoupling dimension can be bounded as $\dimG(\cG,\gamma)\leq \tO(d)$  where $d$ depends on relevant problem parameters.
For the model class $\cM$ with such a low-complexity \belrep~$\cG$, we can apply following bounds for PACDEC and DEC, which in turn imply guarantees for no-regret learning and PAC learning.
\begin{proposition}[Bounding DEC/PACDEC by the complexity of \belrep]
\label{prop:belrep-pac}
Suppose that $\cG$ is a \belrep~of $\cM$. Then we have
\begin{align*}
    \oeec_{\gamma}(\cM)\leq \frac{6LH\dimG(\cG,\gamma/L)+6H}{\gamma}.
\end{align*}
In particular, if the \belrep~satisfies $\piest_M=\pi_M$ (the on-policy case), \aos{}{we have} 
\[
\dec_{\gamma}(\cM)\leq \frac{6LH\dimG(\cG,\gamma/L)+6H}{\gamma}.
\]
\end{proposition}
In particular, for $\cG$ with $\dimG(\cG,\gamma)\leq \tO(d)$,\footnote{Here and henceforth $\tO(\cdot)$ possibly hides $\polylog(\gamma)$ factor.} we have $\eec_{\gamma}(\cM)\leq\tO(dLH/\gamma)$, implying a sample complexity $\tO(dLH/\eps^2)$ of \eetod. If $\cG$ is further an on-policy \belrep, then we also have $\dec_{\gamma}(\cM)\leq\tO(dLH/\gamma)$, which implies a regret bound $\tO(\sqrt{dLH\log\abs{\cM}T})$ of \etod. (If $\cG$ is not on-policy, then we have $\dec_{\gamma}(\cM)\leq\tO(\sqrt{dLH/\gamma})$ by \cref{prop:eec-to-dec}, which gives a regret of order $T^{2/3}$.)

\subsection{\mBelrep} Analogous to \belrep~(\cref{definition:err-rep}), we can define a stronger representation called \mbelrep~that enables controlling the DEC for reward-free learning, all-policy model estimation and preference-based RL.

\begin{definition}[\mBelrep~and its complexity]\label{def:err-rep-model}
The \emph{\mbelrep} $\cG$ of the model class $\cM$ is associated with general index sets $\{ \cT_h\}_{h \in [H]}$, a class of functions $\{ \cerr^{M;\oM}_h: \cT_h \to \R \}_{\oM, M\in\cM, h\in[H]}$, a class of distributions $\{ \dtr_h(M; \oM) \in \Delta(\cT_h) \}_{\oM, M\in\cM, h\in[H]}$, a class of exploration policies $\{ \piest \in\Delta(\Pi) \}_{\pi\in\Pi}$,\footnote{For notational simplicity, here the superscript $\exp$ denotes a mapping from $\Pi$ to $\Delta(\Pi)$ that maps $\pi\to\piest$, i.e. for each policy $\pi\in\Pi$, it is assigned with an exploration policy $\piest\in\Delta(\Pi)$.}%
and a constant $L>0$, satisfying the following
\begin{align}
&\wDTV{ \oM(\pi), M(\pi) }\leq\sum_{h=1}^H \EE_{\tau_h\sim \dtr_h(\pi; \oM)}\brac{\cerr^{M; \oM}_h(\tau_h)}, && \forall M,\pi,\oM, \label{eqn:err-rep-decouple-strong-1}\\
&\sum_{h=1}^H \EE_{\tau_h\sim \dtr_h(\pi; \oM)}\abs{\cerr^{M; \oM}_h(\tau_h)}^2\leq  L\cdot  \tDH{\oM(\piest), M(\piest) } , &&\forall M,\pi,\oM. \label{eqn:err-rep-decouple-strong-2} 
\end{align}
Denote $\cG^{\oM}_h=\{ \cerr^{M;\oM}_h \}_{M\in\cM}$ and $\cQ_h^{\oM}=\{ \dtr_h(\pi; \oM) \}_{\pi\in\Pi}$. The complexity of $\cG$ is measured by $\dimG(\cG,\gamma)\defeq\max_{\oM,h}\dimc(\cG^{\oM}_h,\cQ^{\oM}_h,\gamma)$. 
\end{definition}
Notice that the only difference between the \emph{\mbelrep} and \emph{\belrep} is that the left-hand side of \eqref{eqn:err-rep-decouple-strong-1} gives $\wDTV{ \oM(\pi), M(\pi) }$, whereas the lefthand side of \eqref{eqn:err-rep-decouple-1} gives $f^{M}(\pi_M)-f^{\oM}(\pi_M)$. Note that we always have $f^{M}(\pi_M)-f^{\oM}(\pi_M) \le \wDTV{ \oM(\pi), M(\pi) }$, and hence, a \mbelrep~$\cG$ of $\cM$ is always a \belrep~of $\cM$, with the same complexity.

In the following, we bound RFDEC, AMDEC and PBDEC of any given model class $\cM$ in terms of the complexity of a \mbelrep~of $\cM$.

\begin{proposition}[Bounding \MDEC~by the complexity of \mbelrep]
\label{prop:belrep-am}
Suppose $\cG$ is a \mbelrep~of $\cM$. Then
\begin{align*}
   \medec_{\gamma}(\cM) 
   \leq \frac{ LH\dimG(\cG, \gamma/L ) }{\gamma}.
\end{align*}
\end{proposition}
In particular, if $\cM$ admits a \mbelrep~$\cG$ with $\dimG(\cG,\gamma)\leq \tO(d)$, we have $\medec_{\gamma}(\cM)\leq \tbO{dHL/\gamma}$. Consequently, the \mealg~algorithm returns an estimated model $\hM$ such that $\DTVPi{M^\star, \hM} \le \eps$ using $\tbO{dHL\log\abs{\cM}/\epsilon^2}$ episodes.

As a corollary of \cref{prop:belrep-am}, we also have the following guarantee of RFDEC.
\begin{proposition}\label{prop:belrep-rf}
Suppose $\cG$ is a \mbelrep~of $\cP$, a class of transition models. Then (regardless of the reward function class $\cR$)
\begin{align*}
   \rfec_{\gamma}(\cP) 
   \leq \frac{ 4LH\dimG(\cG, \gamma/(2L) ) }{\gamma}.
\end{align*}
\end{proposition}

Analogously, PBDEC of a given model class can also be bounded by the complexity of a corresponding \mbelrep, as follows.

\begin{proposition}\label{prop:belrep-pb}
Suppose that $\cMpb$ is a PbRL model class, with trajectory class $\cMtraj$ and comparison function class $\cC$. Suppose that $\cG$ is a on-policy \mbelrep~of $\cMtraj$ (i.e. $\piest=\pi$ for all $\pi$), then
\begin{align*}
    \pbdec_{\gamma}(\cMpb)\leq\frac{12LH\dimG(\cG,\gamma/6)+12\ \odimc(\cC,\gamma)+12}{\gamma}.
\end{align*}
\end{proposition}
In particular, suppose $\cG$ admits $\dimG(\cG,\gamma)\leq \tO(d)$ and $\dimG(\cC,\gamma)\leq \tO(d_{\Cmp})$ (e.g. when $\cC$ is a class of generalized linear functions), we have $\pbdec_{\gamma}(\cM)\leq\tO((dLH+d_{\Cmp})/\gamma)$, implying that \pbrlalg~has a regret bound of
\begin{align*}
    \Regpb\leq \tbO{\sqrt{(dLH+d_{\Cmp})\log\abs{\cM}T}}. 
\end{align*}
As we have discussed earlier, $\dimG(\cC,\gamma)$ can be simply controlled for various classes of preference functions.

\subsection{Examples and discussion}
~\cref{prop:belrep-pac}, \ref{prop:belrep-am}, \ref{prop:belrep-rf}, and \ref{prop:belrep-pb} can all be specialized to a wide range of concrete RL problems, for which we provide several illustrative examples and concrete results in \cref{tab:examples} (details in~\crefapp{F.8}{appendix:proof-examples}). Notably, the rates (for each learning goal) are obtained through (almost) a single unified algorithm without further problem-dependent designs. In the following, we briefly discuss their relation to the results in related literatures.

\renewcommand{\arraystretch}{1.6}
\begin{table}[t]
    \centering
    \begin{tabular}{|c|c|c|c|c|c|c|}
    \hline
    Model class & PACDEC & \!PAC Sample Complexity\! & \!RF\! & \!\!AM\!\! & Regret & \!PB\! \\\hline
    Linear bandit & $\frac{d}{\gamma}$ & $\frac{d^2}{\epsilon^2}$ & \yes & \yes & $d\sqrt{T}$ & \yes \\\hline
    Tabular MDP & $\frac{|\cS||\cA|H^2}{\gamma}$ & $\frac{|\cS|^3|\cA|^2H^3}{\epsilon^2}$ & \yes & \yes & $\sqrt{|\cS|^3|\cA|^2H^3T}$ & \yes \\\hline
    \!Linear mixture MDP\! & $\frac{dH^2}{\gamma}$ & $\frac{d^2H^3}{\epsilon^2}$ & \yes & \yes & $\sqrt{d^2H^3T}$ & \yes \\\hline
    Linear MDP & $\frac{dH^2}{\gamma}$ & $\frac{dH^2\log\abs{\cM}}{\epsilon^2}$ & \yes & \yes & $\sqrt{dH^2\log\abs{\cM}\cdot T}$ & \yes \\\hline
    Low-rank MDP & $\frac{d|\cA|H^2}{\gamma}$ & $\frac{d|\cA|H^2\log\abs{\cM}}{\epsilon^2}$ & \yes & \yes & \!$\sqrt{d|\cA|H^2\log\abs{\cM}\cdot T}$\! & \yes \\\hline
    Parametric MDP & $\frac{\kappa^2 d H^2}{\gamma}$ & $\frac{\kappa^2 d^2 H^3}{\epsilon^2}$ & \yes & \yes & $\sqrt{\kappa^2d^2H^3T}$ & \yes \\\hline
    Revealing POMDP & $\frac{|\cS||\cA|^mH^2}{\arev^2 \gamma}$ & \!$\frac{\poly(|\cS|,|\cA|,|\cO|)|\cA|^mH^3}{\arev^2\epsilon^2}$\! & \yes & \yes & $\cO(T^{2/3})$ & \yes \\\hline
    B-stable PSR & \!$\frac{\stab^2 d|\cA|U_AH^2}{\gamma}$\! & $\frac{\stab^2 d|\cA|U_AH^2\log\abs{\cM}}{\epsilon^2}$ & \yes & \yes & $\cO(T^{2/3})$ & \yes \\\hline
    \end{tabular}
    \caption{Upper bounds on sample complexity and regret of \getod~instantiated across various learning tasks and model classes (details and proofs in \crefapp{F.8}{appendix:proof-examples}). For succinctness, we omit logarithmic factors in the sample complexity and regret bounds. The $\cO(T^{2/3})$ regret corresponds to the regret obtained from sample complexity upper bound by explore-then-commit strategy (cf. \cref{prop:eec-to-dec}), and we omit other factors for succinctness. The ``RF/AM'' columns indicate whether the same sample complexity guarantee can be achieved for the settings of reward-free exploration (RF) and all-policy model estimation (AM) as with PAC learning. The ``PB'' column indicates whether a regret bound of the same order is feasible for preference-based RL as in no-regret learning (modulo the extra term $\sqrt{d_{\Cmp}T}$ that is typically of lower order). %
    }
    \label{tab:examples}
\end{table}

\begin{example}[Linear mixture MDPs, ~\citet{ayoub2020model}]\label{example:linear-mixture-demo}
A MDP $M$ is called a linear mixture MDP (with respect to a known $d$-dimensional feature map $\phi$) if there exists parameter $\{\theta_h^M\in\R^d\}_h$, such that for the given features $\phi_h(\cdot|\cdot,\cdot):\cS\times\cS\times\cA\to\R^d$, it holds that
$$
\PP_h^M(s'|s,a)=\iprod{\theta_h^M}{\phi_h(s'|s,a)}.
$$
\end{example}

For linear mixture MDPs, our framework implies a sample complexity of $\tbO{d^2H^3/\epsilon^2}$ for $\eps$-near-optimal reward-free learning, which only has an additional $H^2$ factor over the current best sample complexity of $\tbO{d^2H/\epsilon^2}$ by ~\citet{chen2021near}\footnote{Rescaled to total reward within $[0,1]$.}.

\begin{example}[Low-rank MDPs, \citet{agarwal2020flambe}]\label{example:low-rank-demo}
A MDP $M$ is of low-rank $d$ if 
there exists feature maps $\mu^M=(\mu_h^M:\cS\to\R^d)$ and $\phi^M=(\phi_h^M:\cS\times\cA\to\R^d)$, such that the transition dynamics of $M$ admits the following low-rank factorization: 
$$
\PP_h^M(s'|s,a)=\iprod{\mu_h^M(s')}{\phi_h^M(s,a)},\qquad \forall h\in[H].
$$
\end{example}

For low-rank MDPs with unknown $d$-dimensional features $(\phi, \mu)$ in a given feature class $\Phi\times\Psi$ (i.e. the FLAMBE setting~\citep{agarwal2020flambe}), our PAC result matches the best known sample complexity achieved by, for example, the V-Type Golf Algorithm of~\citet{jin2021bellman}. For reward-free learning, our linear in $d$ dependence improves over the current best $d^2$ dependence achieved by the RFOlive algorithm~\citep{chen2022statistical}, and we do not require linearity or low complexity assumptions on the class of reward functions $\cR$ made in existing work~\citep{wang2020reward,chen2022statistical}. However, we remark that they handle a slightly more general setting where only the $\Phi$ class is known with model-free approach, while our results have to scale with $\log|\Psi|$.

\begin{example}[Parametric MDPs]\label{example:para-mdp-demo}
As further examples, we also consider parametric MDPs (\crefa{Example F.18}{def:parametric-mdp}), which includes MDPs parametrized by exponential families \citep{chowdhury2021reinforcement, li2022exponential} and, in particular, Online Nonlinear Control \citep{kakade2020information, ren2022free}. For the corresponding problem classes, our PAC results match with the best known sample complexities, while we also provide reward-free, model-estimation and preference-based learning guarantees, which are largely unknown in this setting. 
\end{example}

\begin{example}[Partially observable RL]\label{example:pomdp-demo}
For $m$-step $\arev$-revealing Partially Observable MDPs (POMDPs) \citep{liu2022partially}, and more generally B-stable Predictive State Representations (PSRs) \citep{chen2022partially}, we utilize the analysis framework developed in \citet{chen2022partially}, which provides a \mbelrep~for these partially observable problem classes with bounded complexity. 
Detailed results are deferred to \crefapp{F.8.7}{appendix:POMDP}. 
In particular, our framework provides sample-efficient guarantee of \emph{preference-based learning} for a broad class of partially observable RL problems, which is new to our best knowledge.
\end{example}

As a final remark, we emphasize again that the reward-free learning, model estimation, and preference-based RL guarantees provided in the above series of examples are largely unknown beyond linear mixture MDP class.

\section{Connections to optimistic algorithms}
\label{section:optimistic}

E2D is closely related to two other unified algorithm design principles: Model-based Optimistic Posterior Sampling (\mops), and Optimistic Maximum Likelihood Estimation (\omle). In this section, we show that \mops~and \omle---in addition to their algorithmic similarity to E2D---admit efficient sample complexity bound under general structural conditions of model classes, similar to the conditions required for the \eetod~algorithm. 

In establishing these connections, we mainly consider the learning goal of PAC RL. The algorithms will depend on a (user-specified) exploration strategy: $M\mapsto \piest_M\in\Delta(\Pi)$. The \emph{explorative policy} $\piest_M$ for $M$ here is analogous to the one that appears in the definition of \belrep~(\cref{definition:err-rep}). The key difference is that \belrep~only serves as an analysis tool for \eetod, and $\set{\piest_M}_{M\in\cM}$ does not explicitly appear in the algorithm design of \eetod; Instead, the exploration strategy of \eetod~is found through the minimax optimization problem \eqref{equation:vtp-exp}.

\subsection{Model-based Optimistic Posterior Sampling (\mops)}
\label{section:mops}

We consider the MOPS algorithm proposed by~\citet{agarwal2022model}, presented here with a minor modification for notation consistency\footnote{Our version is essentially equivalent to~\citet[Algorithm 1]{agarwal2022model}, except that we look at the full observation and reward vector (of all layers), whereas they only look at a random layer $h^t\sim\Unif([H])$, which is restricted to MDP models.}.  
Similar to \etod, \mops~also maintains a distribution $\mu^t\in\Delta(\cM)$ over models, initialized at a suitable prior distribution $\mu^1$. The policy in the $t$-th episode is directly obtained by posterior sampling: $\pi^t=\piest_{M^t}$, where $M^t\sim \mu^t$. After executing $\pi^t$ and observing $(o^t, r^t)$, the algorithm updates the posterior as
\begin{align}
\label{equation:mops-eq}
    \mu^{t+1}(M) \; \propto_M \; \mu^{t}(M) \cdot 
    \exp\paren{\gamma^{-1}f^M(\pi_M) + \etap\log \P^{M,\pi^t}(o^t) - \etar\ltwo{\br^t-\bR^M(o^t)}^2 }.
\end{align}
This update is similar to \Vovkalg~\cref{equation:tempered-aggregation}, and differs in the additional \emph{optimism} term $\gamma^{-1}f^M(\pi_M)$ which favors models with higher optimal values. After the $T$ episodes, \mops~outputs $\pout=\pi_M$ with  $M\sim\mu^t$, $t\sim\Unif([T])$ (full algorithm in~\cref{alg:MOPS}).

We now define the Posterior sampling coefficient (PSC) and provide theoretical guarantees for the \mops~algorithm. 
\begin{definition}[Posterior sampling coefficient]
\label{definition:psc}
The Posterior Sampling Coefficient (PSC) of model class $\cM$ with respect to reference model $\oM\in\cM$ and parameter $\gamma>0$ is defined as
\begin{align*}
    \psc_{\gamma}(\cM, \oM) \defeq \sup_{\mu\in\Delta(\cM)}\EE_{M\sim \mu}\EE_{M' \sim \mu}\left[f^{M}(\pi_M)-f^{\oM}(\pi_M)-\gamma  \wtdH^2(\oM(\piest_{M'}),M(\piest_{M'}))\right].
\end{align*}
\end{definition}

\begin{theorem}[Guarantees for MOPS]
\label{thm:MOPS}
Choosing $\etap=1/6$, $\etar=0.6$ and the uniform prior $\mu^1=\Unif(\cM)$ in~\cref{alg:MOPS}. Then~\cref{alg:MOPS} achieves the following with probability at least $1-\delta$:
\begin{align*}
    \suboptpac=f^\Ms(\pi_\Ms) - \frac1T\sum_{t=1}^T \E_{M\sim\mu^t}\brac{ f^\Ms(\pi_M)}\leq \psc_{\gamma/6}(\cM, M^\star) + \frac2\gamma + 4\gamma\cdot \frac{\log(\abs{\cM}/\delta)}{T}.
\end{align*}
In particular, in the on-policy scenario when $\piest_M=\pi_M$ for all $M\in\cM$, \cref{alg:MOPS} achieves 
\begin{align*}
    \regdm\leq T\brac{\psc_{\gamma/6}(\cM, M^\star) + 2/\gamma} + 4\gamma\cdot \log(\abs{\cM}/\delta).
\end{align*}
\end{theorem}
\cref{thm:MOPS} %
is similar to~\citet[Theorem 1]{agarwal2022model} and is slightly more general in the assumed structural condition, as the PSC is bounded whenever the ``Hellinger decoupling coefficient'' used in their theorem is bounded (\cref{prop:psc-to-dcp}). %

\paragraph{Bounding the PSC for concrete model classes}
The following proposition (proof in \cref{proof-psc-belrep}) shows that PSC can be bounded in terms of the complexity of any \belrep~of $\cM$. 
\begin{proposition}\label{prop:psc-belrep}
Suppose that $\cM$ admits a \belrep~$\cG$ with exploration policy $M\mapsto \piest_M$. Then for any $\oM\in\cM$, it holds that
\begin{align*}
    \psc_{\gamma}(\cM,\oM)\leq \frac{LH\cdot \max_h\dimc(\cG^{\oM}_h,\gamma/L)}{\gamma}.
\end{align*}
\end{proposition}
In particular, all the upper bounds on PACDEC presented in \cref{tab:examples} are also upper bounds on PSC, providing corresponding sample complexity bounds for \mops. %

\paragraph{Relationship between PACDEC and PSC}
The definition of the PSC (\cref{definition:psc}) looks very similar to that of the PACDEC (\cref{def:edec}). The difference is two-fold: (1) The policy distribution in the definition of PSC directly corresponds to the distribution of models, whereas in the definition of PACDEC, we take infimum over all possible $(\pexp,\pout)$; (2) PSC involves the value difference term $f^{M}(\pi_M)-f^{\oM}(\pi_M)$, whereas PACDEC involves value difference term $\E_{\pi \sim \pout}[f^{M}(\pi_M)-f^{M}(\pi)]$. In the proposition below, we show that the PACDEC can be upper-bounded by the PSC, modulo a (typically lower-order) additive term; in other words, a small PSC implies a small PACDEC. %
\begin{proposition}[Bounding PACDEC by PSC]
  \label{prop:dec-to-psc}
  Suppose $\Pi$ is finite. Then we have for any $\gamma>0$ that
  \begin{align*}
    \eec_{\gamma}(\cM) \le \sup_{\oM \in \cM} \psc_{\gamma/6}(\cM, \oM) + 2(H+1)/\gamma.
  \end{align*}
\end{proposition}

In the proposition below, we show that PACDEC is strictly more general than PSC, in that we cannot upper bound PSC by PACDEC in general. %
\begin{proposition}\label{prop:psc-lower}
    For the model class $\cM$ given in \cref{prop:rev-bandit} and $\piest_M=\pi_M$, we have
    \begin{align*}
        \psc_{\gamma}(\cM,\oM)\asymp \frac{2^d}{\gamma},
    \end{align*}
    and there exists $\Ms\in\cM$ such that $\EE^{\Ms,\mops}[\regdm]\geq\Om{\sqrt{2^d\cdot T}}$. 
\end{proposition}

\subsection{Optimistic Maximum Likelihood Estimation (\omle)} \label{section:omle}

We consider the \omle~algorithm proposed by \citet{liu2022partially}, presented here with a minor modification for notation consistency. The original \omle~in \citet{liu2022partially} utilizes the log-likelihood of the observations and rewards as the objective function. Here we consider the following objective function involving the log-likelihood of the observations and the negative $L_2$ squared loss of the rewards: %
\begin{align}\label{eqn:MLE-risk}
\textstyle
    \cL_t(M):=\sum_{s=1}^{t-1} \brac{ \log \P^{M, \pi^s}(o^s) -\ltwot{ \br^s-\bR^M(o^s) } }. 
\end{align}
In the $t$-th iteration, the \omle~algorithm plays the explorative policy of the most optimistic model within a $\beta$-superlevel set of the above objective: 
\begin{align}
\label{equation:omle-eq-exp-2}
    M^t \defeq \argmax_{M\in\cM}f^M(\pi_M) \quad \textrm{such that} \quad  \cL_t(M) \ge \max_{M'} \cL_t(M') - \beta,
\end{align}
and set $\pi^t=\piest_{M^t}$. After executing $\pi^t$ and observing $(o^t, r^t)$, the algorithm updates the objective function \eqref{eqn:MLE-risk}. The full algorithm description is given in~\cref{alg:OMLE}.

We now define the Maximum likelihood estimation coefficient (MLEC) and provide theoretical guarantees for the \omle~algorithm.
\begin{definition}[Maximum likelihood estimation coefficient]
\label{definition:omlec}
The maximum likelihood estimation coefficient (MLEC) of model class $\cM$ with respect to reference model $\oM\in\cM$, parameter $\gamma>0$, and length $K\in\Z_{\ge 1}$ is defined as %
\begin{align*}
 \omlec_{\gamma,K}(\cM, \oM) 
    \defeq \sup_{\set{M^k}\in\cM} \bypart{ \frac{1}{K} \sum_{k=1}^K \brac{f^{M^k}(\pi_{M^k}) - f^{\oM}(\pi_{M^k})} }{
     - \frac{\gamma}{K} \Big[ 1\vee \max_{k\in[K]} \sum_{t\le k-1} \wtdH^2(\oM(\piest_{M^t}), M^k(\piest_{M^t})) \Big] }.
\end{align*}
\end{definition}

\begin{theorem}[Guarantee for OMLE]
\label{theorem:OMLE}
Choosing $\beta=3\log(|\cM|/\delta)\ge 1$, with probability at least $1-\delta$, Algorithm~\ref{alg:OMLE} achieves
\begin{align*}
    \suboptpac = f^\Ms(\pi_\Ms) - \frac1T\sum_{t=1}^T \brac{ f^\Ms(\pi_{M^t})} \le  \inf_{\gamma>0}\set{ \omlec_{\gamma,T}(\cM, M^\star) + 12\gamma\cdot  \frac{\log(|\cM|/\delta)}{T}}.
\end{align*}
In particular, in the on-policy scenario when $\piest_M=\pi_M$ for all $M\in\cM$, Algorithm~\ref{alg:OMLE} achieves
\begin{align*}
    \regdm \le  \inf_{\gamma>0}\set{ T\cdot \omlec_{\gamma,T}(\cM, M^\star) + 12\gamma\cdot  \log(|\cM|/\delta)}.
\end{align*}
\end{theorem}
Existing sample-complexity guarantees for \omle-type algorithms are only established for specific RL problems through case-by-case analyses~\citep{mete2021reward,uehara2021representation,liu2022partially,liu2022sample}. In contrast,~\cref{theorem:OMLE} shows that \omle~works on any problem with bounded MLEC, thereby offering a more unified understanding. %

\paragraph{Bounding MLEC on concrete model classes}
We show that for any \belrep~$\cG$ of the model class $\cM$, the MLEC of $\cM$ can be bounded in terms of the Eluder dimension of $\cG$.
\begin{proposition}%
\label{proposition:omlec-be}
Suppose that $\cG$ is a \belrep~of $\cM$ with exploration policy $M\mapsto \piest_M$. Then for any $\oM\in\cM$, $\gamma>0, K\in\Z_{\ge 1}$, we have
\begin{align*}
    \omlec_{\gamma, K}(\cM, \oM) \le 5H^2L\cdot \inf_{\Delta>0}\set{\frac{\max_{h\in[H]}\eluder(\cG^{\oM}_h, \Delta/L)}{\gamma} + \Delta} .
\end{align*}
\end{proposition}
The proof of \cref{proposition:omlec-be} is deferred to \cref{appdx:proof-omlec}. As a corollary, suppose that $\max_{h\in[H]}\eluder(\cG^{M^\star}_h, \Delta)\leq \tO(d_{\eluder})$, then \omle~achieves $\EE^{M^\star}[\suboptpac]\leq \tO(\sqrt{d_{\eluder}LH^2\beta/T})$. In particular, for all the concrete problem classes in \cref{tab:examples}, the sample complexity upper bounds also hold for \omle~up to logarithmic factors.

\paragraph{Relationship between MLEC and PSC}
The MLEC resembles the PSC in that they both control a certain \emph{decoupling} error between a family of models and their optimal policies. The main difference is that the MLEC concerns any \emph{triangular sequence} of (model, policy) pairs. In contrast, the PSC concerns \emph{distribution} of models and the corresponding distribution of policies. Intuitively, the sequential nature of the MLEC makes controlling it harder than the PSC. 
Indeed, the following proposition allows us to transfer the upper bound of MLEC to the bound on PSC (a more general version is \cref{prop:mlec-to-psc}).
\begin{proposition}\label{prop:mlec-to-psc-demo}
Suppose that for some constant $\alpha>0$, $\mlec_{\gamma,K}(\cM,\oM)\leq \frac{d\cdot \polylog(\gamma,K)}{\gamma^\alpha}$ for all $\gamma\geq 1$. Then it holds that $\psc_\gamma(\cM,\oM)\leq \frac{d\cdot \polylog(\gamma)}{\gamma^\alpha}$ for all $\gamma\geq 1$.
\end{proposition}
In particular, when $\mlec_{\gamma,K}(\cM,M^\star)\leq \tO(d/\gamma)$ (which implies that OMLE achieves a sub-optimality of $\tO(\sqrt{d\log\abs{\cM}/T})$), PSC also admits an upper bound $\tO(d/\gamma)$ (with additional logarithmic factors), which implies that MOPS also enjoys a similar sub-optimality bound of $\tO(\sqrt{d\log\abs{\cM}/T})$. 

Therefore, we can generally expect MOPS to achieve comparably (or even better) performance than OMLE.
The following proposition further corroborates this intuition by demonstrating an exponential separation between \omle~and \mops~(hence \eetod). %

\begin{proposition}\label{prop:sep-omle-mops}
For any $N\geq 2$, $\Delta\in(0,0.1]$, there exists a class $\cM$ of $N$-arm bandits such that $\abs{\cM}=N$, and
\begin{itemize}
\item[(1)] There exists $\oM\in\cM$, such that \omle~reaches an $\Delta$-optimal policy until $T\geq\Om{\frac{N\beta}{\Delta^2}}$ episodes with probability at least $1/2$ on $\oM$.
\item[(2)] It holds that $\psc_{\gamma}(\cM)\leq\tO(\log^2(N)/\gamma)$, and hence for any model in $\cM$, \mops~(and also \eetod) returns an $\eps$-optimal policy using $\tO(\frac{\log^3(N)}{\eps^2})$ episodes.
\end{itemize}
\end{proposition}

\section{Conclusion}

This paper proposes a unified algorithm framework based on the Decision-Estimation Coefficient (DEC) for handling various learning goals in Reinforcement Learning (RL), such as no-regret RL, PAC RL, reward-free learning, model estimation, and preference-based RL. Our framework builds on a generic complexity measure \DECt~and a generic learning algorithm \getod, which elegantly handles each specific learning goal by simply specifying the goal $\GOAL$. Instantiating our results to each learning goal recovers existing results and yield many new results for a wide range of problem classes. 

We believe our work opens up many important questions, and we list a few of them. \revise{First, several limitations of our framework, such as the computational efficiency of our algorithms and the gaps between our lower and upper bounds, are worth exploration in future work. Specifically, the computational efficiency of the DEC framework is largely unknown. While polynomial-time implementation may be infeasible in general, the computational difficulty can be mitigated by assuming certain computational oracle (e.g. ERM or a sampling oracle) and/or relaxing the min-max optimization~\citep[etc.]{foster2020beyond,amortila2024scalable}. Furthermore, in some cases, even computing the DEC itself may be challenging, especially when the model or policy class involves neural parameterization. These scenarios may require case-by-case study and specialized analysis, and we believe that identifying more examples of DECs in general function approximation, especially with deep neural networks, would be an intriguing direction for future research. Lastly, closing the gaps between the lower and upper bounds \emph{beyond} reward-based learning remains a crucial open question, even for specific learning goals.
}

Second, our complexity measure and algorithm framework are inherently model-based. While model-based algorithms are already powerful enough for achieving sample-efficient learning in many RL learning goals, model-free algorithms are also widely used both in practice and also shown to be sample-efficient theoretically in many scenarios. How to generalize our unified results to model-free algorithms, for example by extending the methods of~\citet{foster2022note}, could be of interest. 

Finally, how to handle multi-agent RL is an interesting open question. While our results on model-estimation already implies a class of multi-agent RL results (\crefapp{G}{appendix:markov-game}), these results (as well as the results of~\citet{foster2023complexity}) are all based on \emph{centralized} model classes, which could become exponentially large in games with a large number of players. How to do sample-efficient multi-agent RL with \emph{independent} model class or value function classes for each player is still an active area of research~\citep{wang2023breaking,cui2023breaking}, and it is an open question if the DEC framework can be extended to produce such a type of algorithm. %

\section*{Acknowledgement}

Song Mei is supported by NSF CCF-2315725, DMS-2210827, an NSF Career award DMS-2339904, and a Google Research Scholar Award.

\bibliographystyle{abbrvnat}
\bibliography{note.bib}

\appendix

\tableofcontents

\section{Additional discussions}

\subsection{Comparison to Foster et al. (2021)}\label{appdx:compare}

\revise{
This work builds on the seminal study by \citet{foster2021statistical}. In this section, we provide a detailed comparison with \citet{foster2021statistical} and discuss our innovations in relation to their findings.

\paragraph{Generalized framework}
While \citet{foster2021statistical} primarily focus on reward-based no-regret decision making within the DEC framework, we extend their techniques to a broader context. By abstracting and re-formulating their
approach, we generalize the DEC framework to address any generalized learning goal (\cref{sec:gen-dec}).

\paragraph{Instantiations to various learning goals} Using this generalized framework, we offer deeper insights into the complexity of model-based reward-free learning (\cref{section:rfec}), model estimation (\cref{section:mdec}), and preference-based learning (\cref{section:pbrl}) — areas that are less understood compared to the reward-based setting. 

\paragraph{Applications to concrete problems}
We propose the \belrep~(\cref{section:bellman-rep}) as a useful tool for establishing guarantees across a range of learning objectives, previously limited mostly to reward-based learning. Additionally, the \belrep~framework enables us to derive sample complexity bounds analogous to those in reward-based settings. We illustrate the effectiveness of this approach by presenting new results for preference-based learning, model estimation, and partially observable RL.

\paragraph{Development of Tempered Aggregation}
We introduce Tempered Aggregation, which offers a stronger online estimation guarantee than Vovk’s
aggregation (see e.g., our discussion in \cref{corollary:restate-lemma-tempered-aggregation}). This stronger guarantee enables us to derive an upper bound based on the randomized
reference model, rather than the improper reference model used in~\citet{foster2021statistical}. Although this
improvement is primarily technical
(see our discussion in \cref{appendix:comparison-dec} below), we believe it provides a more natural and intuitive upper bound
with an interpretation based on posterior sampling.
}

\subsubsection{Detailed comparison: DEC definitions and E2D instantiations}
\label{appendix:comparison-dec}

\newcommand{\coM}{\co(\cM)}

Here we discuss the differences between the Regret DEC definitions used in our \etod~and in the E2D algorithm of~\citet[Section 4.1]{foster2021statistical}, which employs Vovk's aggregating algorithm as the subroutine (henceforth \etodva). Recall that the regret bound of \etod~scales with $\odec_\gamma(\cM)$ defined in~\cref{definition:dec} (cf.~\cref{thm:E2D-TA}). 

We first remark that all the following DECs considered in \citet{foster2021statistical} are defined in terms of the squared Hellinger distance $\dH^2(M(\pi), \oM(\pi))$ between the full distribution of $(o, \br)$ induced by models $M$ and $\oM$ under $\pi$, instead of our $\wtdH^2$ which is defined in terms of squared Hellinger distance in $o$ and squared $L_2$ loss in (the mean of) $\br$. However, all these results hold for $\wtdH^2$ as well, with the DEC definition and algorithms changed correspondingly. \revise{For simplicity, in this section we focus on the setting where the reward distribution is known, i.e., for any $M\in\cM$ the reward distribution $\RM(\cdot|o)=\mathsf{R}_0(\cdot|o)$ for a fixed $\mathsf{R_0}$. In this case, our divergence $\wtdH=\dH$ agrees with Hellinger distance. This simplification will not affect the essence of the comparisons.

We state the original definition of DEC~\citep{foster2021statistical} as follows. The convex hull of a model class $\cM$ is defined as
\begin{align}\label{eqn:def-coM}
    \coM\defeq \set{ \oM_\mu: ~\forall \pi\in\Pi, \oM_\mu(\pi)=\E_{\oM\sim\omu}\brac{\oM(\pi)} }_{\mu\in\Delta(\cM)},
\end{align}
i.e., $\coM\subseteq (\Pi\to\Delta(\cO))$ is the set of all convex combinations of models in $\cM$. The DEC with \emph{improper reference} is then defined as
\begin{align}
\label{equation:dec-co-M}
     \dec_\gamma(\cM,\oM_\mu) 
    \defeq & \inf_{p\in\Delta(\Pi)} \sup_{M\in\cM}
    \E_{\pi\sim p} \brac{ f^M(\pi_M) - f^M(\pi) - \gamma\dH^2\paren{M(\pi), \oM(\pi)}},
\end{align}
and we denote
\begin{align*}
    \dec_\gamma(\cM,\coM)\defeq \sup_{\oM_\mu\in\coM} \dec_\gamma(\cM,\oM_\mu).
\end{align*}
Based on this notion of DEC, \citet[Theorem 3.3 \& 4.1]{foster2021statistical} show that \etodva~achieves the following regret bound with probability at least $1-\delta$:
\begin{align*}
    \regdm\leqsim 
        T\dec_\gamma(\cM, \coM)+\gamma\log(\abs{\cM}/\delta).
\end{align*}

Compared with $\odec_\gamma(\cM,\mu)$, Eq. \cref{equation:dec-co-M} with $\oM_\mu$ is different only in the place where the expectation $\E_{\oM\sim\omu}$ is taken. As $\dH^2$ is convex in the second argument, by Jensen's inequality, we have $\dec_\gamma(\cM, \oM_\mu)\ge \odec_\gamma(\cM,\mu)$ for any $\mu\in\Delta(\cM)$. Therefore,
\begin{align*}
    \dec_\gamma(\cM, \coM)\ge \odec_\gamma(\cM), \quad \forall \gamma>0.
\end{align*}
This indicates that our \cref{thm:E2D-TA} provides a tighter bound than \citet[Theorem 3.3 \& 4.1]{foster2021statistical}. 

After the initial release of this paper, \citet{foster2022note} further prove that
\begin{align*}
    \dec_\gamma(\cM,\coM) \leq \dec_{\gamma/C}(\cM)
\end{align*}
for some universal constant $C$ \citep[Proposition 3.2]{foster2022note}. This result demonstrates that the quantities $\dec_\gamma(\cM,\coM) \asymp \dec_{\gamma/C}(\cM)$ are indeed equivalent. Consequently, it follows that \etodva~achieves the same regret bound as in \cref{thm:E2D-TA}. This implies that throughout our results, the Tempered Aggregation algorithm can be substituted with Vovk’s aggregation algorithm while preserving the same suboptimality/regret bound. Nevertheless, we believe that the Tempered Aggregation algorithm is of independent interest. 
}

\newcommand{\udec}{\underline{\dec}}

\subsection{Connections to the Constrained DECs}\label{appdx:constrained-dec}

\revise{
In this section, we discuss the concurrent work of \citet{foster2023tight} and subsequent work of \citet{foster2023complexity} in detail.

\newcommand{\pdecc}{\operatorname{p-dec}^{\rm c}}
\newcommand{\ueps}{\uline{\eps}}
\newcommand{\beps}{\bar{\eps}}

\newcommand{\creg}{c_{\rm reg}}
\paragraph{Constrained DECs}
\citet{foster2023tight} introduce the concept of the \emph{constrained} DEC, which can be regarded as an variant of the \emph{offset} DEC studied in \citet{foster2021statistical} and this paper. To provide a succint illustration, we considered the constrained PACDEC as an example, and we slightly adapt the notation of \citet{foster2023tight} for clarity.

For a model class $\cM$, $\eps\in[0,1]$, and reference model $\oM$, define the constrained PACDEC of $\cM$ with respect to $\oM$ as
\begin{align*}
    &\pdecc_\eps(\cM,\oM) \\
    &\defeq \inf_{\substack{\pexp\in\Delta(\Pi) \\ \pout\in\Delta(\Pi)}}\sup_{M\in \cM} \set{ \EE_{\pi \sim \pout}\left[f^M(\pi_M)-f^M(\pi)\right] : \EE_{\pi \sim \pexp} \DH{ M(\pi),\oM(\pi) }\leq \eps^2 }.
\end{align*}
The PACDEC of $\cM$ is defined as $\pdecc_\eps(\cM)=\sup_{\oM\in\coM} \pdecc_\eps(\cM,\oM)$, where we recall that $\coM$ is the convex hull of $\cM$, defined in \cref{eqn:def-coM}.

By Langrangian duality, we have 
\begin{align*}
    \pdecc_\eps(\cM)\leq \inf_{\gamma>0}\paren{ \eec_\gamma(\cM)+\gamma\eps^2 },
\end{align*}
and we also have $\eec_\gamma(\cM)\leq \pdecc_{1/\gamma}(\cM)$. Therefore, bounds with (offset) PACDEC and bounds with the constrained PACDEC can be converted to each other (though potentially losing polynomial factors). For a detailed discussion, see e.g. \citet[Section 4]{foster2023tight}.

\paragraph{Bounds with constrained DECs}
Remarkably, using on the constrained PACDEC, \citet{foster2023tight} characterize the $T$-round minimax-optimal sub-optimality for PAC RL as
\begin{align}\label{eqn:constrained-tight}
    \pdecc_{\ueps(T)}(\cM)-C\ueps(T)\leqsim \inf_{\Alg}\sup_{\Ms\in\cM} \EE^{\Ms,\Alg} \brac{ \suboptpac } \leqsim \pdecc_{\beps(T)}(\cM),
\end{align}
where $C$ is a universal constant, the logarithmic factors are omitted, and
\begin{align*}
    \ueps(T)\asymp \sqrt{\frac{1}{T}}, \qquad
    \beps(T)\asymp \sqrt{\frac{\log|\cM|}{T}}.
\end{align*}
Therefore, for PAC RL, the constrained PACDEC provides nearly matching lower and upper bounds, with only a gap of logarithmic factors and $\log|\cM|$, the model class complexity. Notably, the lower and upper bounds described above are tighter than the bounds for PAC RL presented in \cref{section:eec}.

Further, for no-regret learning (\cref{section:e2d}), \citet{foster2021statistical} similarly propose a constrained version of the Regret DEC. Through a sohpisticated analysis, they derive analoguous lower and upper bounds for $T$-round minimax-optimal regret with the constrained Regret DEC.

Later, \citet{foster2023complexity} extend the constrained DEC framework to partial monitoring and equilibrium learning in games. For these two learning goal, they provide non-matching lower and upper bounds based on the corresponding constrained DEC. We also note that both equilibrium learning and partial monitoring fall within the generalized PAC learning goals,\footnote{
Partial monitoring corresponds to the case $\ssp=\Delta(\Pi)$ and $\subopt$ being a linear functional over $\Delta(\Pi)$, and so does the equilibrium learning (see also \cref{appendix:markov-game}, where we apply the AMDEC framework for learning equilibrium). 
} meaning their results are complementary to ours.

\subsubsection{Relation of \citet{foster2023tight, foster2023complexity} with this work}
For \emph{reward-based} decision making (PAC RL and no-regret learning), constrained DECs give tighter bounds than offset DEC (\cref{section:e2d} \& \ref{section:eec}). However, for more general learning goals (e.g. reward-free learning and partial monitoring), the situation is more subtle. 
Below, we examine the tightness of our general lower and upper bounds for general learning goals, building on the insights developed in \citet{foster2023tight,foster2023complexity}.

\paragraph{Improving lower bounds for general learning goals}
While constrained DECs offer tighter lower bounds for reward-based learning, the analysis in \citet{foster2023tight} is highly tailored to this setting. 

Indeed, for more general learning goals, achieveing a similarly tight lower bound turns out to infeasible, as demonstrated by \citet{foster2023complexity}.
More specifically, for partial monitoring, \citet[Theorem D.3 \& D.4]{foster2023complexity} construct two isomorphic model classes $\cM_1, \cM_2$ whose the minimax rates are $\tbO{T^{-1}}$ and $\tbO{T^{-1/2}}$, respectively\footnote{By the term ``isomorphic'', we mean there is a one-to-one mapping that preserves both the risk functional and the pairwise Hellinger distance. For details, see \citet[Appendix D]{foster2023complexity}.}. In particular, for model class $\cM_1$, our lower bound is nearly tight, while for $\cM_2$ our upper bound is nearly tight, indicating that the gap in \cref{eqn:SC} cannot be avoided. Furthermore, for this learning goal, the results of \citet{foster2023complexity} imply a unavoidable gap between the lower and upper bounds for \emph{any} DEC-like complexity measure.

As a result, there is only limited scope for further improving our lower and upper bounds for general learning goals. We believe this argument similarly applies to preference-based learning and reward-free RL.

\paragraph{Improving upper bounds for general learning goals} We note that the PAC E2D algorithm of \citet{foster2023tight} in fact applies to any generalized PAC learning goal. Therefore, for generalized PAC learning, we can always improve the offset DEC upper bounds (\cref{thm:E2D-gen}) using the constrained DECs. 

However, the Regret E2D algorithm introduced by \citet{foster2023tight} (and later refined by \citet{glasgow2023tight}) is specifically designed for reward-based learning, and is challenging to adapt to generalized no-regret learning. In fact, the definition of constrained Regret DEC relies on the value function $f^{\oM}(\pi)$ being well-defined for any $\oM\in\coM$, which may not extend to more general settings. Therefore, for generalized no-regret learning, the only known upper bound is our \cref{thm:E2D-gen}.

On the other hand, offset DECs in fact provide comparable upper bounds for reward-based decision making, under certain (relatively) mild conditions, as discussed in \citet[Section 4.1]{foster2023tight}. Specifically, suppose that $\pdecc_{\eps}(\cM)$ satisfies the following regularity condition: for some constant $\creg\geq 1$,
\begin{align}\label{eqn:def-growth}
    \creg^{-1}\leq \frac{\pdecc_{\eps'}(\cM)}{\eps'}\leq \creg \frac{\pdecc_{\eps}(\cM)}{\eps}, \qquad \forall 0<\eps\leq \eps'\leq 1.
\end{align}
Then, by \citet[Proposition 4.2]{foster2023tight}, we have
\begin{align*}
    \inf_{\gamma>0} \paren{ \eec_\gamma(\cM)+\gamma\eps^2 }\leq \bigO{\creg\log(1/\eps)} \pdecc_\eps(\cM).
\end{align*}
In words, condition \cref{eqn:def-growth} requires that the constrained PACDEC does not decay ``too fast'', i.e., learning model class $\cM$ is \emph{non-trivial}.
Condition \cref{eqn:def-growth} is automatically satisfied if $\pdecc_{\eps}(\cM)~\propto~d\eps^{\rho}$ for some parameter $\rho\leq 1$, which is indeed the case for most natural problem classes (see, e.g. \citet[Section 6]{foster2023tight}).

Therefore, for most applications, with the optimally tuned parameter $\gamma$, the upper bound of \eetod~is at most a logarithmic factor worse than the constrained DEC bound. A similar argument applies to Regret DEC (see, e.g. \citet[Theorem G.5]{chen2024assouad}). 
}

\section{Technical tools}

\subsection{Strong duality}

The following strong duality result for variational forms of bilinear functions is standard, e.g. extracted from the proof of \citet[Proposition 4.2]{foster2021statistical}.
\begin{theorem}[Strong duality]\label{thm:strong-dual}
    Suppose that $\cX$, $\cY$ are two topological spaces, such that $\cX$ is Hausdorff\footnote{
        The Hausdorff space requirement of $\cX$ is only needed to ensure that $\Delta(\cX)$ contains all finitely supported distributions on $\cX$.
    } and $\cY$ is finite (with discrete topology). %
    Then for a bi-continuous function $f:\cX\times\cY\to\RR$ that is uniformly bounded, it holds that
    $$
    \sup_{X\in\Delta(\cX)}\inf_{Y\in\Delta(\cY)}\EE_{x\sim X}\EE_{y\sim Y}\brac{f(x,y)}
    =
    \inf_{Y\in\Delta(\cY)}\sup_{X\in\Delta(\cX)}\EE_{x\sim X}\EE_{y\sim Y}\brac{f(x,y)}.
    $$
\end{theorem}

In this paper, for most applications of \cref{thm:strong-dual}, we take $\cX=\cM$ and $\cY=\Pi$. We will assume that $\Pi$ is finite, which is a natural assumption. For example, in tabular MDPs, it is enough to consider deterministic Markov policies and there are only finitely many of them. Also, the finiteness assumption in \cref{thm:strong-dual} can be relaxed---The strong duality holds as long as both $\cX, \cY$ is Hausdorff, and the function class $\set{f(x,\cdot):\cY\to\R}_{x\in\cX}$ has a finite $\rho$-covering for all $\rho>0$. Such relaxed assumption is always satisfied in our applications.

\subsection{Concentration inequalities}

We will use the following standard concentration inequality in the paper. 

\begin{lemma}[{\citet[Lemma A.4]{foster2021statistical}}]\label{lemma:concen}
For any sequence of real-valued random variables $\left(X_t\right)_{t \leq T}$ adapted to a filtration $\left(\cF_t\right)_{t \leq T}$, it holds that with probability at least $1-\delta$, for all $t \leq T$,
$$
\sum_{s=1}^{t} -\log \cond{\exp(-X_s)}{\cF_{s-1}} \leq \sum_{s=1}^{t} X_s +\log \left(\delta^{-1}\right).
$$
\end{lemma}

\subsection{Properties of the Hellinger distance}

Recall that for two distributions $\P, \Q$ that are absolutely continuous with respect to $\mu$, their squared Hellinger distance is defined as
\[
\dH^2(\P,\Q) \defeq \int (\sqrt{d \P/d\mu}-\sqrt{d \Q/d\mu})^2 d\mu. 
\]
We will use the following properties of the Hellinger distance. 
\begin{lemma}[{\citet[Lemma A.11, A.12]{foster2021statistical}}]
\label{lemma:multiplicative-hellinger}
For distributions $\P,\Q$ defined on $\cX$ and function $h:\cX\to[0,R]$, we have
\begin{align*}
    \abs{ \E_{\P}\brac{h(X)} - \E_{\Q}\brac{h(X)} } \le \sqrt{2R\paren{ \E_{\P}\brac{h(X)} + \E_{\Q}\brac{h(X)} }\cdot \dH^2(\P, \Q)}.
\end{align*}
Therefore, $\E_{\P}\brac{h(X)} \leq 3\E_{\Q}\brac{h(X)} +2R\dH^2(\P, \Q).$
Also, for function $h:\cX\to[-R,R]$, we have
\begin{align*}
    \abs{ \E_{\P}\brac{h(X)} - \E_{\Q}\brac{h(X)} } \le \sqrt{8R\paren{ \E_{\P}\brac{\abs{h(X)}} + \E_{\Q}\brac{\abs{h(X)}} }\cdot \dH^2(\P, \Q)}.
\end{align*}
\end{lemma}

\begin{lemma}\label{lemma:Hellinger-cond}
    For any pair of random variable $(X,Y)$, it holds that
    \begin{align*}
        \EE_{X\sim\PP_X}\brac{\DH{\PP_{Y|X}, \QQ_{Y|X}}}\leq 2\DH{\PP_{X,Y}, \QQ_{X,Y}}.
    \end{align*}
    Conversely, it holds that
    \begin{align*}
    \DH{\PP_{X,Y}, \QQ_{X,Y}}\leq 3\DH{\PP_{X}, \QQ_{X}}+2\EE_{X\sim\PP_X}\brac{\DH{\PP_{Y|X}, \QQ_{Y|X}}}.
    \end{align*}
\end{lemma}
\begin{proof}
Throughout the proof, we slightly abuse notations and write a distribution $\P$ and its density $d \P / d \mu$ interchangeably. By the definition of the Hellinger distance, we have
\begin{align*}
    & \quad \frac12\DH{\PP_{X,Y}, \QQ_{X,Y}} = 1-\int \sqrt{\PP_{X,Y}}\sqrt{\QQ_{X,Y}}\\
    &=1-\int \sqrt{\PP_X\QQ_X}\sqrt{\PP_{Y|X}}\sqrt{\QQ_{Y|X}}\\
    &\geq 1-\int \frac{\PP_X+\QQ_X}{2}\sqrt{\PP_{Y|X}}\sqrt{\QQ_{Y|X}}\\
    &=\int \frac{\PP_X+\QQ_X}{2}\paren{1-\sqrt{\PP_{Y|X}}\sqrt{\QQ_{Y|X}}}\\
    &=\frac14\EE_{X\sim\PP_X}\brac{\DH{\PP_{Y|X}, \QQ_{Y|X}}}+\frac14\EE_{X\sim\QQ_X}\brac{\DH{\PP_{Y|X}, \QQ_{Y|X}}}.
\end{align*}
Similarly, 
\begin{align*}
    \frac12\DH{\PP_{X,Y}, \QQ_{X,Y}}
    &=1-\int \sqrt{\PP_X\QQ_X}+\int \sqrt{\PP_X\QQ_X}(1-\sqrt{\PP_{Y|X}\QQ_{Y|X}})\\
    &\leq \frac12\DH{\PP_{X}, \QQ_{X}}+\int \frac{\PP_X+\QQ_X}{2}\cdot\frac12\DH{\PP_{Y|X}, \QQ_{Y|X}},
\end{align*}
and hence
\begin{align*}
\DH{\PP_{X,Y}, \QQ_{X,Y}}
\leq& \DH{\PP_{X}, \QQ_{X}}+\frac12\EE_{X\sim\PP_X}\brac{\DH{\PP_{Y|X}, \QQ_{Y|X}}}\\
&+\frac12\EE_{X\sim\QQ_X}\brac{\DH{\PP_{Y|X}, \QQ_{Y|X}}}\\
\leq& 3\DH{\PP_{X}, \QQ_{X}}+2\EE_{X\sim\PP_X}\brac{\DH{\PP_{Y|X}, \QQ_{Y|X}}},
\end{align*}
where the last inequality is due to \cref{lemma:multiplicative-hellinger} and $\dH^2\in[0,2]$.
\end{proof}

Next, recall the divergence $\wtdH^2$ defined in~\cref{equation:modified-hellinger}:
\begin{align*}
    \wtdH^2(M(\pi), \oM(\pi)) = \dH^2(\MP(\pi), \oMP(\pi)) + \E_{o\sim \MP(\pi)}\brac{ \ltwot{\bR^M(o) - \bR^{\oM}(o)} }.
\end{align*}

\begin{proposition}\label{prop:modified-to-true}
Recall that $(o, \br) \sim M(\pi)$ is the observation and reward vectors as described in Section \ref{sec:prelim}, with $o \sim \MP(\pi)$ and $\br \sim \RM(\cdot | o)$. 
Suppose that $\br\in[0,1]^H$ almost surely and  $\| \bR^M(o)-\bR^{\oM}(o) \|_2^2 \leq 2$ for all $o\in\cO$.
Then it holds that
\begin{align*}
    \wtdH^2(M(\pi), \oM(\pi))\leq 5\DH{M(\pi), \oM(\pi)},
\end{align*}
where $\DH{M(\pi), \oM(\pi)}$ is the standard squared Hellinger distance between the distribution of $(o,\br)$ under $M(\pi)$ and the distribution of $(o,\br)$ under $\oM(\pi)$. 
\end{proposition}

\begin{proof}
To prove this proposition, we need to bound $\| \bR^M(o)-\bR^{\oM}(o) \|_2^2$ in terms of $\DH{\RM(o),\oMR(o)}$. We denote by $\RM_h(o)$ the distribution of $r_h$. Then by independence, we have
\begin{align*}
    1-\frac{1}{2}\DH{\RM(o),\oMR(o)}
    =&\prod_{h}\paren{1-\frac12\DH{\RM_h(o),\oMR_h(o)}}\\
    \leq &\prod_{h}\paren{1-\frac12\DTVt{\RM_h(o),\oMR_h(o)}}\\
    \leq &\prod_{h}\paren{1-\frac12\abs{R^M_h(o)-R^{\oM}_h(o)}^2}\\
    \leq& \exp\paren{-\frac12\ltwot{\bR^M(o)-\bR^{\oM}(o)}}\\
    \leq& 1-\frac{1}{4}\ltwot{\bR^M(o)-\bR^{\oM}(o)},
\end{align*}
where the last inequality use the fact that $e^{-x}\leq 1-x/2$ for all $x\in[0,1]$. Then by \cref{lemma:Hellinger-cond},
\begin{align*}
\E_{o\sim \MP(\pi)}\brac{ \ltwot{\bR^M(o) - \bR^{\oM}(o)} }
\leq& 2\E_{o\sim \MP(\pi)}\brac{ \dH(\RM(o), \oMR(o))^2 }\\
\leq& 4\DH{M(\pi), \oM(\pi)}.
\end{align*}
Combining the above estimation with the fact that $\dH^2(\MP(\pi), \oMP(\pi))\leq \DH{M(\pi), \oM(\pi)}$ (data-processing inequality) completes the proof. 
\end{proof}

The following lemma shows that, although $\wtdH^2$ is not symmetric with respect to its two arguments (due to the expectation over $o\sim \MP(\pi)$ in the second term), it is \emph{almost} symmetric within a constant multiplicative factor:
\begin{lemma}\label{lemma:diff-hellinger}
For any two models $M, \oM$ and any policy $\pi$, we have
\begin{align*}
    \wtdH^2(\oM(\pi), M(\pi)) \le 5\wtdH^2(M(\pi), \oM(\pi)).
\end{align*}
\end{lemma}
\begin{proof}
For any function $h:\cO\to[0,2]$, by~\cref{lemma:multiplicative-hellinger} we have
\begin{align*}
    \E_{o\sim \oMP(\pi)}\brac{h(o)}
    \leq 3\E_{o\sim \MP(\pi)}\brac{h(o)}+4\dH^2(\MP(\pi), \oMP(\pi)).
\end{align*}
Therefore, we can take $h$ as $h(o)\defeq \ltwot{\bR^M(o) - \bR^{\oM}(o)}$, and the bound above gives
\begin{align*}
    \underbrace{\dH^2(\MP(\pi), \oMP(\pi)) + \E_{o\sim \oMP(\pi)}\brac{h(o)} }_{\wtdH^2(\oM(\pi), M(\pi))} \le \underbrace{5 \dH^2(\MP(\pi), \oMP(\pi)) + 5\E_{o\sim \MP(\pi)}\brac{h(o)} }_{= 5\wtdH^2(M(\pi), \oM(\pi))},
\end{align*}
which is the desired result.
\end{proof}

\begin{lemma}[Bounding value difference by $\wtdH$]
\label{lemma:diff-reward}
For any two models $M, \oM$ and any policy $\pi$, we have
\begin{align*}
    \abs{f^M(\pi) - f^{\oM}(\pi)} \le \sqrt{H+1} \cdot \wtdH(M(\pi), \oM(\pi)).
\end{align*}
\end{lemma}
\begin{proof}
We have
\begin{align*}
    & \quad \abs{f^M(\pi) - f^{\oM}(\pi)} = \abs{ \E_{o\sim \MP(\pi)} \brac{ \sum_h R^M_h(o) } - \E_{o\sim \oMP(\pi)} \brac{ \sum_h R^{\oM}_h(o) } } \\
    & \le \abs{ \E_{o\sim \MP(\pi)} \brac{ \sum_h R^M_h(o) - \sum_h R^{\oM}_h(o) }} + \abs{ \E_{o\sim \MP(\pi)} \brac{ \sum_h R^{\oM}_h(o) } - \E_{o\sim \oMP(\pi)} \brac{ \sum_h R^{\oM}(o) } } \\
    & \stackrel{(i)}{\le} \E_{o\sim \MP(\pi)} \brac{ \sqrt{H}\ltwo{\bR^M(o) - \bR^{\oM}(o)} } + \dH(\MP(\pi), \oMP(\pi)) \\
    & \stackrel{(ii)}{\le} \sqrt{(H+1)\paren{ \E_{o\sim \MP(\pi)} \brac{ \ltwot{\bR^M(o) - \bR^{\oM}(o)} } + \dH^2(\MP(\pi), \oMP(\pi))}} \\
    & = \sqrt{H+1} \cdot \wtdH(M(\pi), \oM(\pi)).
\end{align*}
Above, (i) uses the fact that $R^{\oM}(o)\in[0,1]$ almost surely, and the bound
$$
\abs{ \E_{o\sim \MP(\pi)} \brac{ R^{\oM}(o) } - \E_{o\sim \oMP(\pi)} \brac{ R^{\oM}(o) } }\le \dTV(\MP(\pi), \oMP(\pi))\le \dH(\MP(\pi), \oMP(\pi));
$$
(ii) uses the Cauchy inequality $\sqrt{H}a+b\le \sqrt{(H+1)(a^2+b^2)}$ and the fact that the squared mean is upper bounded by the second moment.
\end{proof}

\section{Proofs for Section~\ref{section:e2d}}
\label{appendix:proof-e2d}

\subsection{Aggregation algorithms as posterior computations}
\label{appendix:discussion-aggregation}

We illustrate that \Vovkalg~is equivalent to computing the \emph{tempered posterior} (or \emph{power posterior})~\citep{bhattacharya2019bayesian,alquier2020concentration} in the following vanilla Bayesian setting.

Consider a model class $\cM$ associated with a prior $\mu^1\in\Delta(\cM)$, and each model specifies a distribution $\P^M(\cdot)\in\Delta(\cO)$ of observations $o\in\cO$. Suppose we receive observations $o^1,\dots,o^t,\dots $ in a sequential fashion. In this setting, the \Vovkalg~updates
\begin{align*}
  \mu^{t+1}(M) \; \propto_M \; \mu^ t(M) \cdot \exp\paren{\etap \log \P^M(o^t)} = \mu^t(M) \cdot \paren{\P^M(o^t)}^{\etap}.
\end{align*}
Therefore, for all $t\ge 1$, 
\begin{align*}
  \mu^{t+1}(M) \; \propto_M \; \mu^1(M) \cdot \paren{\prod_{s=1}^t \P^M(o^s)}^{\etap}.
\end{align*}
If $\etap=1$ as in Vovk's aggregating algorithm~\citep{vovk1995game}, by Bayes' rule, the above $\mu^{t+1}$ is exactly the posterior of $M|o^{1:t}$. As we chose $\etap\le 1/2<1$ in \Vovkalg, $\mu^{t+1}$ gives the tempered posterior, which is a slower variant of the posterior where data likelihoods are weighed less than in the exact posterior.

\subsection{Tempered Aggregation for finite class}
\label{appendix:proof-e2d-ta}

In this section, we analyze the \Vovkalg\ algorithm for finite model classes. For the sake of both generality and simplicity, we state our results in the following general setup of online model estimation. %

\paragraph{Setup: Online model estimation}
In an online model estimation problem, the learner is given a model set $\cM$, a context space $\Pi$, an observation space $\cO$, a family of conditional distributions $(\PP^M(\cdot|\cdot):\Pi\to\Delta(\cO))_{M\in\cM}$\footnote{
We use $\PP^{M,\pi}(o)$ and $\PP^{M}(o|\pi)$ interchangeably in the following.
}, a family of vector-valued mean reward functions $(\bR^{M}: \cO\to[0,1]^H)_{M\in\cM}$. The environment fix a ground truth model $\Ms\in\cM$; for shorthand, let $\PPs:=\PP^{\Ms}, \Rs:=\bR^{\Ms}$. For simplicity (in a measure-theoretic sense) we assume that $\cO$ is finite\footnote{
To extend to the continuous setting, only slight modifications are needed, see e.g. \citet[Section 3.2.3]{foster2021statistical}.
}. For the case where the reward vector is empty, we regard $\br=\mathbf{0}_H$ and $\bR^{M}=\Rs=\mathbf{0}_H$.

At each step $t\in[T]$, the learner first determines a randomized model estimator (i.e. a distribution over models) $\mu^t\in\Delta(\cM)$. Then, the environment reveals the context $\xt\in\Pi$ (that is in general random and possibly depends on $\mu^t$ and history information), generates the observation $o^t\sim\PPs(\cdot|\xt)$, and finally generates the reward $\br^t\in\R^d$ (which is a random vector) such that $\cond{\br^t}{o^t}=\Rs(o^t)$. The information $(\xt,o^t,\br^t)$ may then be used by the learner to obtain the updated estimator $\mu^{t+1}$.

For any $M\in\cM$, we consider the following estimation error of model $M$ with respect to the true model, at step $t$:
\begin{align}\label{equation:err}
\err_{M}^t:=\EE_{t}\brac{\DH{\PP^M(\cdot|\xt), \PPs(\cdot|\xt)}+\ltwot{\bR^{M}(o^t)-\Rs(o^t)}},
\end{align}
where $\EE_t$ is taken with respect to all randomness after prediction $\mu^t$ is made\footnote{
In other words, $\EE_t$ is the conditional expectation on $\cF_{t-1}=\sigma(\mu^1,\pi^1,o^1,\br^1,\cdots,\pi^{t-1},o^{t-1},\br^{t-1},\mu^t)$.
}---in particular it takes the expectation over $(\xt, o^t)$. Note that $\err_{\Ms}^t=0$ by definition.

\paragraph{Algorithm and theoretical guarantee}
The \Vovkalg\ Algorithm is presented in~\cref{alg:TA}. Here we present the case with a finite model class ($\abs{\cM}<\infty$); In~\cref{appendix:ta-covering} we treat the more general case of infinite model classes using covering arguments.

\begin{algorithm}[t]
	\caption{\textsc{\Vovkalg}} 
	\begin{algorithmic}[1]
	\label{alg:TA}
	\REQUIRE Learning rate $\etap\in(0, \frac{1}{2}],\etar\geq 0$.
	\STATE Initialize $\mu^1\setto {\rm Unif}(\cM)$.
	\FOR{$t=1,\dots,T$}
	\STATE Receive $(\xt, o^t, \br^t)$.
    \STATE Update randomized model estimator:
    \begin{align}
        \mu^{t+1}(M) \; \propto_{M} \; \mu^{t}(M) \cdot \exp\paren{\etap\log \PP^M(o^t|\xt) - \etar\ltwot{ \br^t-\bR^{M}(o^t) } }.
    \end{align}
    \ENDFOR
    \end{algorithmic}
\end{algorithm}

\begin{theorem}[\Vovkalg]\label{thm:vovk-finite}
Suppose $\abs{\cM}<\infty$, the reward vector $\br^t$ is $\sigma^2$-sub-Gaussian conditioned on $o^t$, and $\ltwo{\bR^{M}(o^t)-\Rs(o^t)}\leq D$ almost surely for all $t\in[T]$. Then, \cref{alg:TA} with any learning rate $\etap,\etar>0$ such that $2\etap+2\sigma^2\etar<1$ achieves the following with probability at least $1-\delta$:
\begin{align*}
    \sum_{t=1}^T \E_{M\sim \mu^t}\brac{ \err_{M}^t } \le C \log(\abs{\cM}/\delta),
\end{align*}
where $C=C(\etap,\etar)=\max\set{\frac{1}{\etap}, \frac{1}{(1-2\etap)c'}}$, $c'\defeq (1-e^{-c(1-2\sigma^2 c)D^2})/D^2$ and $c\defeq \etar/(1-2\etap)$ are constants depending on $(\etap,\etar,\sigma^2,D)$ only. 
Furthermore, for the special case where reward vectors $\br^t$ are empty and $\etar=0$, $\etap\in(0,\frac12]$, we have
\begin{align*}
    \sum_{t=1}^T \E_{M\sim \mu^t}\brac{ \err_{M}^t } \le \frac{1}{\etap} \log(\abs{\cM}/\delta),
\end{align*}
\end{theorem}
The proof of~\cref{thm:vovk-finite} can be found in~\cref{appendix:proof-vovk-finite}.

\cref{thm:vovk-finite} yields the following corollary, which we state and prove below.
\begin{corollary}[Online estimation guarantee for \Vovkalg{}]
\label{corollary:restate-lemma-tempered-aggregation}
The \Vovkalg~subroutine~\cref{equation:tempered-aggregation} in~\cref{alg:E2D-TA} with $4\etap+\etar<2$ achieves the following bound with probability at least $1-\delta$:
\begin{align}
\label{equation:slow-vovk}
    \EstwtH \defeq \sum_{t=1}^T \E_{\pi^t\sim p^t}\E_{\hat{M}^t\sim \mu^t}\brac{ \wtdH^2(M^\star(\pi^t), \hat{M}^t(\pi^t)) } \le C(\etap,\etar) \cdot \log(\abs{\cM}/\delta),
\end{align}
where $C(\etap,\etar)$ is a constant depends only on $(\etap, \etar)$. Furthermore, for the special case where the reward vectors $\br^t$ are empty, we can also choose $\etar=0$ and $\etap\in(0,\frac12]$, and \eqref{equation:slow-vovk} above still holds true with $C(\etap,0)=\frac{1}{\etap}$. 

Specifically, in our instantiations of \getod, we always choose $\etap=\etar=1/3$ (with $C=10$) in general. In the case that $\br$ is known to be empty (e.g. for reward-free learning and preference-based learning, cf. \cref{sec:gen-dec}), we instead choose $\etap=\frac12, \etar=0$ (with $C=2$).
\end{corollary}

\paragraph{Comparison with Vovk's aggregating algorithm}
We remark that Bound~\cref{equation:slow-vovk} is stronger than the estimation bound for Vovk's aggregating algorithm (e.g.~\citet[Lemma A.15]{foster2021statistical}, adapted to $\wtdH^2$), which only achieves
\begin{align}
\label{equation:fast-vovk}
    \sum_{t=1}^T \E_{\pi^t\sim p^t}\brac{ \wtdH^2\paren{M^\star(\pi^t), \E_{\hM^t\sim \mu^t}\brac{\hat{M}^t(\pi^t)}} } \le C \cdot \log(\abs{\cM}/\delta),
\end{align}
where $\E_{\hM^t\sim \mu^t}\brac{\hM^t(\pi^t)}$ denotes the mixture model of $\hM^t(\pi^t)$ where $\hM^t\sim \mu^t$. Observe that~\cref{equation:slow-vovk} is stronger than~\cref{equation:fast-vovk} by convexity of $\wtdH^2$ in the second argument and Jensen's inequality.

\begin{proofof}[corollary:restate-lemma-tempered-aggregation]
Note that subroutine~\cref{equation:tempered-aggregation} in~\cref{alg:E2D-TA} is exactly an instantiation of the \Vovkalg~algorithm (\cref{alg:TA}) with context $\xt$ sampled from distribution $p^t$ (which depends on $\mu^t$), observation $o^t$, and reward $\br^t$. Therefore, we can apply~\cref{thm:vovk-finite}, where we further note that $\E_{M\sim\mu^t}\brac{\err_{M}^t}$ corresponds exactly to
\begin{align*}
    \E_{M\sim\mu^t}\brac{\err_{M}^t}
    & = \E_{\hat{M}^t\sim \mu^t} \E_{\pi^t\sim p^t}\brac{ \dH^2(\Pm^{M^\star}(\pi^t), \Pm^{\hat{M}^t}(\pi^t)) + \E_{o\sim \Pm^{M^\star}(\pi^t)}\ltwo{\bR^{M^\star}(o) - \bR^{\hat{M}^t}(o) }^2}  \\
    & = \E_{\hat{M}^t\sim \mu^t} \E_{\pi^t\sim p^t} \brac{
    \wtdH^2(M^\star(\pi^t), \hat{M}^t(\pi^t))
    }.
\end{align*}
Notice that we can pick $\sigma^2=1/4$ and $D=\sqrt{2}$, as each individual reward $r_h\in[0,1]$ almost surely (so is $1/4$-sub-Gaussian by Hoeffding's Lemma), and
\begin{align*}
    & \quad \ltwo{\bR^M(o) - \bR^{M'}(o)}^2 = \sum_{h=1}^H \abs{R^M_h(o) - R^{M'}_{h}(o)}^2 \\
    & \le \sum_{h=1}^H \abs{R^M_h(o) - R^{M'}_{h}(o)} \le \sum_{h=1}^H \abs{R^M_h(o)} + \abs{ R^{M'}_{h}(o)} = 2.
\end{align*}
for any two models $M,M'$ and any $o\in\cO$.
Therefore,~\cref{thm:vovk-finite} yields that, as long as $4\etap+\etar<2$, we have with probability at least $1-\delta$ that
\begin{align*}
    \EstwtH \defeq \sum_{t=1}^T \E_{\pi^t\sim p^t}\E_{\hat{M}^t\sim \mu^t}\brac{ \wtdH^2(M^\star(\pi^t), \hat{M}^t(\pi^t)) } \le C \cdot \log(\abs{\cM}/\delta),
\end{align*}
where $C=\max\set{\frac{1}{\etap}, \frac{1}{(1-2\etap)c'}}$, $c'= (1-e^{-c(2-c)})/2$, and $c=\etar/(1-2\etap)$. 
Choosing $\etap=\etar=1/3$, we have $c=1$, $c'=(1-e^{-1})/2$, and $C=\max\set{3, 3/c'}\le 10$ by numerical calculations. This is the desired result. The case $\etar=0$ follows similarly.
\end{proofof}

\subsubsection{Proof of Theorem~\ref{thm:vovk-finite}}
\label{appendix:proof-vovk-finite}

For all $t\in[T]$ define the random variable
$$
\Delta^t:=-\log \EE_{M\sim \mu^t}\left[ \exp\left( \etap\log \frac{\PP^M(o^t|\xt)}{\PPs(o^t|\xt) }+\etar\delta^t_{M}\right) \right],
$$
where 
\begin{align}\label{eqn:vovk-proof-delta}
\delta^t_{M}:=&\ltwot{ \br^t-\Rs(o^t) }-\ltwot{\br^t-\bR^{M}(o^t)}.
\end{align}
Recall that $\EE_t$ is taken with respect to all randomness after prediction $\mu^t$ is made. Then
\begin{align}
\label{eqn:vovk-proof-expminusdelta}
\begin{aligned}
\MoveEqLeft
\EEt{\exp\left(-\Delta^t\right)}
=
\EEt{ \EE_{M\sim \mu^t}\left[ \exp\left( \etap\log \frac{\PP^M(o^t|\xt)}{\PPs(o^t|\xt) } + \etar\delta^t_{M} \right) \right] }\\
=&
\sum_{M\in\cM}\mu^t(M)\EEt{ \exp\left( \etap\log \frac{\PP^M(o^t|\xt)}{\PPs(o^t|\xt) } + \etar\delta^t_{M} \right) }\\
\leq &
\sum_{M\in\cM}\mu^t(M)\EEt{ 2\etap\exp\left( \frac{1}{2} \log \frac{\PP^M(o^t|\xt)}{\PPs(o^t|\xt) } \right) + (1-2\etap)\exp\left( \frac{\etar}{1-2\etap}\delta^t_{M} \right) }\\
=&
2\etap\sum_{M\in\cM}\mu^t(M)\EEt{ \EE_{o\sim\PPs(\cdot|\xt)}\brac{\sqrt{\frac{\PP^M(o|\xt)}{\PPs(o|\xt)}}} }\\
&+(1-2\etap)\sum_{M\in\cM}\mu^t(M)\EEt{ \exp\left(\frac{\etar}{1-2\etap}\delta^t_{M} \right) }.
\end{aligned}
\end{align}
For the first term, by definition
\begin{align}
\label{eqn:vovk-proof-likelihood}
\EE_{o\sim\PPs(\cdot|\xt)}\brac{\sqrt{\frac{\PP^M(o|\xt)}{\PPs(o|\xt)}}}
=1-\frac12\dH^2(\PPs(\cdot|\xt), \PP^M(\cdot|\xt)).
\end{align}
To bound the second term, we abbreviate $c\defeq \frac{\etar}{1-2\etap}$, and invoke the following lemma. The proof can be found in~\cref{appendix:proof-vovk-delta-rew}.
\newcommand{\er}{\overline{\br}}
\newcommand{\rr}{\hat{\br}}
\begin{lemma}\label{lemma:vovk-delta-rew}
    Suppose that $\br\in\R^d$ is a $\sigma^2$-sub-Gaussian random vector, $\er=\EE[\br]$ is the mean of $\br$, and $\rr\in\R^d$ is any fixed vector. Then the random variable
    \begin{align*}
        \delta:=\ltwot{\br-\er}-\ltwot{\br-\rr},
    \end{align*}
    satisfies $\EE\brac{\exp(\lambda\delta)}\leq \exp\paren{-\lambda(1-2\sigma^2\lambda)\ltwot{\er-\rr}}$ for any $\lambda\in\R$.
\end{lemma}
Therefore,
\begin{align}\label{eqn:vovk-proof-reward}
\EEt{ \exp\left(c\delta^t_{M} \right) }
\leq&
\EEt{ \exp\left(-c(1-2\sigma^2c)\ltwot{ \bR^{M}(o^t)-\Rs(o^t) } \right)}\notag\\
\leq&
1- c'\EEt{\ltwot{ \bR^{M}(o^t)-\Rs(o^t) }},
\end{align}
where the second inequality is due to the fact that for all $x\in[0,D^2]$, it holds that $e^{-c(1-2\sigma^2c)x}\le 1-c'x$, which is ensured by our choice of $c\in[0,2/\sigma^2)$ and $c'\defeq (1-e^{-D^2 c(1-2\sigma^2c)})/D^2>0$. Therefore, by flipping ~\cref{eqn:vovk-proof-expminusdelta} and adding one on both sides, and plugging in~\cref{eqn:vovk-proof-likelihood} and~\cref{eqn:vovk-proof-reward}, we get
\begin{align*}
1-\EEt{ \exp\left(-\Delta^t\right) }   
\geq&
\etap\EE_{M\sim \mu^t} \EE_{\xt\sim\cdot|\cF_{t-1}}\brac{\dH^2(\PP^M(\cdot|\xt), \PPs(\cdot|\xt)) }\\
&+ (1-2\etap)c'
\EE_{M\sim \mu^t} \EEt{\ltwot{ \bR^{M}(o^t)-\Rs(o^t) }}.
\end{align*}
Thus, by martingale concentration (\cref{lemma:concen}), we have with probability at least $1-\delta$ that
\begin{align}
\label{equation:deltat-key-bound}
\begin{aligned}
    \sum_{t=1}^{T} \Delta^t+\log(1/\delta)
    \geq& \sum_{t=1}^{T}-\log \EEt{\exp\left(-\Delta^t\right)}
    \geq \sum_{t=1}^{T} 1-\EEt{\exp\left(-\Delta^t\right)}\\
    \geq&
    \etap\sum_{t=1}^{T}
    \EE_{M\sim \mu^t} \EE_{t}\brac{\dH^2(\PP^M(\cdot|\xt), \PPs(\cdot|\xt)) }\\
    &+ (1-2\etap)c'\sum_{t=1}^{T}
    \EE_{M\sim \mu^t} \EEt{\ltwot{ \bR^{M}(o^t)-\Rs(o^t) }} \\
    \geq& \min\set{ \etap, (1-2\etap)c' } \cdot \EE_{M\sim \mu^t}\brac{\err_{M}^t}.
\end{aligned}
\end{align}

It remains to upper bound $\sum_{t=1}^T\Delta^t$. Note that the update rule of~\cref{alg:TA} can be written in the following Follow-The-Regularized-Leader form:
\begin{align*}
    \mu^t(M) = \frac{ \mu^1(M)\exp\left( \sum_{s\le t-1} \etap\log \PP^M(o^s|\xs) + \etar\delta^s_{M} \right) }{\sum_{M'\in\cM} \mu^1(M') \exp\left( \sum_{s\le t-1} \etap \log \PP^{M'}(o^s|\xs) + \etar\delta^s_{M'} \right)},
\end{align*}
where we have used that $\delta^t_{M}=-\ltwot{\br^t-\bR^{M}(o^t)}+\ltwot{ \br^t-\Rs(o^t) }$ in which $\ltwot{ \br^t-\Rs(o^t) }$ is a constant that does not depend on $M$ for all $t\in[T]$. 
Therefore we have
\begin{align}\label{eqn:vovk-proof-sum-delta-eq}
\begin{split}
&\qquad\exp(-\Delta^t)
=\EE_{M\sim \mu^t}\left[ \exp\left( \etap\log \frac{\PP^M(o^t|\xt)}{\PPs(o^t|\xt) } + \etar\delta^t_{M}\right) \right]\\
&=\sum_{M\in\cM} \mu^t(M)\exp\left( \etap\log \frac{\PP^M(o^t|\xt)}{\PPs(o^t|\xt) } + \etar\delta^t_{M}\right)\\
&=\sum_{M\in\cM} \frac{ \mu^1(M)\exp\left( \sum_{s\le t-1} \etap\log \PP^M(o^s|\xs) + \etar\delta^s_{M} \right) }{\sum_{M'\in\cM} \mu^1(M') \exp\left( \sum_{s\le t-1} \etap \log \PP^{M'}(o^s|\xs) + \etar\delta^s_{M'} \right)} \exp\left( \etap\log \frac{\PP^M(o^t|\xt)}{\PPs(o^t|\xt) } + \etar\delta^t_{M}\right)\\
&=\frac{ \sum_{M\in\cM} \mu^1(M) \exp\left(  \sum_{s\le t}\etap\log \frac{\PP^M(o^s|\xs)}{\PPs(o^s|\xs) } + \etar\delta^s_{M}\right) }
{ \sum_{M\in\cM} \mu^1(M) \exp\left( \sum_{s\le t-1} \etap\log \frac{\PP^M(o^s|\xs)}{\PPs(o^s|\xs) } + \etar\delta^s_{M}\right) },
\end{split}
\end{align}
where the last equality used again the fact that $-\etap \log\PPs(o^s|\xs)$ is a constant that does not depend on $M$ for all $s\in[t]$.

Taking $-\log$ on both sides above and summing over $t\in[T]$, we have by telescoping that
\begin{align}
\label{equation:ta-telescope}
\sum_{t=1}^{T} \Delta^t=-\log \sum_{M\in\cM} \mu^1(M) \exp\left( \sum_{t=1}^{T} \etap\log \frac{\PP^M(o^t|\xt)}{\PPs(o^t|\xt) }  + \etar\delta^t_{M}\right).
\end{align}
By realizability $\Ms\in\cM$, we have
$$
\sum_{t=1}^{T} \Delta^t\le -\log \mu^1(\Ms)=\log|\cM|.
$$
Plugging this bound into~\cref{equation:deltat-key-bound} gives the desired high-probability statement. The in-expectation statement follows similarly by further noticing that in~\cref{equation:deltat-key-bound}, taking the expectation $\E\brac{\sum_{t=1}^T\Delta^t}$ gives the same right-hand side, but without the additional $\log(1/\delta)$ term on the left-hand side.
\qed

\subsubsection{Proof of Lemma~\ref{lemma:vovk-delta-rew}}
\label{appendix:proof-vovk-delta-rew}

By definition,
\begin{align*}
    \delta=2\iprod{\br-\er}{\rr-\er}-\ltwot{\rr-\er},
\end{align*}
and therefore,
\begin{align*}
\EE\brac{ \exp\paren{\lambda\delta} }
=&
\exp\paren{- \lambda\ltwot{\rr-\er}}\EE\brac{ \exp\paren{2\lambda\iprod{\br-\er}{\rr-\er} } }\\
\leq&
\exp\paren{ 2\sigma^2\lambda^2\ltwot{\rr-\er} - \lambda\ltwot{\rr-\er} }\\
=&
\exp\paren{-\lambda(1-2\sigma^2\lambda)\ltwot{\er-\rr}},
\end{align*}
where the inequality is due to the definition of $\sigma^2$-sub-Gaussian random vector: For $\v=2\lambda(\rr-\er)\in\R^\dr$, 
$$
\EE\brac{\exp\paren{\iprod{v}{\br}}}\leq \exp\paren{\frac{\sigma^2\ltwot{\v}}{2}}.
$$
\qed

\subsection{General E2D \& Proof of Proposition~\ref{thm:E2D-TA}}
\label{appendix:general-e2d}

We first prove a guarantee for the following E2D meta-algorithm that allows any (randomized) online estimation subroutine, which includes~\cref{alg:E2D-TA} as a special case by instantiating $\Alg_{\Est}$ as the \Vovkalg~subroutine (for finite model classes) and thus proving~\cref{thm:E2D-TA}.

\begin{algorithm}[t]
	\caption{E2D Meta-Algorithm with Randomized Model Estimators} 
	\begin{algorithmic}[1]
	\label{alg:E2D-meta}
	\REQUIRE Parameter $\gamma>0$; Online estimation subroutine $\Alg_{\Est}$; Prior distribution $\mu^1\in\Delta(\cM)$.
	\FOR{$t=1,\ldots,T$}
    \STATE Set $p^t\setto \argmin_{p\in\Delta(\Pi)}\hV^{\mu^t}_{\gamma}(p)$, where $\hV^{\mu^t}_{\gamma}$ is defined in~\cref{equation:vtp}. 
    \STATE Sample $\pi^t\sim p^t$. Execute $\pi^t$ and observe $(o^t,\br^t)$. 
    \STATE Update randomized model estimator by online estimation subroutine:
    \begin{align*}
        \mu^{t+1} \setto \Alg_{\Est}^{t}\paren{\set{(\pi^s, o^s, \br^s)}_{s\in[t]}}.
    \end{align*}
    \ENDFOR
   \end{algorithmic}
\end{algorithm}

The following theorem is an instantiation of~\citet[Theorem 4.3]{foster2021statistical} by choosing the divergence function to be $\wtdH$. It is also an immediate corollary of our results of generalized DEC (see e.g. \cref{thm:E2D-gen} and also its proof in \cref{appendix:proof-gen-dec}). %
Let
\begin{align}
  \label{equation:esth}
    \EstwtH \defeq \sum_{t=1}^T \E_{\pi^t\sim p^t}\E_{\hat{M}^t\sim \mu^t}\brac{ \wtdH^2(M^\star(\pi^t), \hat{M}^t(\pi^t)) }
\end{align}
denote the online estimation error of $\set{\mu^t}_{t=1}^T$ in $\wtdH^2$ divergence (achieved by $\Alg_{\Est}$).
\begin{theorem}[E2D Meta-Algorithm~\citep{foster2021statistical}]
\label{thm:E2D-meta}
\cref{alg:E2D-meta} achieves
\begin{align*}
    \regdm \le T\cdot \odec_\gamma(\cM) + \gamma\cdot \EstwtH.
\end{align*}
\end{theorem}

We are now ready to prove the main theorem (finite $\cM$).

\begin{proofof}[thm:E2D-TA]
Note that~\cref{alg:E2D-TA} is an instantiation of~\cref{alg:E2D-meta} with $\Alg_{\Est}$ chosen as \Vovkalg. By \cref{corollary:restate-lemma-tempered-aggregation}, choosing $\etap=\etar=1/3$, the \Vovkalg~subroutine achieves
\begin{align*}
    \EstwtH \le 10\log(\abs{\cM}/\delta)
\end{align*}
with probability at least $1-\delta$. On this event, by~\cref{thm:E2D-meta} we have that
\begin{align*}
    \regdm \le T\cdot \odec_\gamma(\cM) + \gamma\cdot \EstwtH \le T\cdot \odec_\gamma(\cM) + 10\gamma\log(\abs{\cM}/\delta).
\end{align*}
This is the desired result.
\end{proofof}

\subsection{\Vovkalg~with covering}
\label{appendix:ta-covering}

In many scenarios, we have to work with an infinite model class $\cM$ instead of a finite one. In the following, we define a covering number suitable for divergence $\wtdH$, and provide the analysis of the \Vovkalg\ subroutine (as well as the corresponding \etod~algorithm) with such coverings. 

We consider the following definition of optimistic covering.
\begin{definition}[Optimistic covering]
\label{def:opt-cover}
Given $\rho\in[0,1]$, an optimistic $\rho$-cover of $\cM$ is a tuple $(\tPP,\cM_0)$, where $\cM_0$ is a finite subset of $\cM$, and each $M_0 \in\cM_0$ is assigned with an \emph{optimistic likelihood function} $\tPP^{M_0}$, such that the following holds: 
\begin{enumerate}%
\item[(1)] For $M_0\in\cM_0$, for each $\pi$, $\tPP^{M_0,\pi}(\cdot)$ specifies a un-normalized distribution over $\cO$, and it holds that $\nrm{\PP^{M_0,\pi}(\cdot)-\tPP^{M_0,\pi}(\cdot)}_1\leq\rho^2$.
\item[(2)] For any $M\in\cM$, there exists a $M_0\in\cM_0$ that \emph{covers} $M$: for all $\pi\in\Pi$, $o\in\cO$, it holds $\tPP^{M_0,\pi}(o)\geq \PP^{M,\pi}(o)$\footnote{
An important observation is that, along with (1), this requirement implies $\DTV{\PP^{M,\pi}(\cdot), \PP^{M_0,\pi}(\cdot)}\leq \rho^2$ (for proof, see e.g. \eqref{eqn:proof-mops-tv-opt}). Therefore, a $\rho$-optimistic covering implies a $\rho^2$-covering in TV distance.
}, and $\lone{\bR^M(o)-\bR^{M_0}(o)}\leq \rho$. 
\end{enumerate}
The optimistic covering number $\cN(\cM,\rho)$ is defined as the minimal cardinality of $\cM_0$ such that there exists $\tPP$ such that $(\tPP,\cM_0)$ is an optimistic $\rho$-cover of $\cM$. 
\end{definition}
Define
\begin{align}
\label{equation:est-cover}
    \est(\cM, K):=\inf_{\rho\geq 0}\paren{\log\cN(\cM,\rho)+K\rho} 
\end{align}
which measures the estimation complexity of $\cM$ for $K$-step interaction.
With the above definitions at hand, the \Vovkalg\ algorithm can be directly generalized to infinite model classes by performing the updates on an optimistic cover (\cref{alg:TA-infinite}).
\begin{proposition}[\Vovkalg~with covering for RL]
\label{prop:vovk-cover-demo}
For any model class $\cM$ and an associated optimistic $\rho$-cover $(\tPP,\cM_0)$, the \Vovkalg~subroutine
\begin{align}
\label{eqn:ta-cover-def}
    \mu^{t+1}(M) \; \propto_M \; \mu^{t}(M) \cdot \exp\paren{\etap \log \tPP^{M,\pi^t}(o^{t}) - \etar\ltwo{\br^t - \bR^M(o^t)}^2 }
\end{align}
with $\mu^1=\Unif(\cM_0)$ and $\etap=\etar=1/3$ achieves the following bound with probability at least $1-\delta$: %
\begin{align*}
    \EstwtH %
    \le 10 \cdot \brac{\log\abs{\cM_0}+2T\rho+2\log(2/\delta)}.
\end{align*}
In particular, the \Vovkalg~subroutine can be suitably instantiated such that with probability at least $1-\delta$,
\begin{align*}
    \EstwtH \leq 20 \est(\cM,T)+20\log(2/\delta).
\end{align*}
\end{proposition}

\paragraph{\etod~with covering}
\cref{prop:vovk-cover-demo} implies that, based on the model class $\cM$, we can suitably design the optimistic likelihood function $\tPP$ and the prior $\mu^1$, so that ~\cref{alg:E2D-meta} with $\Alg_{\Est}$ chosen as ~\cref{eqn:ta-cover-def} achieves $\EstwtH=\tbO{\est(\cM,T)}$. Therefore, by~\cref{thm:E2D-meta} we directly have the following guarantee.
\begin{theorem}[\etod~with covering]
\label{thm:E2D-TA-full}
\cref{alg:E2D-meta} with $\Alg_{\Est}$ chosen as \textsc{\Vovkalg~with covering}~\cref{eqn:ta-cover-def} and optimally chosen $\gamma$ achieves
\begin{align*}
    \regdm \le C\inf_{\gamma>0}\paren{ T\cdot \odec_\gamma(\cM) + \gamma\est(\cM,T) + \gamma\log(1/\delta) }
\end{align*}
with probability at least $1-\delta$, where $C$ is a universal constant.
\end{theorem}

\subsubsection{Discussions about optimistic covering}

We make a few remarks regarding our definition of the optimistic covering. 
Examples of optimistic covers on concrete model classes can be found in e.g. \cref{example:cover-tabular}, \cref{prop:cover-linear-mixture}; see also \citep[Appendix B]{liu2022partially}.

\paragraph{A more relaxed definition} 
We first remark that \cref{def:opt-cover}(2) can actually be relaxed to

\textit{(2') For any $M\in\cM$, there exists a $M_0\in\cM_0$, such that $\max_{o\in\cO}\lone{\bR^M(o)-\bR^{M_0}(o)}\leq \rho$, and  
\begin{align}\label{eqn:opt-cover-gen}
    \EE_{o\sim \PP^M(\cdot|\pi)}\brac{\frac{\PP^M(o|\pi)}{\tPP^{M_0}(o|\pi)}}\leq 1+\rho,\quad \forall \pi\in\Pi.\tag{$\dagger$}
\end{align}
}

For the simplicity of presentation, we state all the results in terms of \cref{def:opt-cover}. But the proof of \cref{thm:vovk-cover} can be directly adapted to \eqref{eqn:opt-cover-gen}; see \cref{remark:opt-cover-gen}. %

\paragraph{Relation to \citet[Definition 3.2]{foster2021statistical}} 
We comment on the relationship between our optimistic covering and the covering introduced in \citet[Definition 3.2]{foster2021statistical} (which is also used in their algorithms to handle infinite model classes). First, the covering in \citet{foster2021statistical} needs to cover \emph{the distribution of reward}, while ours only need to cover the mean reward function. More importantly, \citet[Lemma A.16]{foster2021statistical} explicitly introduces a factor $\log B$, where $B\geq \sup_{o\in\cO,\pi\in\Pi,M\in\cM}\frac{\PP^M(o|\pi)}{\nu(o|\pi)}$ with $\nu$ being certain base distribution. Actually, with such a $B$, we can show that
$$
\cN'(\cM, \rho)\leq \cN_{\TV}(\cM, \rho^2/4B),
$$
where $\cN_{\TV}$ is the covering number in the TV sense, and $\cN'$ is the optimistic covering number with respect to \eqref{eqn:opt-cover-gen}. %

\paragraph{Relation to other notions of covering numbers} 
Ignoring the reward component, our optimistic covering number is essentially equivalent to the \emph{bracketing number}. We further remark that optimistic covering can be slightly weaker than the covering in $\chi^2$-distance sense: given a $\rho^2$-covering $\cM_0$ in the latter sense, we can take $\tPP=(1+\rho^2)\PP$ to obtain a $\rho$-optimistic covering defined by \eqref{eqn:opt-cover-gen}.

\subsection{Proof of Proposition~\ref{prop:vovk-cover-demo}}
\label{appendix:proof-e2d-ta-cover}

We first restate the \textsc{\Vovkalg~with covering} subroutine~\cref{eqn:ta-cover-def} in the general setup of online model estimation in \cref{alg:TA-infinite}.

\begin{algorithm}[t]
\caption{\textsc{\Vovkalg~with covering}} 
	\begin{algorithmic}[1]
	\label{alg:TA-infinite}
	\REQUIRE Learning rate $\etap\in(0, \frac{1}{2}),\etar>0$, number of steps $T$, $\rho$-optimistic cover $(\tPP,\cM_0)$.
	\STATE Initialize $\mu^1\setto {\rm Unif}(\cM_0)$.
	\FOR{$t=1,\dots,T$}
	\STATE Receive $(\xt, o^t, \br^t)$.
    \STATE Update randomized model estimator:
    \begin{align*}
        \mu^{t+1}(M) \; \propto_{M} \; \mu^{t}(M) \cdot \exp\paren{\etap\log \tPP^M(o^t|\xt) - \etar\ltwot{ \br^t-\bR^{M}(o^t) } }.
    \end{align*}
    \ENDFOR
    \end{algorithmic}
\end{algorithm}

\begin{theorem}[\Vovkalg\ over covering]\label{thm:vovk-cover}
For any $\cM$ that is not necessarily finite, but otherwise under the same setting as~\cref{thm:vovk-finite},~\cref{alg:TA-infinite} with $2\etap+2\sigma^2\etar<1$ achieves with probability at least $1-\delta$ that
\begin{align*}
    \sum_{t=1}^T \E_{M\sim \mu^t}\brac{ \err_{M}^t } \le C \brac{\log\abs{\cM_0} + 2\log(2/\delta) + 2T\rho(\etar+\etap)},
\end{align*}
where $C$ is defined same as in~\cref{thm:vovk-finite}.%
\end{theorem}

Plugging~\cref{thm:vovk-cover} into the RL setting, picking $(\etap,\etar)$ and performing numerical calculations, we directly have the \cref{prop:vovk-cover-demo}. The proof follows the same arguments as \cref{corollary:restate-lemma-tempered-aggregation} and hence omitted.  Similarly, when $\etap=\eta\in(0,\frac12)$, $\etar=0$, the proof of \cref{thm:vovk-cover} implies that \eqref{eqn:ta-cover-def} with $\mu^1=\Unif(\cM_0)$ achieves the following bound with probability at least $1-\delta$:
\begin{align}
\label{equation:ta-obs-only}
    \sum_{t=1}^T \E_{\pi^t\sim p^t}\E_{\hat{M}^t\sim \mu^t}\brac{ \DH{ \Pm^{M^\star}(\pi^t), \Pm^{\hat{M}^t}(\pi^t) } } %
    \le \frac{1}{\eta}\cdot \brac{\log\abs{\cM_0}+2\eta T\rho+2\log(2/\delta)}.
\end{align}

\begin{proofof}[thm:vovk-cover]
The proof is similar to that of~\cref{thm:vovk-finite}. Consider the random variable
$$
\Delta^t:=-\log \EE_{M\sim \mu^t}\left[ \exp\left( \etap\log \frac{\tPP^M(o^t|\xt)}{\PPs(o^t|\xt) }+\etar\delta^t_{M}\right) \right],
$$
for all $t\in[T]$, where $\delta$ is defined in \eqref{eqn:vovk-proof-delta}. Then by \eqref{eqn:vovk-proof-expminusdelta} and \eqref{eqn:vovk-proof-reward}, we have
\begin{align}\label{eqn:proof-vovk-delta-to}
\begin{split}
\EEt{\exp\left(-\Delta^t\right)}
\leq &
2\etap\EE_{M\sim\mu^t}\EE_t\brac{ \sqrt{\frac{\tPP^M(o^t|\xt)}{\PPs(o^t|\xt)}} }\\
&+(1-2\etap)\paren{ \EE_{M\sim\mu^t}\EE_t\brac{\exp\paren{-c(1-2\sigma^2c)\ltwot{ \bR^{M}(o^t)-\Rs(o^t) }}} }
\\
\leq &
2\etap\EE_{M\sim\mu^t}\EE_t\brac{ \sqrt{\frac{\tPP^M(o^t|\xt)}{\PPs(o^t|\xt)}} }\\
&+(1-2\etap)\paren{ 1- c'\EE_{M\sim\mu^t}\EE_t\brac{\ltwot{ \bR^{M}(o^t)-\Rs(o^t) }} },
\end{split}
\end{align}
where $c'$ is the same as in~\cref{thm:vovk-finite}.
To bound the first term, we notice that for all $\pi\in\Pi$, and $o\sim\PPs(\cdot|\xp)$, we have
\begin{align}\label{eqn:vovk-proof-optlike}
\MoveEqLeft
\EE_{o\sim\PPs(\cdot|\xp)}\brac{\sqrt{\frac{\tPP^M(o|\xp)}{\PPs(o|\xp)}}}
=
\EE_{o\sim\PPs(\cdot|\xp)}\brac{\sqrt{\frac{\PP^M(o|\xp)}{\PPs(o|\xp)}}}+\EE_{o\sim\PPs(\cdot|\xp)}\brac{\frac{\sqrt{\tPP^M(o|\xp)}-\sqrt{\PP^M(o|\xp)}}{\sqrt{\PPs(o|\xp)}}}\notag\\
\leq&
1-\frac12\dH^2(\PP^M(\cdot|\xp),\PPs(\cdot|\xp))+\EE_{o\sim\PPs(\cdot|\xp)}\Bigg[\frac{\paren{\sqrt{\tPP^M(o|\xp)}-\sqrt{\PP^M(o|\xp)} }^2}{\PPs(o|\xp)} \Bigg]^{\frac12}\notag\\
\leq&
1-\frac12\dH^2(\PP^M(\cdot|\xp),\PPs(\cdot|\xp))+\EE_{o\sim\PPs(\cdot|\xp)}\Bigg[\frac{\abs{\tPP^M(o|\xp)-\PP^M(o|\xp) }}{\PPs(o|\xp)} \Bigg]^{\frac12}\notag\\
=&
1-\frac12\dH^2(\PP^M(\cdot|\xp),\PPs(\cdot|\xp))+\nrm{\tPP^M(\cdot|\xp)-\PP^M(\cdot|\xp)}_1^{\frac12}\notag\\
\leq&
1-\frac12\dH^2(\PP^M(\cdot|\xp),\PPs(\cdot|\xp))+\rho,
\end{align}
where the last inequality is due to the fact that $\nrm{\PP^M(\cdot|\xp)-\tPP^M(\cdot|\xp)}_1\leq\rho^2$. \eqref{eqn:vovk-proof-optlike} directly implies that
\begin{align}\label{eqn:vovk-proof-optlike-2}
    \EE_t\brac{ \sqrt{\frac{\tPP^M(o^t|\xt)}{\PPs(o^t|\xt)}} }\leq 1-\frac12\EE_t\brac{\dH^2(\PP^M(\cdot|\xp^t),\PPs(\cdot|\xp^t))}+\rho.
\end{align}
Therefore, by \cref{lemma:concen}, with probability at least $1-\delta/2$, it holds that
\begin{align*}
    \sum_{t=1}^{T} \Delta^t+\log(2/\delta)
    \geq& \sum_{t=1}^{T}-\log \EEt{\exp\left(-\Delta^t\right)}
    \geq \sum_{t=1}^{T} 1-\EEt{\exp\left(-\Delta^t\right)}\\
    \geq&
    \etap \brac{ \sum_{t=1}^{T}
    \EE_{M\sim \mu^t} \EE_t\brac{\dH^2(\PP^M(\cdot|\xt), \PPs(\cdot|\xt)) } -2 T\rho }\\
    &+ (1-2\etap)c'\sum_{t=1}^{T}
    \EE_{M\sim \mu^t} \EE_t\brac{\ltwot{ \bR^{M}(o^t)-\Rs(o^t) }}.
\end{align*}
In the following, we complete the proof by showing that with probability at least $1-\delta/2$,
\begin{align}\label{eqn:vovk-proof-sum-delta-demo}
    \sum_{t=1}^{T} \Delta^t\leq \log|\cM_0|+2T\etar\rho+\log(2/\delta).
\end{align}
By a telescoping argument same as~\cref{equation:ta-telescope}, we have
\begin{align}\label{eqn:vovk-proof-sum-delta}
    \sum_{t=1}^{T} \Delta^t=-\log \sum_{M\in\cM_0} \mu^1(M) \exp\left( \sum_{t=1}^{T} \etap\log \frac{\tPP^M(o^t|\xt)}{\PPs(o^t|\xt) } +\etar\delta^t_{M}\right).
\end{align}
By the definition of $\cM_0$ and the realizability $\Ms\in\cM$, there exists a $M\in\cM_0$ such that $\Ms$ is covered by $M$ (i.e. $\nrm{\bR^{M_0}(o)-\Rs(o)}_{\infty}\leq \rho$ and $\tPP^{M}(\cdot|\pi)\geq \PPs(\cdot|\pi)$ for all $\pi$). 
Then
\begin{align}\label{eqn:proof-vovk-sum-delta-le}
\begin{split}
    \EE\brac{\exp\paren{\sum_{t=1}^{T} \Delta^t}}
    \leq& \abs{\cM_0}\EE\brac{\exp\left(-\sum_{t=1}^{T} \etap\log \frac{\tPP^{M}(o^t|\xt)}{\PPs(o^t|\xt) } - \etar\delta_{M}^t\right)}.
\end{split}
\end{align}
Now%
\begin{align}\label{eqn:vovk-proof-excess}
\begin{split}
\MoveEqLeft
\EE\brac{\exp\left(-\sum_{t=1}^{T} \etap\log \frac{\tPP^{M}(o^t|\xt)}{\PPs(o^t|\xt) } - \etar\delta_{M}^t\right)}
=
\EE\brac{ \prod_{t=1}^{T} \paren{\frac{\PPs(o^t|\xt)}{\tPP^M(o^t|\xt) } }^{\etap} \cdot \exp(-\etar\delta_{M}^t) }\\
&\leq 
\EE\brac{ \prod_{t=1}^{T} \exp(-\etar\delta_{M}^t) }
=
\EE\brac{ \prod_{t=1}^{T-1} \exp(-\etar\delta_{M}^t) \cdot \cond{\exp(-\etar\delta_{M}^T)}{o^T} }\\
&\leq 
\exp\paren{2\rho\etar}\EE\brac{ \prod_{t=1}^{T-1} \exp(-\etar\delta_{M}^t) }\\
&\leq\cdots\leq %
\exp\paren{2T\rho\etar},
\end{split}
\end{align}
where the first inequality is due to $\tPP^{M}\geq \PPs$, the second inequality is because for all $t\in[T]$,
\begin{align*}
    \cond{\exp(-\etar\delta_{M}^t)}{o^t}
    \leq \exp\paren{\etar(1+2\sigma^2\etar)\ltwot{\bR^{M}(o^t)-\Rs(o^t)}}
    \leq \exp(2\rho\etar),
\end{align*}
which is due to \cref{lemma:vovk-delta-rew} and
\begin{align*}
    \ltwo{\bR^{M}(o^t)-\Rs(o^t)}^2 \leq \lone{\bR^{M}(o^t)-\Rs(o^t)} \linf{\bR^{M}(o^t)-\Rs(o^t)} \leq \rho.
\end{align*}
Applying Chernoff's bound completes the proof.
\end{proofof}

\begin{remark}\label{remark:opt-cover-gen}
From the proof above, it is clear that \cref{thm:vovk-cover} also holds for for the alternative definition of covering number in~\cref{eqn:opt-cover-gen}: Under that definition, we can proceed in~\cref{eqn:vovk-proof-excess} by using the fact $\EE_{o\sim \PPs(\cdot|\xp)}\brac{\frac{\PPs(o|\xp)}{\tPP^M(o|\xp)}}\leq 1+\rho$ and the fact $\cond{\exp(-\etar\delta_{M}^t)}{o^t}\leq \exp(2\rho\etar)$ alternately.
\end{remark}

\section{Proofs for Section~\ref{section:eec}}

In this section, we provide proofs for results in \cref{section:eec}, except \cref{thm:E2D-exp} and \cref{prop:eec-lower-bound-demo}, which are encompassed by the generic guarantees provided in \cref{sec:gen-dec} (\cref{thm:E2D-gen} and \cref{prop:gen-lower-bound}, see also \cref{appdx:specifying}). 

\subsection{Proof of Proposition~\ref{prop:eec-to-dec}}
\label{appendix:proof-eec-to-dec}
Fix a $\omu\in\Delta(\cM)$, and we take
\begin{align*}
 (\op_{\expl},\op_{\out})\defeq\argmin_{(\pexp,\pout)\in\Delta(\cM)^2}\sup_{M\in \cM} \Big\{ &\EE_{\pi \sim \pout}\left[f^M(\pi_M)-f^M(\pi)\right]\\
 &-\gamma \EE_{\pi \sim \pexp}\EE_{\oM \sim \omu}\left[ \wtdH^2(M(\pi),\oM(\pi))\right] \Big\}.
\end{align*}
Then consider $\op=\alpha \op_{\expl}+(1-\alpha)\op_{\out}$. By definition,
\begin{align*}
&\dec(\cM,\omu)\\
\leq& \sup_{M\in \cM}\EE_{\pi \sim \op}\left[f^M(\pi_M)-f^M(\pi)\right]-\gamma \EE_{\pi \sim \op}\EE_{\oM \sim \omu}\left[ \wtdH^2(M(\pi),\oM(\pi))\right]\\
=& \sup_{M\in \cM}\Big\{ \alpha\EE_{\pi \sim \op_{\expl}}\left[f^M(\pi_M)-f^M(\pi)\right]-\gamma\alpha \EE_{\pi \sim \op_{\expl}}\EE_{\oM \sim \omu}\left[ \wtdH^2(M(\pi),\oM(\pi))\right]\\
&\qquad\,\,(1-\alpha) \EE_{\pi \sim \op_{\out}}\left[f^M(\pi_M)-f^M(\pi)\right]-\gamma(1-\alpha )\EE_{\pi \sim \op_{\out}}\EE_{\oM \sim \omu}\left[ \wtdH^2(M(\pi),\oM(\pi))\right] \Big\}\\
\leq& \sup_{M\in \cM} \Big\{\alpha+(1-\alpha)\EE_{\pi \sim \op_{\out}}\left[f^M(\pi_M)-f^M(\pi)\right]
-\alpha\gamma \EE_{\pi \sim \op_{\expl}}\EE_{\oM \sim \omu}\left[ \wtdH^2(M(\pi),\oM(\pi))\right] \Big\}\\
=& \alpha+(1-\alpha)\eec_{\alpha\gamma/(1-\alpha)}(\cM,\omu).
\end{align*}
\qed

\subsection{Additional discussions on bounding Regret DEC by PACDEC}
\label{appendix:eec-vs-dec}

Here we argue that, for classes with low PACDEC, obtaining a PAC sample complexity through the implied DEC bound is in general worse than the bound obtained by the PACDEC bound directly.

Consider any model class $\cM$ with $\oeec_{\gamma}(\cM)\lesssim d/\gamma$, where $d$ is some dimension-like complexity measure. Using the \eetod~algorithm, by~\cref{thm:E2D-exp}, the suboptimality of the output policy scales as
\begin{align*}
  & \quad f^{M^\star}(\pi_{M^\star})-\EE_{\pi\sim \hatpout}\left[f^{M^\star}(\pi)\right]\leq \oeec_{\gamma}(\cM)+10\frac{\gamma \log(\abs{\cM}/\delta)}{T} \\
  & \lesssim \frac{d}{\gamma} + \frac{\gamma}{T}\cdot \log(\abs{\cM}/\delta) \lesssim \sqrt{\frac{d\log(\abs{\cM}/\delta)}{T}},
\end{align*}
where the last inequality follows by choosing the optimal $\gamma>0$. This implies a PAC sample complexity $d\log(\abs{\cM}/\delta)/\eps^2$ for finding an $\eps$ near-optimal policy.

By contrast, suppose we use an algorithm designed for low DEC problems (such as \etod). To first bound the DEC by the PACDEC, by~\cref{prop:eec-to-dec}, we have
\begin{align*}
   \quad \odec_{\gamma}(\cM)\leq&~ \inf_{\alpha\in(0,1)}\set{\alpha+(1-\alpha)\oeec_{\gamma\alpha/(1-\alpha)}(\cM)} \\
   \le&~ \inf_{\alpha\in(0,1)}\set{\alpha+(1-\alpha)^2 \frac{d}{\gamma\alpha}} \lesssim \sqrt{\frac{d}{\gamma}}+\frac{d}{\gamma}.
\end{align*}
Then, using the \etod~algorithm, by~\cref{thm:E2D-TA} and the online-to-batch conversion, the suboptimality of the average policy scales as
\begin{align*}
  \frac{\regdm}{T}& = \frac{1}{T}\sum_{t=1}^T f^{M^\star}(\pi_{M^\star})-\EE_{\pi^t\sim p^t}\left[f^{M^\star}(\pi^t)\right]  \\
  & \le \odec_\gamma(\cM) + 10\frac{\gamma\log(\abs{\cM}/\delta)}{T} \\
  & \lesssim \sqrt{\frac{d}{\gamma}} + \frac{\gamma}{T}\cdot \log(\abs{\cM}/\delta) \lesssim \paren{\frac{d\log(\abs{\cM}/\delta)}{T}}^{1/3},
\end{align*}
where the last inequality follows by choosing the optimal $\gamma\geq d$. This implies a PAC sample complexity $d\log(\abs{\cM}/\delta)/\eps^3$ for finding an $\eps$ near-optimal policy, which is an $1/\eps$ factor worse than that obtained from the PACDEC directly. Note that this $1/\eps^3$ rate is the same as obtained from the standard explore-then-commit conversion from PAC algorithms with sample complexity $1/\eps^2$ to no-regret algorithms.

We remark that the same calculations above also hold in general for problems with $\oeec_\gamma(\cM) \lesssim 1/\gamma^\beta$ (when only highlighting dependence on $\gamma$) for some $\beta>0$. In that case, the PACDEC yields PAC sample complexity $(1/\eps)^{\frac{\beta+1}{\beta}}$, whereas the implied DEC bound only yields a slightly worse $(1/\eps)^{\frac{\beta+2}{\beta}}$ sample complexity.

\subsection{Proof of Proposition~\ref{prop:rev-bandit}}\label{appdx:proof-rev-bandit}

We first present the full statement of \cref{prop:rev-bandit} as follows.
\begin{proposition}\label{prop:rev-bandit-full}
    For any integer $n\geq 2$, there exists $\cM$ a class of ``bandits with revealing actions'' such that for all $\gamma>0$, we have
    \begin{align*}
        \eec_{\gamma}(\cM)\leqsim \frac{n}{\gamma}, \qquad
        \dec_{\gamma}(\cM)\geqsim \min\set{\sqrt{\frac{n/\log\gamma+1}{\gamma}}, \frac{2^{n}}{\gamma}, 1}.
    \end{align*}
    Furthermore, the localized version (cf. Eq. \cref{equation:localized-model}) of the above lower bound of $\dec_{\gamma}(\cM)$ also holds true, and hence implies $\EE[\regdm]\geq\tbOm{\min\set{n^{1/3}T^{2/3},\sqrt{2^nT}, T}}$.
\end{proposition}

We now present the proof of \cref{prop:rev-bandit-full} (and \cref{prop:rev-bandit}). Consider $\cM^{\Delta}$ the class of ``bandits with revealing actions'' described as follows.

Denote $\cA=\Arew\bigsqcup\Arev$, $\Arew=\{0,1\}^n$, $\Arev=\set{\acrev{1},\cdots,\acrev{n}}$. For each $a\in\Arew$, we write $a[i]$ to be the $i$-th coordinate of $a$ (as a vector in $\{0,1\}^n\subset \R^n$). Let $\Pi=\cA$.

For each $a\in\Arew$, $\Delta>0$, $M=M_{(a,\Delta)}$ is defined as: 
\begin{enumerate}
    \item If $\pi\in\cA_0$, then $r(\pi)\sim\Bern\paren{\frac12+\Delta\cdot\II(\pi=a)}$, $o\sim \Bern\paren{\frac12}$.
    \item For $\pi=\acrev{i}$, then $r(\pi)=0$, $o\sim \Bern\paren{\frac12+\Delta\cdot a[i] }$.
\end{enumerate}
Notice that for $M=M_{(a,\Delta)}$, we have $\pi_M=a$.

We define $\cM^\Delta\defeq \set{M_{(\Delta,a)}}_{a\in\Arew}$.
Furthermore, let $\oM$ be the model with $o\sim\Bern\paren{\frac12}$, and $r(\pi)\sim\Bern\paren{\frac12}$ for $\pi\in\Arew$, $r(\pi)=0$ for $\pi\in\Arev$. 
Finally, we define the model class  $\cM=\bigcup_{\Delta\in[0,\frac{1}{3}]}\cM^\Delta$ (where we understand $\cM^0=\{\oM\}$), with policy class $\Pi=\cA$.

We next lower bound $\dec_{\gamma}(\cM)\geq \dec_{\gamma}(\cM^\Delta,\oM)$. 
For $M\in\cM^\Delta$, $\pi\in\Arew$, we have
\begin{align}\label{eqn:eec-vs-dec-comp}
    f^M(\pi_M)-f^M(\pi)=\Delta\cdot\II(\pi=a_M), \qquad
    \tDH{ M(\pi), \oM(\pi) }=\Delta^2\cdot\II(\pi=\pi_M).
\end{align}
For $\pi=\acrev{i}$, we have
\begin{align*}
    &f^M(\pi_M)-f^M(\pi)=\frac12+\Delta, \\
    &\tDH{ M(\pi), \oM(\pi) }=\DH{\Bern\paren{\frac12+\Delta\cdot a_M[i]}, \Bern\paren{\frac12}}\leq 3\Delta^2\cdot \pi_M[i].
\end{align*}
Note that by duality we have
\begin{align*}
    \dec_{\gamma}(\cM^\Delta,\oM)
    =&~ \sup_{\mu\in\Delta(\cM)}\inf_{\pi\in\Pi} \EE_{M\sim\mu}\brac{ f^M(\pi_M)-f^M(\pi)-\gamma \tDH{ M(\pi), \oM(\pi) } } \\
    \geq&~
    \sup_{\mu} \set{ \min_{M}\set{\Delta-(\Delta+\gamma\Delta^2)\mu(M)} \wedge \min_i \set{ \frac12+\Delta-3\gamma\Delta^2\cdot \mu(M:\pi_M[i]=1) } }.
\end{align*}
Fix a $p\in(0,\frac{1}{2}]$ and $\Delta>0$ (to be specified later), we consider $\mu\in\Delta(\cM)$ be given by $M=M_{(\Delta,a)}, a\sim\Bern(p)^{\otimes n}$. Then we have $\max_M\mu(M)=(1-p)^n$, $\max_i \mu(M:\pi_M[i]=1)=p$.
Choosing $\Delta=\frac{1}{4}\min\set{\frac{1}{\sqrt{p\gamma}}, \frac{1}{(1-p)^n\gamma}, 1}$, we then have $\dec_{\gamma}(\cM^\Delta,\oM)\geq \frac{\Delta}{2}$. Choosing $p=\min\set{\frac{\log\gamma}{n},\frac{1}{2}}$, it holds
\begin{align*}
    \dec_{\gamma}(\cM)\geq \dec_{\gamma}(\cM^\Delta,\oM)\geq\frac{1}{8}\min\set{\sqrt{\frac{\max\set{n/\log\gamma,2}}{\gamma}}, \frac{2^n}{\gamma}, 1 }.
\end{align*}
This gives the desired lower bound of $\dec_{\gamma}(\cM)$.

Note that under our construction, $\cM^\Delta$ is a localized model class (in the sense of \cref{equation:localized-model}), and hence our lower bound on $\dec_{\gamma}(\cM^\Delta,\oM)$ indeed implies the desired lower bound of regret by \cite[Theorem 3.2]{foster2021statistical} (which is exactly the instantiation of \cref{prop:gen-lower-bound} to no-regret RL).

We next upper bound $\eec_\gamma(\cM)$ and $\dec_\gamma(\cM)$. We only need to show that $\cM$ admits a ``trivial'' \belrep~$\cG$ with complexity 1 (cf. \cref{definition:err-rep} and \cref{prop:belrep-pac}). For each $M\in\cM$, $M$ is parameterized by a tuple $(\Delta_M,a_M)$, where we understand $(\Delta_{\oM},a_{\oM})=(0,\mathbf0)$. Let us consider the \belrep~$\cG$ given by (here $H=1$)
\begin{align*}
    \cE^{M;M'}=f^M(\pi_M)-f^{M'}(\pi_M)=\Delta_M-\Delta_{M'}\cdot\II(a_M=a_{M'})
\end{align*}
with $\cT$ being a singleton.
Clearly, for $\piest=\Unif(\Arev)$, we have
\begin{align*}
    \abs{\cE^{M;M'}}^2 
    \leq&~ \max_{i\in[n]}\paren{ \Delta_M\cdot a_M[i]-\Delta_{M'}\cdot a_{M'}[i]}^2 \\
    \leq&~  \max_{i\in[n]} \DH{ M(\acrev{i}), M'(\acrev{i}) }\\
    \leq&~ n \tDH{ M(\piest), M'(\piest) }.
\end{align*}
Therefore, $\cG$ is indeed a \belrep~with $\dimc(\cG,\gamma)\leq 1\forall\gamma$ and $L=n$, and hence by \cref{prop:belrep-pac},
\begin{align*}
    \eec_{\gamma}(\cM)\leqsim \frac{n}{\gamma}.
\end{align*}
Thus, by \cref{prop:eec-to-dec}, we have $\dec_\gamma(\cM)\leqsim \sqrt{\frac{n}{\gamma}}$. Further, it is clear that $\dec_\gamma(\cM)\leqsim \frac{2^n}{\gamma}$, which can be deduced from the DEC's upper bound of $2^n$-arm bandits. This completes the proof of upper bound.
\qed

\section{Proofs for Section~\ref{sec:gen-dec}}\label{appdx:proof-gen-dec-all}

\subsection{Proof of Theorem~\ref{thm:E2D-gen}}
\label{appendix:proof-gen-dec}
  
By \cref{corollary:restate-lemma-tempered-aggregation}, %
we have the following bound on online estimation error (cf.~\cref{equation:esth}):
\begin{align*}
  \EstRL \defeq \sum_{t=1}^T \E_{\pi^t\sim \pexp^t}\E_{\hat{M}^t\sim \mu^t}\brac{ \wtdH\paren{M^\star(\pi^t), \hat{M}^t(\pi^t)} } \leq 10\log(\abs{\cM}/\delta).
\end{align*}
Then, we can bound
\begin{align*}
  & \quad \sum_{t=1}^T \subopts_{\Ms}(\pout^t) \\
    & = \sum_{t=1}^T \subopts_{\Ms}(\pout^t) -\gamma \EE_{\pi \sim \pexp^t}\EE_{\hM^t\sim\mu^t}\brac{ \wtdH^2\paren{M(\pi),\hM^t(\pi)}} \\
    & \qquad + \gamma\cdot \sum_{t=1}^T \E_{\pi^t\sim \pexp^t}\E_{\hat{M}^t\sim \mu^t}\brac{ \wtdH^2\paren{M^\star(\pi^t), \hat{M}^t(\pi^t)} } \\
    & \stackrel{(i)}{\le} \sum_{t=1}^T \sup_{M\in\cM} \subopts_{M}(\pout^t) -\gamma \EE_{\pi \sim \pexp^t}\EE_{\hM^t\sim\mu^t}\brac{ \wtdH^2\paren{M(\pi),\hM^t(\pi)}}\\
    & \qquad + \gamma\cdot \sum_{t=1}^T \E_{\pi^t\sim \pexp^t}\E_{\hat{M}^t\sim \mu^t}\brac{ \wtdH^2\paren{M^\star(\pi^t), \hat{M}^t(\pi^t)} } \\
    & \stackrel{(ii)}{=} \sum_{t=1}^T \underbrace{\hV^{\mu^t}_{*,\gamma}(\pexp^t,\pout^t)}_{=\inf_{(\pexp,\pout)\in\csps} \hV^{\mu^t}_{\gamma}} + \gamma\cdot \EstRL \\
    & \stackrel{(iii)}{=} \sum_{t=1}^T \DEC_\gamma(\cM, \mu^t) + \gamma\cdot \Est \le T\cdot \DEC_\gamma(\cM) + \gamma\cdot \EstRL.
\end{align*}
Above, (i) follows by the realizability assumption $M^\star\in\cM$; (ii) follows by definition of the risk $\hV^{\mu^t}_{\gamma}$ (cf.~\cref{equation:vtp-gen}) as well as the fact that $(\pexp^t,\pout^t)$ minimizes $\hV^{\mu^t}_{*,\gamma}(\cdot,\cdot)$ in~\cref{alg:E2D-gen}; (iii) follows by definition of $\DEC_\gamma(\cM, \mu^t)$. 

Therefore, by the convexity of $\subopt$, dividing both sides of the above inequality by $T$ gives
\begin{align*}
    \subopts_{\Ms}(\hatpout)\leq \frac{1}{T}\sum_{t=1}^T \subopts_{\Ms}(\pout^t)
    \leq \DEC_{\gamma}(\cM)+\frac{\gamma\EstRL}{T}.
\end{align*}
Furthermore, for the special case $\csps=\cspreg$, we directly have for $p^t=\pexp^t=\pout^t$,
\begin{align*}
    \Regs=\sum_{t=1}^T \subopts_{\Ms}(p^t)
    \leq T\cdot \DEC_{\gamma}(\cM)+\gamma\EstRL.
\end{align*}
Combining the inequalities above completes the proof.
\qed

From the proof above, we can directly generalize \cref{thm:E2D-gen} to any model class $\cM$ with finite covering number (\cref{def:opt-cover}), as follows.
\begin{theorem}\label{thm:E2D-gen-full}
Given a suitable $\rho$-optimistic cover $(\tPP,\cM_0)$ of $\cM$, we can replace the subroutine \eqref{eqn:E2D-gen-TA} in~\cref{alg:E2D-gen} with subroutine \eqref{eqn:ta-cover-def}:
\begin{align*}
    \mu^{t+1}(M) \; \propto_M \; \mu^{t}(M) \cdot \exp\paren{\etap \log \tPP^{M,\pi^t}(o^{t}) - \etar\ltwo{\br^t - \bR^M(o^t)}^2 }.
\end{align*}
Then choosing $\etap$ and $\etar$ as in \cref{corollary:restate-lemma-tempered-aggregation}, \cref{alg:E2D-gen} achieves the following guarantee with probability at least $1-\delta$:
\begin{align*}
   \subopts\leq \DEC_{\gamma}(\cM)+20\gamma\cdot \frac{\est(\cM,T)+ \log(1/\delta)}{T}.
\end{align*}
In particular, for $\csps=\cspreg$, \cref{alg:E2D-gen} also achieves the following guarantee with probability at least $1-\delta$:
\begin{align*}
    \Regs\leq T\cdot \DEC_{\gamma}(\cM)+20\gamma \brac{\est(\cM,T)+ \log(1/\delta)}.
\end{align*}
\end{theorem}

As a remark, we also note that the constant 20 in the above theorem can be replaced by constant 4 in the case $\br$ being empty (e.g. reward-free learning and preference-based learning), using \cref{corollary:restate-lemma-tempered-aggregation}.

\subsection{Proof of Theorem~\ref{prop:gen-lower-bound-demo}}\label{appendix:proof-gen-lower-bound}

In this section, we prove \cref{prop:gen-lower-bound-demo}, which is a generalization of \citet[Theorem 3.2]{foster2021statistical}. Before presenting the proof, we first specify how an algorithm is described, and then we specify the reward distribution for any $M\in\cM$ (note that for $M\in\cM$, only the mean reward function is specified).

\paragraph{Algorithm}
Without loss of generality, we suppose that the algorithm $\Algo$ is given by rules $(p^{(1)}_{\expl}, \cdots, p^{(T)}_{\expl}, p_{\out})$, where for each $t\in[T]$ and any $\cH^{(t-1)}=(\pi^1,o^1,\br^1,\cdots,\pi^{t-1},o^{t-1},\br^{t-1})$ the history up to the $t$-th step, $p_{\expl}^{(t)}\paren{\cH^{(t-1)}}\in \Delta(\Pi)$ specifies the distribution of the policy $\pi^t$ that $\Algo$ executes at $t$-th step (given $\cH^{(t-1)}$), and $p_{\out}\paren{\cH^{(T)}}\in\Delta(\ssp)$ specifies the distribution of the output strategy $\hatpout\in\ssp$ based on the full history $\cH^{(T)}$. For any model\footnote{A ``model'' $M$ here is a map from the policy class $\Pi$ to $\Delta(\cO\times \R^H)$, and it does not necessarily belong to $\cM$ (see also the discussion in \cref{appendix:comparison-dec} and \citet{foster2021statistical}).} $M$, we write $\PP^{M,\Algo}$ the probability distribution of $\cH^{(T)}$ induced by $\Algo$ when interacting with $M$, and $\EE^{M,\Algo}$ is the expectation under $\PP^{M,\Algo}$.
For any model $M$, we define
$$
p_{M,\expl}=\EE^{M,\Algo}\left[\frac{1}{T} \sum_{t=1}^{T} p^{(t)}_{\expl}\left(\cH^{(t-1)}\right)\right] \in \Delta(\Pi), \qquad 
p_{M,\out}=\EE^{M,\Algo}\left[p_{\out}(\cH^{T})\right] \in \ssp,
$$
where we identify $\Delta(\ssp)$, the space of probability distribution over $\ssp$, with $\ssp$ itself by convexity.
With the above definition, we may invoke the following chain rule of Hellinger distance \citep[Lemma A.13]{foster2021statistical} (a more detailed derivation can be found in \citet[Appendix C.1.1]{foster2021statistical}).
\begin{proposition}\label{prop:Hellinger-chain}
For any two model $M, \oM$, it holds that
\begin{align*}
    \DH{\PP^{M,\Algo}, \PP^{\oM,\Algo}} \leq 100\log(2T)\cdot \EE^{\oM,\Algo}\brac{ \sum_{t=1}^T \DH{ M(\pi^t), \oM(\pi^t) } }.
\end{align*}
\end{proposition}
Therefore, by the definition of $p_{M,\expl}$, we have
\begin{align}\label{eqn:Hellinger-chain}
    \DH{\PP^{M,\Algo}, \PP^{\oM,\Algo}} \leq 100\log(2T)T\cdot  \EE_{\pi\sim p_{\oM,\expl}}\brac{ \DH{ M(\pi), \oM(\pi) } }.
\end{align}

\paragraph{Reward distribution}
We show that, for the given model class $\cM$, we can suitably assign each model $M=(\MP,\bR^M)$ with a reward distribution, so that $\wtdH$ is equivalent to $\dH$ (see also \cref{prop:modified-to-true}). Therefore, in the remaining part of the proof of \cref{prop:gen-lower-bound}, we assume that \cref{eqn:Hellinger-to-DRL} holds for the model class $\cM$, and we only need to deal with the Hellinger distance $\dH$. 
\begin{lemma}\label{lem:reward-dist}
    For a model class $\cM$, we can assign each model $M\in\cM$ with a reward distribution $\RM$, such that for each $o\in\cO$, $\br\sim\RM(o)$ is 1-sub-Gaussian with mean $\EE[\br|o]=\bR^M(o)$, and
    \begin{align}\label{eqn:Hellinger-to-DRL}
        \DH{M(\pi),\oM(\pi)}\leq 3\tDH{M(\pi),\oM(\pi)}, \qquad \forall M,\oM\in\cM.
    \end{align}
\end{lemma}
The proof of \cref{lem:reward-dist} is deferred to the end of this section.

With the above preparation, we now prove \cref{prop:gen-lower-bound-demo}.
\begin{proof}[Proof of \cref{prop:gen-lower-bound-demo} (1)]
\renewcommand{\tM}{\widetilde{M}}
Let us first prove the case $\csps=\csppac$. We abbreviate $\gamma=\gammaT$ and 
\begin{align*}
    \delta:=\DEC_{\gamma}(\cM)=\sup_{\omu\in\Delta(\cM)}\DEC_{\gamma}(\cM,\omu).
\end{align*}
Let $\omu \in \Delta(\cM)$ attains this supremum. To proceed, we follow \citet{foster2021statistical} and consider an \emph{improper} reference model $\tM=\EE_{\oM\sim \omu} \oM$ induced by $\omu$\footnote{More formally, we define $\tM\in (\Pi\to\Delta(\cO\times\R^H))$ as $\tM(\pi)=\EE_{\oM\sim \omu} \brac{\oM(\pi)}$. For more details, see discussion in \cref{appendix:comparison-dec}.}. Then by definition,
$$
\sup _{M \in \cM} \subopts_M(p_{\tM,\out}) - \gamma \EE_{\pi \sim p_{\tM,\expl},\oM\sim\omu}\left[ \wtdH^2\paren{M(\pi), \oM(\pi)}\right] \geq \delta.
$$
Let $M \in \cM$ attains the supremum above. Then by the convexity of $\subopts$ and \cref{lem:reward-dist}, we have
\begin{align*}
\EE^{\tM,\Algo}\brac{\subopts_M(\hatpout)}
\geq&~ \subopts_M(p_{\tM,\out}) \\
\geq&~ \frac{\gamma}3 \cdot \EE_{\pi \sim p_{\tM,\expl},\oM\sim\omu}\left[\DH{M(\pi), \oM(\pi)}\right]+\delta.
\end{align*}
By \cref{lemma:multiplicative-hellinger}, it holds that
$$
\begin{aligned}
&\left|\EE^{M,\Algo}\left[\subopts_M(\hatpout)\right]-\EE^{\tM,\Algo}\left[\subopts_M(\hatpout)\right]\right| \\
&\leq \sqrt{2\left(\EE^{M,\Algo}\left[\subopts_M(\hatpout)\right]+\EE^{\tM,\Algo}\left[\subopts_M(\hatpout)\right]\right) \cdot \DH{\PP^{M,\Algo}, \PP^{\tM,\Algo}}} \\
&\leq 2\DH{\PP^{M,\Algo}, \PP^{\tM,\Algo}}+\frac{1}{2}\left(\EE^{M,\Algo}\left[\subopts_M(\hatpout)\right]+\EE^{\tM,\Algo}\left[\subopts_M(\hatpout)\right]\right),
\end{aligned}
$$
which implies
\begin{align*}
    \EE^{M,\Algo}\left[\subopts_M(\hatpout)\right]
    \geq &\frac{1}{3}\EE^{\tM,\Algo}\left[\subopts_M(\hatpout)\right]-\frac{4}{3}\DH{\PP^{M,\Algo}, \PP^{\tM,\Algo}}.
\end{align*}
Using \cref{eqn:Hellinger-chain}, we have 
\begin{align*}
    \DH{\PP^{M,\Algo}, \PP^{\tM,\Algo}}
    \leq&~
    100\log(2T)T\cdot  \EE_{\pi\sim p_{\oM,\expl}}\brac{ \DH{ M(\pi), \tM(\pi) } } \\
    \leq&~ 100\log(2T)T\cdot  \EE_{\pi\sim p_{\oM,\expl}, \oM\sim \omu}\brac{ \DH{ M(\pi), \oM(\pi) } },
\end{align*}
where we use the convexity of squared Hellinger distance:
\begin{align*}
     \DH{M(\pi), \tM(\pi)}
    = \DH{M(\pi), \EE_{\oM\sim\omu}\brac{\oM(\pi)}}
    \leq \EE_{\oM\sim\omu}\brac{ \DH{M(\pi),\oM(\pi)} }.
\end{align*}
Hence, using the fact that $\gamma=\gammaT\geq 400\log(2T)T$ (by choosing $c_0=400$), we can conclude that
$$
\EE^{M,\Algo}\left[\subopts_M(\hatpout)\right]\geq \frac{1}{3} \delta=\frac{1}{3}\cdot\DEC_{\gammaT}(\cM).
$$
This completes the proof.
\end{proof}

\begin{proof}[Proof of \cref{prop:gen-lower-bound-demo} (2)]
For the case where $\csps=\cspreg$  (i.e. $\ssps=\Delta(\Pi)$), the proof is essentially analogous. We can consider the algorithm $\Algo$ with the same interaction rules $\pexp$ and (modified) output rule $\hatpout=\frac1T \sum_{t=1}^T \pexp^t$. Under such an algorithm, we know $p_{M,\expl}=p_{M,\out}$ for any model $M$. Therefore, the argument above yields that there exists $M^\star\in\cM$ such that
\begin{align*}
    \EE^{M^\star,\Algo}\left[\subopts_{M^\star}(\hatpout)\right]\geq\frac{1}{3}\cdot\DEC_{\gammaT}(\cM).
\end{align*}
Then by the convexity of $\subopt$, it holds that
\begin{align}\label{eqn:reg-lower-bound-strong}
\begin{split}
    \EE^{M^\star,\Algo}\brac{\Regs}
    =&~ \EE^{M^\star,\Algo}\brac{ \sum_{t=1}^T \subopts_{M^\star}(\pexp^t) } \\
    \geq&~ T\cdot \EE^{M^\star,\Algo}\brac{\subopts_{M^\star}\paren{ \frac1T\sum_{t=1}^T\pexp^t } } \\
    \geq&~ \frac{T}{3}\cdot\DEC_{\gammaT}(\cM).
\end{split}
\end{align}
This gives the desired result.
\end{proof}

\begin{proof}[Proof of \cref{lem:reward-dist}]
For any model $M\in\cM$, we specify the reward distribution of $M$ as follows: given its mean reward function $\bR^M:\cO\to[0,1]^H$, conditional on the observation $o$, we let $\br=(r^h)$ be a random vector, with each entry $r^h$ independently sampled from%
$$
\RM_h\left(r^h=-\frac12\Big|o\right)=\frac34-\frac{R^M_h(o)}{2},\qquad
\RM_h\left(r^h=\frac32\Big|o\right)=\frac14+\frac{R^M_h(o)}{2}.
$$
Then a simple calculation gives $\DH{\RM_h(o),\oMR_h(o)}\leq \frac12\abs{R^M_h(o)-R^{\oM}_h(o)}^2$. Therefore, by the fact that $(r^h)$ are mutually independent conditional on $o$, we have
\begin{align*}
    1-\frac{1}{2}\DH{\RM(o),\oMR(o)}
    =&~\prod_{h}\paren{1-\frac12\DH{\RM_h(o),\oMR_h(o)}}\\
    \geq&~ \prod_{h}\paren{1-\frac14\abs{R^M_h(o)-R^{\oM}_h(o)}^2}\\
    \geq&~ 1 - \sum_h \frac14\abs{R^M_h(o)-R^{\oM}_h(o)}^2 \\
    =&~ 1-\frac14\ltwot{\bR^M(o)-\bR^{\oM}(o)},
\end{align*}
Therefore, by \cref{lemma:Hellinger-cond},
\begin{align*}
&\DH{M(\pi),\oM(\pi)}\leq
3\DH{\MP(\pi),\oMP(\pi)}+2\EE_{o\sim \MP(\pi)}\brac{\DH{\RM(o),\oMR(o)}}\\
&\qquad\leq 
3\DH{\MP(\pi),\oMP(\pi)}+\EE_{o\sim \MP(\pi)}\brac{\ltwot{\bR^M(o)-\bR^{\oM}(o)}}
\leq 3\tDH{M(\pi),\oM(\pi)}.
\end{align*}
\end{proof}

\subsection{Specifications for concrete learning goals}\label{appdx:specifying}

In this section, we briefly discuss how our general framework encompasses the specific learning goals presented in
\cref{tab:goals}.
Recall that in \cref{def:gen-dec}, the \DECt~is defined as
\begin{align*}
    \DEC_{\gamma}(\cM,\omu) \defeq & \inf_{(\pexp,\pout)\in\csps}\sup_{M\in \cM}\subopts_{M}(\pout)-\gamma \EE_{\pi \sim \pexp,\oM \sim \omu}\brac{ \wtdH^2\paren{M(\pi),\oM(\pi)}}.
\end{align*}

\paragraph{No-regret learning and PAC RL} For these two learning goals, we consider
\begin{align}
    \subopt_{M}(p)\defeq f^M(\pi_M)-\EE_{\pi\sim p} f^M(\pi).
\end{align}
Then for no-regret learning (\cref{sec:prelim}), DEC defined in \cref{definition:dec} indeed follows from \cref{def:gen-dec}:
\begin{align*}
    \dec_{\gamma}(\cM,\omu) = & 
    \inf_{p\in\Delta(\Pi)}\sup_{M\in \cM}\subopt_{M}(p)-\gamma \EE_{\pi \sim p,\oM \sim \omu}\brac{ \wtdH^2\paren{M(\pi),\oM(\pi)}} \\
    =& \inf_{(\pexp,\pout)\in\cspreg}\sup_{M\in \cM}\subopt_{M}(\pout)-\gamma \EE_{\pi \sim \pexp,\oM \sim \omu}\brac{ \wtdH^2\paren{M(\pi),\oM(\pi)}},
\end{align*}
Similarly, for PAC learning ($\GOAL=\mathrm{pac}$, \cref{section:eec}), PACDEC defined in \cref{def:edec} also follows from \cref{def:gen-dec}:
\begin{align*}
    \eec_{\gamma}(\cM,\omu) = & \inf_{(\pexp,\pout)\in\csppac}\sup_{M\in \cM}\subopt_{M}(\pout)-\gamma \EE_{\pi \sim \pexp,\oM \sim \omu}\brac{ \wtdH^2\paren{M(\pi),\oM(\pi)}}.
\end{align*}

\paragraph{Reward-free learning} For reward-free learning ($\GOAL=\mathrm{rf}$, \cref{section:rfec}), the strategy space of interest is $\ssprf=\set{ p: \cR\to\Pi}$. The sub-optimality of any strategy $p\in\ssprf$ is then defined as
\begin{align*}
    \suboptrf_{\Pm}(p)\defeq \sup_{R\in\cR} \set{ f^{\Pm,R}(\pi_{\Pm,R})-\EE_{\pi\sim p(R)}\brac{f^{\Pm,R}(\pi)} }.
\end{align*}
To see why RFDEC (\cref{def:rfdec}) is a specification of \DECt~with $\GOAL=\mathrm{rf}$, we have
\begin{align*}
    &~ \inf_{\pout\in\ssprf}\sup_{\Pm\in \cP}\subopts_{\Pm}(\pout)-\gamma \EE_{\pi \sim \pexp,\oPm \sim \omu}\brac{ \DH{\Pm(\pi),\oPm(\pi)}} \\
    =&~
    \inf_{\pout:\cR\to\Delta(\Pi)}\sup_{\Pm\in \cP}\sup_{R\in\cR} \set{ f^{\Pm,R}(\pi_{\Pm,R})-\EE_{\pi\sim \pout(R)}\brac{f^{\Pm,R}(\pi)} }-\gamma \EE_{\pi \sim \pexp,\oPm \sim \omu}\brac{ \DH{\Pm(\pi),\oPm(\pi)}} \\
    =&~
    \inf_{\pout:\cR\to\Delta(\Pi)}\sup_{R\in\cR}\sup_{\Pm\in \cP} \set{ f^{\Pm,R}(\pi_{\Pm,R})-\EE_{\pi\sim \pout(R)}\brac{f^{\Pm,R}(\pi)} }-\gamma \EE_{\pi \sim \pexp,\oPm \sim \omu}\brac{ \DH{\Pm(\pi),\oPm(\pi)}} \\
    =&~
    \sup_{R\in\cR} \inf_{\pout\in\Delta(\Pi)}\sup_{\Pm\in \cP}\set{ f^{\Pm,R}(\pi_{\Pm,R})-\EE_{\pi\sim \pout}\brac{f^{\Pm,R}(\pi)} }-\gamma \EE_{\pi \sim \pexp,\oPm \sim \omu}\brac{ \DH{\Pm(\pi),\oPm(\pi)}},
\end{align*}
Hence, we can conclude that for RFDEC defined in \cref{def:rfdec}, 
\begin{align*}
    \rfecg(\cP,\omu)
    =\inf_{(\pexp,\pout)\in\Delta(\Pi)\times \ssprf}\sup_{\Pm\in \cP}\subopts_{\Pm}(\pout)-\gamma \EE_{\pi \sim \pexp,\oPm \sim \omu}\brac{ \DH{\Pm(\pi),\oPm(\pi)}}.
  \end{align*}
The same reasoning also applies for \cref{alg:E2D-RF}.

\paragraph{All-policy model estimation} For this task, the space of strategy (corresponding to an \emph{estimation of the model}) is $\sspme=\Delta(\cM)$, i.e. the set of all distributions over $\cM$, and
\begin{align*}
    \suboptme_M(p)\defeq \max_{\pi\in\Pi}\EE_{\hM\sim p}\big[ \wdTV\big( M(\pi), \hM(\pi) \big) \big].
\end{align*}
Along with $\csp=\csppac$, \cref{definition:mdec} (\cref{alg:E2D-ME}) immediately follows from \cref{def:gen-dec} (\cref{alg:E2D-gen}).

\paragraph{Preference-based RL} For preference-based RL ($\GOAL=\mathrm{pb}$), the strategy space is $\ssppb=\Delta(\Piab)$, and the sub-optimality measure is given by
\begin{align*}
    \suboptpb_M(p)=\sup_{\pis\in\Pia}\EE_{(\pi_1,\pi_2)\sim p}\brac{ \Cmp^M(\pi_\star, \pi_{1})+\Cmp^M(\pi_\star, \pi_{2})-1 }.
\end{align*}
Along with $\csp=\cspreg$, \cref{def:pbdec} (\cref{alg:E2D-PbRL}) directly follows from \cref{def:gen-dec} (\cref{alg:E2D-gen}). We also remark that the inequalities in \cref{prop:pbrl-lower-bound-demo} indeed follows from Eq. \cref{eqn:reg-lower-bound-strong} in the proof of \cref{prop:gen-lower-bound-demo}.

\subsection{Lower bound with localized G-DECs}\label{appdx:gen-lower-local}

The main result of this section is the following localized version of \cref{prop:gen-lower-bound-demo}. For any model class $\cM$ and $\oM\in\cM$, we define the shorthand $\DEC_\gamma(\cM, \oM)\defeq \DEC_\gamma(\cM, \delta_{\oM})$, where $\delta_{\oM}$ denotes the point mass at $\oM$.

\begin{proposition}
\label{prop:gen-lower-bound}
Consider a general learning goal $\GOAL$, a model class $\cM$, and $T \geq 1$ a fixed integer. Define $\oeps_{\gamma}:= \frac{\gamma}{c_0\log(2T)T}$ for a large absolute constant $c_0$
and the localization
\begin{align}
\label{equation:gen-localized-model}
\cM_{\varepsilon}^{\GOAL}(\oM)=\set{M \in \cM:\big|\subopts_M(p)-\subopts_{\oM}(p)\big| \leq \varepsilon,\ \forall p\in\ssp }.
\end{align}
Then for any $T$-round algorithm $\Algo$, the following holds:

(1) For $\csps=\csppac$, there exists a $M^\star\in\cM$ such that
\begin{align*}
\EE^{M^\star,\Algo}\brac{\subopts} \geq \frac16\cdot \max _{\gamma>0} \sup _{\oM \in \cM} \DEC_{\gamma}(\cM_{\oeps_{\gamma}}^{\GOAL}(\oM), \oM),
\end{align*}

(2) For $\csps=\cspreg$, there exists a $M^\star\in\cM$ such that
\begin{align*}
\EE^{M^\star,\Algo}\brac{\Regs} 
\geq&~ \frac{T}{6}\cdot \max _{\gamma>0} \sup _{\oM \in \cM} \DEC_{\gamma}(\cM_{\oeps_{\gamma}}^{\GOAL}(\oM), \oM).
\end{align*}
\end{proposition}

The proof of \cref{prop:gen-lower-bound} is a direct generalization of \citet[Theorem 3.2]{foster2021statistical} and is deferred to \cref{appdx:gen-lower-local}. \cref{prop:gen-lower-bound} can be directly instantiated to all the learning goals of interest to provide possibly tighter lower bound. As an example, we instantiate \cref{prop:gen-lower-bound} to PAC RL, as follows.

\begin{proposition}[Lower bound for PAC RL]
\label{prop:eec-lower-bound-local}
For any model class $\cM$, $T\in\Z_{\ge 1}$, and any $T$-round algorithm $\Algo$, there exists a $M^\star\in\cM$ such that
\begin{align*}
\EE^{M^\star,\Algo}\brac{\suboptpac} \geq \frac16\cdot \max _{\gamma>0} \sup _{\oM \in \cM} \eec_{\gamma}(\cM_{\oeps_{\gamma}}^{\rm pac}(\oM), \oM),
\end{align*}
where we abbreviate $\vgap^{M}(\pi):=f^{M}(\pi_{M})-f^{M}(\pi)$, and define the PAC localization as
\begin{align}
\label{equation:localized-model}
\cM_{\varepsilon}^{\rm pac}(\oM)=\left\{M \in \cM:\left|\vgap^{M}(\pi)-\vgap^{\oM}(\pi)\right| \leq \varepsilon,\ \forall \pi \in \Pi\right\}.
\end{align}
\end{proposition}

In the following proposition, we show that there exists a class of tabular MDPs whose (localized) PACDEC has a desired lower bound, and hence establish a $\tilde{\Omega}(HSA/\epsilon^2)$ lower bound of sample complexity for PAC learning in tabular MDPs, recovering the result of \cite{domingues2021episodic}.
\begin{proposition}[PACDEC lower bound for tabular MDPs]
\label{prop:eec-lower-bound-MDP}
There exists $\cM$ a class of MDPs with $S\geq 4$ states, $A\geq 2$ actions, horizon $H\geq 2\log_2(S)$ and the same reward function, such that
\begin{align*}
    \sup_{\oM\in\cM} \eec_{\gamma}(\cM_{\varepsilon}^{\pac}(\oM),\oM)\geq  c_1\min\set{1, \frac{HSA}{\gamma}},
\end{align*}
for all $\gamma>0$ such that $\varepsilon\geq c_2HSA/\gamma$, where $c_1,c_2$ are two universal constants. As a corollary, applying the PAC lower bound in localized PACDEC (\cref{prop:eec-lower-bound-local}), we have that for any algorithm $\Algo$ that interacts with the environment for $T$ episodes,
    \begin{align*}
    \sup_{M\in\cM} \EE^{M,\Algo}\brac{\subopt}\geq c_0\min\set{1, \sqrt{\frac{HSA}{\log(2T)T}}}, 
    \end{align*}
    where $c_0$ is a universal constant.
\end{proposition}

\cref{prop:eec-lower-bound-MDP} (proof in \cref{appdx:proof-eec-lower-MDP}) implies a sample complexity lower bound of $\tilde{\Omega}(HSA/\epsilon^2)$\footnote{The logarithm factor here is an artifact of the proof of \cref{appdx:gen-lower-local}, where under certain condition \cref{prop:Hellinger-chain} can be refined to remove the $\log(2T)$ factors (see e.g. \citet[Theorem 3.2]{foster2021statistical}).} for learning $\epsilon$-optimal policy in tabular MDPs. This simple example illustrates the power of $\eec$ as a lower bound for PAC learning, analogously to the DEC for no-regret learning.

\subsubsection{Proof of Proposition~\ref{prop:gen-lower-bound}}\label{appdx:proof-gen-lower-local}
The proof in this section adopts the technique of \citet{foster2021statistical} for proving lower bounds with the localized DEC. %
We recall the notations in \cref{appendix:proof-gen-lower-bound} and the results there (\cref{prop:Hellinger-chain} and \cref{lem:reward-dist}).

Let us first prove \cref{prop:gen-lower-bound} (1), the case $\csps=\csppac$. We abbreviate $\eps=\oeps_{\gamma}$, $\delta:=\sup_{\oM\in\cM}\DEC_{\gamma}\left(\cM_{\varepsilon}^{\GOAL}(\oM), \oM\right)$ for a fixed $\gamma>0$. Let $\oM \in \cM$ attains this supremum, and then by definition,
$$
\sup _{M \in \cM_{\varepsilon}^{\GOAL}(\oM)} \subopts_M(p_{\oM,\out}) - \gamma \EE_{\pi \sim p_{\oM,\expl}}\left[ \wtdH^2\paren{M(\pi), \oM(\pi)}\right] \geq \delta.
$$
Let $M \in \cM_{\varepsilon}^{\GOAL}(\oM)$ attain the supremum above. Then by the convexity of $\subopts$ and \cref{lem:reward-dist}, we have
$$
\EE^{\oM,\Algo}\brac{\subopts_M(\hatpout)}
\geq \subopts_M(p_{\oM,\out})
\geq \frac{\gamma}3 \cdot \EE_{\pi \sim p_{\oM,\expl}}\left[\DH{M(\pi), \oM(\pi)}\right]+\delta.
$$
Recall from the definition of $\cM_{\varepsilon}^{\GOAL}(\oM)$ that $\left|\subopts_M(\pout)-\subopts_{\oM}(\pout)\right| \leq \varepsilon$ for all $\pout\in\ssp$. Hence, we can now apply \cref{lemma:multiplicative-hellinger} to the random variable
\begin{align*}
    X_{M,\oM}\defeq X_{M,\oM}(\hatpout)=\subopts_M(\hatpout)-\subopts_{\oM}(\hatpout),
\end{align*}
and it holds that
$$
\begin{aligned}
&\left|\EE^{M,\Algo}\left[X_{M,\oM}\right]-\EE^{\oM,\Algo}\left[X_{M,\oM}\right]\right| \\
&\leq \sqrt{8 \varepsilon \cdot\left(\EE^{M,\Algo}\abs{X_{M,\oM}}+\EE^{\oM,\Algo}\abs{X_{M,\oM}}\right) \cdot \DH{\PP^{M,\Algo}, \PP^{\oM,\Algo}}} \\
&\leq 4 \varepsilon \DH{\PP^{M,\Algo}, \PP^{\oM,\Algo}}+\frac{1}{2}\left(\EE^{M,\Algo}\abs{X_{M,\oM}}+\EE^{\oM,\Algo}\abs{X_{M,\oM}}\right),
\end{aligned}
$$
which implies (by re-arranging)
\begin{align*}
    &~\EE^{M,\Algo}\left[\subopts_M(\hatpout)\right]+\EE^{\oM,\Algo}\left[\subopts_{\oM}(\hatpout)\right] \\
    \geq &~\frac{1}{3}\EE^{\oM,\Algo}\left[\subopts_M(\hatpout)\right]+\frac{1}{3}\EE^{M,\Algo}\left[\subopts_{\oM}(\hatpout)\right]
    -\frac{8}{3}\varepsilon \DH{\PP^{M,\Algo}, \PP^{\oM,\Algo}}.
\end{align*}
Furthermore, by the chain rule of Hellinger distance (cf. \cref{eqn:Hellinger-chain} and \cref{eqn:Hellinger-chain}), we have
\begin{align*}
\DH{\PP^{M,\Algo}, \PP^{\oM,\Algo}}
\leq 100\log(2T)T \cdot \EE_{\pi \sim p_{\oM,\expl}}\left[\DH{M(\pi), \oM(\pi)}\right].
\end{align*}
As long as $2400\log(2T) T \varepsilon \leq \gamma$, it holds that
$$
\EE^{M,\Algo}\left[\subopts_M(\hatpout)\right]+\EE^{\oM,\Algo}\left[\subopts_{\oM}(\hatpout)\right] \geq \frac{1}{3} \delta.
$$
This completes the proof of the case $\csps=\csppac$. 

For (2), the case where $\csps=\cspreg$ (i.e. $\ssps=\Delta(\Pi)$), we follow the strategy that we prove \cref{prop:gen-lower-bound-demo} (2). Consider the algorithm $\Algo$ with the same interaction rules $\pexp$ and (modified) output rule $\hatpout=\frac1T \sum_{t=1}^T \pexp^t$. Repeating the argument above yields that there exists $M^\star\in\cM$ such that
\begin{align*}
    \EE^{M^\star,\Algo}\left[\subopts_{M^\star}(\hatpout)\right]\geq\frac16\sup_{\oM\in\cM}\DEC_{\gamma}\left(\cM_{\varepsilon}^{\GOAL}(\oM), \oM\right).
\end{align*}
Then by the convexity of $\subopt$, it holds that
\begin{align*}%
\begin{split}
    \EE^{M^\star,\Algo}\brac{\Regs}
    =&~ \EE^{M^\star,\Algo}\brac{ \sum_{t=1}^T \subopts_{M^\star}(\pexp^t) } \\
    \geq&~ T\cdot \EE^{M^\star,\Algo}\brac{\subopts_{M^\star}\paren{ \frac1T\sum_{t=1}^T\pexp^t } } \\
    \geq&~ \frac{1}{6}\max _{\gamma>0} \sup _{\oM \in \cM} \DEC_{\gamma}(\cM_{\oeps_{\gamma}}^{\GOAL}(\oM), \oM).
\end{split}
\end{align*}
This gives the desired result. \qed

\subsubsection{Proof of Proposition~\ref{prop:eec-lower-bound-MDP}}\label{appdx:proof-eec-lower-MDP}
In this section, we follow \cite{domingues2021episodic, foster2021statistical} to construct a class of tabular MDPs whose (localized) PACDEC has a desired lower bound.

Without loss of generality, we assume that $S=2^{n+1}+1$ and let $S'=2^n$. We also write $A'=A-1$, $H'=H-n\geq H/2$. 

Fix a $\Delta\in(0,\frac13]$, we consider $\cM^{\Delta}$ the class of MDPs described as follows.
\newcommand{\Stree}{\cS_{\rm tree}}
\newcommand{\Sl}{S'}
\newcommand{\hs}{h^\star}
\renewcommand{\ss}{s^\star}
\newcommand{\as}{a^\star}
\newcommand{\Al}{[A']}
\renewcommand{\sp}{s_{\oplus}}
\newcommand{\sq}{s_{\ominus}}
\newcommand{\aleft}{{\sf left}}
\newcommand{\aright}{{\sf right}}
\newcommand{\await}{{\sf wait}}
\begin{enumerate}[leftmargin=2em]
    \item The state space $\cS=\Stree \bigsqcup \set{\sp, \sq}$, where $\Stree$ is a binary tree of level $n+1$ (hence $\abs{\Stree}=2^{n+1}-1$), and $\sp, \sq$ are two auxiliary nodes. Let $s_0$ be the root of $\Stree$, and $\Sl$ be the set of leaves of $\Stree$ (hence $\abs{\Sl}=2^n$).
    \item Each episode has horizon $H$. 
    \item The reward function is fixed and known: arriving at $\sp$ emits a reward 1, and at all other states the reward is 0.
    \item For $\hs\in\cH'\defeq \set{n+1,\cdots,H}, \ss\in\Sl, \as\in\Al$, the transition dynamic of $M=M_{\hs,\ss,\as}$ is defined as follows:
    \begin{itemize}[leftmargin=1em]
        \item The initial state is always $s_0$.
        \item At a node $s\in\Stree$ such that $s$ is not leaf of $\Stree$, there are two available actions $\aleft$ and $\aright$, with $\aleft$ leads to the left children of $s$ and $\aright$ leads to the right children of $s$.
        \item At leaf $s\in\Sl$, there are $A$ actions: $\await, 1, \cdots, A-1$. The dynamic of $M=M_{\hs,\ss,\as}\in\cM$ at $s$ is given by: $\PP^M_h(s|s,\await)=1$, and for $a\in\Al$, $h\in[H]$
        \begin{align*}
            &\PP^{M}_h(\sp|s,a)=\frac12+\Delta\cdot\II(h=\hs,s=\ss,a=\as), \\
            &\PP^{M}_h(\sq|s,a)=\frac12-\Delta\cdot\II(h=\hs,s=\ss,a=\as),
        \end{align*}
        \item The state $\sp$ always transits to $\sq$, and $\sq$ is the absorbing state (i.e. $\PP(\sq|\sp,\cdot)=1$, $\PP(\sq|\sq,\cdot)=1$).
    \end{itemize}
\end{enumerate}
Let $\oM$ be the MDP model with the same transition dynamic and reward function as above, except that for all $h\in[H], s\in\Sl, a\in\Al$ it holds $\PP^{M}_h(\sp|s,a)=\PP^{M}_h(\sq|s,a)=\frac12$. Note that $\oM$ does not depend on $\Delta$. We then define
$$
\cM^{\Delta}=\set{\oM}\bigcup 
\set{M_{h,s,a}: (h,s,a)\in\cH'\times\Sl\times\Al}.
$$

Before lower bounding $\eecg(\cM^{\Delta}, \oM)$, we make some preparations. Define
\begin{align*}
    \nu_{\pi}(h,s,a)=\PP^{\oM,\pi}(s_h=s,a_h=a), \qquad \forall (h,s,a)\in\cH'\times\Sl\times\Al.
\end{align*}
Note that due to the structure of $\oM$, the events $A_{h,s,a}\defeq \{s_h=s,a_h=a\}$ are disjoint for $(h,s,a)\in\cH'\times\Sl\times\Al$; therefore,
$$
\sum_{h\in\cH',s\in\Sl,a\in\Al} \nu_{\pi}(h,s,a)\leq 1.
$$
Furthermore, for $M=M_{h,s,a}\in\cM^\Delta$, we have
\begin{align*}
    \DH{M(\pi), \oM(\pi)}=\PP^{\oM,\pi}(s_{h}=s, a_{h}=a) \DH{\PP_h^M(\cdot|s,a), \PP_h^{\oM}(\cdot|s,a)}
\end{align*}
because $M(\pi)$ and $\oM(\pi)$ only differs at the conditional probability of $s_{h+1}|s_h=s,a_h=a$. Therefore, due to the fact that $\DH{\PP_h^M(\cdot|s,a), \PP_h^{\oM}(\cdot|s,a)}=\DH{\Bern(\frac12+\Delta), \Bern(\frac12)}\leq 3\Delta^2$, we have
$$
\DH{M(\pi), \oM(\pi)}\leq 3\nu_{\pi}(h,s,a)\Delta^2.
$$

Now, that for $\pexp,\pout\in\Delta(\pi)$ and $M\in\cM'$, we have
\begin{align*}
    &\EE_{\pi\sim\pout}\brac{f^M(\pi_M)-f^M(\pi)}-\gamma\EE_{\pi\sim\pexp}\brac{\DH{M(\pi), \oM(\pi)}}\\
    =&
    \frac12+\Delta-\EE_{\pi\sim\pout}\brac{\frac12+\Delta\PP^{M,\pi}(\hs,\ss,\as)}
    -\gamma\EE_{\pi\sim\pexp}\brac{\DH{M(\pi), \oM(\pi)}}\\
    \geq& \Delta\paren{1-\EE_{\pi\sim\pout}\brac{\nu_{\pi}(\hs,\ss,\as)}}-3\gamma\Delta^2\EE_{\pi\sim\pexp}\brac{\nu_{\pi}(\hs,\ss,\as)}.
\end{align*}
Therefore, we define
\begin{align*}
    \nu_{\pexp}(h,s,a)=\EE_{\pi\sim\pexp}\brac{\nu_{\pi}(h,s,a)}, \qquad \nu_{\pout}(h,s,a)=\EE_{\pi\sim\pout}\brac{\nu_{\pi}(h,s,a)}.
\end{align*}
Then, for any fixed $\pexp,\pout\in\Delta(\Pi)$, by the fact that
$$
\sum_{h\in\cH',s\in\Sl,a\in\Al}\set{ \nu_{\pexp}(h,s,a)+3\gamma\Delta\nu_{\pout}(h,s,a) }\leq 1+3\gamma\Delta,
$$
we know that there exists $(h',s',a')\in\cH'\times\Sl\times\Al$ such that 
$$
\nu_{\pexp}(h',s',a')+3\gamma\Delta\nu_{\pout}(h',s',a')\leq \frac{1+3\gamma\Delta}{H'S'A'}.
$$
Then we can consider $M'=M_{h',s',a'}$, and 
\begin{align*}
    &\sup_{M\in\cM'}\EE_{\pi\sim\pout}\brac{f^M(\pi_M)-f^M(\pi)}-\gamma\EE_{\pi\sim\pexp}\brac{\DH{M(\pi), \oM(\pi)}}\\
    \geq &
    \EE_{\pi\sim\pout}\brac{f^{M'}(\pi_{M'})-f^{M'}(\pi)}-\gamma\EE_{\pi\sim\pexp}\brac{\DH{M'(\pi), \oM(\pi)}}\\
    \geq& \Delta\paren{1-\nu_{\pexp}(h',s',a')}-3\gamma\Delta^2\nu_{\pout}(h',s',a')\\
    \geq& \Delta-\Delta\cdot\frac{1+3\gamma\Delta}{H'S'A'}.
\end{align*}
By the arbitrariness of $\pexp,\pout\in\Delta(\Pi)$, we derive that
$$
\eec_{\gamma}(\cM^{\Delta},\oM)\geq \Delta-\Delta\cdot\frac{1+3\gamma\Delta}{H'S'A'}.
$$
Therefore, we consider $\cM=\bigcup_{\Delta>0} \cM^{\Delta}$. By definition, it holds that $\cM^{\Delta}\subset \cM^{\pac}_{\eps}(\oM)$ for any $\eps\geq\Delta$ (cf.~\eqref{equation:localized-model}). Hence, for any given $\gamma>0$, we can take $\Delta=\min\set{\frac13, \frac{H'S'A'}{12\gamma}}$, and then as long as $\eps\geq\Delta$
\begin{align*}
    \eec_{\gamma}(\cM_{\varepsilon}^{\pac}(\oM),\oM)\geq  \eec_{\gamma}(\cM^{\Delta},\oM)
    \geq \Delta-\Delta\cdot\frac{1+3\gamma\Delta}{H'S'A'}\geq \frac{\Delta}{4}=\frac14\min\set{\frac13, \frac{H'S'A'}{12\gamma}}.
\end{align*}
This completes the proof of \cref{prop:eec-lower-bound-MDP}. \qed

\subsection{Algorithm Reward-Free E2D}

In this section, we present the complete description of \rfalg~algorithm (\cref{alg:E2D-RF}) sketched in \cref{section:rfec}. %

\begin{algorithm}[t]
\caption{\rfalg} 
\begin{algorithmic}[1]
\label{alg:E2D-RF}
\STATE \textbf{Input:} Parameters $\eta=1/2$, $\gamma>0$; prior distribution $\mu^1=\unif(\cP)$.
\STATE \textit{// Exploration phase}
\FOR{$t=1,\ldots,T$}
    \STATE Set $\pexp^t=\argmin_{\pexp\in\Delta(\Pi)}\sup_{R\in\cR}\hV^{\mu^t}_{\rf,\gamma}(\pexp,R)$, where (cf.~\cref{equation:vtp-rf})
    \begin{align*}
       \hV^{\mu^t}_{\rf,\gamma}(\pexp,R) \defeq \inf_{\pout}\sup_{\Pm\in\cP}\E_{\pi \sim \pout}  \brac{ f^{\Pm,R}(\pi_{\Pm,R})-f^{\Pm,R}(\pi) } - \gamma \E_{\pi \sim \pexp}\E_{\hPm^t\sim \mu^t}\brac{ \dH^2(\Pm(\pi), \hPm^t(\pi)) }.
    \end{align*}
    \STATE Sample $\pi^t\sim p^t_{\expl}$. Execute $\pi^t$ and observe $o^t$.
    \STATE Compute $\mu^{t+1}\in\Delta(\cP)$ by \Vovkalg~with observations only:
    \begin{align}
      \label{equation:tempered-aggregation-obs-only}
      \mu^{t+1}(\Pm) \; \propto_{\Pm} \; \mu^{t}(\Pm) \cdot \exp\paren{\eta \log \Pm^{\pi^t}(o^{t}) }.
    \end{align}
\ENDFOR
\STATE \textit{// Planning phase}
\STATE \textbf{Input:}  $R^\star\in\cR$
\FOR{$t=1,\ldots,T$}
    \STATE Compute
    \begin{align*}
        \pout^t(R^\star):=\argmin_{\pout\in\Delta(\Pi)} \sup_{\Pm\in\cP}\E_{\pi \sim \pout}  \brac{ f^{\Pm,R^\star}(\pi_{\Pm,R})-f^{\Pm,R^\star}(\pi) } - \gamma \E_{\pi \sim \pexp^t}\E_{\hPm^t\sim \mu^t}\brac{ \dH^2(\Pm(\pi), \hPm^t(\pi)) }.
    \end{align*}
\ENDFOR
\STATE \textbf{Output:} $\hatpout(R^\star)=\frac{1}{T}\sum_{t=1}^T \pout^t(R^\star).$
\end{algorithmic}
\end{algorithm}

\newcommand{\tPm}{\wt{\Pm}}
More generally, \cref{alg:E2D-RF} also applies to the case when $\cM$ only admits a finite optimistic covering (\cref{thm:E2D-gen-full}). 
The guarantee of \cref{alg:E2D-RF} in this general setting is stated as follows.
\begin{theorem}[\rfalg]
\label{thm:E2D-RF-full}
Given a suitable $\rho$-optimistic cover $(\tPm,\cP_0)$ of $\cP$, we can replace the subroutine \eqref{equation:tempered-aggregation-obs-only} in~\cref{alg:E2D-RF} with \begin{align*}
  \mu^{t+1}(\Pm) \; \propto_{\Pm} \; \mu^{t}(\Pm) \cdot \exp\paren{\eta \log \tPm^{\pi^t}(o^{t}) }.
\end{align*}
and let $\eta=1/2$, $\mu^1=\Unif(\cP_0)$, then~\cref{alg:E2D-RF} achieves the following with probability at least $1-\delta$:
\begin{align*}
    \suboptrf\leq \orfecg(\cP)+\frac{2\gamma}{T}\brac{\est(\cP,T)+2\log(2/\delta)}.
\end{align*}
\end{theorem}

\subsection{Algorithm All-Policy Model-Estimation E2D}
\label{appendix:mdec-alg}\label{appendix:mdec}

In this section, we present the complete description of \mealg~algorithm (\cref{alg:E2D-ME}) sketched in \cref{section:mdec}.
\begin{algorithm}[t]
\caption{\MEalg} 
\begin{algorithmic}[1]
\label{alg:E2D-ME}
\STATE \textbf{Input:} Parameters $\etap=\etar=1/3$, $\gamma>0$; prior distribution $\mu^1=\Unif(\cM)$.
\STATE Initialize $\mu^1\setto \unif(\cM_0)$
\FOR{$t=1,\ldots,T$}
    \STATE Set $(\pexp^t,\muout^t)=\argmin_{(\pexp,\muout)\in\Delta(\Pi)\times\Delta(\cM)}\hV^{\mu^t}_{\ME,\gamma}(\pexp,\muout)$, where (cf. \cref{equation:vtp-me})
    \begin{align*}
       \hV^{\mu^t}_{\ME,\gamma}(\pexp,\muout) \defeq \sup_{M\in\cM}\sup_{\opi\in\Pi}\E_{\oM \sim \muout}  \brac{ \wDTV{ M(\opi), \oM(\opi)} } - \gamma \E_{\pi \sim \pexp}\E_{\hM^t\sim \mu^t}\brac{ \wtdH^2(M(\pi), \hM^t(\pi)) }.
    \end{align*}
    \STATE Sample $\pi^t\sim p^t_{\expl}$. Execute $\pi^t$ and observe $(o^t,\br^t)$.
    \STATE Compute $\mu^{t+1}\in\Delta(\cM)$ by \Vovkalg:
    \begin{align}\label{eqn:TA-AMDEC}
    \mu^{t+1}(M) \; \propto_M \; \mu^{t}(M) \cdot \exp\paren{\etap \log \PP^{M,\pi^t}(o^{t}) - \etar\ltwo{\br^t - \bR^M(o^t)}^2 }.
    \end{align}
\ENDFOR
\STATE Compute $\muout=\frac1T\sum_{t=1}^T\muout^t\in\Delta(\cM)$.
\STATE \textbf{Output:} $\hM=\argmin_{M\in\cM} \sup_{\opi\in\Pi} \E_{\oM \sim \muout}  \brac{ \wDTV{ M(\opi), \oM(\opi)} }$
\end{algorithmic}
\end{algorithm}

More generally, \cref{alg:E2D-ME} also applies to the case when $\cM$ only admits a finite optimistic covering (\cref{thm:E2D-gen-full}). 
The guarantee of \cref{alg:E2D-ME} in this general setting is stated as follows.

\begin{theorem}[\mealg]\label{thm:E2D-ME-full}
Given a suitable $\rho$-optimistic cover $(\tPP,\cM_0)$ of $\cM$, we can replace the subroutine \eqref{eqn:TA-AMDEC} in~\cref{alg:E2D-ME} with \begin{align*}
    \mu^{t+1}(M) \; \propto_M \; \mu^{t}(M) \cdot \exp\paren{\etap \log \tPP^{M,\pi^t}(o^{t}) - \etar\ltwo{\br^t - \bR^M(o^t)}^2 }.
\end{align*}
and let $\etap=\etar=1/3$, $\mu^1=\Unif(\cM_0)$, then~\cref{alg:E2D-ME} achieves the following with probability at least $1-\delta$:
\begin{align*}
    \DTVPi{M^\star, \hM}\leq 6\omdec_{\gamma}(\cM)+\frac{60\gamma}{T}\brac{\est(\cM,T)+2\log(2/\delta)}.
\end{align*}
\end{theorem}

\begin{proof}[Proof of \cref{thm:E2D-ME} and \cref{thm:E2D-ME-full}]
We only need to relate the guarantee of $\hM$ to the guarantee of $\muout$ (implied by \cref{thm:E2D-gen}). 
By definition of $\hM$, it holds that
$$
\sup_{\opi\in\Pi}\E_{\oM \sim \muout}  \brac{ \wDTV{ \hM(\opi), \oM(\opi)} }\leq \sup_{\opi\in\Pi}\E_{\oM \sim \muout}  \brac{ \wDTV{ M^\star(\opi), \oM(\opi)} },
$$
and therefore by \cref{lemma:tDTV-triagnle},
\begin{align*}
\max_{\opi\in\Pi}\wDTV{ M^\star(\opi), \hM(\opi) }
\leq& 3\sup_{\opi\in\Pi}\set{ \E_{\oM \sim \muout}  \brac{ \wDTV{ \hM(\opi), \oM(\opi)} +  \wDTV{ M^\star(\opi), \oM(\opi)} }}\\
\leq&  6\sup_{\opi\in\Pi}\E_{\oM \sim \muout}  \brac{ \wDTV{ M^\star(\opi), \oM(\opi)} }.
\end{align*}
This completes the proof of \cref{thm:E2D-ME} and \cref{thm:E2D-ME-full} (by invoking \cref{thm:E2D-gen} and \cref{thm:E2D-gen-full}, respectively).
\end{proof}

\begin{lemma}\label{lemma:tDTV-triagnle}
For all $M,\oM,\hM\in\cM$ and policy $\pi$, it holds that
\begin{align}\label{eqn:triangle-wdTV}
    \wDTV{ M(\pi), \oM(\pi)} \leq 3\wDTV{ M(\pi), \hM(\pi)} + 3\wDTV{ \oM(\pi), \hM(\pi)}.
\end{align}
Furthermore, we also have
\begin{align}\label{eqn:triangle-wdTV-2}
    \abs{ \wDTV{ M(\pi), \hM(\pi) } - \wDTV{ \oM(\pi), \hM(\pi) } } \leq 3\wDTV{ M(\pi), \oM(\pi) }
\end{align}
\end{lemma}
\begin{proof}
Note that by triangle inequality,
$$
\DTV{ \MP(\pi), \oMP(\pi)}\leq \DTV{ \MP(\pi), \Pm^{\hM}(\pi)}+\DTV{ \oMP(\pi), \Pm^{\hM}(\pi)},
$$
and similarly,
\begin{align*}
&\E_{o\sim \MP(\pi)}\brac{ \big\|\bR^M(o) - \bR^{\oM}(o)\big\|_1 } \\
\leq&~ \E_{o\sim \MP(\pi)}\brac{ \big\|\bR^M(o) - \bR^{\hM}(o)\big\|_1 }+\E_{o\sim \MP(\pi)}\brac{ \big\|\bR^{\oM}(o) - \bR^{\hM}(o)\big\|_1 }\\
\leq&~ \E_{o\sim \MP(\pi)}\brac{ \big\|\bR^M(o) - \bR^{\hM}(o)\big\|_1 } \\
&~+\E_{o\sim \oMP(\pi)}\brac{ \big\|\bR^{\oM}(o) - \bR^{\hM}(o)\big\|_1 } 
+2\DTV{\MP(\pi), \oMP(\pi)},
\end{align*}
where the second inequality is because $\big\|\bR^{\oM}(o) - \bR^{\hM}(o)\big\|_1\leq 2$ for all $o\in\cO$ and $\oM,\hM\in\cM$. Combining the inequalities above completes the proof of \eqref{eqn:triangle-wdTV}. The proof of \cref{eqn:triangle-wdTV-2} follows similarly.
\end{proof}

\subsection{Algorithm Preference-based E2D}

In this section, we present the complete description of \pbrlalg~algorithm (\cref{alg:E2D-PbRL}) sketched in \cref{section:pbrl}.

\begin{algorithm}[t]
	\caption{E2D for Preference-based RL (\pbrlalg)}
	\begin{algorithmic}[1]
	\label{alg:E2D-PbRL}
	\REQUIRE Parameters $\eta=1/2$, $\gamma>0$; prior distribution $\mu^1=\Unif(\cMpb)$.
	\STATE Initialize $\mu^1\setto {\rm Unif}(\cM)$.
	\FOR{$t=1,\ldots,T$}
    \STATE Set $p^t\setto \argmin_{p\in\Delta(\Pi)}\hV^{\mu^t}_{\mathrm{pb},\gamma}(p)$, where (cf. \cref{equation:vtp-pbrl})%
    \begin{align*}
        \hV^{\mu^t}_{\mathrm{pb},\gamma}(p) \defeq \sup_{M\in\cMpb} \suboptpb_M(p)-\gamma\E_{\bpi\sim p} \E_{\oM\sim\mu^t}\brac{\dH^2(M(\bpi), \oM(\bpi))}.
    \end{align*}
    \label{line:e2d-pbrl-pt}
    \STATE Sample $\bpi^t=(\pi^t_1,\pi^t_2)\sim p^t$. Execute $\pi^t_1$ and observe $\tau^t_{1}$, execute $\pi^t_2$ and observe $\tau^t_{2}$, and then observe $b^t$. Set $o^t=(\tau^t_1,\tau^t_2,b^t)$. \label{line:e2d-pbrl-execute}
    \STATE Update randomized model estimator by \Vovkalg: %
    \begin{align}\label{eqn:TA-PbRL}
        \mu^{t+1}(M) \; \propto_M \; \mu^{t}(M) \cdot \exp\paren{\eta \log \P^{M,\bpi^t}(o^{t}) }.
    \end{align}
    \label{line:e2d-pbrl}
    \ENDFOR
   \end{algorithmic}
\end{algorithm}

More generally, \cref{alg:E2D-PbRL} also applies to the case when $\cM$ only admits a finite optimistic covering (\cref{thm:E2D-gen-full}). 
The guarantee of \cref{alg:E2D-PbRL} in this general setting is stated as follows.

\begin{theorem}[\pbrlalg]
\label{thm:E2D-PbRL-full}
Given a suitable $\rho$-optimistic cover $(\tPm,\cM_0)$ of $\cMpb$, we can replace the subroutine \eqref{eqn:TA-PbRL} in~\cref{alg:E2D-RF} with \begin{align*}
  \mu^{t+1}(M) \; \propto_{M} \; \mu^{t}(M) \cdot \exp\paren{\eta \log \tPP^{M,\pi^t}(o^{t}) }.
\end{align*}
and let $\eta=1/2$, $\mu^1=\Unif(\cM_0)$, then~\cref{alg:E2D-PbRL} achieves the following with probability at least $1-\delta$:
\begin{align*}
    \Regpb\leq T\cdot\pbdec_\gamma(\cM)+2\gamma\cdot\brac{\est(\cP,T)+2\log(2/\delta)}.
\end{align*}
\end{theorem}

\subsection{Proof of Proposition~\ref{prop:rfdec-to-mdec}}
\label{appendix:relation-mdec-rfdec}

Fix any $\tmu\in\Delta(\cP)$. By definition,
\begin{align*}
\rfecg(\cP,\omu)
\stackrel{(i)}{=}&
\inf_{\pexp\in\Delta(\Pi)}\sup_{R\in\cR}\sup_{\mu\in\Delta(\cP)}\inf_{\pout\in\Delta(\Pi)}\EE_{\Pm\sim\mu}\EE_{\pi\sim \pout}\left[f^{\Pm, R}(\pi_{\Pm, R})-f^{\Pm, R}(\pi)\right]\\
&\qquad\qquad\qquad\qquad\qquad\qquad\quad-\gamma \EE_{\Pm\sim\mu}\EE_{\oPm \sim \omu}\EE_{\pi \sim \pexp}\left[ \dH^2(\Pm(\pi),\oPm(\pi))\right]\\
=&
\inf_{\pexp\in\Delta(\Pi)}\sup_{R\in\cR}\sup_{\mu\in\Delta(\cP)}\inf_{\pout\in\Delta(\Pi)}\EE_{\Pm\sim\mu}\left[f^{\Pm, R}(\pi_{\Pm, R})\right]
-\EE_{\pi \sim \pout}\EE_{\oPm \sim \tmu}\left[f^{\oPm, R}(\pi)\right]\\
&\qquad\qquad\qquad\qquad\qquad\qquad\quad+\EE_{\Pm\sim\mu}\EE_{\pi \sim \pout}\EE_{\oPm \sim \tmu}\left[f^{\oPm, R}(\pi)-f^{\Pm, R}(\pi)\right]\\
&\qquad\qquad\qquad\qquad\qquad\qquad\quad-\gamma \EE_{\Pm\sim\mu}\EE_{\oPm \sim \omu}\EE_{\pi \sim \pexp}\left[ \dH^2(\Pm(\pi),\oPm(\pi))\right]\\
\stackrel{(ii)}{\leq}&
\inf_{\pexp\in\Delta(\Pi)}\sup_{R\in\cR}\sup_{\mu\in\Delta(\cP)}\EE_{\Pm\sim\mu}\EE_{\oPm \sim \tmu}\left[f^{\Pm, R}(\pi_{\Pm, R})-f^{\oPm, R}(\pi_{\Pm, R})\right]\\
&\qquad\qquad\qquad\qquad\quad+\EE_{\Pm\sim\mu}\EE_{\Pm'\sim\mu}\EE_{\oPm \sim \tmu}\left[f^{\oPm, R}(\pi_{\Pm',R})-f^{\Pm, R}(\pi_{\Pm',R})\right]\\
&\qquad\qquad\qquad\qquad\quad-\gamma \EE_{\Pm\sim\mu}\EE_{\oPm \sim \omu}\EE_{\pi \sim \pexp}\left[ \dH^2(\Pm(\pi),\oPm(\pi))\right]\\
\stackrel{(iii)}{\leq}&
\inf_{\pexp\in\Delta(\Pi)}\sup_{R\in\cR}\sup_{\mu\in\Delta(\cP)}2\EE_{\Pm\sim\mu}\brac{\sup_{\pi'\in\Pi}\abs{\EE_{\oPm \sim \tmu}\left[f^{\Pm, R}(\pi')-f^{\oPm, R}(\pi')\right]}}\\
&\qquad\qquad\qquad\qquad\quad-\gamma \EE_{\Pm\sim\mu}\EE_{\oPm \sim \omu}\EE_{\pi \sim \pexp}\left[ \dH^2(\Pm(\pi),\oPm(\pi))\right]\\
=&
\inf_{\pexp\in\Delta(\Pi)}\sup_{R\in\cR,\Pm\in\cP} \Big\{ \sup_{\pi'\in\Pi}\abs{\EE_{\oPm \sim \tmu}\left[f^{\oPm, R}(\pi')-f^{\Pm, R}(\pi')\right]}\\
&\qquad\qquad\qquad\qquad\quad-\gamma \EE_{\oPm \sim \omu}\EE_{\pi \sim \pexp}\left[ \dH^2(\Pm(\pi),\oPm(\pi))\right] \Big\},
\end{align*}
where (i) is due to strong duality (\cref{thm:strong-dual}), in (ii) we upper bound $\inf_{\pout}$ by letting $\pout\in\Delta(\Pi)$ be defined by $\pout(\pi)=\mu(\{\Pm: \pi_{\Pm,R}=\pi\})$, and in (iii) we upper bound
\begin{align*}
\EE_{\oPm \sim \tmu}\left[f^{\Pm, R}(\pi_{\Pm, R})-f^{\oPm, R}(\pi_{\Pm, R})\right]
&\leq \sup_{\pi'\in\Pi}\abs{\EE_{\oPm \sim \tmu}\left[f^{\oPm, R}(\pi')-f^{\Pm, R}(\pi')\right]},
\\
\EE_{\Pm'\sim\mu}\EE_{\oPm \sim \tmu}\left[f^{\oPm, R}(\pi_{\Pm',R})-f^{\Pm, R}(\pi_{\Pm',R})\right]&\leq \sup_{\pi'\in\Pi}\abs{\EE_{\oPm \sim \tmu}\left[f^{\oPm, R}(\pi')-f^{\Pm, R}(\pi')\right]}.
\end{align*}
Taking $\inf_{\tmu}$ over $\tmu\in\Delta(\cP)$ gives
\begin{align*}
\rfecg(\cP,\omu)\leq&~
\inf_{\substack{p\in\Delta(\Pi), \\ \tmu\in\Delta(\cP)}}\sup_{\Pm,R,\pi'} 2\abs{\EE_{\oPm\sim\tmu}\brac{ f^{\Pm, R}(\pi')-f^{\oPm, R}(\pi')} }-\gamma\EE_{\oPm\sim\omu}\EE_{\pi\sim p}\brac{\dH^2(\Pm(\pi),\oPm(\pi))} \\
\leq&~
\inf_{\substack{p\in\Delta(\Pi), \\ \tmu\in\Delta(\cP)}}\sup_{\Pm,\pi'} 2\abs{\EE_{\oPm\sim\tmu}\brac{   \dTV(\Pm(\pi'),\oPm(\pi'))  } }-\gamma\EE_{\oPm\sim\omu}\EE_{\pi\sim p}\brac{\dH^2(\Pm(\pi),\oPm(\pi))} \\
=&~
2\mdec_{\gamma/2}(\cP,\omu).
\end{align*}
This is the desired result.

\subsection{Proof of Theorem~\ref{thm:PbRL-to-RL}}\label{appdx:proof-PbRL-to-RL}

Fix a $\omu\in\Delta(\cMpb)$.
For any $q\in\Delta(\Pi)$, consider the following function $L_q:\Delta(\Pi)\to \R$
\begin{align*}
    p\mapsto L_q(p):=\max_{M,\pis} \EE_{\pi\sim p}\EE_{\pi_0\sim q}\brac{ \Cmp^M(\pis,\pi_0)-\Cmp^M(\pi,\pi_0) - \gamma \EE_{\oM\sim\omu}\brac{\dH^2(M(\pi,\pi_0), \oM(\pi,\pi_0))}  }.
\end{align*}
Notice that for any model $M\in\cMpb$, $\pi\in\Pia$, the expected reward of $\pi$ under model $M_q$ is exactly $f^{M_q}(\pi)=\EE_{\pi_0\sim q}\brac{ \Cmp^M(\pi,\pi_0) }$, and
\begin{align*}
    \EE_{\pi_0\sim q}\brac{\dH^2(M(\pi,\pi_0), \oM(\pi,\pi_0))}=\dH^2(M_q(\pi), \oM_q(\pi)), \qquad\forall \oM\in\cMpb.
\end{align*}
This is because under model $M_q$ and $\oM_q$ the distribution of $\pi_0\sim q$ is the same, and hence the equality above follows from the definition of Hellinger distance. Therefore, we have
\begin{align*}
    L_q(p)=\max_{M} \EE_{\pi\sim p}\brac{ f^{M_q}(\pi_{M_q})-f^{M_q}(\pi) - \gamma \EE_{\oM\sim\omu}\brac{\dH^2(M_q(\pi), \oM_q(\pi))}  }.
\end{align*}
By the definition of $\dec$, we know
\begin{align*}
    \min_{p\in\Delta(\Pi)} L_q(p)\leq \dec_{\gamma}(\cM_q,\omu).
\end{align*}

Now, consider the map
\begin{align*}
    F: q\in\Delta(\Pia) \mapsto \argmin_{p\in\Delta(\Pia)} L_q(p).
\end{align*}
$F$ is a set-valued map such that for all $q\in\Delta(\Pia)$, $F(q)\subset\Delta(\Pia)$ is non-empty and convex, and the graph of $F$ is clearly closed. Therefore, we can apply Kakutani’s fixed point theorem \cite[Lemma 20.1]{osborne1994course} to show that there exists a fixed point $q$ such that $q\in F(q)$, i.e. $L_q(q)\leq \dec_{\gamma}(\cM_q,\omu)$. Equivalently, we have
\begin{align*}
    \max_{M,\pis} \EE_{\pi_1\sim q}\EE_{\pi_2\sim q}\brac{ \Cmp^M(\pis,\pi_2)-\Cmp^M(\pi_1,\pi_2) - \gamma \EE_{\oM\sim\omu}\brac{\dH^2(M(\pi_1,\pi_2), \oM(\pi_1,\pi_2))}  }\leq \dec_{\gamma}(\cM_q,\omu).
\end{align*}
Notice that $\EE_{\pi_1\sim q}\EE_{\pi_2\sim q}\brac{ \Cmp^M(\pi_1,\pi_2)}=\frac12$ by the symmetric property of $\Cmp^M$. Therefore, for the distribution $p=q\times q\in\Delta(\Piab)$, we have
\begin{align*}
    &~\max_{M,\pis} \EE_{\bpi=(\pi_1,\pi_2)\sim p}\brac{ \Cmp^M(\pis,\pi_1)+\Cmp^M(\pis,\pi_2)-1- 2\gamma \EE_{\oM\sim\omu}\brac{\dH^2(M(\pi_1,\pi_2), \oM(\pi_1,\pi_2))}  } \\
    =&~
    \max_{M,\pis} \EE_{\pi\sim q}\brac{ 2\Cmp^M(\pis,\pi)-1} - 2\gamma \EE_{\pi_1\sim q,\pi_2\sim q}\EE_{\oM\sim\omu}\brac{\dH^2(M(\pi_1,\pi_2), \oM(\pi_1,\pi_2))}  \\
    \leq&~ 2\dec_{\gamma}(\cM_q,\omu).
\end{align*}
This inequality directly implies $\pbdec_{2\gamma}(\cMpb,\omu)\leq 2\dec_{\gamma}(\cM_q,\omu)$. By the arbitrariness of $\omu$, the proof is completed by rescaling $\gamma$ to $\gamma/2$.
\qed

\section{Proofs for Section~\ref{section:bellman-rep}}
\label{appendix:bellman-rep}

This section provides the proofs for~\cref{section:bellman-rep} along with some additional discussions, organized as follows. We begin by presenting some useful definition and intermediate results in~\cref{appendix:complexity-measures} for complexity measures; 
~\cref{appdx:bilinear-eluder} presents some discussions on Bellman Eluder dimension~\citep{jin2021bellman}, Bilinear class~\citep{du2021bilinear}, coverability~\citep{xie2022role}, and Bellman representation~\citep{foster2021statistical}. The subsequent sections provide proofs for the propositions of \cref{section:bellman-rep}. The proofs of results in \cref{tab:examples} are provided in~\cref{appendix:proof-examples}. Finally, unless otherwise specified, the proofs of all new results in this section are presented in~\cref{appendix:proof-bellman-rep-props}. 

\subsection{Complexity measures}
\label{appendix:complexity-measures}

In this section, we revisit several well-known complexity measures of function classes, which are sufficient for bounding the decoupling dimension.

\begin{example}[Linear function class]\label{example:dc-linear}
Suppose that there exists $\phi:\cX\to\R^d$ and $\theta:\cF\to \R^d$ such that $f(x)=\iprod{\theta(f)}{\phi(x)}$ for all $(f,x)\in\cF\times\cX$. Then %
$\odimc(\cF,\gamma)\leq d$. %
\exend 
\end{example}
In particular, we always have $\odimc(\cF,\gamma)\leq\abs{\cX}$. \cref{example:dc-linear} can also be extended to generalized linear functions, as follows.
\begin{example}[Generalized linear function class]\label{example:dc-linear-gen}
Suppose that $\sigma:\R\to\R$ satisfies $\mu\leq \sigma'(t)\leq L$ for all $t\in\R$, and there exists $\phi:\cX\to\R^d$ and $\theta:\cF\to \R^d$ such that $f(x)=\sigma(\iprod{\theta(f)}{\phi(x)})$ for all $(f,x)\in\cF\times\cX$. Then $\odimc(\cF,\gamma)\leq \kappa^2d$, where $\kappa=L/\mu$.
\exend 
\end{example}

Next, we recall the definition of Eluder dimension~\citep{russo2013eluder} and star number~\citep{foster2020instance}.

\begin{definition}[Eluder dimension]
\label{definition:eluder}
The \emph{eluder dimension} $\eluder(\cF,\Delta)$ is the maximum of the length of sequence $(f_1,x_1), \cdots, (f_n,x_n)\in\cF\times\cX$ such that there is a $\Delta'\geq \Delta$, and
\begin{align*}
    \abs{f_i(x_i)}> \Delta', \qquad \sum_{j< i} \abs{f_i(x_j)}^2\leq (\Delta')^2, \qquad \forall i.
\end{align*}
We also define $\oeluder(\cF,\Delta)=\sup_{f\in\cF}\eluder(\cF-f,\Delta)$.
\end{definition}

\begin{definition}[Star number]
\label{definition:starn}
The \emph{star number} $\starn(\cF,\Delta)$ is the maximum of the length of sequence $(f_1,x_1), \cdots,(f_n,x_n)\in\cF\times\cX$ such that there is a $\Delta'\geq \Delta$, and 
\begin{align*}
    \abs{f_i(x_i)}> \Delta', \qquad \sum_{j\neq i} \abs{f_i(x_j)}^2\leq (\Delta')^2, \qquad \forall i.
\end{align*}
\end{definition}

\begin{example}[{\citet{foster2021statistical}}]\label{example:dc-to-es}
\label{def:star-eluder}
When $\cF\subset( \cX \rightarrow [-1,1])$, it holds that
\begin{align*}
\dimc(\cF,\gamma)\leq 24\inf_{\Delta> 0}\set{\min\set{ \starn^2(\cF, \Delta), \eluder(\cF, \Delta) }\log^2(\gamma\vee e)+\gamma\Delta}.
\end{align*}
\end{example}
More generally, the decoupling dimension can be bounded by the \textit{disagreement coefficient} introduced in \citep[Definition 6.3]{foster2021statistical}. The proof of~\cref{example:dc-to-es} (\cref{lemma:dc-to-disagree}, which follows directly from \cite{foster2021statistical}) along with some further discussions can be found in~\cref{appdx:bilinear-eluder}. %

Notice that the examples above relate the decoupling dimension to certain structural complexity measure of the function class $\cF$, i.e. they all hold for any $\cQ\subset \Delta(\cX)$. The following two examples illustrate that the decoupling dimension can also be bounded by certain complexity measures of the distribution class $\cQ$. 

\newcommand{\Cov}{C_{\sf cov}}
\begin{example}[Coverability, \citet{xie2022role}]\label{example:dc-to-cov}
For a class $\cQ$ of distributions over $\cX$, the \emph{coverability} of $\cQ$ is defined as
\begin{align*}
    \Cov(\cQ)=\inf_{\mu\in\Delta(\cX)}\sup_{q\in\cQ}\sup_{x\in\cX}\frac{q(x)}{\mu(x)}.
\end{align*}
Then for any function class $\cF$ over $\cX$, it holds that $\dimc(\cF,\cQ,\gamma)\leq \Cov(\cQ)$. \exend
\end{example}

\cref{example:dc-to-cov} recovers the learnability results under coverability (see discussion in \cref{appdx:bilinear-eluder}).

\begin{example}\label{example:dc-to-rank}
For a class $\cQ$ of distributions over $\cX$, $\rank(\cQ)$ is defined as the rank of the matrix $[q(x)]_{(q,x)}\in\R^{\cQ\times\cX}$. Then for any function class $\cF$ over $\cX$, it holds that $\dimc(\cF,\cQ,\gamma)\leq \rank(\cQ)$.
\exend
\end{example}

\subsection{Relation to known structural conditions}\label{appdx:bilinear-eluder}

In this section, we briefly discuss how \belrep~framework recovers several known structural conditions. %

\subsubsection{Relation to Bellman-Eluder dimension}\label{appdx:BE-dim}
Take Q-type Bellman-Eluder dimension~\citep{jin2021bellman} as an example. Using \cref{example:dc-to-es}, we know that for the \belrep~$\cG=\GBE$,
\begin{align*}
    \dimG(\cG,\gamma)\leqsim \max_{\oM,h} \dim_{\rm DE}(\cG_h^{\oM}, \cQ_h^{\oM}, 1/\gamma)\cdot\log^2(1/\gamma),
\end{align*}
where $\dim_{\rm DE}$ is the \emph{distribution Eluder dimension} studied in \citet{jin2021bellman}. From the definition of $\GBE$ (\cref{def:bellman-err}), it is not hard to see the quantity $\dim_{\rm DE}(\cG_h^{\oM}, \cQ_h^{\oM}, 1/\gamma)$ is indeed the Q-type (model-induced) Bellman Eluder dimension with respect to model $\oM$~\citep{jin2021bellman}.

\subsubsection{Relation to Model-based Bilinear Class}\label{appdx:bilinear}
Consider the following model-based version of Bilinear class~\citep{du2021bilinear}, which is introduced in \citet{foster2021statistical}.

\newcommand{\piestt}{\pi^{\rm est}}
\newcommand{\lest}{\ell^{\rm est}}
\begin{definition}[Model-based Bilinear class]\label{def:bilinear}
A MDP model class $\cM$ is a $d$-dimensional bilinear class if there exists a collection of maps $\{X_h:\cM\times\cM\to\R^d\}_{h\in[H]}$, $\{W_h:\cM\times\cM\to\R^d\}_{h\in[H]}$, a class of \emph{estimation policies} $\{\piestt_M\}_{M\in\cM}$, a collection of estimation functional $\{\lest_{M,h}(\cdot;\cdot)\}_{M\in\cM,h\in[H]}$, such that the following holds:

(1) For each $M, \oM\in\cM$, $h\in[H]$,
\begin{align*}
    \abs{\EE^{\oM,\pi_M}\brac{Q^{M, \star}_h(s_h, a_h) - r_h - V_{h+1}^{M, \star}(s_{h+1})}}
    \leq \abs{\iprod{X_h(M;\oM)}{W_h(M;\oM)}},
\end{align*}
and $W_h(\oM;\oM)=0$ always.

(2) For each $M, M', \oM\in\cM$, $h\in[H]$, it holds that
\begin{align*}
    \iprod{X_h(M;\oM)}{W_h(M';\oM)}=\EE^{\oM,\pi_M\circ_h\piestt_M}\brac{\lest_{M,h}(M';s_h,a_h,r_h,s_{h+1})},
\end{align*}
where for policy $\pi, \pi'$, the policy $\pi\circ_h\pi'$ is given by executing $\pi$ for the first $h$ steps and then following $\pi'$ afterwards.
\end{definition}
Given a Bilinear class $\cM$ described above, we immediately have a \belrep~$\cG$ given as follows:
\begin{itemize}%
    \item (Index set) For each $h\in[H]$, $\cT_h=\cM$, and $\dtr(M;\oM)=\delta_M\in\Delta(\cM)$.
    \item For each $M, M', \oM\in\cM$, $h\in[H]$, 
    \begin{align*}
        \cE^{M;\oM}(M')=\iprod{X_h(M';\oM)}{W_h(M;\oM)}.
    \end{align*}
    \item The exploration policies are given by $\piest_M=\Unif(\{ \pi_M\circ_h\piestt_M \}_{h\in[H]})$, and $L=2HL_0^2$, where $L_0$ is a upper bound of $\abs{\lest_{M,h}}$ for all $M\in\cM, h\in[H]$.
\end{itemize}
For $\cG$ described as above, we have $\dimG(\cG,\gamma)\leq d$ (\cref{example:dc-linear}). Therefore, \belrep~indeed encompasses the model-based version of Bilinear class.

\subsubsection{Relation to coverability} The following definition of MDP with coverability is introduced in \citet{xie2022role}.
\begin{definition}[Coverability]\label{def:mdp-cov}
For a MDP $M$, we define the coverability coefficient of $M$ as
\begin{align*}
    C(M)=\inf_{\mu_1,\cdots,\mu_H\in\Delta(\cS\times\cA)}\sup_{\pi\in\Pi,s\in\cS,a\in\cA,h\in[H]} \frac{\PP^{M,\pi}(s_h=s,a_h=a)}{\mu_h(s,a)}.
\end{align*}
\end{definition}
Suppose that $\cM$ is a class of MDPs such that $C(M)\leq C$ for all $M\in\cM$. Then, by definition, the \belrep~$\GBE$ of $\cM$ (\cref{def:bellman-err}) has $\Cov(\cQ_h^{\oM})\leq C$ for all $\oM\in\cM, h\in[H]$. Thus, by \cref{example:dc-to-cov}, we have $\dimG(\GBE,\gamma)\leq C$, and hence the \belrep~framework indeed recovers the results of the coverability~\citep{xie2022role}.  
In particular, we have $\dec_\gamma(\cM)\leq CH^2/\gamma$, and thus \etod~achieves a regret bound of $\cO(\sqrt{CH^2\log\abs{\cM}T})$, matching the result of \citet{xie2022role}, except that the upper bound there depends on $\log\abs{\cF}$, the log-cardinality of certain value function class $\cF$.
A more detailed discussion is deferred to \cref{appendix:proof-cov-mdp}, where we show that a bounded coverability also implies bounded DECs of reward-free learning, model estimation and preference-based learning.

\subsubsection{Relation to Bellman representation}\label{appdx:relation-bel-rep}
\Belrep~(\cref{definition:err-rep}) can be regarded as a generalization of the Bellman representation~\citep[Definition F.1]{foster2021statistical}, by considering $\cT_h=\cM$, $\dtr(M;\oM)=\delta_{M}$, and choosing $\cerr^{M';\oM}_h(M)$ to be the expectation of the discrepancy function $\ell_M(M';s_h,a_h,r_h,s_{h+1})$ considered there.

In~\citet{foster2021statistical}, the complexity of a \belrep~is measured in terms of \textit{disagreement coefficient}, which can be upper bounded by Eluder dimension or star number. 
\begin{definition}
The disagreement coefficient of a function class $\cF\subset( \Pi \rightarrow [-1,1])$ is defined as
$$
\btheta\left(\cF, \Delta_0, \varepsilon_0; \rho \right)=\sup _{\Delta \geq \Delta_0, \varepsilon \geq \varepsilon_0}\left\{\frac{\Delta^2}{\varepsilon^2} \cdot \mathbb{P}_{\pi \sim \rho}\left(\exists f \in \cF:|f(\pi)|>\Delta, \mathbb{E}_{\pi \sim \rho}\left[f^2(\pi)\right] \leq \varepsilon^2\right)\right\} \vee 1.
$$
By \citet[Lemma 6.1]{foster2021statistical}, for $\Delta, \varepsilon>0, \rho\in\Delta(\Pi)$, it holds that
$$
\btheta(\cF, \Delta, \varepsilon; \rho) \leq 4\min\set{ \starn^2(\cF, \Delta), \eluder(\cF, \Delta) }.
$$
\end{definition}

It turns out that our decoupling dimension can be upper bounded by the disagreement coefficient: the following result follows immediately from \citet[Lemma E.3]{foster2021statistical}. 
\begin{lemma}\label{lemma:dc-to-disagree}
For function class $\cF\subset( \Pi \rightarrow [-1,1])$, we have
$$
\dimc(\cF,\gamma)
\leq \inf _{\Delta>0}\left\{2\gamma \Delta+6 \btheta\left(\cF, \Delta, \gamma^{-1}\right) \log ^2( \gamma \vee e)\right\},
$$
where $\btheta\left(\cF, \Delta, \gamma^{-1}\right)\defeq\sup_{\rho\in\Delta(\Pi)} \btheta(\cF, \Delta, \varepsilon; \rho)$.
\end{lemma}
\cref{example:dc-to-es} is now a direct corollary of \cref{lemma:dc-to-disagree}.

\subsection{Proof of Example~\ref{def:bellman-err}}\label{appendix:GBE-proof}
We first prove \cref{eqn:Bellman-decomp}, as follows:
\begin{align*}
    &~\sum_{h=1}^H \EE^{\oM,\pi_M}\brac{ Q^{M, \star}_h(s_h, a_h) - [\T^{\oM}_h V^{M,\star}_{h+1}](s_h,a_h) } \\
    =&~
    \sum_{h=1}^H \EE^{\oM,\pi_M}\brac{ Q^{M, \star}_h(s_h, a_h) - r_h - V^{M,\star}_{h+1}(s_{h+1}) } \\
    =&~
    \sum_{h=1}^H \EE^{\oM,\pi_M}\brac{ V^{M, \star}_h(s_h) - V^{M,\star}_{h+1}(s_{h+1}) } - \EE^{\oM,\pi_M}\brac{\sum_{h=1}^H r_h} \\
    =&~ 
    \EE\brac{V_1^{M,\star}(s_1)} - f^{\oM}(\pi_M) \\
    =&~
    f^M(\pi_M)-f^{\oM}(\pi_M),
\end{align*}
where we use the assumption that $M$ and $\oM$ has the same initial distribution.

We next prove the condition \cref{eqn:err-rep-decouple-2} for $\GBE$. By definition,
\begin{align*}
    \abs{\cerr_h^{M; \oM}(s_h,a_h)} =&~ \abs{Q^{M, \star}_h(s_h, a_h) - [\T^{\oM}_h V^{M,\star}](s_h,a_h)} \\
    =&~ \abs{[\T^{M}_h V^{M,\star}](s_h,a_h) - [\T^{\oM}_h V^{M,\star}](s_h,a_h)} \\
    \leq&~ \DTV{\PP^M_h(\cdot|s_h,a_h),\PP^\oM_h(\cdot|s_h,a_h)}+\abs{R^M_h(s_h,a_h)-R^\oM_h(s_h,a_h)}.
\end{align*}
Therefore,
\begin{align*}
    \EE^{\oM,\pi} \abs{\cerr_h^{M; \oM}(s_h,a_h)}^2
    \leq&~ 2\EE^{\oM,\pi}\DTV{\PP^M_h(\cdot|s_h,a_h),\PP^\oM_h(\cdot|s_h,a_h)}^2+2\EE^{\oM,\pi}\abs{R^M_h(s_h,a_h)-R^\oM_h(s_h,a_h)}^2.
\end{align*}
By $\dTV\leq \dH$ and \cref{lemma:Hellinger-cond}, we have
\begin{align*}
    \EE^{\oM,\pi}\DTV{\PP^M_h(\cdot|s_h,a_h),\PP^\oM_h(\cdot|s_h,a_h)}^2
    \leq 2\DH{ M(\pi), \oM(\pi) }.
\end{align*}
Therefore, taking summation over $h\in[H]$ yields
\begin{align*}
    \sum_{h=1}^H \EE^{\oM,\pi} \abs{\cerr_h^{M; \oM}(s_h,a_h)}^2
    \leq&~ 4H \DH{ M(\pi), \oM(\pi) } + 2 \EE^{\oM,\pi} \ltwo{\bR^M(o)-\bR^\oM(o)}^2 \\
    \leq&~ 4H\tDH{\oM(\pi),M(\pi)}.
\end{align*}
Combining the results above completes the proof.
\qed

\subsection{Proof of Proposition~\ref{prop:belrep-pac}}\label{appdx:proof-belrep-pac}

\cref{prop:belrep-pac} is an immediate corollary of the combination of \cref{prop:dec-to-psc} and \cref{prop:psc-belrep}. As a remark, a similar strategy of using probability matching to bound DEC is also adopted in \citet{foster2021statistical}.
\qed

\subsection{Proof of Proposition~\ref{prop:belrep-am}}\label{appendix:bellman-rep-medec-dcpl}

By definition, we have
\begin{align*}
&~\mdec(\cM,\omu)\\
\defeq&~ \inf_{\pexp\in\Delta(\Pi),\muout\in\Delta(\cM)}\sup_{M\in\cM,\opi\in\Pi}\E_{\hM \sim \muout}  \brac{ \wDTV{ M(\opi), \hM(\opi)} } - \gamma \E_{\pi \sim \pexp}\E_{\oM\sim \omu}\brac{ \tDH{M(\pi), \oM(\pi)} }\\
\leq&~ \inf_{\pexp\in\Delta(\Pi)}\sup_{M\in\cM,\opi\in\Pi}\E_{\oM \sim \omu}\brac{ \wDTV{ M(\opi), \oM(\opi)}  - \gamma \E_{\pi \sim \pexp}\brac{ \tDH{M(\pi), \oM(\pi)} } }\\
=&~ \sup_{\nu\in\Delta(\cM\times\Pi)}\inf_{\pexp\in\Delta(\Pi)} \E_{\oM \sim \omu}\brac{ \EE_{(M,\opi)\sim\nu}\brac{\wDTV{ M(\opi), \oM(\opi)}}  - \gamma \EE_{M\sim\nu}\E_{\pi \sim \pexp}\brac{ \tDH{M(\pi), \oM(\pi)} } },
\end{align*}
where in the first inequality we take $\muout=\omu$, and the last line is due to strong duality. Now, for any $\nu\in\Delta(\cM\times\Pi)$, we have
\begin{align*}
\EE_{(M,\opi)\sim\nu}\brac{\wDTV{ M(\opi), \oM(\opi)}}
\leq&~
\sum_{h=1}^H 
\EE_{(M,\opi)\sim\nu}\EE_{\tau_h\sim \dtr_h(\opi;\oM)}\brac{\abs{\cE^{M; \oM}_h(\tau_h)}}\\
=&~
\sum_{h=1}^H 
\EE_{(M,\tau_h)\sim\nu_h}\brac{\abs{\cE^{M; \oM}_h(\tau_h)}}\\
\leq&~
\sum_{h=1}^H 
\eta\EE_{M\sim \nu_h,\tau_h\sim\nu_h}\abs{\cE^{M; \oM}_h(\tau_h)}^2+\eta^{-1}\dimc(\cG^{\oM}_h,\eta),
\end{align*}
where we define $\nu_h\in\Delta(\cM\times\cT_h)$ as the joint distribution of $(M,\tau_h)$ where $(M,\opi)\sim\nu$, $\tau_h\sim \dtr_h(\opi;\oM)$. By definition, $M\sim \nu\stackrel{d}{=} M\sim \nu_h$, and hence
\begin{align*}
\sum_h\EE_{M\sim \nu_h,\tau_h\sim\nu_h}\abs{\cE^{M; \oM}_h(\tau_h)}^2
=&~
\EE_{M\sim \nu}\EE_{\pi\sim\nu} \brac{ \sum_h\EE_{\tau_h\sim \dtr_h(\opi;\oM)}\abs{\cE^{M; \oM}_h(\tau_h)}^2 } \\
\leq&~
L\EE_{M\sim \nu}\EE_{\pi\sim\nu}\brac{\tDH{M(\piest), \oM(\piest)}}.
\end{align*}
Therefore, taking $\eta=\gamma/L$ gives
\begin{align*}
\EE_{(M,\opi)\sim\nu}\brac{\wDTV{ M(\opi), \oM(\opi)}}
\leq&~ \gamma \EE_{M\sim \nu}\EE_{\pi\sim\nu}\brac{\tDH{M(\piest), \oM(\piest)}} + \gamma^{-1}LH\max_h \dimc(\cG^{\oM}_h,\gamma/L).
\end{align*}
To finalize the proof, we note that
\begin{align*}
&~\mdec_\gamma(\cM,\omu)\\
\leq&~ \sup_{\nu\in\Delta(\cM\times\Pi)}\inf_{\pexp\in\Delta(\Pi)} \E_{\oM \sim \omu}\brac{ \EE_{(M,\opi)\sim\nu}\brac{\wDTV{ M(\opi), \oM(\opi)}}  - \gamma \EE_{M\sim\nu}\E_{\pi \sim \pexp}\brac{ \tDH{M(\pi), \oM(\pi)} } }\\
\leq&~ \sup_{\nu\in\Delta(\cM\times\Pi)} \E_{\oM \sim \omu}\brac{ \EE_{(M,\opi)\sim\nu}\brac{\wDTV{ M(\opi), \oM(\opi)}}  - \gamma \EE_{M\sim\nu}\E_{\pi \sim \nu}\brac{ \tDH{M(\piest), \oM(\piest)} } }\\
\leq&~
\gamma^{-1}LH\dimG(\cG,\gamma/L).
\end{align*}
This is the desired result.
\qed

\subsection{Proof of Proposition~\ref{prop:belrep-pb}}\label{appdx:proof-pbrl-belrep}

Before presenting the proof, we first remark that \cref{prop:belrep-pb} can be proven by utilizing \cref{thm:PbRL-to-RL} and establishing \mbelrep~for each model class $\cM_q$ ($q\in\Delta(\Pi)$). However, the following proof is conceptually simpler, as it is analogous to our proof of \cref{prop:belrep-pac} (by combining \cref{prop:dec-to-psc} and \cref{prop:psc-belrep}).

\begin{proof}
By definition, it suffices to bound $\decpb_\gamma(\cMpb, \omu)$ for any fixed $\omu\in\Delta(\cMpb)$. For notational simplicity, for $\bpi=(\pi_1,\pi_2)\in\Piab$, we denote
\begin{align*}
    \Gap_M(\pi_\star,\bpi)
    =
    \Cmp^M(\pi_\star,\pi_1)+\Cmp^M(\pi_\star,\pi_2)-1.
\end{align*}
Then we can rewrite
\begin{align*}
    \pbdec_\gamma(\cMpb, \omu)  \defeq \inf_{p\in\Delta(\Piab)} \sup_{\substack{M\in\cMpb\\\pi_\star\in\Pi}} \E_{\bpi\sim p} \E_{\oM\sim\omu}\brac{ \Gap_M(\pi_\star,\bpi) - \gamma\dH^2(M(\bpi), \oM(\bpi))}.
\end{align*}
By strong duality (\cref{thm:strong-dual}), we have
\begin{align*}
    \pbdec_\gamma(\cMpb, \omu)  
    & = \inf_{p\in\Delta(\Piab)} \sup_{\mu\in\Delta(\cMpb\times\Pia)} \E_{(M,\pi_\star)\sim\mu,\oM\sim\omu}\E_{\bpi\sim p} \brac{ \Gap_M(\pi_\star,\bpi) - \gamma\dH^2(M(\bpi), \oM(\bpi))}\\
    & = \sup_{\mu\in\Delta(\cMpb\times\Pia)} \inf_{p\in\Delta(\Piab)}  \E_{(M,\pi_\star)\sim\mu,\oM\sim\omu}\E_{\bpi\sim p} \brac{ \Gap_M(\pi_\star,\bpi) - \gamma\dH^2(M(\bpi), \oM(\bpi))}.
\end{align*}
Now, fix any $\mu\in\Delta(\cMpb\times\Pi)$, we pick $p\in\Delta(\Piab)$ to be the distribution of $\bpi=(\pi_1,\pi_2)$, $\pi_1\sim \mu, \pi_2\sim\mu$. For this choice of $p$, we can compute the quantity inside the sup-inf above as follows. First, by probability matching, we have
\begin{align*}
    \E_{(M,\pi_\star)\sim\mu}\E_{\bpi\sim p} \brac{ \Gap_M(\pi_\star,\bpi) }
    =&~
    \E_{(M,\pi_\star)\sim\mu}\E_{\bpi\sim p} \brac{ \Cmp^M(\pi_\star,\pi_1)+\Cmp^M(\pi_\star,\pi_2) } - 1\\
    =&~
    2\E_{(M,\pi_\star)\sim\mu}\E_{\pi\sim \mu} \brac{ \Cmp^M(\pi_\star,\pi) } - 1\\
    =&~
    2\E_{(M,\pi_\star)\sim\mu}\E_{\pi\sim \mu} \brac{ \Cmp^M(\pi_\star,\pi) } - 2\E_{\pi_\star\sim\mu,\pi\sim \mu} \brac{ \Cmp^\oM(\pi_\star,\pi) }\\
    =&~
    2\E_{(M,\pi_\star)\sim\mu}\E_{\pi\sim \mu} \brac{ \Cmp^M(\pi_\star,\pi)-\Cmp^\oM(\pi_\star,\pi) }\\
    \leq&~
    2\E_{(M,\pi_\star)\sim\mu}\E_{\pi\sim \mu} \brac{ \DTV{M(\pi_\star,\pi),\oM(\pi_\star,\pi)} },
\end{align*}
where the third line is because for any $q\in\Delta(\Pia)$, we have
\begin{align*}
    \EE_{\pi_1\sim q, \pi_2\sim q}\brac{ \Cmp^\oM(\pi_1,\pi_2)} = \frac12
\end{align*}
by symmetry.
Therefore,
\begin{align*}
    \pbdec_\gamma(\cMpb,\omu) & = \sup_{\mu\in\Delta(\cMpb\times\Pia)} \inf_{p\in\Delta(\Piab)}  \E_{(M,\pi_\star)\sim\mu,\oM\sim\omu}\E_{\bpi\sim p} \brac{ \Gap_M(\pi_\star,\bpi) - \gamma\dH^2(M(\bpi), \oM(\bpi))}\\
    &\leq 
    \sup_{\mu\in\Delta(\cMpb\times\Pia)} 2\E_{(M,\pi_\star)\sim\mu, \pi\sim \mu, \oM\sim\omu} \brac{ \DTV{M(\pi_\star,\pi),\oM(\pi_\star,\pi)} }\\
    &\qquad\qquad\qquad\qquad\qquad-\gamma \E_{M\sim\mu,\pi_1\sim\mu,\pi_2\sim\mu,\oM\sim\omu}\brac{ \DH{M(\pi_1,\pi_2), \oM(\pi_1,\pi_2)}}\\
    &\leq \sup_{\oM\in\cMpb}
    \sup_{\mu\in\Delta(\cMpb\times\Pia)} 2\E_{(M,\pi_\star)\sim\mu, \pi\sim \mu} \brac{ \DTV{M(\pi_\star,\pi),\oM(\pi_\star,\pi)} }\\
    &\qquad\qquad\qquad\qquad\qquad-\gamma \E_{M\sim\mu,\pi_1\sim\mu,\pi_2\sim\mu}\brac{ \DH{M(\pi_1,\pi_2), \oM(\pi_1,\pi_2)}}.
\end{align*}
Notice that
\begin{align*}
    \DTV{M(\pi_\star,\pi),\oM(\pi_\star,\pi)}
    \leq&~ \DTV{ \Mtraj(\pi_\star), \oMtraj(\pi_\star)} + \DTV{ \Mtraj(\pi), \oMtraj(\pi)} \\
    &~+ \EE_{(\tau_1,\tau_2)\sim \oM(\pi_\star,\pi)}\abs{\Cmp^M(\tau_1,\tau_2)-\Cmp^\oM(\tau_1,\tau_2)}.
\end{align*}
By the proof of \cref{prop:belrep-am} (\cref{appendix:bellman-rep-medec-dcpl}), we have
\begin{align}\label{eqn:pbrl-proof-part-1}
\begin{aligned}
    \E_{(M,\pi_\star)\sim\mu} \brac{ \DTV{ \Mtraj(\pi_\star), \oMtraj(\pi_\star)} } 
    \leq&~ \frac{LH\dimG(\cG,\eta_1)}{\eta_1} + \eta_1\E_{M\sim\mu,\pi_\star\sim\mu} \brac{ \DH{ \Mtraj(\pi_\star), \oMtraj(\pi_\star)} }.
\end{aligned}
\end{align}
Second, by AM-GM inequality, we have
\begin{align}\label{eqn:pbrl-proof-part-2}
\begin{aligned}
    \E_{M\sim \mu,\pi\sim\mu}\brac{ \DTV{ \Mtraj(\pi), \oMtraj(\pi)} } 
    \leq&~ \E_{M\sim \mu,\pi\sim\mu} \brac{ \dH\paren{ \Mtraj(\pi), \oMtraj(\pi)} }\\
    \leq&~ \frac{1}{\eta_2}+\eta_2
    \E_{M\sim\mu,\pi\sim\mu} \brac{ \DH{ \Mtraj(\pi), \oMtraj(\pi)} }\\
\end{aligned}
\end{align}
Third, we can consider $\nu\in\Delta(\cF\times\cT)$ that equals distribution of $(\Cmp^M,(\tau_1,\tau_2))$ with $(M,\pi_\star)\sim\mu,\pi\sim\mu, (\tau_1,\tau_2)\sim \oM(\pi_\star,\pi)$. Then
\begin{align}\label{eqn:pbrl-proof-part-3}
\begin{aligned}
    \MoveEqLeft
    \E_{(M,\pi_\star)\sim\mu,\pi\sim\mu} \EE_{(\tau_1,\tau_2)\sim \oM(\pi_\star,\pi)}\abs{\Cmp^M(\tau_1,\tau_2)-\Cmp^\oM(\tau_1,\tau_2)} 
    = \EE_{(\Cmp,\tau)\sim \nu} \abs{\Cmp(\tau)-\Cmp^\oM(\tau)} \\
    \leq&~ \frac{\dimc(\cC-\Cmp^\oM,\eta_3)}{\eta_3}
    +\eta_3\EE_{\Cmp\sim \nu,\tau\sim \nu} \abs{\Cmp(\tau)-\Cmp^\oM(\tau)}^2 \\
    =&~  \frac{\dimc(\cC-\Cmp^\oM,\eta_3)}{\eta_3}
    +\eta_3\EE_{M\sim \mu}\EE_{\bpi=(\pi_1,\pi_2)\sim\mu, (\tau_1,\tau_2)\sim \oM(\bpi)} \abs{\Cmp^M(\tau_1,\tau_2)-\Cmp^\oM(\tau_1,\tau_2)}^2,
\end{aligned}
\end{align}
where the inequality follows from the definition of $\dimc(\cC-\Cmp^\oM,\eta_3)$.

On the other hand, for $M\in\cMpb$, $\bpi=(\pi_1,\pi_2)\in\Piab$, by data-processing inequality, it holds that
\begin{align}\label{eqn:pbrl-hell}
    \DH{M(\bpi),\oM(\bpi)}
    \geq \DH{\Mtraj(\pi_i),\oMtraj(\pi_i)}, \qquad i=1,2.
\end{align}
By \cref{lemma:Hellinger-cond}, we also have
\begin{align}\label{eqn:pbrl-hell-2}
    2\DH{M(\bpi),\oM(\bpi)}
    \geq& \EE_{(\tau_1,\tau_2)\sim \oM(\bpi)} \brac{ \DH{\MP(b=\cdot|\tau_1,\tau_2), \oMP(b=\cdot|\tau_1,\tau_2)} }\\
    =&
    \EE_{(\tau_1,\tau_2)\sim \oM(\bpi)} \brac{ \DH{\Cmp^M(\tau_1,\tau_2), \Cmp^\oM(\tau_1,\tau_2)} }\\
    \geq&
    \EE_{(\tau_1,\tau_2)\sim \oM(\bpi)} \brac{ \abs{\Cmp^M(\tau_1,\tau_2)-\Cmp^\oM(\tau_1,\tau_2)}^2 }.
\end{align}
Combining \eqref{eqn:pbrl-proof-part-1} \eqref{eqn:pbrl-proof-part-2} \eqref{eqn:pbrl-proof-part-3} \cref{eqn:pbrl-hell} \cref{eqn:pbrl-hell-2}, we obtain that
\begin{align*}
    \E_{(M,\pi_\star)\sim\mu, \pi\sim \mu} \brac{ \DTV{M(\pi_\star,\pi),\oM(\pi_\star,\pi)} }
    \leq &~ \frac{LH\dimG(\cG,\eta_1)}{\eta_1}+\frac{1}{\eta_2}+\frac{\odimc(\cC,\eta_3)}{\eta_3} \\
    &~+ (\eta_1+\eta_2+\eta_3) \E_{M\sim\mu,\bpi=(\pi_1,\pi_2)\sim\mu} \brac{ \DH{ M(\bpi), \oM(\bpi)} }.
\end{align*}
We now take $\eta_1=\eta_2=\eta_3=\gamma/6$. By the arbitrariness of $\oM$, the proof is completed.
\end{proof}

\subsection{Additional examples from Table~\ref{tab:examples}}\label{} 
In this section, we briefly discuss our results for tabular MDPs and linear MDPs (as described in \cref{tab:examples}).

\begin{example}[Tabular MDPs]\label{example:tabular-demo}
For tabular MDPs, our framework implies a regret bound of $\tO(\sqrt{|\cS|^3|\cA|^2H^3T})$ and reward-free bound $\tO(|\cS|^3|\cA|^2H^3/\eps^2)$ (cf. \cref{appendix:proof-tabular}). These bounds are worse than the optimal $\tO(\sqrt{|\cS||\cA|HT})$ regret bound~\citep{azar2017minimax} %
and $\tO(|\cS|^2|\cA|\poly(H)/\eps^2)$ reward-free bound~\citep{jin2020reward}, which is expected as our unified algorithms do not utilize the specific structure of tabular problems, and is a worthy direction for future study.
\end{example}

\begin{example}[Linear MDPs, \citet{jin2020provably}]\label{example:linear-mdp-demo}
A MDP $M$ is called a linear MDP (with respect to a known $d$-dimensional feature map $\phi$) if there exists maps $\mu_h^M:\cS\to\R^d$, such that for the given features $\phi_h:\cS\times\cA\to\R^d$, it holds that
$$
\PP_h^M(s'|s,a)=\iprod{\mu_h^M(s')}{\phi_h(s,a)}.
$$
\end{example}
For linear MDPs, our framework implies a sample complexity of $\tO(dH^2\log|\cM|/\eps^2)$ for learning an $\eps$-approximate model, which is new in this setting. 

\subsection{Proofs for problem classes in Table~\ref{tab:examples}}
\label{appendix:proof-examples}

\newcommand{\alldec}{\set{\text{dec, edec, rfdec, amdec, pbdec}}_{\gamma}(\cM)}
\newcommand{\allexpdec}{\set{\text{edec, rfdec, amdec}}_{\gamma}(\cM)}
\newcommand{\allregdec}{\set{\text{dec, pbdec}}_{\gamma}(\cM)}

For each problem class in \cref{tab:examples}, the proof of its results is presented in the corresponding subsections. The proofs follow the following ``streamline'': (1) identify a suitable (strong) \belrep~$\cG$ of the model class, (2) upper bound the decoupling dimension $\dimG(\cG,\gamma)$ by investigating the structure of $\cG$ (e.g. linear structure or low coverability), and (3) apply \cref{prop:belrep-pac} (or \cref{prop:belrep-rf}, \cref{prop:belrep-am}, \cref{prop:belrep-pb}) to conclude the desired results. It is worth noting that our bounds on (strong) \belrep~do not require the policy to be Markov. Therefore, we may work with any policy class $\Pi$ that possibly contains general history-dependent policies, e.g. in the setting of PbRL with general trajectory preferences and partially observable RL. 

As a remark, the upper bounds of PBDEC (cf. \cref{prop:belrep-pb}) additionally have an extra term $\tbO{\frac{d_{\Cmp}}{\gamma}}$, as long as $\odimc(\cC,\gamma)\leq \tO(d_{\Cmp})$ for the corresponding comparison function class $\cC$. To simplify the presentation, we assume that $\odimc(\cC,\gamma)$ is of lower order compared to $\dimG(\cG,\gamma)$ for the remainder of this section. %

\paragraph{A useful \mbelrep}
In the following, we define a \mbelrep~for any MDP model class, given by per-state TV distance, which is useful for the proofs in this section.

\newcommand{\GTV}{\cG_{\rm TV}}
\begin{definition}
\label{example:mer}
For a model class $\cM$ of MDPs, we consider the following \mbelrep~of $\cM$, which we term as $\GTV$: 
\begin{itemize}%
\item (Index set) For each $h\in[H]$, the index set is $\cT_h=\cS\times\cA$.
\item (Error functions and distributions) For each $M,\oM\in\cM$,
\begin{align*}
    \dtr_h(M;\oM)\defeq&~\PP^{\oM,\pi_M}(s_h=\cdot,a_h=\cdot)\in\Delta(\cS\times\cA), \\
    \cerr_h^{M; \oM}(s_h,a_h) \defeq&~ \DTV{ \PP^M_h(\cdot|s_h,a_h), \PP^{\oM}_h(\cdot|s_h,a_h) }+ \abs{R^M_h(s_h,a_h) - R^{\oM}_h(s_h,a_h)}.
\end{align*}
\item The exploration policies are given by $\piest_M=\pi_M$ and the constant $L=4H$.
\end{itemize}
\end{definition}

\subsubsection{Example \ref{example:tabular-demo}: tabular MDPs}
\label{appendix:proof-tabular}

Consider the model class $\cM$ of tabular MDPs with state space $\cS$, action space $\cA$, its covering number $\log\cN(\cM,\rho)=\tbO{S^2AH}$, where $S=\abs{\cS}$ and $A=\abs{\cA}$ (see the following \cref{example:cover-tabular}).
Applying \cref{example:dc-linear}, we know that its \belrep~$\GBE$ (\cref{def:bellman-err}) admits $\dimG(\GBE,\gamma)\leq SA$, and its \mbelrep~$\GTV$ (\cref{example:mer}) also admits $\dimG(\GTV,\gamma)\leq SA$. Therefore, we have
\begin{align*}
    \alldec\leqsim \frac{SAH^2}{\gamma}, \qquad 
    \forall \gamma>0.
\end{align*}
The results of tabular MDPs in \cref{tab:examples} now follow. \qed

In the following, we demonstrate briefly how to construct an optimistic covering of the class of tabular MDPs. Without loss of generality, we only cover the class of transition dynamic $\Pm$.
\begin{example}[Optimistic covering of tabular MDP]\label{example:cover-tabular}
Consider $\cM$, the class of MDPs with $S$ states, $A$ actions, $H$ steps. Fix a $\rho_1\in(0,1]$, and $\rho=\rho_1^2/eHS$. For $M\in\cM$, we compute its $\rho_1$-optimistic likelihood function as follows: define
\begin{align}\label{eqn:opt-cover-tabular}
\tPP^{M}_h(s'|s,a)\defeq \rho\ceil{\frac{1}{\rho}\PP^M_h(s'|s,a)}, \qquad
\tPP^M_1(s)\defeq \rho\ceil{\frac{1}{\rho}\PP^M_1(s)},
\end{align}
and for any policy $\pi$, we let
\begin{align}\label{eqn:proof-cover-tabular-opt-like}
\begin{split}
&\tPP^{M,\pi}(s_1,a_1,\cdots,s_{H},a_{H}) \\
:=&\tPP^M_1(s_1)\tPP^M_1(s_2|s_1,a_1)\cdots\tPP^M_{H-1}(s_H|s_{H-1},a_{H-1})\times \pi(s_1,a_1,\cdots,s_H,a_H)\\
=&\pi(s_1,a_1,\cdots,s_H,a_H)\times \tPP^M_1(s_1)\times\prod_{h=1}^{H-1}\tPP^M(s_{h+1}|s_h,a_h),
\end{split}
\end{align}
where for general (possibly non-Markovian) policy $\pi$, we write 
\begin{align}\label{eqn:proof-cover-tabular-policy}
    \pi(s_1,a_1,\cdots,s_H,a_H)\defeq \prod_{h=1}^{H}\pi(a_h|s_{1:h},a_{1:h-1}).
\end{align}
A direct calculation shows that $\tPP^{M,\pi}\geq \PP^{M,\pi}$ for all $\pi$, and $\|\tPP^{M,\pi}(\cdot)-\PP^{M,\pi}(\cdot)\|_1\leq \rho_1^2$. Clearly, there are at most $\ceil{1/\rho}^{S^2AH}$ different optimistic likelihood functions defined by \eqref{eqn:opt-cover-tabular}, and we can form $\cM_0$ by picking a representative in $\cM$ for each optimistic likelihood function (if possible). Then, $\log\abs{\cM_0}=\bigO{S^2AH\log(SH/\rho_1)}$. \exend
\end{example}

\subsubsection{Example \ref{example:linear-mixture-demo}: linear mixture MDPs}
\label{appendix:proof-linear-mixture}

Following the commonly used definition of linear mixture MDPs~\citep{chen2021near}, we also assume that the mean reward function has the form
\begin{align}
\label{equation:linear-mixture-linear-reward}
    R_h^M(s,a)=\iprod{\theta_h^M}{\phi_h'(s,a)}, \qquad \forall (h,s,a)\in[H]\times\cS\times\cA,
\end{align}
where $\phi_h':\cS\times\cA\to\R^d$ are also known maps. 
We remark that the linear reward assumption~\cref{equation:linear-mixture-linear-reward} is needed for the no-regret and model-estimation settings, but not needed for the reward-free setting (where the rewards can be arbitrary measurable functions).

Now, suppose that $\cM$ is a class of linear mixture MDP models with the given feature map $\phi$. 
We then have $\log\cN(\cM,\rho)=\tbO{dH}$ (see \cref{prop:cover-linear-mixture}). It remains to bound the decoupling dimension of a \mbelrep~of $\cM$.

\paragraph{\mBelrep~of $\cM$}
Consider the \mbelrep~$\GTV$ (\cref{example:mer}) of $\cM$.
By definition,
\begin{align*}
    \cE^{M;\oM}_h(s_h,a_h)
    =&\DTV{ \PP^M_h(\cdot|s_h,a_h), \PP^{\oM}_h(\cdot|s_h,a_h) }+ \abs{R^M_h(s_h,a_h) - R^{\oM}_h(s_h,a_h)}\\
    =&\abs{\iprod{\theta_h^M-\theta_h^{\oM}}{\phi_h'(s_h,a_h)}}+\sum_{s'} \abs{\iprod{\theta_h^M-\theta_h^{\oM}}{\phi_h(s'|s_h,a_h)}}.
\end{align*}
Therefore, \cref{cor:dc-linear-max}(2) gives that $\dimG(\cG^{\oM}_h,\gamma)\leq d$ for all $h\in[H]$. Thus, 
\begin{align*}
    \alldec\leqsim \frac{dH^2}{\gamma}.
\end{align*}
\qed

The following proposition provides an upper bound on the covering number of $\cM$ via a concrete construction. We assume that the initial state distribution is known. %
\begin{proposition}[Optimistic covering for linear mixture MDPs]
\label{prop:cover-linear-mixture}
Suppose that $\cM$ consists of linear mixture MDPs with $d$-dimensional feature map $\phi$. 
Further assume that $\ltwo{\sum_{s'}\phi_h(s'|s,a)V(s')}\leq 1$ for all $V:\cS\to[0,1]$ and tuple $(s,a,h)\in\cS\times\cA\times[H]$ (as in \citet{ayoub2020model}), and for any $M\in\cM$, $M$ is parameterized by parameter $(\theta_h)_h$ such that $\ltwo{\theta_h}\leq B$ for all $h\in[H]$.
Then for any $\rho>0$, there exists a $\rho$-optimistic covering $(\tPP,\cM_0)$ with $\log\abs{\cM_0}=\tbO{dH}$.
\end{proposition}

\subsubsection{Example \ref{example:linear-mdp-demo}: linear MDPs}\label{appendix:proof-linear-mdps}

Following the common definition of linear MDP \citep{jin2020provably}, we also assume that the mean reward function has the form
\begin{align}\label{equation:linear-mdp-linear-reward}
    R_h^M(s,a)=\iprod{\theta_h^M}{\phi_h(s,a)}, \qquad \forall (h,s,a)\in[H]\times\cS\times\cA,
\end{align}
where $(\theta_h^M)_h$ are parameters associated with $M$.
We remark again that the linear reward assumption~\cref{equation:linear-mdp-linear-reward} is needed for the no-regret and model-estimation settings, but not needed for the reward-free setting (where the rewards can be arbitrary measurable functions).

Now, suppose that $\cM$ is a class of linear MDPs with the given feature map $\phi$.

\paragraph{\Belrep~of $\cM$} It is direct to see the \belrep~$\GBE$ (\cref{def:bellman-err}) of $\cM$ has $\dimG(\GBE,\gamma)\leq d$ (\cref{example:dc-linear}).

\paragraph{\mBelrep~of $\cM$} 
Consider the \mbelrep~$\GTV$ (\cref{example:mer}) of $\cM$. By definition,
\begin{align*}
    \cE^{M;\oM}_h(s_h,a_h)
    =&\DTV{ \PP^M_h(\cdot|s_h,a_h), \PP^{\oM}_h(\cdot|s_h,a_h) }+ \abs{R^M_h(s_h,a_h) - R^{\oM}_h(s_h,a_h)}\\
    =&\abs{\iprod{\theta_h^M-\theta_h^{\oM}}{\phi_h(s_h,a_h)}}+\sum_{s'} \abs{\iprod{\mu_h^M(s')-\mu_h^{\oM}(s')}{\phi_h(s_h,a_h)}}.
\end{align*}
Hence, by \cref{cor:dc-linear-max}, $\dimG(\cG^{\oM}_h,\gamma)\leq d$ for all $h\in[H]$. Thus, 
\begin{align*}
    \alldec\leqsim \frac{dH^2}{\gamma}.
\end{align*}
\qed

\subsubsection{Example \ref{example:low-rank-demo}: low-rank MDPs}
\label{appendix:proof-low-rank}

We also consider the broader class of MDPs with low occupancy rank~\citep{du2021bilinear}:
\begin{definition}[Occupancy rank]
We say a MDP model $M$ is of occupancy rank $d$ if for all $h\in[H]$, there exists map $\phi_h^M:\Pi\to\R^d$, $\psi_h^M:\cS\to\R^d$, such that
\begin{align*}
    \PP^{M,\pi}(s_h=s)=\iprod{\psi_h^M(s)}{\phi_h^M(\pi)},\qquad \forall s\in\cS,\pi\in\Pi.
\end{align*}
\end{definition}
By definition, low-rank MDP with rank $d$ is of occupancy rank $d$. 

\paragraph{\mBelrep~of low-rank MDPs} For a model class $\cM$ consisting of MDPs with occupancy rank $d$, its \mbelrep~$\GTV$ satisfies that $\Cov(\cQ_h^{M})\leq dA$ for all $M\in\cM$ and $h\in[H]$. This is because, by \cref{lem:cov-to-rank} there exists $\mu^M_h\in\Delta(\cS)$ such that $\PP^{M,\pi}(s_h=s) \leq d\mu^M_h(s)$, and hence
\begin{align*}
    \PP^{M,\pi}(s_h=s,a_h=a)\leq dA \cdot \frac{\mu^M_h(s)}{A}.
\end{align*}
Therefore, by \cref{example:dc-to-cov}, we have $\dimG(\GTV,\gamma)\leq d\abs{\cA}$.
Thus,
\begin{align*}
    \alldec\leqsim \frac{d|\cA|H^2}{\gamma}.
\end{align*}
\qed

\subsubsection{MDP with coverability}\label{appendix:proof-cov-mdp}

Recall that the coverability $C(M)$ of a MDP $M$ is defined in \cref{def:mdp-cov}.

\paragraph{\mBelrep~for MDPs with coverability} Suppose that $\cM$ is a class of MDPs such that $C(M)\leq C$ for all $M\in\cM$. Then the \mbelrep~$\GTV$ of $\cM$ (\cref{example:mer}) has $\Cov(\cQ_h^{\oM})\leq C$ for all $\oM\in\cM, h\in[H]$ (\cref{example:dc-to-cov}). Hence, by \cref{example:dc-to-cov}, $\dimG(\GTV,\gamma)\leq C$, and
\begin{align*}
    \alldec\leqsim \frac{CH^2}{\gamma}.
\end{align*}
\qed

\subsubsection{Example \ref{example:para-mdp-demo}: parametric MDPs}\label{appendix:proof-exp-family}

We consider the following general definition of parametric MDPs (with a general family of transition distributions), which generalizes the parametric MDPs with exponential family \citep{chowdhury2021reinforcement}.
\begin{example}[Parametric MDPs]\label{def:parametric-mdp}
    Let $\cP=\{P_{s,a}(\cdot|\theta)\}_{s,a,\theta}$ be a given distribution family. A MDP model $M$ is a parametric MDP (with respect to $\cP$) if there is a feature map $\phi^M$ such that
    \begin{align*}
        \PP^M_h(\cdot|s,a)=P_{s,a}(\cdot|\phi^M_h(s,a)), \qquad \forall (s,a)\in\cS\times\cA\times[H].
    \end{align*}
    We assume that $\cP$ is $(\alpha,\beta)$-\emph{smooth}:
    \begin{align*}
        \dTV\paren{P_{s,a}(\cdot|\theta), P_{s,a}(\cdot|\theta')}\leq \beta\nrm{\theta-\theta'}, \qquad
        \dH\paren{P_{s,a}(\cdot|\theta), P_{s,a}(\cdot|\theta')}\geq \min\set{\alpha\nrm{\theta-\theta'}, 1}, \qquad \forall s,a, \theta,\theta'.
    \end{align*}
    Let $\kappa=\beta/\alpha$. Consider the \mbelrep~$\cG$ of $\cM$ given by
    \begin{align*}
        \cerr^{M;\oM}_h(s_h,a_h)=\beta\nrm{ \phi^M_h(s_h,a_h)-\phi^{\oM}_h(s_h,a_h) } \wedge 1, \qquad \dtr_h(\pi;\oM)=\PP^{\oM,\pi}(s_h=\cdot,a_h=\cdot),
    \end{align*}
    and $L=2H\kappa^2$. We write $\cH_h=\set{\phi^M_h:M\in\cM}$ for $h\in[H]$. Then by \cref{example:dc-to-es} it holds that $\dimG(\cG,\gamma)\leqsim \max_{h}\oeluder(\cH_h,1/\gamma)\log^2(\gamma)$, and hence
    \begin{align*}
        \alldec\leqsim \frac{\max_h\oeluder(\cH_h,1/\gamma)\cdot \kappa^2H^2 \log^2(\gamma)}{\gamma}, \qquad \forall \gamma\geq e.
    \end{align*}
    In particular, suppose that for some $d>0$ we have $\oeluder(\cH_h,\Delta)=\tO(d)$ and $\log \cN(\cH_h,\rho)=\tO(d)$ (e.g. $\cH_h$ being a linear class \citep{chowdhury2021reinforcement}), then we have the sample complexity upper bound $\tO(\kappa^2d^2H^3/\eps^2)$ for PAC RL (and also reward-free learning and model estimation), and we also have a $\tO(\sqrt{\kappa^2d^2H^3T})$ regret for no-regret learning (and also preference-based RL if we further assume $d_{\Cmp}\leq d$), as promised in \cref{tab:examples}.
\exend
\end{example}

\newcommand{\lammin}{\lambda_{\min}}
\newcommand{\lammax}{\lambda_{\max}}

In the following example, we instantiate \cref{def:parametric-mdp} to parametric MDPs with exponential family, where we translate the assumption on minimum eigenvalues in \citet{chowdhury2021reinforcement} to \cref{eqn:assumption-para-exp} (see also \cref{eqn:nabla-Z}).
\begin{example}[Parametric MDPs with exponential family \citep{chowdhury2021reinforcement}]\label{def:exp-family-mdp}
Consider the distribution family $\cP_{\exp}=\{P_{s,a}(\cdot|\theta)\}_{s,a,\theta}$ given by
\begin{align*}
    P_{s,a}(ds'|\theta)=\exp\paren{ \iprod{ \mu(s') }{ \theta } - Z_{s,a}(\theta) }\cdot p(ds'|s,a), \qquad
    Z_{s,a}(\theta) \defeq \log \int_{\cS} \exp\paren{ \iprod{ \mu(s') }{ \theta }}\cdot p(ds'|s,a),
\end{align*}
where $\mu:\cS\to\R^n$ is a known feature map, $p:\cS\times\cA\to L^1(\cS)$ is a known base measure. 

Suppose that $\cM$ is a class of MDP parameterized by $\cP_{\exp}$ and feature map $\phi$, and assume that for each pair of $(s,a)$, it holds that
\begin{align}\label{eqn:assumption-para-exp}
    \lammin I\preceq \nabla^2 Z_{s,a}(\theta)\preceq \lammax I, 
    \qquad \forall\theta\in\conv\set{\phi^M_h(s,a): M\in\cM, h\in[H]}.
\end{align}
Then clearly $\cP$ is $(\Theta(\sqrt{\lammin}),\Theta(\sqrt{\lammax}))$-smooth (see \cref{appdx:proof-exp-family}), and by \cref{def:parametric-mdp} we have
\begin{align*}
    \alldec\leqsim \frac{\max_h\oeluder(\cH_h,1/\gamma)\cdot \kappa^2H^2 \log^2(\gamma)}{\gamma}, \qquad \forall \gamma\geq e,
\end{align*}
where we write $\kappa=\beta/\alpha$, $\cH_h=\set{\phi^M_h:M\in\cM}$.
\exend
\end{example}
We remark that 
\begin{align}\label{eqn:nabla-Z}
    \nabla^2 Z_{s,a}(\theta) = \EE_{s'\sim P_{s,a}(\cdot|\theta)} \brac{ \mu(s')\mu(s')^\top } - \EE_{s'\sim P_{s,a}(\cdot|\theta)} \brac{ \mu(s')} \EE_{s'\sim P_{s,a}(\cdot|\theta)} \brac{ \mu(s')}^\top,
\end{align}
and hence \eqref{eqn:assumption-para-exp} agrees with the common assumptions considered in \citet{chowdhury2021reinforcement,li2022exponential,ouhamma2023bilinear}.

\cref{def:exp-family-mdp} also encompasses the problem of Online Nonlinear Control \citep{kakade2020information, ren2022free,chowdhury2021reinforcement}, an important concrete class of MDPs whose transition is given by
\begin{align}\label{eqn:ONR}
    s_{h+1}=\phi^M(s_h,a_h)+\sigma_h, \qquad \sigma_h\sim\normal{0,\bSigma}.
\end{align}
The results above directly recover the regret bounds of \citet{ren2022free},
and also gives sample complexity of reward-free learning, model estimation and preference-based RL in Online Nonlinear Control (see also the discussion in \cref{appdx:proof-exp-family}).

\subsubsection{Example \ref{example:pomdp-demo}: Partially observable RL}\label{appendix:POMDP}

\newcommand{\md}{{M}}
\newcommand{\dPSR}{d_{\rm PSR}}
\newcommand{\Test}{\mathfrak{T}}

In the following, we first introduce the notations of partially observable RL, following \citet{chen2022partially}.

\paragraph{Sequential decision processes with observations} 
An episodic sequential decision process model $M$ can be specified by a tuple $\set{H,\cO,\cA,\PP^M,\{R_{h}\}_{h \in [H]}}$, where $H\in\Z_{\ge 1}$ is the horizon length; $\cO$ is the observation space; $\cA$ is the action space; $\PP^M$ specifies the transition dynamics in the model $M$, such that the initial observation follows $o_1\sim \PP_0^M(\cdot) \in \Delta(\mathcal{O})$, and given the \emph{history} $\tau_h\defeq (o_1,a_1,\cdots,o_h,a_h)$ up to step $h$, the observation follows $o_{h+1}\sim\PP^M(\cdot|\tau_{h})$;
$R_h:\cO\times\cA\to[0,1]$ is the reward function at $h$-th step, which we assume is a known deterministic function of $(o_h,a_h)$ for simplicity.

An episodic sequential decision process $M$ can be cast as a DMSO problem directly, as follows. The observation is $\tau_H=(o_1,a_1,\cdots,o_H,a_H)$. A policy $\pi = \{\pi_h: (\cO\times\cA)^{h-1}\times\cO\to\Delta(\cA) \}_{h \in [H]}$ is a collection of $H$ functions. At step $h\in[H]$, an agent running policy $\pi$ observes the observation $o_h$ and takes action $a_{h}\sim \pi_h(\cdot|\tau_{h-1}, o_h)\in\Delta(\cA)$ based on the history $(\tau_{h-1},o_h)=(o_1,a_1,\dots,o_{h-1},a_{h-1},o_h)$. The environment then generates the next observation $o_{h+1}\sim\PP^M(\cdot|\tau_h)$ based on $\tau_h=(o_1,a_1,\cdots,o_h,a_h)$ (if $h < H $). The episode terminates immediately after $a_H$ is taken. The policy class $\Pi$ is the set of all such history-dependent policies.

\paragraph{POMDPs} A Partially Observable Markov Decision Process (POMDP) is a special sequential decision process whose transition dynamics are governed by \emph{latent states}. An episodic POMDP model $M$ is specified by a tuple $\{H,\cS,\cO,\cA,\{\T_h^M\}_{h=1}^{H},\{\O_h^M\}_{h=1}^{H},\{R_{h}\}_{h=1}^{H},\mu_1 \}$, where $\cS$ is the latent state space with $\abs{\cS}=S$,  $\O_h^M(\cdot|\cdot):\cS\to\Delta(\cO)$ is the emission dynamics at step $h$, $\T_h^M(\cdot|\cdot,\cdot):\cS\times\cA\to\Delta(\cS)$ is the transition dynamics over the latent states (which we identify as transition matrices $\T_h^M(\cdot|\cdot,a)\in\R^{\cS\times \cS}$ for each $a\in\cA$), and $\mu_1^M\in\Delta(\cS)$ specifies the distribution of initial state. At each step $h$, given latent state $s_h$ (which the agent cannot observe), the system emits observation $o_h\sim \O_h^M(\cdot|s_h)$, receives action $a_h\in\cA$ from the agent, emits the (known) reward $R_h(o_h,a_h)$, and then transits to the next latent state $s_{h+1}\sim \T_h^M(\cdot|s_h, a_h)$ in a Markov fashion. 
Note that (with known rewards) a POMDP $M$ can be fully described by the parameter $(\T^M,\O^M,\mu_1^M)$. %

\paragraph{PSR, core test sets, and predictive states} 
A \emph{test} $t$ is a sequence of future observations and actions (i.e. $t\in\Test:=\bigcup_{W \in\Z_{\ge 1}}\cO^W \times\cA^{W -1}$). For some test $t_h=(o_{h:h+W-1},a_{h:h+W-2})$ with length $W\ge 1$, we define the probability of test $t_h$ being successful conditioned on (reachable) history $\tau_{h-1}$ as $\PP^M(t_h|\tau_{h-1})\defeq \PP^M(o_{h:h+W-1}|\tau_{h-1};\doac(a_{h:h+W-2}))$,
i.e., the probability of observing $o_{h:h+W-1}$ under model $M$ if the agent deterministically executes actions $a_{h:h+W-2}$, conditioned on history $\tau_{h-1}$. %

\begin{definition}[PSR, core test sets, and predictive states]\label{def:core-test}
For any $h\in[H]$, we say a set $\Uh\subset\Test$ is a \emph{core test set} for model $M$ at step $h$ if the following holds: For any $W\in\Z_{\ge 1}$, any possible future (i.e., test) $t_h=(o_{h:h+W-1},a_{h:h+W-2})\in\cO^W\times\cA^{W-1}$, there exists a vector $b_{t_h,h}^M\in\mathbb{R}^{\Uh}$ such that 
\begin{align}
	\label{eqn:psr-def}
	\PP^M(t_h|\tau_{h-1})=\langle b_{t_h,h}^M,[\PP^M(t|\tau_{h-1})]_{t\in\Uh}\rangle, \qquad \forall \tau_{h-1} \in \cT^{h-1}:=(\cO \times \cA)^{h-1}. 
\end{align}
We refer to the vector $\bq^M(\tau_{h-1})\defeq [\PP^M(t|\tau_{h-1})]_{t\in\Uh}$ as the \emph{predictive state} at step $h$ (with convention $\bq^M(\tau_{h-1})=0$ if $\tau_{h-1}$ is not reachable), and $\bq_0^M\defeq [\PP^M(t)]_{t\in\cU_{1}}$ as the initial predictive state. 

A (linear) PSR $M$ is a sequential decision process equipped with a core test set $\{ \Uh \}_{h \in [H]}$. The core test sets $(\Uh)_{h\in[H]}$ are assumed to be known.
\end{definition}

Define $\QAh\defeq \{\a:(\o,\a)\in\Uh~\textrm{for some}~\o\in\bigcup_{W \in\mathbb{N}^+}\cO^{W} \}$ as the set of ``core actions'' (possibly including an empty sequence) in $\Uh$, with $\nUA\defeq \max_{h\in[H]}\nUAh$. Further define $\cU_{H+1}\defeq \set{o_{\dum}}$ for notational simplicity. 

\begin{definition}[PSR rank]
\label{def:rank}
Given a PSR, its \emph{PSR rank} is defined as $\dPSR^M:=\max_{h\in [H]} \rank(D_h^M)$, where $D_h^M:=\left[ \bq^M(\tau_{h} ) \right]_{\tau_{h}\in \cT^h}\in\R^{\Uhp\times \cT^{h}}$ is the matrix formed by predictive states at step $h\in[H]$.
\end{definition}
For POMDP model $M$ with latent state space $\cS$, it is clear that $\dPSR^M\leq \abs{\cS}$, regardless of the core test sets.

\begin{definition}[\Bpara]
\label{def:Bpara}
A {\Bpara} of a PSR model $M$ is a set of matrices $\{ (\BB_h^M(o_h,a_h)\in\R^{\Uhp\times\Uh})_{h,o_h,a_h}, \bq_0^M \in \R^{\Uone} \}$ such that for any $0\leq h\leq H$, policy $\pi$, history $\tau_{h} = (o_{1:h}, a_{1:h})\in \cT^{h}$, and core test $t_{h+1} = (o_{h+1:h+W}, a_{h+1:h+W-1}) \in \Uhp$, the quantity $\PP^M(\tau_{h},t_{h+1})$, i.e.
the probability of observing $o_{1:h+W}$ upon taking actions $a_{1:h+W-1}$ under model $M$, admits the decomposition 
\begin{align}\label{eqn:psr-op-prod-h}
    \PP^M(\tau_{h},t_{h+1}) = \PP^M(o_{1:h+W} | \doac(a_{1:h+W-1}))
    = \e_{t_{h+1}}^\top \cdot \BB_{h:1}^M(\tau_{h}) \cdot \bq_0^M,
\end{align}
where $\e_{t_{h+1}} \in \R^{\Uhp}$ is the indicator vector of $t_{h+1} \in \Uhp$, and 
$$
\BB_{h:1}^M(\tau_{h}) \defeq \BB_{h}^M(o_{h},a_{h}) \BB_{h-1}^M(o_{h-1},a_{h-1}) \cdots \BB_{1}^M(o_{1},a_{1}).
$$
\end{definition}

\begin{definition}[{B-stable PSR \citep[Definition 4]{chen2022partially}}]
For a PSR $M$, a \Bpara~$\{\BB_h^M, \bq_0^M \}$ of $M$ is $\stab$-stable if for all $h\in[H]$,
\begin{align*}
    \max_{\pi} \sum_{\tau_{h+1:H}} \pi(\tau_{h+1:H}) \times \left| \B_{H:h+1}^M(\tau_{h+1:H}) \bq\right| \leq \stab \nrmonetwo{\bq}, \qquad \forall \bq\in\R^{\cU_h},
\end{align*}
where the maximization is over all policies $\pi$ starting from step $h$ (i.e., ignoring the history $\tau_{h-1}$), and $\R^{\cU_h}$ is equipped with the following $\nrmonetwo{\cdot}$-norm:
\begin{align*}
    \nrmonetwo{\bq} \defeq&~ \textstyle  \big(\sum_{\a \in \UAh } \big(\sum_{\o: (\o, \a) \in \Uh} \vert \bq(\o,\a) \vert \big)^2 \big)^{1/2}
\end{align*}
A PSR is $\stab$-stable if it admits a $\stab$-stable \Bpara.
\end{definition}

An important subclass of B-stable PSR is the weakly revealing POMDPs~\citep{jin2020sample,golowich2022learning,liu2022partially}. The following definition of weakly revealing POMDP is taken from \citet{chen2022partially,chen2023lower}, which is slightly more general.
\begin{definition}[Weakly revealing POMDP]\label{example:rev-POMDP}
For a POMDP $M$, the $m$-step emission-action matrices $\M_h \in \R^{\cO^m\cA^{m-1} \times \cS} $ of $M$ are defined as
\begin{align*}
    [\M_h^M]_{(\o,\a), s} \defeq \P^M(o_{h:h+m-1} = \o | s_h = s, a_{h:h+m-2} = \a), \forall (\o, \a) \in \cO^m\times\cA^{m-1}, s \in \cS.
\end{align*}
The POMDP $M$ is called a $m$-step $\arev$-weakly revealing POMDP if for each $1\leq h\leq H-m+1$, $\M_h^M$ admits a left inverse $\M_h^{M,+}$ such that $ \| \M_h^{M,+} \|_{1 \to 1} \le \arev^{-1}$.

\citet[Proposition D.2]{chen2022partially} shows that any $m$-step $\arev$-weakly revealing POMDP is a $\stab$-stable PSR with core test sets $\Uh=(\cO\times\cA)^{\minop{m-1,H-h}}\times\cO$, and $ \stab\leq \sqrt{|\cA|^{m-1}} \arev^{-1}$.
\end{definition}

\paragraph{\mBelrep~of B-stable PSRs}
Suppose that $\cM$ is a class of PSR such that for each $M\in\cM$, $M$ admits a \Bpara~$\{\BB_h^M, \bq_0^M \}$ %
that is $\stab$-stable. Then by \citet[Proposition D.1 \& D.2]{chen2022partially}, $\cM$ admits a \mbelrep~$\cG$ given as follows: 
\begin{itemize}%
\item (Index set and distributions) $\cT_h=(\cO\times\cA)^{h}$ is the set of all histories up to step $h$, and for each policy $\pi$,
\begin{align*}
    \dtr_h(\pi;\oM)=\PP^{\oM,\pi}(\tau_h=\cdot)\in\Delta(\cT_h),
\end{align*}
i.e. $\dtr_h(\pi;\oM)$ is the distribution of trajectory $\tau_h$ induced by executing $\pi$ in model $\oM$.
\item (Error functions) For each $M,\oM\in\cM$, $\pi\in\Pi$,
\begin{align*}
    &~ \cerr_0^{M;\oM}=\max_{\pi} \sum_{\tau_{1:H}} \pi(\tau_{1:H}) \times \left| \B_{H:1}^M(\tau_{1:H}) \left(\bq^M_0 - \bq^\oM_0\right)\right|, \\
    &~ \cerr_h^{M;\oM}(\tau_h)=\max_{\pi} \sum_{\tau_{h+1:H}} \pi(\tau_{h+1:H}) \times \left| \B_{H:h+1}^M(\tau_{h+1:H}) \left(\B^M_h(o_h,a_h) - \B^\oM_h(o_h,a_h)\right)\bq^\oM(\tau_{h-1})\right|.
\end{align*}
\item The exploration policies are given by
\begin{align*}
    \piest=\frac{1}{H}\sum_{h=0}^{H-1} \pi\circ_h \Unif(\cA)\circ_{h+1}\Unif(\cU_{A,h+1}),
\end{align*}
i.e. $\piest$ is the policy that uniformly samples a $h$, and then executes $\pi$ for the first $h-1$ steps, take $a_h\sim \Unif(\cA)$, and then take $\ba\sim\Unif(\cU_{A,h+1})$.
\item $L=\bigO{H|\cA|U_A\stab^2}$.
\end{itemize} 

Note that by \cref{cor:dc-linear-max}(3), we have
\begin{align*}
    \dimG(\cG,\gamma)\leq 
    \max_{\oM} \rank(D_h^{\oM})=
    \max_{\oM} \dPSR^{\oM}=:\dPSR
\end{align*}
Therefore, 
\begin{align*}
    \allexpdec\leqsim \frac{\dPSR H^2AU_A\stab^2}{\gamma}, \qquad
    \dec_\gamma(\cM)\leqsim \paren{\frac{\dPSR H^2AU_A\stab^2}{\gamma}}^{1/2}.
\end{align*}

\paragraph{Implications to weakly revealing POMDPs} In particular, if $\cM$ is a class of $m$-step $\arev$-revealing POMDPs, then
\begin{align*}
    \allexpdec\leqsim \frac{SA^mH^2}{\arev^2\gamma}, \qquad
    \dec_\gamma(\cM)\leqsim \paren{\frac{SA^mH^2}{\arev^2\gamma}}^{1/2}.
\end{align*}
As a remark, the DEC bound above implies that %
\etod~achieves a regret of order $T^{2/3}$. It turns out that such a scaling of $T$ is actually \emph{tight} for no-regret learning in multi-step revealing POMDPs \citep[Theorem 6]{chen2023lower}.

\subsection{Proofs of additional results}
\label{appendix:proof-bellman-rep-props}

\subsubsection{Proof of Example~\ref{example:dc-linear}}\label{appendix:bellman-rep-proof-linear}

Under the linearity assumption, we can consider $f\mapsto\theta_f\in\R^d$ such that $f(x)=\iprod{\theta_f}{\phi(x)}\forall x\in\cX$.

Given a $\nu\in\Delta(\cF\times\cX)$, let us set $\Phi_{\lambda}:=\lambda I_d+\EE_{x\sim \nu}\brac{\phi(x)\phi(x)^\top}$ for $\lambda>0$. Then
\begin{align*}
\EE_{(f,x)\sim\nu}\brac{\abs{f(x)}} 
\leq
\EE_{(f,x)\sim\nu}\brac{\nrm{\theta_f}_{\Phi_\lambda}\nrm{\phi(x)}_{\Phi_{\lambda}^{-1}}}
\leq\gamma\EE_{f\sim\nu}\brac{\nrm{\theta_f}_{\Phi_\lambda}^2} + \frac{1}{4\gamma}\EE_{x\sim\nu}\brac{\nrm{\phi(x)}_{\Phi_{\lambda}^{-1}}^2}.
\end{align*}
For the first term, we have
\begin{align*}
\EE_{f\sim\nu}\brac{\nrm{\theta_f}_{\Phi_\lambda}^2}
=&\EE_{f\sim\nu}\brac{\theta_f^\top\paren{\EE_{x\sim\nu}\brac{\phi(x)\phi(x)^\top}}\theta_f}+\lambda\EE_{f\sim\nu}\nrm{\theta_f}^2\\
=&\EE_{f\sim\nu}\EE_{x\sim\nu}\brac{\abs{f(x)}^2}+\lambda\EE_{f\sim\nu}\nrm{\theta_f}^2.
\end{align*}
For the second term, we have
\begin{align*}
\EE_{x\sim\nu}\brac{\nrm{\phi(x)}_{\Phi_{\lambda}^{-1}}^2}
=&
\EE_{x\sim\nu}\brac{\tr\paren{\Phi_{\lambda}^{-1/2}\phi(x)\phi(x)^\top\Phi_{\lambda}^{-1/2}}}\\
=&
\tr\paren{\Phi_{\lambda}^{-1/2}\EE_{x\sim\nu}\brac{\phi(x)\phi(x)^\top}\Phi_{\lambda}^{-1/2}}\\
=&
\tr\paren{\Phi_{\lambda}^{-1/2}\Phi_0\Phi_{\lambda}^{-1/2}}
\leq d.
\end{align*}
Letting $\lambda\to0^+$ and then taking $\sup_{\nu}$ completes the proof.
\qed

As a corollary, we have the following result.
\begin{corollary}\label{cor:dc-linear-max}
If one of the following statements holds, then $\dimc(\cF,\gamma)\leq d$ holds.\\
(1). There exists $\phi=(\phi_i:\cX\to\R^d)$ and $\theta:\cF\to \R^d$ such that $f(x)=\max_i\abs{\iprod{\theta(f)}{\phi_i(x)}}$ for all $(f,x)\in\cF\times\cX$.\\
(2). There exists $\phi=(\phi_i:\cX\to\R^d)$ and $\theta:\cF\to \R^d$ such that $f(x)=\sum_i \abs{\iprod{\theta(f)}{\phi_i(x)}}$ for all $(f,x)\in\cF\times\cX$.
(3). There exists $\phi=(\phi_{i,r}:\cX\to\R^d)$ and $\theta:\cF\to \R^d$ such that $f(x)=\max_r \sum_i \abs{\iprod{\theta(f)}{\phi_{i,r}(x)}}$ for all $(f,x)\in\cF\times\cX$.
\end{corollary}

\begin{proof}
\cref{cor:dc-linear-max}(1) is actually a direct implication of \cref{example:dc-linear}: for $(f,x)\in\cF\times\cX$, we consider $i(f,x)\defeq \argmax_i \abs{\iprod{\theta(f)}{\phi_i(x)}}$. Then for any $\nu\in\Delta(\cF\times\cX)$, we have
\begin{align*}
    \EE_{(f,x)\sim \nu}\brac{\abs{f(x)}}
    =&~\EE_{(f,x)\sim \nu, i=i(f,x)}\brac{\abs{\iprod{\theta(f)}{\phi_i(x)}}}\\
    \leq&~ \frac{d}{4\gamma}+\gamma\EE_{f'\sim \nu}\EE_{(f,x)\sim \nu, i=i(f,x)}\brac{\abs{\iprod{\theta(f')}{\phi_i(x)}}^2}\\
    \leq&~ \frac{d}{4\gamma}+\gamma\EE_{f'\sim \nu}\EE_{x\sim \nu}\brac{\abs{f'(x)}^2},
\end{align*}
where the first inequality is due to \cref{example:dc-linear}.

We next reduce \cref{cor:dc-linear-max}(2) to (1). Consider $\Phi=(\Phi_f:\cX\to\R^d)_{f\in\cF}$ given by
$$
\Phi_f(x)\defeq \sum_{i} \phi_i(x)\sign\iprod{\theta(f)}{\phi_i(x)}.
$$
Then $f(x)=\max_{g\in\cF}\abs{\iprod{\theta(f)}{\Phi_g(x)}}$. Applying \cref{cor:dc-linear-max}(1) completes the proof of \cref{cor:dc-linear-max}(2).

Similarly, we can reduce \cref{cor:dc-linear-max}(3) to (1) by considering $\Phi=(\Phi_{f,r}:\cX\to\R^d)_{f\in\cF}$ given by
$$
\Phi_{f,r}(x)\defeq \sum_{i} \phi_{i,r}(x)\sign\iprod{\theta(f)}{\phi_{i,r}(x)}.
$$
Then $f(x)=\max_{g\in\cF,r}\abs{\iprod{\theta(f)}{\Phi_{g,r}(x)}}$ and applying \cref{cor:dc-linear-max}(1) completes the proof of (3).
\end{proof}

\subsubsection{Proof of Example~\ref{example:dc-to-cov}}\label{appendix:proof-dc-to-cov}

\newcommand{\oq}{\overline{q}}
Fix any $\nu\in\Delta(\cF\times\cQ)$. Let $\oq\in\Delta(\cX)$ be given by $\oq(x)=\EE_{q\sim \mu}[q(x)]$. Then
\begin{align*}
    \EE_{(f,q)\sim \nu}\EE_{x\sim q}[f(x)]
    =&~
    \EE_{(f,q)\sim \nu}\EE_{x\sim \oq}\brac{\frac{q(x)}{\oq(x)}\cdot f(x)} \\
    \leq&~ \sqrt{ \EE_{q\sim\nu}\EE_{x\sim\oq} \brac{ \frac{q(x)^2}{\oq(x)^2} } \cdot \EE_{f\sim\nu} \EE_{x\sim \oq} f(x)^2 } \\
    \leq&~ \sqrt{ \Cov(\cQ) \cdot \EE_{f\sim\nu} \EE_{q\sim\nu, x\sim q} f(x)^2 },
\end{align*}
where the last inequality is because the distribution of $x\sim \oq$ agrees with the distribution of $x\sim q$, $q\sim \nu$, and for any $\mu\in\Delta(\cX)$, we also have
\begin{align*}
    \EE_{q\sim\nu}\EE_{x\sim\oq} \brac{ \frac{q(x)^2}{\oq(x)^2} } 
    =&~ \EE_{q\sim\nu}\brac{ \sum_{x\in\cX} \frac{q(x)^2}{\oq(x)} }
    = \sum_{x\in\cX} \frac{\EE_{q\sim \nu} q(x)^2}{\EE_{q\sim \nu} q(x)} \\
    =&~ \sum_{x\in\cX} \frac{\EE_{q\sim \nu} q(x)\cdot \mu(x)\cdot q(x)/\mu(x)}{\EE_{q\sim \nu} q(x)} \\
    \leq&~ \sum_{x\in\cX} \frac{\EE_{q\sim \nu} q(x)\cdot \mu(x)}{\EE_{q\sim \nu} q(x)} \cdot \sup_{q\in\cQ}\linf{\frac{q}{\mu}} \\
    =&~ \sum_{x\in\cX} \mu(x) \cdot \sup_{q\in\cQ}\linf{\frac{q}{\mu}} \\
    =&~ \sup_{q\in\cQ}\linf{\frac{q}{\mu}}.
\end{align*}
This gives the desired result.
\qed

\subsubsection{Proof of Example~\ref{example:dc-to-rank}}\label{appendix:proof-dc-to-rank}

Under the assumption of \cref{example:dc-to-rank}, there exists $\mu:\cX\to\R^d$ and $\phi:\cQ\to\R^d$ such that
\begin{align*}
    q(x)=\iprod{\mu(x)}{\phi(q)} \qquad \forall x\in\cX, q\in\cQ.
\end{align*}
Therefore, for any $\nu\in\Delta(\cF\times\cQ)$,
\begin{align*}
    \EE_{(f,q)\in \nu}\EE_{x\sim q} f(x)
    =&~ \EE_{(f,q)\sim \nu} \sum_{x\in\cX} f(x) \iprod{ \mu(x) }{\phi(q)} \\
    =&~ \EE_{(f,q)\sim \nu}  \iprod{ \sum_{x\in\cX} f(x) \mu(x) }{\phi(q)} \\
    \leq&~ \frac{d}{4\gamma} + \gamma \EE_{f\sim \nu}\EE_{q\sim \nu}  \iprod{ \sum_{x\in\cX} f(x) \mu(x) }{\phi(q)}^2 \\
    =&~ \frac{d}{4\gamma} + \gamma \EE_{f\sim \nu}\EE_{q\sim \nu}  \paren{ \EE_{x\sim q} f(x) }^2 \\
    \leq&~ \frac{d}{4\gamma} + \gamma \EE_{f\sim \nu}\EE_{q\sim \nu,x\sim q}  f(x)^2,
\end{align*}
where the third line follows from \cref{example:dc-linear}. This gives the desired result.
\qed

We remark that \cref{example:dc-to-rank} can also be proved using the following lemma.
\begin{lemma}\label{lem:cov-to-rank}
For $\cQ$ a class of distributions over $\cX$, it holds that $\Cov(\cQ)\leq \rank(\cQ)$.
\end{lemma}

\begin{proof}
Let $d=\rank(\cQ)$. Then $\cQ$ spans a $d$-dimensional subspace of $L^1(\cS)$, and hence the closure of $\cQ$ admits a barycentric spanner \citep{awerbuch2008online}, i.e. there exists $q_1,\cdots,q_d\in\Delta(\cX)$ such that for any $q\in\cQ$, there are $\lambda_1,\cdots,\lambda_d\in[-1,1]$ and
\begin{align*}
    q=\lambda_1q_1+\cdots+\lambda_dq_d.
\end{align*}
Hence, for $\mu=\frac{1}{d}\sum_{i=1}^d q_i$, it holds that $q(x)\leq d\mu(x)\forall x\in\cX$, and thus $\Cov(\cQ)\leq d$.
\end{proof}

\subsubsection{Proof of Proposition~\ref{prop:cover-linear-mixture}}

We construct such a covering directly, which is a generalization of the construction in \cref{example:cover-tabular}. The covering of mean reward function is standard, and in the following we assume that the mean reward function is known and fixed, without loss of generality.

An important observation is that, by our assumption, it holds that
\begin{align*}
    \sum_{s'} \nrm{\phi(s'|s,a)}_1\leq 2d, \qquad \forall (s,a)\in\cS\times\cA.
\end{align*}
Then, we set $N=\ceil{B/\rho}$ and let $B'=N\rho$. For $\theta\in[-B',B']^d$, we define the $\rho$-neighborhood of $\theta$ as $\cB(\theta,\rho):=\rho\floor{\theta/\rho}+[0,\rho]^d$, and let
\begin{align*}
    \tPP_{\theta}(s'|s,a):=\max_{\theta'\in\cB(\theta,\rho)} \iprod{\theta'}{\phi(s'|s,a)}.
\end{align*}
Then, if $\theta$ induces a transition dynamic $\PP_{\theta}$, then $\tPP_{\theta}\geq \PP_{\theta}$, and
\begin{align*}
    \sum_{s'} \abs{\tPP_{\theta}(s'|s,a)-\PP_{\theta}(s'|s,a)}
    =
    \sum_{s'} \max_{\theta'\in\cB(\theta,\rho)} \abs{\iprod{\theta'-\theta}{\phi(s'|s,a)}}
    \leq \rho\sum_{s'} \nrm{\phi(s'|s,a)}_1\leq 2\rho d.
\end{align*}
Now, similarly to \eqref{eqn:proof-cover-tabular-opt-like}, for any $\Theta=(\theta_h)_h\in(\R^d)^{H-1}$, we define
\begin{align*}
\tPP_{\Theta}^{\pi}(s_1,a_1,\cdots,s_{H},a_{H}):=&\PP_1(s_1)\tPP_{\theta_1}(s_2|s_1,a_1)\cdots\tPP_{\theta_{H-1}}(s_H|s_{H-1},a_{H-1})\times \pi(s_1,a_1,\cdots,s_H,a_H)\\
=&\pi(s_1,a_1,\cdots,s_H,a_H)\times\PP_1(s_1)\times\prod_{h=1}^{H-1}\tPP_{\theta_h}(s_{h+1}|s_h,a_h).
\end{align*}
Suppose that $\rho\leq 1/(2Hd)$. Then for $M\in\cM$, if $M\in\cM$ is a linear mixture MDP induced by $\Theta$, then a simple calculation shows, 
\begin{align*}
    \lone{\tPP_{\Theta}^{\pi}-\PP_{\Theta}^{\pi}}\leq 2eHd\rho.
\end{align*}
Therefore, we let $\rho_1=\sqrt{2eHd\rho}$, then by picking representative in each $\ell_{\infty}$-$\rho$-ball, we can construct a $\rho_1$-optimistic covering with $\abs{\cM_0}\leq (2N)^{Hd}=\paren{2\ceil{B/\rho}}^{Hd}=\paren{2\ceil{2eHdB/\rho_1^2}}^{Hd}$, which implies that $\log\abs{\cM_0}\leq \bigO{dH\log(dHB/\rho_1)}$.
\qed

\subsubsection{Proof for Example~\ref{def:exp-family-mdp}}\label{appdx:proof-exp-family}

We only need to verify that $\cP_{\exp}$ is indeed $(\alpha,\beta)$-smooth under \eqref{eqn:assumption-para-exp} with $\alpha=\sqrt{\frac{\lammin}{4\log 2}}$, $\beta=\sqrt{\frac{\lammax}{2}}$. Fix a pair of $(s,a)$. Notice that
\begin{align}\label{eqn:exp-family-1}
\begin{aligned}
    1-\frac12\DH{ P_{s,a}(\cdot|\theta_1), P_{s,a}(\cdot|\theta_2) }
    =&~ \int_{\cS} \sqrt{P_{s,a}(ds'|\theta_1)P_{s,a}(ds'|\theta_2)} \\
    =&~ \int_{\cS} \exp\paren{ \iprod{ \mu(s') }{ \frac{\theta_1+\theta_2}{2} } - \frac12 Z_{s,a}(\theta_1)-\frac12 Z_{s,a}(\theta_2) } p(ds'|s,a) \\
    =&~ \exp\paren{ Z_{s,a}\paren{\frac{\theta_1+\theta_2}{2}} - \frac12 Z_{s,a}(\theta_1)-\frac12 Z_{s,a}(\theta_2) }.
\end{aligned}
\end{align}
Therefore, under the assumption that $\lammin I\preceq \nabla^2 Z_{s,a}\preceq \lammax I$, we know $Z_{s,a}$ (as a function of $\theta$) is $(2\lammin)$-strong convex and $(2\lammax)$-smooth, which implies
\begin{align}\label{eqn:exp-family-2}
    -\frac{\lammax}{4}\nrm{\theta_1-\theta_2}^2 \leq
    Z_{s,a}\paren{\frac{\theta_1+\theta_2}{2}} - \frac12 Z_{s,a}(\theta_1)-\frac12 Z_{s,a}(\theta_2) \leq -\frac{\lammin}{4}\nrm{\theta_1-\theta_2}^2.
\end{align}
Notice that for $x\geq 0$, we have
\begin{align*}
    \min\set{x,1}\geq 1-e^{-x}\geq \min\set{\frac{x}{2\log 2},\frac12}.
\end{align*}
Therefore, combining \eqref{eqn:exp-family-1} and \eqref{eqn:exp-family-2}, we obtain
\begin{align*}
    \min\set{\frac{\lammin}{4\log 2}\nrm{\theta_1-\theta_2}^2,1}
    \leq
    \DH{ P_{s,a}(\cdot|\theta_1), P_{s,a}(\cdot|\theta_2) } 
    \leq 
    2\min\set{ \frac{\lammax}{4}\nrm{\theta_1-\theta_2}^2, 1 }.
\end{align*}
Then using the fact that $\dTV\paren{ P_{s,a}(\cdot|\theta_1), P_{s,a}(\cdot|\theta_2) }\leq \dH\paren{ P_{s,a}(\cdot|\theta_1), P_{s,a}(\cdot|\theta_2) }$, both conditions of \eqref{eqn:assumption-para-exp} are fulfilled.
\qed

As a remark, for Online Nonlinear Control \eqref{eqn:ONR}, we can consider $\cP$ given by $P_{s,a}(\cdot|\theta)=\normal{\theta,\bSigma}$, and a direct computation yields
\begin{align*}
    1-\frac12\DH{ P_{s,a}(\cdot|\theta_1), P_{s,a}(\cdot|\theta_2) }=\exp\paren{ -\frac18\nrm{\theta_1-\theta_2}_{\bSigma^{-1}}^2 }.
\end{align*}
Therefore, we can obtain upper bounds on DECs that scale with $\max_h\oeluder(\cH_h',1/\gamma)$ as in \cref{def:exp-family-mdp}, where $\cH_h'\defeq\set{\bSigma^{-1/2}\phi^M_h: M\in\cM}$.

\section{Learning equilibria in Markov Games via AMDEC}
\label{appendix:markov-game}

In this section, we adapt the AMDEC and the \mealg~algorithm (cf.~\cref{section:mdec}) to design unified multi-agent RL algorithms for learning various equilibria in Markov Games. 

\subsection{Preliminaries}

\paragraph{Markov Games}
We consider the model of episodic Markov Games (MGs)~\citep{littman1994markov}, a standard multi-agent generalization of MDPs. An $m$-player MG $M=(m, H, \cS, (\cA_i)_{i=1}^m, \P^M, (r^M_i)_{i=1}^m)$ can be cast as a DMSO problem as follows. The observation $o=(s_1,\ba_1,\dots,s_H,\ba_H)\in(\cS\times\cA)^H$, where each $\ba_h=(\ba_{h,1},\dots,\ba_{h,m})\in\cA\defeq \prod_{i\in[m]}\cA_i$ denotes the \emph{joint action} taken simultaneously by all players at step $h$. A joint policy $\pi$ (which may in general be non-Markovian) is denoted by $\pi=\set{\pi_h:(\cS\times\cA)^{h-1}\times\cS\to\Delta(\cA)}_{h\in[H]}$. Upon executing $\pi$ in $M$, the (centralized) learner observes $o=(s_1,\ba_1\dots,s_H,\ba_H)\sim \MP(\pi)$, which sequentially samples $s_1\sim \P^M_0(\cdot)$, $\ba_h\sim \pi_h(\cdot|s_1,\ba_1,\dots,s_{h-1},\ba_{h-1},s_h)$, and $s_{h+1}\sim \P^M_{h}(\cdot|s_h, 
\ba_h)$ for all $h\in[H]$. The learner then receives $m$ reward vectors $(\br_i)_{i=1}^m$, where each $\br_i=[r_{1,i},\dots,r_{h,i}]^\top\in[0,1]^H$, and each $r_{h,i}$ is the (possibly random) instantaneous reward for the $i$-th player at the $h$-th step. Note that the mean reward function $R_{h,i}^M(o)=\E^M[r_{h,i}|o]\eqdef R^M_{h,i}(s_h,\ba_h)$ is the expected reward at $(s_h,\ba_h)$. We assume that $\sum_{h=1}^H R^M_{h,i}(s_h,\ba_h)\in[0,1]$ almost surely for all $M$, all $o\in\cO$, and all $i\in[m]$. Let $f_i^M(\pi)\defeq \E^{M, \pi}\brac{\sum_{h=1}^H r_{h,i}}$ denote player $i$'s value (expected cumulative reward) of $\pi$ under $M$, for all $i\in[m]$.

For MGs, we let $\Pi$ denote the set of all possible joint policies. Let $\Pi^{\det}$ denote the set of all deterministic policies, where each $\pi\in\Pi^{\det}$ has form $\pi=\set{\pi_h:(\cS\times\cA)^{h-1}\times\cS\to\cA}_{h\in[H]}$. Additionally, we say $\pi\in\Pi$ is a \emph{product policy} if all agents take their actions independently from each other, i.e. $\pi=\pi_1\times\dots\times \pi_m$, where each $\pi_i=\set{\pi_{i,h}:(\cS\times\cA)^{h-1}\times\cS\to\Delta(\cA_i)}_{h\in[H]}\in\Pi_i$, where $\Pi_i$ denotes the set of all policies for player $i$.

\paragraph{Definitions of equilibria}
We consider three commonly studied notions of equilibria in Markov Games: Nash Equilibria (NE), Correlated Equilibria (CE), and Coarse Correlated Equilibria (CCE). 

\begin{definition}[Nash Equilibrium]
A product policy $\pi=\pi_1\times\dots\times\pi_m$ is an $\eps$-approximate Nash Equilibrium (NE) of $M$ if
\begin{align*}
    \Gap_{\NE}(\pi, M) \defeq \max_{i\in[m]} \max_{\pi_i^\dagger\in\Pi_i} f_i^M(\pi_i^\dagger, \pi_{-i}) - f_i^M(\pi) \le \eps.
\end{align*}
\end{definition}

To define CEs, we define a strategy modification for Markov Games~\citep{song2021can}.

\begin{definition}[Strategy modification for $i$-th player]
A strategy modification $\phi_i$ for the $i$-th player is a set of $H$ mappings $\phi_i=\set{\phi_{i,h}:(\cS\times\cA)^{h-1}\times (\cS\times\cA_i)\to \cA_i}_{h\in[H]}$. For any joint policy $\pi$, the modified policy $\phi_i\circ \pi$ is defined as follows: At any step $h\in[H]$, state $s_h\in\cS$, and history $\tau_{h-1}\defeq (s_1,\ba_1,\dots,s_{h-1}, \ba_{h-1})$, we sample a joint action $\ba_h\sim \pi_h(\cdot|\tau_{h-1}, s_h)$, the $i$-th player takes the modified action $\wt{a}_{h,i}=\phi_{i,h}(\tau_{h-1}, s_h, a_{h,i})$, and all other players take the unmodified action $\ba_{h,-i}$.
\end{definition}

\begin{definition}[Correlated Equilibrium]
A policy $\pi$ is an $\eps$-approximate Correlated Equilibrium (CE) of $M$ if
\begin{align*}
    \Gap_{\CE}(\pi, M) \defeq \max_{i\in[m]} \max_{\phi_i\in\Phi_i} f_i^M\paren{ \phi_i\circ\pi} - f_i^M(\pi) \le \eps.
\end{align*}
\end{definition}
 
\begin{definition}[Coarse Correlated Equilibrium]
A policy $\pi$ is an $\eps$-approximate Coarse Correlated Equilibrium (CCE) of $M$ if
\begin{align*}
    \Gap_{\CCE}(\pi, M) \defeq \max_{i\in[m]} \max_{\pi_i^\dagger\in\Pi_i} f_i^M(\pi_i^\dagger, \pi_{-i}) - f_i^M(\pi) \le \eps.
\end{align*}
\end{definition}

\paragraph{Learning goal}
Our learning goal is to output an $\eps$-approximate NE/CE/CCE in as few episodes of play as possible. We consider the centralized learning setting, in which the learner controls all agents when interacting with the environment, but may output either correlated policies or independent (product) policies depending on which equilibrium is desired.

\subsection{Model estimation for Markov Games}

We now show that the equilibrium gaps $\Gap_{\{\NE,\CE,\CCE\}}$ can be directly bounded by the model divergence between them. This in turn allows us to derive algorithms for learning NE/CE/CCE in MGs by a simple reduction to \mealg, which we present in~\cref{appendix:mg-learn-equilibria}.

Different from MDPs, each episode of an $m$-player MG $M$ yields a collection of $m$ reward vectors $\br=(\br_i\in[0,1]^{H})_{i\in[m]}$, instead of a single reward vector. We slightly adapt the definitions of the various divergence functions for this setting: For any two MGs $M,M'$ and any policy $\pi$, we define divergences
\begin{align}
    \wtdH^2(M(\pi), \oM(\pi)) & \defeq \dH^2(\MP(\pi), \oMP(\pi)) + \E_{o\sim \MP(\pi)}\brac{ \max_{i\in[m]}\big\|\bR^M_i(o) - \bR^{\oM}_i(o)\big\|^2_2 }, \label{equation:def-drl-mg} \\
    \wDTV{M(\pi), \oM(\pi)} & \defeq \dTV(\MP(\pi), \oMP(\pi)) + \E_{o\sim \MP(\pi)}\brac{ \max_{i\in[m]} \lone{\bR^M_i(o) - \bR^{\oM}_i(o)} }, \label{equation:def-dtv-mg} \\
    \DTVPid{M, M'} & \defeq \max_{\opi\in\Pi^{\det}} \wDTV{M(\opi), M'(\opi)}.
\end{align}

The following result is a direct adaptation of~\cref{lemma:diff-reward} to divergence $\wdTV$ using the fact that the $\max_{i\in[m]}$ in the definition of $\wdTV$ is lower bounded by the $i$-th term.

\begin{lemma}[Bounding value difference by $\wdTV$]
\label{lemma:diff-reward-mg}
For any two MGs $M,M'$, $i\in[m]$, and $\pi\in\Pi$, we have
\begin{align*}
    \abs{f^M_i(\pi) - f^{M'}_i(\pi)} \le \wDTV{M(\pi), M'(\pi)}.
\end{align*}
\end{lemma}

\begin{proposition}[Bounding equilibrium gaps by model estimation error]
\label{prop:eqgap-to-model-error}
    For $\EQ\in\set{\NE,\CE,\CCE}$, it holds that
    \begin{align*}
        \Gap_{\EQ}(\pi,M)\leq \Gap_{\EQ}(\pi,\oM)+2\DTVPid{M,\oM},
    \end{align*}
    where we recall the definition of $\DTVPid{\cdot, \cdot}$ in~\cref{equation:def-dtvpi}.
\end{proposition}
\begin{proof}
    We first deal with the case $\EQ\in\set{\NE,\CCE}$. 
    First, note that the divergence $\wdTV(M(\pi), \oM(\pi))$ is convex in $\pi$ (where the linear combination $\lambda\pi+(1-\lambda)\pi'$ is understood as the corresponding mixture policy), as both $\dTV(\MP(\pi), \oMP(\pi))$ and $\E_{o\sim \MP(\pi)}\brac{ \max_{i\in[m]} \lone{\bR^M_i(o) - \bR^{\oM}_i(o)} }$ are linear in $\pi$. Therefore, as $\Pi=\Delta(\Pi^{\det})$, we have 
    \begin{align*}
        \DTVPi{M,\oM}=\DTVPid{M,\oM}.
    \end{align*}
    Now, for any $i\in[m]$, $\pi_i^\dagger\in\Pi_{i}$, by~\cref{lemma:diff-reward-mg}, it holds that
    \begin{align*}
        \abs{f_i^M(\pi_i^\dagger, \pi_{-i})-f_i^{\oM}(\pi_i^\dagger, \pi_{-i})}\leq \wDTV{M(\pi_i^\dagger, \pi_{-i}), \oM(\pi_i^\dagger, \pi_{-i})}\leq \DTVPid{M,\oM},
    \end{align*}
    and similarly $\abs{f^M_i(\pi)-f^{\oM}_i(\pi)}\leq \DTVPid{M,\oM}$. Therefore,
    \begin{align*}
        f_i^M(\pi_i^\dagger, \pi_{-i}) - f_i^M(\pi)
        \leq
        f_i^{\oM}(\pi_i^\dagger, \pi_{-i}) - f_i^{\oM}(\pi)+2\DTVPid{M,\oM}, \qquad \forall i\in[m], \pi_i^\dagger\in\Pi_{-i}.
    \end{align*}
    Taking $\max_{i\in[m]} \max_{\pi_i^\dagger\in\Pi_{-i}}$ completes the proof for the NE/CCE case.
    
    Similarly, for $\EQ=\CE$, we have for any $\phi_i\in\Phi_i$ and $\pi\in\Pi$ that
    \begin{align*}
        \abs{f_i^M(\phi_i\circ \pi)-f_i^{\oM}(\phi_i\circ \pi)}\leq \wDTV{M(\phi_i\circ \pi), \oM(\phi_i\circ \pi)}\leq \DTVPid{M,\oM},
    \end{align*}
    and thus
    \begin{align*}
        f_i^M(\phi_i\circ \pi) - f_i^M(\pi)
        \leq
        f_i^{\oM}(\phi_i\circ \pi) - f_i^{\oM}(\pi)+2\DTVPid{M,\oM}, \qquad \forall i\in[m], \phi_i\in\Phi_i.
    \end{align*}
    Taking $\max_{i\in[m]} \max_{\phi_i\in\Phi_i}$ completes the proof for the CE case.
\end{proof}

The following result is an adaptation of~\cref{thm:E2D-ME}. %
\begin{theorem}[\Vovkalg~for Markov Games]
\label{thm:E2D-ME-MG}
Given a $\rho$-optimistic cover $(\tPP,\cM_0)$ of a family of Markov Games $\cM$, choosing $\etap=\etar=1/3$, $\mu^1=\Unif(\cM_0)$, then the variant of~\cref{alg:E2D-ME} with TA subroutine
\begin{align}\label{eqn:TA-MG}
    \mu^{t+1}(M) \; \propto_{M} \; \mu^{t}(M) \cdot \exp\paren{\etap\log \tPP^M(o^t|\xt) - \etar\sum_{i=1}^m\ltwot{ \br^t_i-\bR^{M}_i(o^t) } },
\end{align}
$\wtdH^2$ defined in~\cref{equation:def-drl-mg}, and $\wdTV$ defined in~\cref{equation:def-dtv-mg} achieves the following with probability at least $1-\delta$:
\begin{align}
\label{equation:vovkalg-mg}
    \DTVPid{M^\star, \hM}\leq \cO\paren{ \omdec_{\gamma}(\cM)+\frac{\gamma}{T}\brac{\log\abs{\cM_0}+T\rho+\log(1/\delta)} }.
\end{align}
\end{theorem}

\begin{proof}
Note that \cref{lemma:tDTV-triagnle} still holds for the $\wdTV$ we defined in \eqref{equation:def-dtv-mg}. Therefore, repeating the argument in the proof of \cref{thm:E2D-ME} gives that for the output model $\hM$,
\begin{align*}
\max_{\opi\in\Pi}\wDTV{ \hM(\opi), M^\star(\opi) }
\leq& 6\omdec_{\gamma}(\cM) + 6\gamma\cdot\frac{\EstwtH}{T},
\end{align*}
where $\EstwtH \defeq \sum_{t=1}^T \EE_{\oM\sim\mu^t}\EE_{\pi \sim \pexp^t}\left[ \wtdH^2(M^\star(\pi), \oM(\pi))\right]$ here is defined with respect to the $\wtdH$ given by \eqref{equation:def-drl-mg}. It remains to upper bound $\EstwtH$ under subroutine \eqref{eqn:TA-MG}.

Note that the guarantee provided by \cref{thm:vovk-cover} is actually that: suppose the reward vector $\br^t$ is $\sigma^2$-sub-Gaussian, and $2\etap+2\sigma^2\etar<1$, then \cref{alg:TA-infinite} achieves that with probability at least $1-\delta$,
\begin{align*}
    \MoveEqLeft
    \sum_{t=1}^T\EE_{M\sim \mu^t}\brac{\etap\EE_t\brac{\DH{\PP^M(\pi^t), \PP^{\Ms}(\pi^t)}}+(1-2\etap)\EE_{t}\brac{1-\exp\paren{-c(1-2\sigma^2c)\ltwot{\bR^M(o^t)-\bR^{\Ms}(o^t)}}} }\\
    &\leq \log\abs{\cM_0}+2T\rho(\etap+\etar)+2\log(2/\delta),
\end{align*}
where $c=\etar/(1-2\etap)$ as in \cref{thm:vovk-finite}, e.g. by combining \eqref{eqn:proof-vovk-delta-to}, \eqref{eqn:vovk-proof-optlike-2} and \eqref{eqn:vovk-proof-sum-delta-demo}. Therefore, we can plug in $\etap=\etar=1/3$, $c=1$ and $\br^t$ being $[\br^t_1,\cdots,\br^t_m]$ (and hence $\sigma^2=1/4$) in the guarantee above to derive that: subroutine \eqref{eqn:TA-MG} achieves that with probability at least $1-\delta$,
\begin{align*}
    \MoveEqLeft
    \sum_{t=1}^T\EE_{M\sim \mu^t}\EE_{\pi^t\sim \pexp^t}\brac{\DH{\PP^M(\pi^t), \PP^{\Ms}(\pi^t)}+\EE_{o^t\sim \Ms(\pi^t)}\brac{1-\exp\paren{-\frac12\ltwot{\bR^M(o^t)-\bR^{\Ms}(o^t)}}} }\\
    &\leq 3\brac{\log\abs{\cM_0}+2T\rho+2\log(2/\delta)}.
\end{align*}
Now, it holds that for all $o\in\cO$,
\begin{align*}
1-\exp\paren{-\frac12\ltwot{\bR^M(o)-\bR^{\Ms}(o)}}
\geq &
1-\exp\paren{-\frac12\max_{i\in[m]}\ltwot{\bR^M_i(o)-\bR^{\Ms}_i(o)}}\\
\geq &\frac{1-e^{-1}}{2}\max_{i\in[m]}\ltwot{\bR^M_i(o)-\bR^{\Ms}_i(o)},
\end{align*}
which is due to $\ltwot{\bR^M_i(o)-\bR^{\Ms}_i(o)}\leq 2$ for all $i\in[m], o\in\cO$.
Thus, under the assumption of our \cref{thm:E2D-ME}, with probability at least $1-\delta$,
\begin{align*}
\EstwtH
=&\sum_{t=1}^T \EE_{\oM\sim\mu^t}\EE_{\pi \sim \pexp^t}\left[ \wtdH^2(M^\star(\pi), \oM(\pi))\right] \\
=&\sum_{t=1}^T\EE_{M\sim \mu^t}\EE_{\pi^t\sim \pexp^t}\brac{\DH{\PP^M(\pi^t), \PP^{\Ms}(\pi^t)}+\EE_{o\sim \Ms(\pi^t)}\brac{\max_{i\in[m]}\ltwot{\bR^M_i(o)-\bR^{\Ms}_i(o)}}}\\
\leq& 10\brac{\log\abs{\cM_0}+2T\rho+2\log(2/\delta)}.
\end{align*}
This completes the proof of \cref{thm:E2D-ME-MG}.
\end{proof}

\subsection{Learning equilibria in Markov Games}
\label{appendix:mg-learn-equilibria}

Similar as MDPs, we define linear mixture/low-rank/linear MGs as follows: View the MG as an ``MDP'' with a single ``mega-agent'' taking joint action $\ba_h=(a_{h,1},\dots,a_{h,m})$ at the $h$-th step. Note that the action space has size $A=\prod_{i\in[m]} A_i$. We say a class $\cM$ of MGs is a class of linear mixture/low-rank/linear MGs, if their corresponding class of MDPs is a linear mixture/low-rank/linear MDP satisfying the definitions in~\cref{example:linear-mixture-demo}, \cref{example:linear-mdp-demo} and \cref{example:low-rank-demo}. For example, $\cM$ is a class of linear mixture MG with feature dimension $d$ if there exists fixed feature maps $(\phi_h:\cS\times\cA\times\cS\to\R^d)_{h\in[H]}$ such that for any $M\in\cM$, there exists vectors $(\theta_h)_{h\in[H]}$ such that $\P_h(s'|s,a)=\<\theta_h, \phi_h(s'|s,\ba)\>$ for all $(h,s,\ba,s')$. We remark the following additional reference for Linear MGs~\citep{xie2020learning}, Linear Mixture MGs~\citep{chen2021almost}, and MGs with general function approximation~\citep{huang2021towards,jin2022power}.

We consider the following reduction to model estimation as a unified algorithm for learning equilibria in Markov Games: Run algorithm \mealg~with model class $\cM$ for $T$ episodes to obtain an estimated model $\hM$. Then, simply return an \{NE,CE,CCE\} of $\hM$. We summarize this in~\cref{alg:MG}.

\begin{algorithm}[t]
\caption{Learning NE/CE/CCE in Markov Games by reduction to model estimation} 
\begin{algorithmic}[1]
\label{alg:MG}
\STATE \textbf{Input:} Model class $\cM$. Desired equilibrium $\EQ\in\set{\NE,\CE,\CCE}$.
\STATE Run \mealg~(\cref{alg:E2D-ME}) over $\cM$ for $T$ episodes, and obtain estimated model $\hM$.
\STATE \textbf{Output:} Policy $\hpi \defeq \EQ(\hM)$.
\end{algorithmic}
\end{algorithm}

For these Markov Games problem classes, we can adapt the guarantee for the \mealg~algorithm as follows:
\begin{itemize}[leftmargin=2em]
    \item Prove similar AMDEC bounds through \cref{prop:belrep-am} as in \cref{appendix:proof-examples}, but with newly defined divergences~\cref{equation:def-drl-mg} \&~\cref{equation:def-dtv-mg} for MGs. The essential difference is that we have the additional $\max_{i\in[m]}$ in the divergence, which however does not affect the arguments in \cref{appendix:proof-examples}. 
    \item By~\cref{thm:E2D-ME-MG}, the \Vovkalg~subrouine achieves estimation guarantee~\cref{equation:vovkalg-mg}. 
\end{itemize}
Consequently, the \mealg~algorithm for these Markov Games achieve bounds of the form (where $\cM_0$ is an $1/T$-optimistic covering of $\cM$):
\begin{align}
\label{equation:model-estimation-mg}
    \DTVPid{M^\star,\hM} \le \tO\paren{ \sqrt{D \log\abs{\cM_0} / T}},
\end{align}
where $D=\set{SAH^2,dH^2,dH^2,dAH^2}$ for \{tabular, linear mixture, linear, low-occupancy rank\} MGs.

Then, as~\cref{alg:MG} returns $\hpi=\EQ(\hM)$, by~\cref{prop:eqgap-to-model-error}, we obtain
\begin{align*}
    \Gap_{\EQ}(\hpi, M^\star) \le \Gap_{\EQ}(\hpi, \hM) + 2\DTVPid{M^\star,\hM} \le 2\DTVPid{M^\star,\hM}.
\end{align*}
for $\EQ\in\set{\NE,\CE,\CCE}$. Combined with~\cref{equation:model-estimation-mg} and covering number bounds, we directly obtain the following result.
\begin{theorem}[Learning NE/CE/CCE in Markov Games]
\label{thm:MG}
For $\EQ\in\set{\NE,\CE,\CCE}$, with probability at least $1-\delta$, the output policy $\hpi$ of~\cref{alg:MG} achieves $\Gap_{\EQ}(\hpi)\le\eps$ within $T$ episodes of play, where
\begin{enumerate}[leftmargin=2em]
    \item For tabular MGs, $T\le \tbO{S^3A^2H^3/\epsilon^2}$;
    \item For linear mixture MGs,  $T\le \tbO{d^2H^3/\epsilon^2}$;
    \item For linear MGs, $T\le \tbO{dH^2\log\abs{\cM}/\epsilon^2}$;
    \item For MGs with occupancy rank at most $d$ (including low-rank MGs with rank $d$), $T\le \tbO{dAH^2\log\abs{\cM}/\epsilon^2}$. %
\end{enumerate}
\end{theorem}

For succinctness, here we only present guarantees for these four problem classes; The analogous results for other problem classes in \cref{appendix:proof-examples} also hold true. 
To our best knowledge, \cref{thm:MG} provides the first unified algorithm for learning equilibrium in Markov Games with general model classes, building the model estimation algorithm \mealg~through the DEC framework. 

We remark that the dependence on $A=\prod_{i\in[m]}A_i$ (which is exponential in the number of players $m$) in~\cref{thm:MG} happens due to the model-based nature of~\cref{alg:MG} (in particular, the need of estimating the model behavior on \emph{any} policy $\pi\in\Pi$). This dependence is similar as existing work for learning MGs using model-based approaches~\citep{bai2020provable,liu2021sharp}, but is worse than model-free approaches such as the V-Learning algorithm~\citep{bai2020near,song2021can,jin2021v,mao2022provably} whose sample complexity only depends polynomially on $\max_{i\in[m]}A_i$. Extending the DEC framework to obtain ${\rm poly}(\max_{i\in[m]}A_i)$ dependence in the sample complexity would be an interesting direction for future work.

\section{Proofs for Section~\ref{section:mops}}\label{appendix:MOPS}

\subsection{Algorithm MOPS}
\label{appendix:proof-mops}

Here we present a more general version of the \mops~algorithm where we allow $\cM$ to be a possibly infinite model class, and require a prior $\mu^1\in\Delta(\cM)$ and an optimistic likelihood function $\tPP$ (cf.~\cref{def:opt-cover}) as inputs. The algorithm stated in~\cref{section:mops} is a special case of~\cref{alg:MOPS} with $\abs{\cM}<\infty$, $\tPP=\PP$, and $\mu^1=\unif(\cM)$.

\begin{algorithm}[t]
	\caption{\textsc{MOPS}~\citep{agarwal2022model}} 
	\begin{algorithmic}[1]
	\label{alg:MOPS}
	\STATE \textbf{Input:} Parameters $\etap,\etar,\gamma>0$; prior distribution $\mu^1\in\Delta(\cM)$; optimistic likelihood function $\tPP$.
	\FOR{$t=1,\ldots,T$}
    \STATE Sample $M^t\sim \mu^t$ and set $\pi^t=\piest_{M^t}$.
    \STATE Execute $\pi^t$ and observe $(o^t,r^t)$.
    \STATE Update posterior of models by Optimistic Posterior Sampling (OPS):
    \begin{align}
    \label{equation:ops}
    \!\!\!\!\mu^{t+1}(M) \; \propto_M \; \mu^{t}(M) \cdot \exp\paren{ \gamma^{-1} f^M(\pi_M) + \etap\log \tPP^{M,\pi^t}(o^t) - \etar\ltwot{\br^t-\bR^M(o^t)} }.
    \end{align}    
    \ENDFOR
    \STATE \textbf{Output:} $\pout$ the distribution of $\pi=\pi_M$, where $M\sim\mu^t$, $t\sim\Unif([T])$. %
   \end{algorithmic}
\end{algorithm}

We state the theoretical guarantee for~\cref{alg:MOPS} as follows.
\begin{theorem}[MOPS]
\label{thm:MOPS-full}
Given a $\rho$-optimistic cover $(\tPP,\cM_0)$,~\cref{alg:MOPS} with $\etap=1/6$, $\etar=0.6$ and $\mu^1=\Unif(\cM_0)$ achieves the following with probability at least $1-\delta$:
\begin{align*}
    \regdm\leq T\brac{\psc_{\gamma/6}(\cM, M^\star) + \frac{2}{\gamma}} + \gamma\brac{\log\abs{\cM_0}+3T\rho+2\log(2/\delta)}.
\end{align*}
Choosing the optimal $\gamma>0$, with probability at least $1-\delta$, suitable implementation of~\cref{alg:MOPS} achieves
\begin{align*}
    \regdm\leq 12 \inf_{\gamma>0}\set{T\psc_{\gamma}(\cM, M^\star) + \frac{T}{\gamma}+\gamma\brac{\est(\cM,T)+\log(1/\delta)}}.
\end{align*}
\end{theorem}

When $\cM$ is finite, clearly $(\PP,\cM)$ itself is a $0$-optimistic covering, and hence \cref{thm:MOPS-full} implies \cref{thm:MOPS} directly.

It is worth noting that \cite{agarwal2022model} states the guarantee of MOPS in terms of a general prior, with the regret depending on a certain ``prior around true model'' like quantity.
The proof of \cref{thm:mops-online} can be directly adapted to work in their setting; however, we remark that, obtaining an explicit upper bound on their ``prior around true model'' in a concrete problem likely requires constructing an explicit covering, similar as in \cref{thm:MOPS-full}. %

\begin{proofof}[thm:MOPS-full]
By definition,
\begin{align*}
\regdm=&
    \sum_{t=1}^T \E_{M\sim \mu^t}\brac{ f^\Ms(\pi_\Ms)-f^\Ms(\pi_M)}\\
    =
    &\underbrace{\sum_{t=1}^T \E_{M\sim \mu^t}\brac{ f^\Ms(\pi_\Ms)-f^M(\pi_M) + \frac{\gamma}{6}\EE_{\pi\sim p^t}\brac{\tDH{\Ms(\pi),M(\pi)}}}}_{\text{Bounded by \cref{cor:MOPS-online}}}\\
    &+\sum_{t=1}^T \underbrace{\E_{M\sim \mu^t}\brac{ f^M(\pi_M)-f^\Ms(\pi_M) - \frac{\gamma}{6}\EE_{\pi\sim p^t}\brac{\tDH{\Ms(\pi),M(\pi)}}}}_{\text{Bounded by }\psc}\\
    \le& \gamma\brac{\log\abs{\cM_0}+3T\rho+2\log(2/\delta)}+\frac{2T}{\gamma}+T\psc_{\gamma/6}(\cM,\Ms). \qedhere
\end{align*}
\end{proofof}

\subsection{Optimistic posterior sampling}

In this section, we analyze the following Optimistic Posterior Sampling algorithm under a more general setting. 
The problem setting and notation are the same as the online model estimation problem introduced in~\cref{appendix:proof-e2d-ta}. Additionally, we assume that each $M\in\cM$ is assigned with a scalar $V_M\in[0,1]$; in our application, $V_M$ is going to be optimal value of model $M$.

\begin{theorem}[Analysis of posterior in OPS]\label{thm:mops-online}
Fix a $\rho>0$ and a $\rho$-optimistic covering $(\tPP,\cM_0)$ of $\cM$. Under the assumption of \cref{thm:vovk-finite}, the following update rule
\begin{align}\label{eqn:mops-online-def}
    \mu^{t+1}(M) \; \propto_M \; \mu^{t}(M) \cdot \exp\paren{\gamma^{-1}V_M+\etap\log \tPP^M(o^t|\xp^t) - \etar\ltwot{ \br^t-\bR^M(o^t) } }.
\end{align}
with $2\etap+4\sigma^2\etar<1$ and $\mu^1=\unif(\cM_0)$ achieves with probability at least $1-\delta$ that
\begin{align*}
    \sum_{t=1}^T \E_{M\sim \mu^t}\brac{ \Vs-V_M+c_0\gamma\err_M^t } 
    \le& \frac{T}{8\gamma(1-2\etap-4\sigma^2\etar)} + \gamma\log\abs{\cM_0}\\
    &+ \gamma\brac{T\rho(2\gamma^{-1}+2\etap+\etar)+2\log(2/\delta)},
\end{align*}
where $\Vs=V_{M^\star}$ and $c_0=\min\{\etap, 4\sigma^2\etar(1-e^{-D^2/8\sigma^2})/D^2\}$, as long as there exists $M\in\cM_0$ such that $\Ms$ is covered by $M$ (cf.~\cref{def:opt-cover}) and $V_M\geq \Vs-2\rho$.
\end{theorem}
The proof of~\cref{thm:mops-online} can be found in~\cref{appendix:proof-mops-online}. 

As a direct corollary of~\cref{thm:mops-online}, the posterior $\mu^t$ maintained in the \mops~algorithm (\cref{alg:MOPS}) achieves the following guarantee.
\begin{corollary}\label{cor:MOPS-online}
Given a $\rho$-optimistic covering $(\tPP,\cM_0)$, subroutine~\cref{equation:ops} within \cref{alg:MOPS} with $\etap=1/6, \etar=0.6, \gamma\geq 1$ and uniform prior $\mu^1=\Unif(\cM_0)$ achieves with probability at least $1-\delta$ that
\begin{align*}
    \MoveEqLeft \sum_{t=1}^T \E_{M\sim \mu^t}\brac{ f^\Ms(\pi_\Ms)-f^M(\pi_M) + \frac{\gamma}{6}\EE_{\pi\sim p^t}\brac{\tDH{\Ms(\pi),M(\pi)}}} \\
    & \le \frac{2T}{\gamma}+\gamma\brac{\log\abs{\cM_0}+3T\rho+2\log(2/\delta)}.
\end{align*}
\end{corollary}
\begin{proof}
Note that subroutine~\cref{equation:ops} in~\cref{alg:MOPS} is exactly an instantiation of \eqref{eqn:mops-online-def} with context $\xp^t$ sampled from distribution $p^t$ (which depends on $\mu^t$), observation $o^t$, reward $\br^t$, and $V_M=f^M(\pi_M)$. Furthermore, $\E_{M\sim\mu^t}\brac{\err_M^t}$ corresponds to $\E_{\hat{M}^t\sim \mu^t} \E_{\pi^t\sim p^t} \brac{\wtdH^2(M^\star(\pi^t), \hat{M}^t(\pi^t))}$ (cf.~\cref{corollary:restate-lemma-tempered-aggregation}).

Therefore, in order to apply~\cref{thm:mops-online}, we have to verify: as long as $M\in\cM_0$ covers the ground truth model $\Ms$ (i.e. $\nrm{\bR^{M_0}(o)-\Rs(o)}_{1}\leq \rho$ and $\tPP^{M}(\cdot|\pi)\geq \PPs(\cdot|\pi)$ for all $\pi$), it holds that $V_M\geq \Vs-2\rho$. We note that $\Vs\geq f^{\Ms}(\pi_M)$, thus
\begin{align}\label{eqn:proof-mops-diff-v}
\Vs- V_M\leq \sup_{\pi}\abs{f^M(\pi)-f^{\Ms}(\pi)} \leq \sup_{\pi}\DTV{\PPs(\cdot|\pi), \PP^M(\cdot|\pi)}+\rho
\end{align}
by definition.
An important observation is that, for $\pi\in\Pi$,
\begin{align}\label{eqn:proof-mops-tv-opt}
\DTV{\PPs(\cdot|\pi), \PP^M(\cdot|\pi)}=\sum_{o\in\cO} \brac{\PPs(o|\pi)-\PP^M(o|\pi)}_{+}\leq \sum_{o\in\cO} \tPP^M(o|\pi)-\PP^M(o|\pi)\leq \rho^2.
\end{align}
Therefore, $\Vs- V_M\leq \rho+\rho^2\leq 2\rho$. 
Now we can apply \cref{thm:mops-online} and plug in $\sigma^2=1/4$, $D^2=2$ as in \cref{corollary:restate-lemma-tempered-aggregation}. Choosing $\etap=1/6$, $\etar=0.6$, and $\gamma\ge 1$, we have $4\etar+\gamma^{-1}+\etap\le 3$, $8(1-2\etap-\etar)\ge 1/2$, and $c_0=1/6$.
This completes the proof.
\end{proof}

\subsubsection{Proof of Theorem~\ref{thm:mops-online}}
\label{appendix:proof-mops-online}

For all $t\in[T]$ define the random variable
$$
\Delta^t:=-\log \EE_{M\sim \mu^t}\left[ \exp\left( \gamma^{-1}(V_M-\Vs) + \etap\log \frac{\tPP^M(o^t|\xp^t)}{\PPs(o^t|\xp^t) }+\etar\delta^t_M\right) \right],
$$
where $\delta_M^t$ is defined as in \eqref{eqn:vovk-proof-delta}. 

Similar as the proof of~\cref{thm:vovk-finite}, we begin by noticing that
\begin{align}\label{eqn:mops-proof-decomp}
\log \EEt{\exp\left(-\Delta^t\right)}
=&
\log \EE_{M\sim \mu^t}\EEt{\exp\left( \gamma^{-1}(V_M-\Vs) + \etap\log \frac{\tPP^M(o^t|\xp^t)}{\PPs(o^t|\xp^t) } + \etar\delta^t_M \right)} \notag\\
\leq &
(1-2\etap-4\sigma^2\etar)\log \EE_{M\sim \mu^t} \brac{\exp\paren{\frac{V_M-\Vs}{\gamma(1-2\etap-4\sigma^2\etar)}}} \notag\\
&+2\etap\log \EE_{M\sim \mu^t}\EEt{ \exp\left( \frac12 \log \frac{\tPP^M(o^t|\xp^t)}{\PPs(o^t|\xp^t) } \right) } \notag\\
&+4\sigma^2\etar\log \EE_{M\sim \mu^t}\EEt{ \exp\left(\frac{1}{4\sigma^2}\delta^t_M \right) },
\end{align}
which is due to Jensen's inequality.
For the first term, we abbreviate $\eta_0=1-2\etap-4\sigma^2\etar$ and consider $a_M:=(\Vs-V_M)/\gamma\eta_0$. Then by the boundedness of $a_M$ and Hoeffding's Lemma,
\begin{align}
\label{eqn:ops-proof-value}
\EE_{M\sim \mu^t}\brac{\exp(-a_M)}
\leq \exp\paren{\frac{\EE_{M\sim\mu^t}[V_M]-\Vs}{\gamma\eta_0}} \cdot\exp\paren{\frac{1}{8\gamma^2\eta_0^2}}.
\end{align}
The second term can be bounded as in~\cref{eqn:vovk-proof-optlike}:
\begin{align}\label{eqn:mops-proof-likelihood}
\MoveEqLeft
    \log \EE_{M\sim \mu^t}\EE_{t}\brac{ \exp\left( \frac12 \log \frac{\tPP^M(o^t|\xp^t)}{\PPs(o^t|\xp^t) } \right) } 
    \leq \log\EE_{M\sim \mu^t} \brac{1-\frac12\EEt{\dH^2(\PP^M(\cdot|\xp^t), \PPs(\cdot|\xp^t))}+\rho} \notag\\
    &\leq -\frac12\EE_{M\sim \mu^t}\EEt{\dH^2(\PP^M(\cdot|\xp^t), \PPs(\cdot|\xp^t))}+\rho,
\end{align}
and the third term can be bounded by \cref{lemma:vovk-delta-rew} (similar to~\cref{eqn:vovk-proof-reward}):
\begin{align}\label{eqn:mops-proof-rew}
\begin{split}
\MoveEqLeft
    \log \EE_{M\sim \mu^t}\EEt{ \exp\left(\frac{1}{4\sigma^2}\delta^t_M \right) }
    \leq \log\EE_{M\sim \mu^t}\EEt{ \exp\left(\frac{1}{8\sigma^2}\ltwot{ \bR^M(o^t)-\Rs(o^t) } \right) }\\
    &\leq \log\EE_{M\sim \mu^t}\EEt{ 1-(1-e^{-D^2/8\sigma^2})/D^2\ltwot{ \bR^M(o^t)-\Rs(o^t) } }\\
    &\leq -(1-e^{-D^2/8\sigma^2})/D^2\EE_{M\sim \mu^t}\EEt{\ltwot{ \bR^M(o^t)-\Rs(o^t) }}.
\end{split}
\end{align}
Plugging~\cref{eqn:ops-proof-value},~\cref{eqn:mops-proof-likelihood}, and~\cref{eqn:mops-proof-rew}
into \eqref{eqn:mops-proof-decomp} gives
\begin{align}
\label{eqn:ops-proof-regrettomgf}
\begin{aligned}
-\log \EEt{\exp\left(-\Delta^t\right)}
\geq& \frac{\Vs-\EE_{M\sim\mu^t}[V_M]}{\gamma}-\frac{1}{8\gamma^2\eta_0} \\
& +
2\etap\brac{\frac12\EE_{M\sim \mu^t} \EEt{\dH^2(\PP^M(\cdot|\xp^t), \PPs(\cdot|\xp^t))} -\rho}\\
&+4\sigma^2\etar\paren{1-e^{-D^2/8\sigma^2}}/D^2\cdot \EE_{M\sim \mu^t}\EEt{\ltwot{ \bR^M(o^t)-\Rs(o^t) }} \\
 \geq & \EE_{M\sim\mu^t}\brac{\gamma^{-1}(\Vs-V_M) + c_0\err_M^t}-\frac{1}{8\gamma^2\eta_0}-2\etap\rho.
\end{aligned}
\end{align}
On the other hand, by \cref{lemma:concen}, we have with probability at least $1-\delta/2$ that
\begin{align}
\label{eqn:ops-proof-concentration}
    \sum_{t=1}^{T} \Delta^t+\log(2/\delta)
    \geq& \sum_{t=1}^{T}-\log \EEt{\exp\left(-\Delta^t\right)}.
\end{align}
It remains to bound $\sum_{t=1}^{T} \Delta^t$. By the update rule~\cref{eqn:mops-online-def} and a telescoping argument similar to \eqref{eqn:vovk-proof-sum-delta-eq}, we have
$$
\sum_{t=1}^{T} \Delta^t=-\log \EE_{M\sim\mu^1}\exp\left( \sum_{t=1}^{T} \gamma^{-1}(V_M-\Vs)+\etap\log \frac{\tPP^M(o^t|\xp^t)}{\PPs(o^t|\xp^t) }  + \etar\delta^t_M\right).
$$
The following argument is almost the same as the argument we make to bound \eqref{eqn:vovk-proof-sum-delta}. Fix a $M\in\cM_0$ that covers $\Ms$ and $V_M-\Vs\geq -2\rho$.
We bound the following moment generating function
\begin{align*}
\EE\brac{\exp\left( \sum_{t=1}^T \Delta^t \right)}
=& \EE\brac{\frac{1}{\EE_{M\sim \mu^1}\exp\left( \sum_{t=1}^{T}\gamma^{-1}(V_M-\Vs)+\etap\log \frac{\tPP_M(o^t|\xp^t)}{\PPs(o^t|\xp^t) }  + \etar\delta^t_M\right)}}\\
\leq& \abs{\cM_0}\EE\brac{\exp\left( -\sum_{t=1}^{T}\gamma^{-1}(V_M-\Vs)+\etap\log \frac{\tPP_M(o^t|\xp^t)}{\PPs(o^t|\xp^t) }  + \etar\delta^t_M \right)}\\
\leq&
\exp\paren{2T\gamma^{-1}\rho}\abs{\cM_0}\EE\brac{ \prod_{t=1}^{T} \exp\paren{-\etar\delta^t_M}}\\
\leq&
\exp\paren{2T\gamma^{-1}\rho+T\rho\etar}\abs{\cM_0},
\end{align*}
where the first inequality is because $\mu^1(M)=1/\abs{\cM_0}$, the second inequality is due to $\tPP^{M}\geq \PPs$, and the last inequality follows from the same argument as \eqref{eqn:vovk-proof-excess}: by \cref{lemma:vovk-delta-rew} we have $\cond{\exp(-\etar\delta^t_M)}{o^t}\leq \exp(\etar\rho)$, and applying this inequality recursively yields the desired result. 

Therefore, with at least with probability $1-\delta/2$,
\begin{align}
\label{equation:ops-delta-bound}
\sum_{t=1}^{T} \Delta^t\leq \log\abs{\cM_0}+T\rho\paren{2\gamma^{-1} + 2\etar} + \log(2/\delta).
\end{align}
Summing~\cref{eqn:ops-proof-regrettomgf} over $t\in[T]$, then taking union of~\cref{equation:ops-delta-bound} and~\cref{eqn:ops-proof-concentration} establish the theorem. 
\qed

\subsection{Bounding PSC by Hellinger decoupling coefficient}

The Hellinger decoupling coefficient is 
introduced by~\citet{agarwal2022model} as a structural condition for the sample efficiency of the \mops~algorithm. \footnote{We remark that \cite{agarwal2022model} defines the Hellinger decoupling coefficient in terms of $H$ general functions $\pi_{\rm gen}(h,\mu)$ that map a $\mu\in\Delta(\cM)$ to a distribution of policies. Here, we only consider a simplified version (\cref{def:hel-dcp}) by assuming $\pi_{\rm gen}$ is linear with respect to $\mu$.
}

\begin{definition}[Hellinger decoupling coefficient]\label{def:hel-dcp}
Given $\alpha\in(0,1)$, $\epsilon\geq0$, the coefficient $\dcpha(\cM,\oM,\epsilon,\piest)$ is the smallest positive number $c^h\geq 0$ such that for all $\mu\in\Delta(\cM)$,
\begin{align*}
    &\EE_{M\sim\mu}\EE^{\oM,\pi_M}\brac{Q^{M,\pi_M}_h(s_h,a_h)-r_h-V^{M, \pi_M}_{h+1}(s_{h+1})}\\
    &\leq \paren{c^h\ \EE_{M,M'\sim\mu}\EE^{\oM,\piest_{M'}}\brac{\DH{\PP^M_h(\cdot|s_h,a_h), \PP^{\oM}_h(\cdot|s_h,a_h)}+\abs{R^M_h(s_h,a_h)-R^{\oM}_h(s_h,a_h)}^2}}^{\alpha}+\epsilon.
\end{align*}
The Hellinger decoupling coefficient $\dcp$ is defined as
\begin{align*}
    \dcpa(\cM,\oM,\epsilon):=\paren{\frac1H\sum_{h=1}^H \dcpha(\cM,\oM,\epsilon,\piest)^{\alpha/(1-\alpha)}}^{(1-\alpha)/\alpha}.
\end{align*}
\end{definition}

We remark that the main difference between the PSC and the Hellinger decoupling coefficient is that, the Hellinger decoupling coefficient is defined in terms of Bellman errors and Hellinger distances within each layer $h\in[H]$ separately, whereas the PSC is defined in terms of the overall value function and Hellinger distances of the entire observable ($o, \br$).

The following result shows that the PSC can be upper bounded by the Hellinger decoupling coefficient, and thus is a more general definition.

\begin{proposition}[Bounding PSC by Hellinger decoupling coefficient]
\label{prop:psc-to-dcp}
For any $\alpha\in(0,1)$, we have
    \begin{align*}
        \psc_{\gamma}(\cM,\oM)\leq H\inf_{\epsilon\geq 0}\paren{ (1-\alpha)\paren{\frac{2\alpha H\dcpa(\cM,\oM,\epsilon,\piest)}{\gamma}}^{\alpha/(1-\alpha)}+\epsilon }.
    \end{align*}
\end{proposition}
\begin{proof}
Fix $\cM$, $\oM\in\cM$, and $\alpha\in(0,1)$. Consider
\begin{align*}
    \cE^{M;\oM}_h(s_h,a_h):=\DH{\PP^M_h(\cdot|s_h,a_h), \PP^{\oM}_h(\cdot|s_h,a_h)}+\abs{R^M_h(s_h,a_h)-R^{\oM}_h(s_h,a_h)}^2.
\end{align*}
By the definition of $\wtdH$ and \cref{lemma:Hellinger-cond}, we have for any $h\in[H]$ that
\begin{align*}
    \EE^{\oM,\piest_{M'}}\brac{\cE^{M;\oM}_h(s_h,a_h)}\leq 2\tDH{\oM(\piest_{M'}), M(\piest_{M'})}.
\end{align*}

Fix $\eps\ge 0$ and and write $c^h=\dcpha(\cM,\oM,\epsilon,\piest)$. For any $\mu\in\Delta(\cM)$, we have
\begin{align*}
&\EE_{M\sim\mu}\brac{f^M(\pi_M)-f^{\oM}(\pi_M)}\\
=&\sum_{h=1}^H\EE_{M\sim\mu}\EE^{\oM,\pi_M}\brac{Q^{M,\pi_M}_h(s_h,a_h)-r_h-V^{M, \pi_M}_{h+1}(s_{h+1})}\\
\leq &\sum_{h=1}^H\paren{c^h\ \EE_{M,M'\sim\mu}\EE^{\oM,\piest_{M'}}\brac{\cE^{M;\oM}_h(s_h,a_h)}}^{\alpha}+H\epsilon\\
\leq& \sum_{h=1}^H\paren{c^h}^{\alpha}\paren{ 2\EE_{M,M'\sim\mu}\brac{\tDH{\oM(\piest_{M'}), M(\piest_{M'})}}}^{\alpha}+H\epsilon\\
\leq& \gamma\EE_{M,M'\sim\mu}\brac{\tDH{\oM(\piest_{M'}), M(\piest_{M'})}}+(1-\alpha)\paren{\frac{2\alpha H}{\gamma}}^{\alpha/(1-\alpha)}\sum_{h=1}^H \paren{c^h}^{\alpha/(1-\alpha)}+H\epsilon \\
=& \gamma\EE_{M,M'\sim\mu}\brac{\tDH{\oM(\piest_{M'}), M(\piest_{M'})}}+(1-\alpha)H^{1/(1-\alpha)}\paren{\frac{2\alpha \dcpa(\cM,\oM,\epsilon)}{\gamma}}^{\alpha/(1-\alpha)}+H\epsilon,
\end{align*}
where the last inequality is due to the fact that for all $x,y\geq 0$, $\alpha\in(0,1)$,
$$
x^{\alpha}y^{\alpha}\leq \alpha\cdot\frac{\gamma x}{\alpha H} + (1-\alpha)\paren{\frac{\alpha H y}{\gamma} }^{\alpha/(1-\alpha)}
=\frac{\gamma x}{H} + (1-\alpha)\paren{\frac{\alpha H}{\gamma} }^{\alpha/(1-\alpha)} \cdot y^{\alpha/(1-\alpha)},
$$
by weighted AM-GM inequality.
\end{proof}

\subsection{Proof of Proposition~\ref{prop:psc-belrep}}\label{proof-psc-belrep}
Fix a $\oM\in\cM$ and $\mu\in\Delta(\cM)$. By the definition of \belrep~(\cref{definition:err-rep}), for any $\eta> 0$,
\begin{align*}
    \MoveEqLeft \EE_{M\sim\mu} \brac{f^{M}(\pi_{M}) - f^{\oM}(\pi_{M})} 
     \leq \EE_{M\sim\mu} \brac{ \sum_{h=1}^H \EE_{\tau_h\sim \dtr_h(M; \oM)}\brac{\cE^{M; \oM}_h(\tau_h)} } \\ 
    =&~ \sum_{h=1}^H   \underbrace{ \EE_{M\sim\mu,\tau_h\sim \dtr_h(M;\oM)}\brac{ \cE_h^{M; \oM}(\tau_h) } - \eta \EE_{M\sim\mu}\EE_{M'\sim\mu,\tau_h\sim \dtr_h(M';\oM)} \brac{ \abs{ \cE_h^{M; \oM}(\tau_h) }^2 } }_{\leq \eta^{-1}\dimc(\cG^{\oM}_h, \cQ^{\oM}_h, \eta)} \\
    &~ + \eta \EE_{M\sim\mu}\EE_{M'\sim\mu,\tau_h\sim \dtr_h(M';\oM)} \brac{ \abs{ \cE_h^{M; \oM}(\tau_h) }^2 }\\
    \leq&~ \sum_{h=1}^H \eta^{-1}\dimc(\cG^{\oM}_h, \cQ^{\oM}_h, \eta) + \eta \EE_{M,M'\sim\mu}\brac{ \sum_{h=1}^H \EE_{\tau_h\sim \dtr_h(M';\oM)} \brac{ \abs{ \cE_h^{M; \oM}(\tau_h) }^2 } }\\
    \leq&~ \sum_{h=1}^H \eta^{-1}\dimc(\cG^{\oM}_h, \cQ^{\oM}_h, \eta) + \eta L \EE_{M,M'\sim\mu}\brac{ \tDH{ \oM(\piest_{M'}), M(\piest_{M'}) } }.
\end{align*}
Now taking $\eta=\gamma/L$ completes the proof.
\qed

\subsection{Proof of Proposition~\ref{prop:dec-to-psc}}
\label{appendix:proof-dec-to-psc}

First, we consider the on-policy case $\piest_M=\pi_M$ for simplicity. In this case, we show the following bound on DEC (stronger than the bound on PACDEC):
\begin{align}\label{eqn:dec-to-psc}
    \dec_{\gamma}(\cM) \le \sup_{\oM \in \cM} \psc_{\gamma/6}(\cM, \oM) + 2(H+1)/\gamma.
\end{align}
Our overall argument is to bound the DEC by strong duality and the probability matching argument similar as~\citep[Section 4.2]{foster2021statistical}, after which we show that the resulting quantity is related nicely to the PSC. 

By definition, it suffices to bound $\dec_\gamma(\cM, \omu)$ for any fixed $\omu\in\Delta(\cM)$. We have
\begin{align*}
    & \quad \dec(\cM, \omu) = \inf_{p\in\Delta(\Pi)} \sup_{M\in\cM} \EE_{\oM\sim\omu} \EE_{\pi \sim p}\brac{ f^M(\pi_M)-f^M(\pi)-\gamma  \wtdH^2(M(\pi),\oM(\pi)) } \\
    & = \inf_{p\in\Delta(\Pi)} \sup_{\mu\in\Delta(\cM)} \EE_{M\sim\mu,\oM\sim\omu} \EE_{\pi \sim p}\brac{ f^M(\pi_M)-f^M(\pi)-\gamma  \wtdH^2(M(\pi),\oM(\pi)) } \\
    & = \sup_{\mu\in\Delta(\cM)} \inf_{p\in\Delta(\Pi)} \EE_{M\sim\mu,\oM\sim\omu} \EE_{\pi \sim p}\brac{ f^M(\pi_M)-f^M(\pi)-\gamma  \wtdH^2(M(\pi),\oM(\pi)) },
\end{align*}
where the last equality follows by strong duality (\cref{thm:strong-dual}).

Now, fix any $\mu\in\Delta(\cM)$, we pick $p\in\Delta(\Pi)$ by probability matching: $\pi\sim p$ is equal in distribution to $\pi=\pi_{M'}$ where $M'\sim\mu$ is an independent copy of $M$. For this choice of $p$, the quantity inside the sup-inf above is
\newcommand{\cga}{\frac{5\gamma}{6}}
\newcommand{\cgb}{\frac{\gamma}{6}}
\newcommand{\cgc}{\gamma/6}
\begin{align*}
    &
    \EE_{M\sim\mu,M'\sim\mu} \EE_{\oM\sim\omu} \brac{ f^M(\pi_M)-f^M(\pi_{M'})-\gamma  \wtdH^2(M(\pi_{M'}),\oM(\pi_{M'})) } \\
    = &\EE_{M\sim\mu,M'\sim\mu} \EE_{\oM\sim\omu} \brac{ f^M(\pi_M)- f^{\oM}(\pi_{M'}) - \cga\wtdH^2(M(\pi_{M'}),\oM(\pi_{M'})) }  \\
    & +\EE_{M\sim\mu} \EE_{\oM\sim\omu}\brac{ f^{\oM}(\pi_{M'}) - f^M(\pi_{M'}) - \cgb\wtdH^2(M(\pi_{M'}),\oM(\pi_{M'})) } \\
    \stackrel{(i)}= &
    \EE_{M\sim\mu,M'\sim\mu} \EE_{\oM\sim\omu} \brac{ f^M(\pi_M)- f^{\oM}(\pi_{M}) - \cga\wtdH^2(M(\pi_{M'}),\oM(\pi_{M'})) }\\
    &+
    \EE_{M\sim\mu} \EE_{\oM\sim\omu}\brac{ f^{\oM}(\pi_{M'}) - f^M(\pi_{M'}) - \cgb\wtdH^2(M(\pi_{M'}),\oM(\pi_{M'})) } \\
    \stackrel{(ii)} \leq &\underbrace{\EE_{M\sim\mu,M'\sim\mu} \EE_{\oM\sim\omu} \brac{ f^M(\pi_M)- f^{\oM}(\pi_{M}) - \cgb\wtdH^2(\oM(\pi_{M'}),M(\pi_{M'})) }}_{\le \E_{\oM\sim\omu}\brac{\psc_{\cgc}(\cM, \oM)}} \\
    & + \EE_{M\sim\mu,M'\sim\mu} \EE_{\oM\sim\omu}\brac{ \sqrt{H+1}\wtdH(M(\pi_{M'}), \oM(\pi_{M'})) - \cgb\wtdH^2(M(\pi_{M'}),\oM(\pi_{M'})) } \\
    \stackrel{(iii)}{\le}  &\sup_{\oM \in \cM} \psc_{\gamma/6}(\cM, \oM) + \frac{2(H+1)}{\gamma}.
\end{align*}
Above, (i) uses the fact that $f^{\oM}(\pi_{M'})$ is equal in distribution to $f^{\oM}(\pi_M)$ (since $M\sim\mu$ and $M'\sim\mu$); 
(ii) uses \cref{lemma:diff-reward} and \cref{lemma:diff-hellinger}; 
(iii) uses the inequality $\sqrt{H+1}x\le \cgb x^2 + \frac{3(H+1)}{2\gamma}$ for any $x\in\R$. Finally, by the arbitrariness of $\omu\in\Delta(\Pi)$, we have shown that $\odec_\gamma(\cM)\le\psc_{\cgc}(\cM)+2(H+1)/\gamma$. This completes the proof of \cref{eqn:dec-to-psc}. 

The upper bound of $\eec$ follows similarly, because we can prove that
\begin{align*}
    \eec_\gamma(\cM,\omu)\leq \sup_{\mu\in\Delta(\cM)}\EE_{M\sim\mu,M'\sim\mu} \EE_{\oM\sim\omu} \brac{ f^M(\pi_M)-f^M(\pi_{M'})-\gamma  \wtdH^2(M(\piest_{M'}),\oM(\piest_{M'})) },
\end{align*}
using the same probability matching argument,
and then repeat the proof above. %
\qed

\subsection{Proof of Proposition~\ref{prop:psc-lower}}\label{appdx:proof-psc-lower}

In the following, we resume the notations and definitions in \cref{appdx:proof-rev-bandit}. Notice that for each model $M\in\cM$, $\pi_M\in\cA_0$, and hence if the agent is forced not to take the ``revealing actions'' in $\Arev$, then intuitively learning $\cM$ is equivalent to learning a class of $2^n$-arms bandit. The lower bound of regret of \mops~follows immediately by such a reduction to multi-arm bandit.

To prove a lower bound of PSC rigorously, we can consider $\mu\in\Delta(\cM)$ be the distribution of $M_{(\Delta,a)}, a\sim\Unif(\cA_0)$, for a fixed $\Delta>0$. Then by \eqref{eqn:eec-vs-dec-comp},
\begin{align*}
    \psc_\gamma(\cM,\oM)
    \geq&~ \EE_{M\sim \mu}\EE_{M' \sim \mu}\left[f^{M}(\pi_M)-f^{\oM}(\pi_M)-\gamma  \wtdH^2(\oM(\piest_{M'}),M(\piest_{M'}))\right] \\
    =&~ \Delta - \gamma \EE_{M\sim \mu}\EE_{M' \sim \mu}\brac{ \Delta^2 \II(\pi_{M'}=\pi_M) } \\
    =&~ \Delta - \frac{\gamma}{2^n}\Delta^2.
\end{align*}
Therefore, we can take $\Delta=\min\set{ \frac{2^n}{2\gamma}, \frac{1}{3} }$, which gives
\begin{align*}
    \psc_\gamma(\cM,\oM)\geq \min\set{ \frac{2^n}{4\gamma}, \frac{1}{6} }.
\end{align*}
This is the desired result.
\qed

\section{Proofs for Section~\ref{section:omle}}\label{appendix:OMLE}

\subsection{Algorithm OMLE}

In this section, we present the Algorithm \textsc{OMLE} (\cref{alg:OMLE}), and then state the basic guarantees of its confidence sets, as follows.

\begin{algorithm}[t]
	\caption{\textsc{OMLE}} 
	\begin{algorithmic}[1]
	\label{alg:OMLE}
	\STATE \textbf{Input:} Parameter $\beta>0$.
	\STATE Initialize confidence set $\cM^1= \cM$.
	\FOR{$t=1,\ldots,T$}
    \STATE Compute $(M^t, \pi^t) = \argmax_{M\in \cM^t, \pi\in\Pi} f^M(\pi)$.
    \STATE Execute $\pi^t$ and observe $\tau^t=(o^t,\br^t)$.
    \STATE Update confidence set with \eqref{eqn:MLE-risk}:
    \begin{align*}
        \cM^{t+1} \defeq \set{ M\in\cM: \cL_{t+1}(M)\ge \max_{M'\in\cM} \cL_{t+1}(M') - \beta }.
    \end{align*}
    \ENDFOR
   \end{algorithmic}
\end{algorithm}

\begin{theorem}[Guarantee of MLE]\label{thm:MLE}
By choosing $\beta\geq3\est(\cM, 2T) +3\log(1/\delta)$, \cref{alg:OMLE} achieves the following with probability at least $1-\delta$: 
for all $t\in[T]$, $\Ms\in\cM^t$, and it holds that 
\begin{align*}
    \sum_{s<t} \wtdH^2(\Ms(\pi^s),M(\pi^s))\leq 2\beta+6\est(\cM, 2T)+6\log(1/\delta)\leq 4\beta, \qquad\forall M\in\cM^t.
\end{align*}
\end{theorem}

\begin{proofof}[thm:MLE]
The proof of \cref{thm:MLE} is mainly based on the following lemma. 
\begin{lemma}\label{lemma:MLE}
    Fix a $\rho>0$. With probability at least $1-\delta$, it holds that for all $t\in[T]$ and $M\in\cM$,
    \begin{align*}
        \sum_{s<t} \tDH{\Ms(\pi^s),M(\pi^s)}\leq 2(\cL_t(M^\star)-\cL_t(M))+6\log\frac{\cN(\cM,\rho)}{\delta}+8T\rho.
    \end{align*}
\end{lemma}

Now, we can take $\rho$ that attains $\est(\cM,2T)$ and apply \cref{lemma:MLE}. Conditional on the success of \cref{lemma:MLE}, it holds that for all $t\in[T]$ and $M\in\cM$,
\begin{align*}
    \cL_t(M)-\cL_t(M^\star)\leq 3\est(\cM,2T)+3\log(1/\delta).
\end{align*}
Therefore, our choice of $\beta$ is enough to ensure that $M^\star\in\cM^t$. Then, for $M\in\cM^t$, we have
$$
\cL_t(M)\geq \max_{M'\in\cM} \cL_t(M')-\beta\geq \cL_t(M^\star)-\beta.
$$
Applying \cref{lemma:MLE} again completes the proof. %
\end{proofof}

The proof of \cref{lemma:MLE} is mostly a direct adaption of the proof of \cref{thm:vovk-finite} and \cref{thm:vovk-cover}.

\renewcommand{\PPs}{\PP^\star}
\renewcommand{\Rs}{\bR^\star}
\begin{proofof}[lemma:MLE]
For simplicity, we denote $\PPs(o|\pi):=\PP^{\Ms,\pi}(o)$ and $\Rs(o):=\bR^M(o)$.

We pick a $\rho$-optimistic covering $(\tPP,\cM_0)$ of $\cM$ such that $\abs{\cM_0}=\cN(\cM,\rho)$. 

Recall that the MLE functional is defined as
\begin{align*}%
    \cL_t(M):=\sum_{s=1}^{t-1} \log \P^{M, \pi^s}(o^s) -\ltwot{ \br^s-\bR^M(o^s) }.
\end{align*}
For $M\in\cM_0$, we consider
\begin{align*}
\ell^t_M:=\log\frac{\tPP^M(o^t|\pi^t)}{\PPs(o^t|\pi^t)}+\delta^t_M,\qquad
\delta^t_M:=\ltwot{\br^t-\Rs(o^t)}-\ltwot{\br^t-\bR^M(o^t)},
\end{align*}
where the definition of $\delta$ agrees with \eqref{eqn:vovk-proof-delta}.
We first show that with probability at least $1-\delta$, for all $M\in\cM_0$ and all $t\in[T]$,
\begin{align}\label{eqn:MLE-proof-1}
     \sum_{s<t} 1-\EE_{s}\brac{\sqrt{\frac{\tPP^M(o^s|\pi^s)}{\PPs(o^s|\pi^s)}}}+\frac12\EE_s\brac{\ltwot{\Rs(o^s)-\bR^M(o^s)}}\leq -\sum_{s<t} \ell^s_M + 3\log\frac{\abs{\cM_0}}{\delta}.
\end{align}
This is because
by \cref{lemma:concen}, it holds that with probability at least $1-\delta$, for all $t\in[T]$ and $M\in\cM_0$,
\begin{align*}
    \sum_{s<t} -\frac{1}{3}\ell^s_M+\log(|\cM_0|/\delta)\geq \sum_{s<t} -\log\EE_s\brac{\exp\paren{\frac{1}{3}\ell^s_M }}.
\end{align*}
Further,
\begin{align*}
    -\log\EE_s\brac{\exp\paren{\frac{1}{3}\ell^s_M }}
    \geq& -\frac{2}{3}\log \EE_s\brac{\exp\paren{\frac{1}{2}\log\frac{\tPP^M(o^s|\pi^s)}{\PPs(o^s|\pi^s)} }}-\frac{1}{3}\log \EE_s\brac{\exp\paren{ \delta^s_M }}\\
    \geq& \frac{1}{3}\paren{1-\EE_{s}\brac{\sqrt{\frac{\tPP^M(o^s|\pi^s)}{\PPs(o^s|\pi^s)}}}}+\frac16\EE_s\brac{\ltwot{ \Rs(o^s)-\bR^M(o^s) }},
\end{align*}
where the second inequality is due to the fact that $-\log x\geq 1-x$ and \cref{lemma:vovk-delta-rew} (with $\sigma^2=1/4$).
Hence \eqref{eqn:MLE-proof-1} is proven. 

Now condition on the success of \eqref{eqn:MLE-proof-1} for all $M_0\in\cM_0$. Fix a $M\in\cM$, there is a $M_0\in\cM_0$ such that $M$ is covered by $M_0$ (i.e. $\lone{\bR^{M_0}(o)-\bR^{M}(o)}\leq \rho$ and $\tPP^{M_0}(\cdot|\pi)\geq \PP^M(\cdot|\pi)$ for all $\pi$). Notice that $\sum_{o\in\cO} \tPP^{M_0}(o|\pi)\leq 1+\rho^2$, 
and therefore $\lone{\tPP^{M_0}(\cdot|\pi)-\PP^M(\cdot|\pi)}\leq \rho^2$.
Then the first term in \eqref{eqn:MLE-proof-1} (plug in $M_0$) can be lower bounded as
\begin{align*}
    1-\EE_{s}\brac{\sqrt{\frac{\tPP^{M_0}(o^s|\pi^s)}{\PPs(o^s|\pi^s)}}}\geq \frac12\EE_s\brac{\DH{\PP^M(\cdot|\pi^s), \PPs(\cdot|\pi^s)}}-\rho,
\end{align*}
by \eqref{eqn:vovk-proof-optlike}. For the second term,
by the fact that $\bR\in[0,1]^H$ and $\lone{\bR^{M_0}(o)-\bR^{M}(o)}\leq \rho$, we have
\begin{align*}
    \EE_s\brac{\ltwot{ \Rs(o^s)-\bR^M(o^s) }}
    \geq &\EE_s\brac{\ltwot{ \Rs(o^s)-\bR^M(o^s) }}-2\rho.
\end{align*}
Similarly, $\delta_{M_0}^s\geq \delta_{M}^s-2\rho$, and hence $-\sum_{s<t} \ell^s_{M_0}\leq \cL^t(\Ms)-\cL^t(M)+2T\rho$, which completes the proof.
\end{proofof}

\subsection{Proof of Theorem~\ref{theorem:OMLE}}\label{appendix:proof-omle}

In the following, we show the following general result.
\begin{theorem}[Full version of \cref{theorem:OMLE}]
\label{theorem:OMLE-full}
Choosing $\beta\geq 3\est(\cM, 2T) +3\log(1/\delta)$, with probability at least $1-\delta$, Algorithm~\ref{alg:OMLE} achieves
\begin{align*}
    \sum_{t=1}^T \brac{ f^{\Ms}(\pi_{\Ms})-f^{\Ms}(\pi_{M^t}) } \le  \inf_{\gamma>0}\set{ T\cdot \omlec_{\gamma,T}(\cM, M^\star) + 4\gamma\beta}.
\end{align*}
\end{theorem}
Especially, when $\cM$ is finite, we can take $\beta=3\log(\abs{\cM}/\delta)$ (because $\est(\cM,2T)\leq \log\abs{\cM}$), and \cref{theorem:OMLE-full} implies \cref{theorem:OMLE} directly.

\begin{proofof}[theorem:OMLE-full]
Condition on the success of \cref{thm:MLE}. Then, for $t\in[T]$, it holds that $\Ms\in\cM^t$. Therefore, by the choice of $M^t$, it holds that $f^{M^t}(\pi_{M^t})\geq f^{\Ms}(\pi_{\Ms})$. Then,
\begin{align*}
\MoveEqLeft
\sum_{t=1}^T \brac{ f^{\Ms}(\pi_{\Ms})-f^{\Ms}(\pi_{M^t}) }
\leq \sum_{t=1}^T \brac{ f^{M^t}(\pi^t)-f^{\Ms}(\pi_{M^t}) }\\
=& T\cdot \underbrace{ \set{ \frac{1}{T} \sum_{t=1}^T \brac{ f^{M^t}(\pi_{M^t})-f^{\Ms}(\pi_{M^t}) }
- \frac{\gamma}{T} \cdot \brac{ 1\vee \max_{t\in[T]} \sum_{s\le t-1} \wtdH^2(\Ms(\piest_{M^s}), M^t(\piest_{M^s}))  } } }_{\text{bounded by }\omlec_{\gamma,T}(\cM,\Ms)}\\
&+\gamma \cdot \underbrace{ \brac{ 1\vee \max_{t\in[T]} \sum_{s\le t-1} \wtdH^2(\Ms(\piest_{M^s}), M^t(\piest_{M^s})) } }_{\text{bounded by }4\beta}\\
\leq & T\omlec_{\gamma,T}(\cM, M^\star) + 4\gamma\beta.
\end{align*}
Taking $\inf_{\gamma>0}$ completes the proof. %
\end{proofof}

\subsection{Proof of Proposition~\ref{proposition:omlec-be}}\label{appdx:proof-omlec}
Fix any set of models $\set{M^k}_{k\in[K]}\in\cM$. By the definition of \belrep, we have
\begin{align*}
    \frac{1}{K} \sum_{k=1}^K \brac{f^{M^k}(\pi_{M^k}) - f^{\oM}(\pi_{M^k})} 
    \le \frac{1}{K} \sum_{h=1}^H   \sum_{k=1}^K \EE_{\tau_h\sim\dtr(M^k;\oM)}\brac{ \cE_h^{M^k; \oM}(\tau_h) },
\end{align*}
and
\begin{align*}
    \sum_{h=1}^H\sum_{t=1}^{k-1} \EE_{\tau_h\sim\dtr(M^t;\oM)}\abs{ \cE_h^{M^k; \oM}(\tau_h) }^2 \leq \LM\sum_{t=1}^{k-1} \wtdH^2(M^k(\piest_{M^t}), \oM(\piest_{M^t}))\leq \LM\wt{\beta},
\end{align*}
where we define
\begin{align*}
    \wt{\beta}\defeq 1\vee\max_{k\in[K]} \sum_{t=1}^{k-1} \wtdH^2(M^k(\piest_{M^t}), \oM(\piest_{M^t})).
\end{align*}
To proceed, we invoke the following result of \citet[Lemma 41]{jin2021bellman}.

\begin{lemma}\label{lem:eluder}
Suppose that $\cF\subset(f:\cX\to[-R,R])$, and $(x_1,f_1),\cdots,(x_K,f_K)\in\cX\times\cF$ is a sequence such that $\sum_{t<k} f_k(x_t)^2\leq \beta$. Then it holds that
\begin{align*}
    \sum_{t=1}^T \abs{f_t(x_t)}\leq 2\sqrt{\eluder(\cF,\Delta)T\beta}+R\min\set{\eluder(\cF,\Delta),T}+T\Delta.
\end{align*}
\end{lemma}

In the following, we denote $d=\max_{h}\eluder(\cG_h^{\oM}, \Delta)$ and suppose $R$ is a uniform upper bound for functions in $\cG^{\oM}_h$ for all $h\in[H]$.
To apply \cref{lem:eluder}, we consider the following procedure. Fix a large integer $N$, for each $(k,h)\in[K]\times[H]$, we sample $N$ i.i.d $\tau_h^{k,n}\sim \dtr(M^k;\oM)$. Then by Hoeffding inequality and union bound, we have with probability at least $\frac12$,
\begin{align*}
\abs{ \EE_{\tau_h\sim\dtr(M^k;\oM)} \cE_h^{M^k; \oM}(\tau_h) - \frac{1}{N}\sum_{n=1}^N  \cE_h^{M^k; \oM}(\tau_h^{k,n}) } 
\leq&~ \eps_N, \forall 1\leq k\leq K,
\end{align*}
\begin{align}\label{eqn:omlec-bound-tri}
\abs{ \EE_{\tau_h\sim\dtr(M^t;\oM)}\abs{ \cE_h^{M^k; \oM}(\tau_h) }^2 -  \frac{1}{N}\sum_{n=1}^N  \abs{ \cE_h^{M^k; \oM}(\tau_h^{t,n}) }^2  } \leq&~ \eps_N, \forall 1\leq t\leq k\leq K,
\end{align}
where $\eps_N=\bigO{\sqrt{\frac{R^2\log(KN)}{N}}}$.
Thus, for $h\in[H]$, we consider the sequence given by $(f_{Nk+n},x_{Nk+n})=(\cE_h^{M^k;\oM},\tau_h^{k,n})\in\cG_h^{\oM}\times\cT_h$, which satisfies that for all $r\in[KN]$, $k=\floor{r/N}$,
\begin{align*}
    \sum_{s=1}^{r-1} \abs{f_r(x_s)}^2
    \leq \sum_{t\leq k} \sum_{n=1}^N \abs{ \cE_h^{M^k; \oM}(\tau_h^{t,n}) }^2 
    \leq N(L\wt{\beta}+4L+k\eps_N)=:\beta,
\end{align*}
where we use \eqref{eqn:omlec-bound-tri} and the fact that $\wtdH\leq 4$.
Applying \cref{lem:eluder} then yields
\begin{align*}
    \frac{1}{NK}\sum_{r=1}^{NK} \abs{f_r(x_r)}\leq 2\sqrt{\frac{d\beta}{NK}}+\frac{d}{NK}+\Delta.
\end{align*}
Plugging in the definition of $\beta$ and $(f_r,x_r)$ gives
\begin{align*}
    \frac{1}{K}  \sum_{k=1}^K \EE_{\tau_h\sim\dtr(M^k;\oM)}\brac{ \cE_h^{M^k; \oM}(\tau_h) } %
    \leq 2\sqrt{\frac{d(L\wt{\beta}+4L+K\eps_N)}{K}}+\frac{d}{NK}+\Delta+\eps_N.
\end{align*}
Hence, the above inequality holds deterministically (for all $h$ and $N$). Letting $N\to\infty$ and taking summation over $h$, we obtain
\begin{align*}
    \frac{1}{K} \sum_{h=1}^H \sum_{k=1}^K \EE_{\tau_h\sim\dtr(M^k;\oM)}\brac{ \cE_h^{M^k; \oM}(\tau_h) }
    \leq H\cdot\brac{2\sqrt{\frac{dL(\wt{\beta}+4)}{K}}+\Delta}.
\end{align*}
To finalize, we have
\begin{align*}
    \frac{1}{K} \sum_{k=1}^K \brac{f^{M^k}(\pi_{M^k}) - f^{\oM}(\pi_{M^k})} - \frac{\gamma}{K}\wt{\beta} \leq&~ 2H\sqrt{\frac{d(L\wt{\beta}+4)}{K}}+H\Delta- \frac{\gamma}{K}\wt{\beta} \\
    \leq &~
    2H\sqrt{\frac{5dL\wt{\beta}}{K}}-\frac{\gamma}{K}\wt{\beta}+H\Delta \\
    \leq&~ 
    \frac{5H^2dL}{\gamma}+H\Delta,
\end{align*}
where the last inequality follow from AM-GM inequality. By definition of $\omlec_{\gamma, K}$, taking $\inf_{\Delta>0}$ and $\inf_{\set{M^k}}$ completes the proof.
\qed

\subsection{Proof of Proposition~\ref{prop:mlec-to-psc-demo}}\label{appdx:proof-mlec-to-psc}
\newcommand{\tomlec}{\widetilde{\omlec}}
\newcommand{\hmlec}{\widehat{\mlec}}
In this section, we state and prove a more general result.
Define
\begin{align*}
    \tomlec_{\gamma}(\cM,\oM)\defeq&~ \inf_{K\geq 1} \paren{  \mlec_{\gamma,K}(\cM,\oM)+\gamma\sqrt{\frac{2\log K}{K}} }, \\
    \hmlec_{\gamma}(\cM,\oM)\defeq&~ \frac{\sup_{y\leq \gamma} \tomlec_y(\cM,\oM)}{ \gamma },
\end{align*}
and let
\begin{align}\label{def:tomlec}
    \mlec_{\gamma}(\cM,\oM) = \inf_{\eps\in[0,1]} \paren{ \eps+4\log^2(1/\eps)\cdot\hmlec_{\gamma}(\cM,\oM) }
\end{align}
Then, the following proposition bounds PSC in terms of the above \emph{enveloped} version of MLEC. \cref{prop:mlec-to-psc-demo} is then an immediate corollary. %

\begin{proposition}\label{prop:mlec-to-psc}
It holds that
\begin{align*}
    \psc_\gamma(\cM,\oM)\leq \mlec_\gamma(\cM,\oM), \qquad \forall \gamma>0,
\end{align*}
\end{proposition}

\begin{proofof}[prop:mlec-to-psc]
We only need to prove the following fact: for any fixed $\oM\in\cM$ and $\mu\in\Delta(\cM)$, it holds that
\begin{align*}
    \EE_{M\sim \mu}\brac{ f^{M}(\pi_M)-f^{\oM}(\pi_M) } - \gamma \EE_{M\sim \mu}\EE_{M' \sim \mu} \wtdH^2(\oM(\piest_{M'}),M(\piest_{M'})) \leq  \mlec_\gamma(\cM,\oM).
\end{align*}
For notational simplicity, in the following we denote
\begin{align*}
    \delta(M)=f^{M}(\pi_M)-f^{\oM}(\pi_M), \qquad
    \Delta(M,M')=\wtdH^2(\oM(\piest_{M'}),M(\piest_{M'})).
\end{align*}
Then, we need to upper bound the quantity $\EE_{M\sim \mu} \delta(M)-\gamma\EE_{M,M'\sim \mu}\Delta(M,M')$, using the fact below:
\begin{align}\label{eqn:tomlec-proof}
    \frac{1}{K} \sum_{k=1}^K \delta(M^k)\leq \mlec_{\gamma,K}(\cM,\oM)+\frac{\gamma}{K}\max_{k\in[K]} \sum_{t<k} \Delta(M^k,M^t), \qquad \forall M^1,\cdots,M^K\in\cM.
\end{align}
For any $x>0$, consider the event $\cE_x=\set{ M: \delta(M)\geq x }$, $p_x=\mu(\cE_x)$. We first proceed to upper bound $p_x$ for all $x\in[0,1]$. 

Let $\mu_x=\mu(\cdot|\cE_x)$. Then%
\begin{align*}
\EE_{M\sim \mu_x}\EE_{M'\sim \mu_x} \Delta(M,M')\leq \frac{1}{p_x^2}\EE_{M\sim \mu}\EE_{M'\sim \mu} \Delta(M,M')=:\sigma_x.
\end{align*}
Further consider the following event:
\begin{align*}
\cE_x'=\set{ M: \EE_{M'\sim \mu_x} \Delta(M,M') \leq 2\sigma_x }.
\end{align*}
By Markov's inequality, we know $\mu_x(\cE_x')\geq \frac{1}{2}$. Thus, we consider $\tmu_x=\mu_x(\cdot|\cE_x')$. Then for $M\sim \tmu_x$, almost surely
\begin{align*}
    \delta(M)\geq x, \qquad
    \EE_{M'\sim \tmu_x} \Delta(M,M') \leq 4\sigma_x.
\end{align*}

Next, for any fixed integer $K\geq 1$, we generate i.i.d. samples $M^1, \cdots, M^K\sim \tmu_x$. Notice that for each $k\geq 2$, $M^1,\cdots,M^{k-1}$ are i.i.d. samples from $\tmu_x$ conditional on $M^k$. Therefore, for any $\delta\in(0,1)$, with probability at least $1-\delta$, the following holds simultaneously:
\begin{align*}
    \frac{1}{k-1}\sum_{t<k} \Delta(M,M^t)-\EE_{M'\sim\tmu_x}\brac{\Delta(M,M')}\leq \sqrt{\frac{2\log(K/\delta)}{k-1}},\qquad \forall 2\leq k\leq K.
\end{align*}
Letting $\delta\uparrow 1$, we have shown that there exists $M^1,\cdots,M^K\in\cM$ such that 
\begin{align}\label{eqn:psc-to-mlec-exist}
    \delta(M^k)\geq x, \qquad \frac{1}{K}\sum_{t<k} \Delta(M^k,M^t) \leq 4\sigma_x+\sqrt{\frac{2\log K}{K}}, \qquad \forall k\in[K].
\end{align}
Then \cref{eqn:tomlec-proof} implies that%
\begin{align*}
    x\leq \frac{1}{K}\sum_{k=1}^K \delta(M^k) \leq \mlec_{\gamma,K}(\cM,\oM)+\gamma \paren{ 4\sigma_x+\sqrt{\frac{2\log K}{K}}}.
\end{align*}
Taking $\inf_{K\geq 1}$, we now derive
\begin{align*}
    x\leq \tomlec_{\gamma}(\cM,\oM)+\frac{4\gamma\sigma}{p_x^2}.
\end{align*}
Rescaling $\gamma$ to $\frac{cp_x\gamma}{4}$ for some $c\in(0,1]$, and using the definition of $\hmlec$, we have
\begin{align*}
    xp_x\leq p_x\tomlec_{cp_x\gamma/4}(\cM,\oM)+\gamma \sigma
    \leq 4c^{-1}\hmlec_{\gamma}(\cM,\oM)+c\gamma\sigma.
\end{align*}
Therefore, for any $\eps\in(0,1]$, it holds that
\begin{align*}
\EE_{M\sim \mu} \delta(M)
\leq &~\int_0^1 \mu(\delta(M)\geq x)dx \\
\leq&~ \eps+\int_\eps^1 p_x dx \\
\leq&~ \eps+\log(1/\eps)\paren{ 4c^{-1}\tomlec_{\gamma}(\cM,\oM)+c\gamma\sigma }.
\end{align*}
Taking $c^{-1}=\log(1/\eps)$ completes the proof.
\end{proofof}

\subsection{Proof of Proposition~\ref{prop:sep-omle-mops}}\label{appdx:proof-sep-omle-mops}

Let us consider $\cA=[N]$ and $\Delta\leq \frac{1}{10}$.
\begin{lemma}[{\citet[Theorem 9]{li2022understanding}}]
    There exists a function class $\cH\subset(\cA\to\{-1,1\})$ such that $\cH=\{h_1,\cdots,h_N\}$,
    \begin{align*}
        h_i(j)=\begin{cases}
            -1, & j<i \\
            1, & j=i,
        \end{cases}
    \end{align*}
    and $\starn(\cH-\bar{h},\gamma)\leq C\log_2 N$ for all $\bar{h}\in\cH$ and $\gamma>0$, where $C$ is an universal constant.
\end{lemma}

Fix a sequence $\Delta_1>\Delta_2>\cdots>\Delta_N=\Delta$ with $\Delta_1\leq 2\Delta$.
For each $i\in[N]$, we consider $M=M_i$ a $N$-arm bandit instance given by $r^M(a)\sim\Bern\paren{\frac{1}{2}+\Delta_i h_i(a)}$ for all $a\in\cA$, and we take $\Pi=\cA$. Under such construction, the canonical choice of $\pi_M$ is $\pi_{M}=i$ and we have $f^M(\pi_M)=\frac12+\Delta_i$.

\paragraph{Proof of \cref{prop:sep-omle-mops} (1).} Consider $\oM=M_N$ and running \omle~on $\oM$: for each $t=1,\cdots,T$, we set
\begin{align*}
    \cL_t(M)\defeq \sum_{s<t} \ell_M(\pi^t,r^t), \qquad \cM^t\defeq \set{ M: \cL_t(M)\geq \max_{M'} \cL_t(M')-\beta },
\end{align*}
and $(M^t,\pi^t)=\argmax_{\pi,M\in\cM^t} f^M(\pi)$. There are two choices of $\ell$: (a) $\ell_M(\pi,r)=\log \MP(r(\pi)=r)$, or (b) $\ell_M(\pi,r)=-(r-r^M(\pi))^2$. For both of these cases, the argument is similar, and hence we focus on case (a).

Consider the random variables $T_1,\cdots,T_N$ defined by
\begin{align*}
    T_i=\min\set{t: M_j\not \in\cM^t\ \forall j\leq i}.
\end{align*}
By the definition of $T_i$, for $t<T_i$, there exists $j\leq i$ such that $M_j\in\cM^t$, which implies $M^t\in\set{M_1,\cdots,M_i}$, and hence $\pi^t\in\set{1,\cdots,i}$. 

Also, by definition, we have
\begin{align*}
    \max_{M\in\cM}\cL_{T_i}(M)=\max_{M\in\set{M_{i+1},\cdots,M_{N}}}\cL_{T_i}(M).
\end{align*}
Now, for each $j>i$, we can compute
\begin{align*}
    \cL_{T_i}(M_i)-\cL_{T_i}(M_j)
    =&~ \sum_{t=1}^{T_i-1} \ell_{M_i}(\pi^t,r^t)-\ell_{M_j}(\pi^t,r^t) \\
    =&~ \sum_{t\in[T_i-1]:\pi^t=i} \ell_{M_i}(i,r^t)-\ell_{M_j}(i,r^t),
\end{align*}
where we use the fact that $\ell_{M_i}(\pi,r^t)=\ell_{M_j}(\pi,r^t)$ for all $\pi\in[i-1]$. Notice that 
\begin{align*}
    \EE_{r\sim \oM(i)}\brac{\ell_{M_i}(i,r)-\ell_{M_j}(i,r)}\geq -\KL{ \Bern\paren{\frac12-\Delta_N} }{ \Bern\paren{\frac12+\Delta_i} }\geq -48\Delta^2,
\end{align*}
which is due to the definition of $\ell$ and the fact that $\Delta_i\leq2\Delta\leq \frac14.$ Therefore, by Hoeffding's inequality, for $U_i\defeq \#\set{t\in[T_i-1]:\pi^t=i}$, the following holds with probability at least $1-\delta$ for each $i<j$:
\begin{align*}
    \cL_{T_i}(M_i)-\cL_{T_i}(M_j)\geq -U_i\cdot 48\Delta^2-\bigO{\sqrt{U_i\Delta^2\log(NT/\delta)}}.
\end{align*}
Thus, notice that $\cL_{T_i}(M_i)-\max_{j>i}\cL_{T_i}(M_j)\leq -\beta$, we derive
\begin{align*}
    U_i\geq c\min\set{ \frac{\beta}{\Delta^2} , \frac{\beta^2}{\Delta^2\log(NT/\delta)}}
\end{align*}
for some small universal constant $c$ with probability at least $1-\delta$. 
Taking summation over $i\in[N-1]$ gives
\begin{align*}
    T\geq \sum_{i=1}^{N-1} U_i \geqsim N\min\set{ \frac{\beta}{\Delta^2} , \frac{\beta^2}{\Delta^2\log(NT)}}
\end{align*}
with probability $\frac12$. This is the desired result as long as $\beta\geq \log(N/\Delta)$, which is indeed the case. \qed

\paragraph{Proof of \cref{prop:sep-omle-mops} (2).} Fix a $\oM\in\cM$. For each $M=M_i\in\cM$, we have $\pi_M=i$, and we also write $h^M=h_i$, $\Delta^M=\Delta_i$. Then we know
\begin{align*}
    \wtdH(\oM(\pi),M(\pi))
    \geq&~ \abs{\Delta^Mh^M(\pi)-\Delta^{\oM}h^{\oM}(\pi)} \\
    =&~ \max\set{ \frac{\Delta^{M}+\Delta^{\oM}}{2}\abs{h^M(\pi)-h^{\oM}(\pi)}, \abs{\Delta^{M}-\Delta^{\oM}}} \\
    \geq&~ \frac{\Delta}2\abs{h^M(\pi)-h^{\oM}(\pi)} + \frac12\abs{\Delta^{M}-\Delta^{\oM}},
\end{align*}
and we also have 
\begin{align*}
    f^M(\pi)-f^{\oM}(\pi)
    =&~ \Delta^Mh^M(\pi)-\Delta^{\oM}h^{\oM}(\pi)\leq 2\Delta\abs{h^M(\pi)-h^{\oM}(\pi)} + \abs{\Delta^{M}-\Delta^{\oM}}.
\end{align*}
Therefore, we have
\begin{align*}
    \psc_{\gamma}(\cM,\oM)\leq&~ \sup_{\mu\in\Delta(\cM)} \Big\{ 2\Delta\EE_{M\sim \mu} \abs{h^M(\pi_M)-h^{\oM}(\pi_M)}-\frac{\gamma\Delta^2}{4}\EE_{M,M'\sim \mu} \abs{h^M(\pi_{M'})-h^{\oM}(\pi_{M'})}^2 \\
    &~ \qquad + \EE_{M\sim \mu}\abs{\Delta^{M}-\Delta^{\oM}} - \frac\gamma2\EE_{M\sim \mu}\abs{\Delta^{M}-\Delta^{\oM}}^2 \Big\} \\
    \leq&~ \frac{2\dimc(\cH-h^{\oM},\gamma\Delta/2)+1}{\gamma}.
\end{align*}
Applying \cref{example:dc-to-es} gives the desired result. \qed

\end{document}